\documentclass{article}

\usepackage{etoolbox}
\newcommand{\arxiv}[1]{\iftoggle{colt}{}{#1}}
\newcommand{\colt}[1]{\iftoggle{colt}{#1}{}}
\newtoggle{colt}
\global\togglefalse{colt}

\newcommand{\loose}{\looseness=-1}

\usepackage[utf8]{inputenc} 
\usepackage[T1]{fontenc}    
\usepackage{url}            
\usepackage{booktabs}       
\usepackage{amsfonts}       
\usepackage{nicefrac}       
\usepackage{microtype}      

\usepackage{tocloft}            

\usepackage{enumitem}

\usepackage{breakcites}

\usepackage{comment}
\newtoggle{draft}
\togglefalse{draft}

\usepackage[normalem]{ulem}

\usepackage{mathrsfs}

\usepackage{algorithm}
\usepackage{verbatim}
\usepackage[noend]{algpseudocode}

\usepackage{multicol}

\usepackage{colortbl}

\usepackage{setspace}

\usepackage{transparent}

\usepackage{inconsolata}
\usepackage[scaled=.90]{helvet}
\usepackage{xspace}

\usepackage{pifont}
\arxiv{\usepackage[letterpaper, left=1in, right=1in, top=1in, bottom=1in]{geometry}}

\PassOptionsToPackage{hypertexnames=false}{hyperref}  
\arxiv{\usepackage{parskip}}

\usepackage[dvipsnames]{xcolor}
\arxiv{\usepackage[colorlinks=true, linkcolor=blue!70!black, citecolor=blue!70!black,urlcolor=black,breaklinks=true]{hyperref}}
\usepackage{microtype}
\usepackage{hhline}

\makeatletter
\newcommand{\neutralize}[1]{\expandafter\let\csname c@#1\endcsname\count@}
\makeatother

\usepackage{algorithm}

\usepackage{natbib}
\bibpunct{(}{)}{;}{a}{,}{,}

\usepackage{amsthm}
\usepackage{mathtools}
\usepackage{amsmath}
\usepackage{bbm}
\usepackage{amsfonts}
\usepackage{amssymb}

\usepackage{xpatch}

\usepackage{thmtools}
\usepackage{thm-restate}
\declaretheorem[name=Theorem,parent=section]{theorem}
\declaretheorem[name=Lemma,parent=section]{lemma}
\declaretheorem[name=Assumption, parent=section]{assumption}
\declaretheorem[name=Condition, parent=section]{condition}
\declaretheorem[qed=$\triangleleft$,name=Example,style=definition, parent=section]{example}
\declaretheorem[name=Remark, parent=section]{remark}
\declaretheorem[name=Proposition, parent=section]{proposition}

\declaretheorem[name=Corollary, parent=section]{corollary}

\makeatletter
  \renewenvironment{proof}[1][Proof]%
  {%
   \par\noindent{\bfseries\upshape {#1.}\ }%
  }%
  {\qed\newline}
  \makeatother

\theoremstyle{definition}  

\theoremstyle{plain}
\newtheorem{definition}{Definition}[section]

\xpatchcmd{\proof}{\itshape}{\normalfont\proofnameformat}{}{}
\newcommand{\proofnameformat}{\bfseries}

\usepackage[nameinlink,capitalize]{cleveref}

\newcommand{\pref}[1]{\cref{#1}}
\newcommand{\pfref}[1]{Proof of \pref{#1}}

\renewcommand{\eqref}[1]{\texorpdfstring{\hyperref[#1]{Eq. (\ref*{#1})}}{Eq. (\ref*{#1})}}

\crefformat{equation}{#2(#1)#3}
\Crefformat{equation}{#2(#1)#3}

\Crefformat{figure}{#2Figure #1#3}
\Crefname{assumption}{Assumption}{Assumptions}
\Crefformat{assumption}{#2Assumption #1#3}

\usepackage{crossreftools}
\newcommand{\creftitle}[1]{\crtcref{#1}}

\usepackage{xparse}

\ExplSyntaxOn
\DeclareDocumentCommand{\XDeclarePairedDelimiter}{mm}
 {
  \__egreg_delimiter_clear_keys: 
  \keys_set:nn { egreg/delimiters } { #2 }
  \use:x 
   {
    \exp_not:n {\NewDocumentCommand{#1}{sO{}m} }
     {
      \exp_not:n { \IfBooleanTF{##1} }
       {
        \exp_not:N \egreg_paired_delimiter_expand:nnnn
         { \exp_not:V \l_egreg_delimiter_left_tl }
         { \exp_not:V \l_egreg_delimiter_right_tl }
         { \exp_not:n { ##3 } }
         { \exp_not:V \l_egreg_delimiter_subscript_tl }
       }
       {
        \exp_not:N \egreg_paired_delimiter_fixed:nnnnn 
         { \exp_not:n { ##2 } }
         { \exp_not:V \l_egreg_delimiter_left_tl }
         { \exp_not:V \l_egreg_delimiter_right_tl }
         { \exp_not:n { ##3 } }
         { \exp_not:V \l_egreg_delimiter_subscript_tl }
       }
     }
   }
 }

\keys_define:nn { egreg/delimiters }
 {
  left      .tl_set:N = \l_egreg_delimiter_left_tl,
  right     .tl_set:N = \l_egreg_delimiter_right_tl,
  subscript .tl_set:N = \l_egreg_delimiter_subscript_tl,
 }

\cs_new_protected:Npn \__egreg_delimiter_clear_keys:
 {
  \keys_set:nn { egreg/delimiters } { left=.,right=.,subscript={} }
 }

\cs_new_protected:Npn \egreg_paired_delimiter_expand:nnnn #1 #2 #3 #4
 {
  \mathopen{}
  \mathclose\c_group_begin_token
   \left#1
   #3
   \group_insert_after:N \c_group_end_token
   \right#2
   \tl_if_empty:nF {#4} { \c_math_subscript_token {#4} }
 }
\cs_new_protected:Npn \egreg_paired_delimiter_fixed:nnnnn #1 #2 #3 #4 #5
 {
  \mathopen{#1#2}#4\mathclose{#1#3}
  \tl_if_empty:nF {#5} { \c_math_subscript_token {#5} }
 }
\ExplSyntaxOff

\XDeclarePairedDelimiter{\supnorm}{
  left=\lVert,
  right=\rVert,
  subscript=\infty
  }

\DeclarePairedDelimiter{\abs}{\lvert}{\rvert} %
\DeclarePairedDelimiter{\brk}{[}{]}
\DeclarePairedDelimiter{\crl}{\{}{\}}
\DeclarePairedDelimiter{\prn}{(}{)}
\DeclarePairedDelimiter{\nrm}{\|}{\|}
\DeclarePairedDelimiter{\tri}{\langle}{\rangle}

\DeclarePairedDelimiter{\ceil}{\lceil}{\rceil}

\let\Pr\undefined

\DeclareMathOperator{\En}{\mathbb{E}}

\DeclareMathOperator{\Pr}{Pr}

\DeclareMathOperator*{\argmin}{arg\,min} 
\DeclareMathOperator*{\argmax}{arg\,max}             

\newcommand{\mb}[1]{\boldsymbol{#1}}
\newcommand{\wt}[1]{\widetilde{#1}}
\newcommand{\wh}[1]{\widehat{#1}}
\newcommand{\wb}[1]{\widebar{#1}}

\def\ddefloop#1{\ifx\ddefloop#1\else\ddef{#1}\expandafter\ddefloop\fi}
\def\ddef#1{\expandafter\def\csname bb#1\endcsname{\ensuremath{\mathbb{#1}}}}
\ddefloop ABCDEFGHIJKLMNOPQRSTUVWXYZ\ddefloop
\def\ddefloop#1{\ifx\ddefloop#1\else\ddef{#1}\expandafter\ddefloop\fi}
\def\ddef#1{\expandafter\def\csname b#1\endcsname{\ensuremath{\mathbf{#1}}}}
\ddefloop ABCDEFGHIJKLMNOPQRSTUVWXYZ\ddefloop
\def\ddef#1{\expandafter\def\csname f#1\endcsname{\ensuremath{\mathfrak{#1}}}}
\ddefloop ABCDEFGHIJKLMNOPQRSTUVWXYZ\ddefloop
\def\ddef#1{\expandafter\def\csname sf#1\endcsname{\ensuremath{\mathsf{#1}}}}
\ddefloop ABCDEFGHIJKLMNOPQRSTUVWXYZ\ddefloop
\def\ddef#1{\expandafter\def\csname c#1\endcsname{\ensuremath{\mathcal{#1}}}}
\ddefloop ABCDEFGHIJKLMNOPQRSTUVWXYZ\ddefloop
\def\ddef#1{\expandafter\def\csname h#1\endcsname{\ensuremath{\widehat{#1}}}}
\ddefloop ABCDEFGHIJKLMNOPQRSTUVWXYZ\ddefloop
\def\ddef#1{\expandafter\def\csname hc#1\endcsname{\ensuremath{\widehat{\mathcal{#1}}}}}
\ddefloop ABCDEFGHIJKLMNOPQRSTUVWXYZ\ddefloop
\def\ddef#1{\expandafter\def\csname t#1\endcsname{\ensuremath{\widetilde{#1}}}}
\ddefloop ABCDEFGHIJKLMNOPQRSTUVWXYZ\ddefloop
\def\ddef#1{\expandafter\def\csname tc#1\endcsname{\ensuremath{\widetilde{\mathcal{#1}}}}}
\ddefloop ABCDEFGHIJKLMNOPQRSTUVWXYZ\ddefloop
\def\ddefloop#1{\ifx\ddefloop#1\else\ddef{#1}\expandafter\ddefloop\fi}
\def\ddef#1{\expandafter\def\csname scr#1\endcsname{\ensuremath{\mathscr{#1}}}}
\ddefloop ABCDEFGHIJKLMNOPQRSTUVWXYZ\ddefloop

\newcommand{\ind}{\mathbbm{1}}    

\newcommand{\veps}{\varepsilon}

\newcommand{\ldef}{\vcentcolon=}
\newcommand{\rdef}{=\vcentcolon}

  \newcommand{\gsf}{\mathsf{g}}%

\newcommand{\etacheck}{\check{\eta}}

\newcommand{\Lambdammax}[1][M]{\Lambda(#1;\veps,\nbar)}
\newcommand{\Lambdammaxm}[1][M]{\Lambda(#1;\veps,\nbarm)}
\newcommand{\Imfull}[2][M]{I\sups{#1}(#2; \cM)}

\newcommand{\nbarm}[1][M]{\nbar\sups{#1}}

\newcommand{\ommbar}[1][\Mbar]{\omega\sups{#1}}

\newcommand{\lambdabar}{\wb{\lambda}}

\newcommand{\Tgl}{T^{\mathrm{gl}}}

\newcommand{\Pma}[1][M]{\bbP^{\sss{\smash{#1},\,\bbA}}}
\newcommand{\Ema}[1][M]{\En^{\sss{\smash{#1},\,\bbA}}}
\newcommand{\Pmap}[1][M]{\bbP^{\sss{\smash{#1},\,\bbA'}}}
\newcommand{\Emap}[1][M]{\En^{\sss{\smash{#1},\,\bbA'}}}
\newcommand{\Pmpa}[1][M']{\bbP^{\sss{\smash{#1},\,\bbA}}}
\newcommand{\Empa}[1][M']{\En^{\sss{\smash{#1},\,\bbA}}}

\newcommand{\Pmbara}[1][\Mbar]{\bbP^{\sss{\smash{#1},\,\bbA}}}
\newcommand{\Embara}[1][\Mbar]{\En^{\sss{\smash{#1},\,\bbA}}}

\newcommand{\Alg}{\mathbb{A}}

\newcommand{\lambdahat}{\wh{\lambda}}

\newcommand{\Lambdam}[1][M]{\Lambda(#1;\veps)}

\newcommand{\vepsstat}{\veps_{\mathrm{stat}}}

\newcommand{\Nm}{N_{\neg\pibm}}
\newcommand{\Nmu}{N_{\neg\pim}}
\newcommand{\Nmbar}{N_{\neg\pibmbar}}

\newcommand{\pist}{\pistar}
\newcommand{\KL}{D_{\mathsf{KL}}}
\newcommand{\Delmin}{\Delta_{\mathrm{min}}}

\newcommand{\Aspace}{\Rplus^{\Pi}}

\newcommand{\cMalt}{\cM^{\mathrm{alt}}}
\newcommand{\cMgl}[1][\veps]{\cM_{#1}^{\mathrm{gl}}}

\newcommand{\lam}{\lambda}
\newcommand{\om}{\omega}

\newcommand{\gm}[1][M]{\mathsf{g}\sups{#1}}

\newcommand{\gmi}[1][M_i]{\mathsf{g}\sups{#1}}
\newcommand{\gmj}[1][M_j]{\mathsf{g}\sups{#1}}

\newcommand{\gmstar}[1][\Mstar]{\mathsf{g}\sups{#1}}
\newcommand{\gstar}{\mathsf{g}^{\star}}

\newcommand{\nm}[1][M]{\mathsf{n}\sups{#1}}
\newcommand{\nmax}{\mathsf{n}_{\mathrm{max}}}

\newcommand{\qmbar}[1][\Mbar]{q\sups{#1}}

\let\Im\undefined
\newcommand{\Im}[1][M]{I\sups{#1}}
\newcommand{\Imstar}[1][\Mstar]{I\sups{#1}}

\newcommand{\lambdam}[1][M]{\lambda\sups{#1}}

\newcommand{\etam}[1][M]{\eta\sups{#1}}
\newcommand{\etamstar}[1][\Mstar]{\eta\sups{#1}}

\newcommand{\Rplus}{\bbR_{+}}

\newcommand{\delm}[1][M]{\Delta\sups{#1}}
\newcommand{\delmstar}[1][\Mstar]{\Delta\sups{#1}}

\newcommand{\delmin}{\Delta_{\mathrm{min}}}
\newcommand{\delminm}[1][M]{\Delta_{\mathrm{min}}\sups{#1}}
\newcommand{\delminmstar}[1][\Mstar]{\Delta_{\mathrm{min}}\sups{#1}}

\newcommand{\alloc}{Graves-Lai allocation\xspace}
\newcommand{\glc}{\mathsf{glc}}

\newcommand{\GLText}{Graves-Lai Coefficient\xspace}

\newcommand{\pibm}[1][M]{\mb{\pi}\subs{#1}}
\newcommand{\pibmbar}[1][\Mbar]{\mb{\pi}\subs{#1}}

\newcommand{\cMopt}{\cM^{\mathrm{opt}}}
\newcommand{\cMoptsub}{\cM^{\mathrm{opt}}_0}
\newcommand{\cMsub}{\cM_0}

\newcommand{\CompShort}{AEC\xspace}
\newcommand{\CompText}{Allocation-Estimation Coefficient\xspace}
\newcommand{\comp}[1][\veps]{\mathsf{aec}_{#1}}
\newcommand{\aec}[1][\veps]{\mathsf{aec}_{#1}}
\newcommand{\aecM}[2]{\mathsf{aec}_{#1}\sups{#2}}
\newcommand{\aecD}[1][\veps]{\mathsf{aec}_{#1}\sups{\mathsf{D}}}
\newcommand{\aecflip}[1][\veps]{\wb{\mathsf{aec}}_{#1}}
\newcommand{\aecflipD}[1][\veps]{\wb{\mathsf{aec}}_{#1}\sups{\mathsf{D}}}
\newcommand{\aecflipDM}[2]{\wb{\mathsf{aec}}_{#1}\sups{\mathsf{D},#2}}
\newcommand{\aecflipM}[2]{\wb{\mathsf{aec}}_{#1}\sups{#2}}

\newcommand{\mainalg}{$\textsf{AE}^{2}$\xspace}
\newcommand{\mainalgb}{$\textsf{AE}_\star^{2}$\xspace}
\newcommand{\MainAlg}{Allocation Estimation via Adaptive Exploration\xspace}

\newcommand{\Mtil}{\wt{M}}

\newcommand{\picirc}{\pi^{\circ}}

\newcommand{\cMall}{\cM^{+}}

\newcommand{\RegDM}[1][T]{\Reg(#1)}

\newcommand{\Empi}[1][M]{\En^{\sss{#1},\pi}}

\renewcommand{\c}{\mathrm{c}}

\newcommand{\fmp}{f\sups{M'}}
\newcommand{\pimp}{\pi\subs{M'}}

\renewcommand{\emptyset}{\varnothing}

\newcommand{\filt}{\mathscr{F}}

\newcommand{\hist}{\mathcal{H}}

\newcommand{\Asig}{\mathscr{P}}
\newcommand{\Rsig}{\mathscr{R}}
\newcommand{\Osig}{\mathscr{O}}

\newcommand{\Hspace}{\Omega}
\newcommand{\Hsig}{\filt}

\newcommand{\Framework}{Decision Making with Structured Observations\xspace}
\newcommand{\FrameworkShort}{DMSO\xspace}

\newcommand{\learner}{learner\xspace}

\newcommand{\act}{\pi}
\newcommand{\Act}{\Pi}

\newcommand{\obs}{o}
\newcommand{\Obs}{\mathcal{\cO}}

\newcommand{\RewardSpace}{\cR}

\newcommand{\Rspace}{\RewardSpace}

\newcommand{\M}[1]{^{{\scriptscriptstyle M}}}  

\newcommand{\sups}[1]{^{{\scriptscriptstyle#1}}}
\newcommand{\subs}[1]{_{{\scriptscriptstyle#1}}}

\newcommand{\sss}[1]{{\scriptscriptstyle#1}}

\newcommand{\fm}[1][M]{f\sups{#1}}
\newcommand{\pim}[1][M]{\pi_{\sss{#1}}}

\newcommand{\fmbar}{f\sups{\Mbar}}

\newcommand{\fmstar}{f\sups{\Mstar}}
\newcommand{\pimstar}{\pi\subs{\Mstar}}

\newcommand{\fstar}{f^{\star}}
\newcommand{\pistar}{\pi_{\star}}

\newcommand{\pihat}{\wh{\pi}}

\newcommand{\Mbar}{\wb{M}}

\newcommand{\fmi}{f\sups{M_i}}

\newcommand{\Rm}[1][M]{R\sups{#1}}
\newcommand{\Pm}[1][M]{P\sups{#1}}

\newcommand{\Prm}[2]{\bbP^{\sss{#1},#2}}

\newcommand{\Qm}[2]{Q^{\sss{#1},#2}}
\newcommand{\Vm}[2]{V^{\sss{#1},#2}}

\newcommand{\Reg}{\mathrm{\mathbf{Reg}}}

\newcommand{\EstD}{\mathrm{\mathbf{Est}}_{\mathsf{D}}}
\newcommand{\EstDbar}{\widehat{\mathrm{\mathbf{Est}}}_{\mathsf{D}}}

\newcommand{\EstKL}{\mathrm{\mathbf{Est}}_{\mathsf{KL}}}

\newcommand{\Mhat}{\wh{M}}
\newcommand{\Mstar}{M^{\star}}

\newcommand{\algcommentlight}[1]{\textcolor{blue!70!black}{\transparent{0.5}\footnotesize{\texttt{\textbf{//\hspace{2pt}#1}}}}}

\newcommand{\midsem}{\,;}

\newcommand{\approxleq}{\lesssim}
\newcommand{\approxgeq}{\gtrsim}

\renewcommand{\ind}[1]{^{#1}}

\newcommand{\bigoh}{O}
\newcommand{\bigoht}{\wt{O}}

\newcommand{\bigom}{\Omega}
\newcommand{\bigomt}{\wt{\Omega}}

\newcommand{\bigtheta}{\Theta}

\newcommand{\indic}{\mathbb{I}}

\renewcommand{\Pr}{\bbP}

\newcommand{\poly}{\mathrm{poly}}
\newcommand{\polylog}{\mathrm{polylog}}

\newcommand{\kl}[2]{D_{\mathsf{KL}}\prn*{#1\,\|\,#2}}
\newcommand{\Dkl}[2]{D_{\mathsf{KL}}\prn*{#1\,\|\,#2}}

\newcommand{\Dklbig}[2]{D_{\mathsf{KL}}\big ( #1\,\|\,#2 \big )}

\newcommand{\Dhel}[2]{D_{\mathsf{H}}\prn*{#1,#2}}

\newcommand{\Dhels}[2]{D^{2}_{\mathsf{H}}\prn*{#1,#2}}
\newcommand{\Dhelsbig}[2]{D^{2}_{\mathsf{H}}\big ( #1,#2 \big )}

\newcommand{\Dtv}[2]{D_{\mathsf{TV}}\prn*{#1,#2}}

\newcommand{\conv}{\mathrm{co}}

\newcommand{\unif}{\mathrm{Unif}}

\newcommand{\mathand}{\quad\text{and}\quad}

\def\multiset#1#2{\ensuremath{\left(\kern-.3em\left(\genfrac{}{}{0pt}{}{#1}{#2}\right)\kern-.3em\right)}}

\renewcommand{\emptyset}{\varnothing}

\newcommand{\thetast}{\theta^\star}

\newcommand{\fst}{f^\star}

\newcommand{\xst}{x^\star}

\newcommand{\mutil}{\widetilde{\mu}}

\renewcommand{\ast}{a^\star}

\newcommand{\cMaltst}{\cM_{\mathrm{alt}}^\star}
\newcommand{\Mst}{M^\star}

\newcommand{\cst}{\gstar}

\newcommand{\Cexp}{C_{\mathrm{exp}}}

\newcommand{\muhat}{\widehat{\mu}}
\newcommand{\pitil}{\widetilde{\pi}}

\newcommand{\lamtil}{\widetilde{\lambda}}

\newcommand{\stil}{\widetilde{s}}

\newcommand{\etahat}{\widehat{\eta}}

\newcommand{\etatil}{\widetilde{\eta}}

\newcommand{\cEtil}{\widetilde{\cE}}

\newcommand{\Delminst}{\Delmin^\star}
\newcommand{\cMst}{\cM^{\star}}

\newcommand{\Clo}{C_{\mathsf{low}}}
\newcommand{\pexp}{p_{\mathrm{exp}}}

\newcommand{\betast}{\mathsf{n}^\star}

\newcommand{\piM}{\pim}

\newcommand{\piMhats}{\pim[\Mhat^s]}
\newcommand{\lamGL}{\lambda}

\newcommand{\cSGLest}{\cS_{\mathrm{exp}}}

\newcommand{\LKL}{L_{\mathsf{KL}}}

\newcommand{\Mhats}{\Mhat^s}
\newcommand{\cTexploit}{\cT_{\mathrm{exploit}}}
\newcommand{\AlgEstD}{\mathrm{\mathbf{Alg}}_{D}}
\newcommand{\AlgEstKL}{\mathrm{\mathbf{Alg}}_{\mathsf{KL}}}

\newcommand{\linear}{\mathrm{lin}}

\newcommand{\LM}[1][\cM]{L\subs{#1}}

\newcommand{\cMp}{\cM^+}
\newcommand{\DelM}{\Delta\sups{M}}

\newcommand{\simplex}{\triangle}
\newcommand{\Exp}{\mathbb{E}}
\newcommand{\R}{\mathbb{R}}

\newcommand{\D}[2]{D \big (#1\,\|\,#2 \big )}

\newcommand{\nsf}{\mathsf{n}}

\newcommand{\rmd}{\mathrm{d}}

\newcommand{\Caec}{C_{\aec[]}}
\newcommand{\cMcov}{\cM_{\mathrm{cov}}}
\newcommand{\Ncov}{\mathsf{N}_{\mathrm{cov}}}
\newcommand{\ntil}{\widetilde{\nsf}}
\newcommand{\nbar}{\widebar{\nsf}}

\newcommand{\wmpi}[1][M]{w_{\pi}\sups{#1}}

\newcommand{\wmb}[2]{w^{\sss{#1},#2}}
\newcommand{\htil}{\widetilde{h}}
\newcommand{\rem}[1][M]{r\sups{#1}}
\newcommand{\Probm}[1][M]{P\sups{#1}}

\newcommand{\atil}{\widetilde{a}}
\newcommand{\Prmb}[1][M]{\mathbb{P}\sups{#1}}

\newcommand{\am}{\widetilde{a}\subs{\Mst}}

\newcommand{\pmin}{P_{\min}}
\newcommand{\Ccov}{C_{\mathsf{cov}}}
\newcommand{\dcov}{d_{\mathsf{cov}}}
\newcommand{\frakD}{\mathfrak{D}}
\newcommand{\dE}{d_{\mathrm{E}}}
\newcommand{\gst}{\cst}
\newcommand{\sst}{s^{\star}}
\newcommand{\alphazeta}{\alpha_{\zeta}}
\newcommand{\alphaq}{\alpha_q}

\newcommand{\cMtab}{\cM_{\mathrm{tab}}}
\newcommand{\Pmin}{\pmin}

\newcommand{\Cexpxi}[1][\xi]{\Cexp^\xi}
\newcommand{\pexpxi}[1][\xi]{\pexp^\xi}

\newcommand{\CexpD}{\Cexp^{\mathsf{D}}}
\newcommand{\alphaM}{\alpha\subs{\cM}}
\newcommand{\delminst}{\Delminst}

\newcommand{\delst}{\Delta^\star}
\newcommand{\Emb}[2]{\Exp^{\sss{#1},#2}}
\newcommand{\VM}{V\subs{\cM}}
\newcommand{\Delst}{\Delta^\star}
\newcommand{\inner}[2]{\langle #1, #2 \rangle}

\newcommand{\Thetast}{\Theta^\star}
\newcommand{\Thetaalt}{\Theta^{\mathrm{alt}}}

\newcommand{\Thetabar}{\widebar{\Theta}}
\newcommand{\Thetastbar}{\widebar{\Theta}^\star}
\newcommand{\pX}{p_{\cX}}
\newcommand{\piexp}{\pi_{\mathrm{exp}}}

\newcommand{\LMc}{L\subs{\cM}\sups{M}}

\newcommand{\LMst}{\LM^\star}
\newcommand{\cMLdelst}{\cM^\star}
\newcommand{\xstbar}{\bar{x}^\star}
\newcommand{\sigmamin}{g_{\min}}
\newcommand{\sigmamax}{g_{\max}}
\newcommand{\link}{g}
\newcommand{\ncM}{\nm[\cM]}

\newcommand{\nst}{\nsf^\star}
\newcommand{\guncM}{\underline{\mathsf{g}}\sups{\cM}}
\newcommand{\cMbar}{\widebar{\cM}}
\newcommand{\nmeps}[1][M]{\nsf_{\veps}\sups{#1}}
\newcommand{\nepsst}{\nsf_{\veps}^\star}
\newcommand{\ncMeps}[1][\veps]{\ncM_{#1}}

\newcommand{\DelLst}{\Delta_\star}
\newcommand{\delLst}{\DelLst}

\newcommand{\alphan}{\alpha_{\nsf}}
\newcommand{\nmepsb}[1][M]{\nsf_{\veps}\sups{#1}}
\newcommand{\nmepsc}[2]{\nsf_{#1}\sups{#2}}

\newcommand{\specialO}{\bigoh\sups{+}}
\newcommand{\specialOt}{\bigoht\sups{+}}

\newcommand{\specialOmt}{\bigomt\sups{+}}

\let\underbar\undefined

\makeatletter
\let\save@mathaccent\mathaccent
\newcommand*\if@single[3]{%
  \setbox0\hbox{${\mathaccent"0362{#1}}^H$}%
  \setbox2\hbox{${\mathaccent"0362{\kern0pt#1}}^H$}%
  \ifdim\ht0=\ht2 #3\else #2\fi
  }
\newcommand*\rel@kern[1]{\kern#1\dimexpr\macc@kerna}
\newcommand*\widebar[1]{\@ifnextchar^{{\wide@bar{#1}{0}}}{\wide@bar{#1}{1}}}
\newcommand*\underbar[1]{\@ifnextchar_{{\under@bar{#1}{0}}}{\under@bar{#1}{1}}}
\newcommand*\wide@bar[2]{\if@single{#1}{\wide@bar@{#1}{#2}{1}}{\wide@bar@{#1}{#2}{2}}}
\newcommand*\under@bar[2]{\if@single{#1}{\under@bar@{#1}{#2}{1}}{\under@bar@{#1}{#2}{2}}}
\newcommand*\wide@bar@[3]{%
  \begingroup
  \def\mathaccent##1##2{%
    \let\mathaccent\save@mathaccent
    \if#32 \let\macc@nucleus\first@char \fi
    \setbox\z@\hbox{$\macc@style{\macc@nucleus}_{}$}%
    \setbox\tw@\hbox{$\macc@style{\macc@nucleus}{}_{}$}%
    \dimen@\wd\tw@
    \advance\dimen@-\wd\z@
    \divide\dimen@ 3
    \@tempdima\wd\tw@
    \advance\@tempdima-\scriptspace
    \divide\@tempdima 10
    \advance\dimen@-\@tempdima
    \ifdim\dimen@>\z@ \dimen@0pt\fi
    \rel@kern{0.6}\kern-\dimen@
    \if#31
      \overline{\rel@kern{-0.6}\kern\dimen@\macc@nucleus\rel@kern{0.4}\kern\dimen@}%
      \advance\dimen@0.4\dimexpr\macc@kerna
      \let\final@kern#2%
      \ifdim\dimen@<\z@ \let\final@kern1\fi
      \if\final@kern1 \kern-\dimen@\fi
    \else
      \overline{\rel@kern{-0.6}\kern\dimen@#1}%
    \fi
  }%
  \macc@depth\@ne
  \let\math@bgroup\@empty \let\math@egroup\macc@set@skewchar
  \mathsurround\z@ \frozen@everymath{\mathgroup\macc@group\relax}%
  \macc@set@skewchar\relax
  \let\mathaccentV\macc@nested@a
  \if#31
    \macc@nested@a\relax111{#1}%
  \else
    \def\gobble@till@marker##1\endmarker{}%
    \futurelet\first@char\gobble@till@marker#1\endmarker
    \ifcat\noexpand\first@char A\else
      \def\first@char{}%
    \fi
    \macc@nested@a\relax111{\first@char}%
  \fi
  \endgroup
}
\newcommand*\under@bar@[3]{%
  \begingroup
  \def\mathaccent##1##2{%
    \let\mathaccent\save@mathaccent
    \if#32 \let\macc@nucleus\first@char \fi
    \setbox\z@\hbox{$\macc@style{\macc@nucleus}_{}$}%
    \setbox\tw@\hbox{$\macc@style{\macc@nucleus}{}_{}$}%
    \dimen@\wd\tw@
    \advance\dimen@-\wd\z@
    \divide\dimen@ 3
    \@tempdima\wd\tw@
    \advance\@tempdima-\scriptspace
    \divide\@tempdima 10
    \advance\dimen@-\@tempdima
    \ifdim\dimen@>\z@ \dimen@0pt\fi
    \rel@kern{0.6}\kern-\dimen@
    \if#31
      \underline{\rel@kern{-0.6}\kern\dimen@\macc@nucleus\rel@kern{0.4}\kern\dimen@}%
      \advance\dimen@0.4\dimexpr\macc@kerna
      \let\final@kern#2%
      \ifdim\dimen@<\z@ \let\final@kern1\fi
      \if\final@kern1 \kern-\dimen@\fi
    \else
      \underline{\rel@kern{-0.6}\kern\dimen@#1}%
    \fi
  }%
  \macc@depth\@ne
  \let\math@bgroup\@empty \let\math@egroup\macc@set@skewchar
  \mathsurround\z@ \frozen@everymath{\mathgroup\macc@group\relax}%
  \macc@set@skewchar\relax
  \let\mathaccentV\macc@nested@a
  \if#31
    \macc@nested@a\relax111{#1}%
  \else
    \def\gobble@till@marker##1\endmarker{}%
    \futurelet\first@char\gobble@till@marker#1\endmarker
    \ifcat\noexpand\first@char A\else
      \def\first@char{}%
    \fi
    \macc@nested@a\relax111{\first@char}%
  \fi
  \endgroup
}
\makeatother

\allowdisplaybreaks

\usepackage{color-edits}
 \addauthor{df}{ForestGreen}
 \usepackage[textsize=tiny]{todonotes}
 
\newcommand{\awcomment}[1]{{\color{red}[AW: #1]}}

\makeatletter
\let\OldStatex\Statex
\renewcommand{\Statex}[1][3]{%
  \setlength\@tempdima{\algorithmicindent}%
  \OldStatex\hskip\dimexpr#1\@tempdima\relax}
\makeatother

 \usepackage{accents}

\let\oldparagraph\paragraph
\renewcommand{\paragraph}[1]{\oldparagraph{#1.}}

\title{Instance-Optimality in Interactive Decision Making: \\ Toward a Non-Asymptotic Theory}
\author{%
Andrew Wagenmaker \\ University of Washington \\
{\normalsize\texttt{ajwagen@cs.washington.edu}} \and Dylan J. Foster \\ Microsoft Research \\ {\normalsize\texttt{dylanfoster@microsoft.com}}
}

\date{}

\begin{document}
\maketitle

\begin{abstract}

We consider the development of adaptive, instance-dependent algorithms for interactive decision making (bandits,
  reinforcement learning, and beyond) that, rather than only performing well in the worst case,
  adapt to favorable properties of real-world instances for improved
  performance. We aim for \emph{instance-optimality}, a strong notion
  of adaptivity which asserts that, on any particular problem instance, the algorithm under consideration
  outperforms all consistent algorithms. Instance-optimality enjoys a rich asymptotic
  theory originating from the work of \citet{lai1985asymptotically}
  and \citet{graves1997asymptotically}, but \emph{non-asymptotic}
  guarantees have remained elusive outside of certain special cases.
  Even for problems as simple as tabular reinforcement learning,
  existing algorithms do not attain instance-optimal performance until
  the number of rounds of interaction is \emph{doubly exponential} in
  the number of states.

In this paper, we take the first step toward developing a
non-asymptotic theory of instance-optimal decision making with
  general function approximation. We
introduce a new complexity measure, the \CompText (\CompShort), and provide a new
algorithm, \mainalg, which attains non-asymptotic instance-optimal performance
at a rate controlled by the \CompShort. Our results recover the best known
guarantees for well-studied problems such as finite-armed and linear bandits
and, when specialized to tabular reinforcement learning, attain the first instance-optimal regret
bounds with polynomial dependence on all problem parameters, improving
over prior work exponentially.
We complement these results with lower bounds that show that i) existing
notions of statistical complexity are insufficient to
derive non-asymptotic guarantees,  and
ii) under
certain technical conditions, boundedness of the \CompText is
\emph{necessary} to learn an instance-optimal allocation of
  decisions in finite time.\loose

\end{abstract}

\addtocontents{toc}{\protect\setcounter{tocdepth}{2}}
{
\hypersetup{hidelinks}
\tableofcontents
}

\section{Introduction}
\label{sec:intro}

\newcommand{\expconst}{c}

We consider the development of adaptive, sample-efficient
algorithms for \emph{interactive decision making}, encompassing bandit
problems and reinforcement learning with general function
approximation. For decision making in high-dimensional spaces with a
long horizon, existing approaches \citep{lillicrap2015continuous,mnih2015human,silver2016mastering} are
sample-hungry, which presents an obstacle for real-world deployment in
settings where data is scarce or high-quality simulators are not
available. To overcome this challenge, algorithms should both i)
flexibly incorporate users' domain knowledge, as expressed via
modeling and function approximation, and ii) explore the environment in a
deliberate, adaptive fashion, taking advantage of favorable structure
whenever possible.

Toward achieving these goals, a major area of research aims to develop algorithms with optimal sample complexity and understand
the fundamental limits for such algorithms
\citep{russo2013eluder,jiang2017contextual,sun2019model,wang2020provably,du2021bilinear,jin2021bellman,foster2021statistical}, and the foundations
are beginning to fall into place. In particular, focusing on \emph{minimax regret} (that
is, the best regret that can achieved for a worst-case problem
instance in a given class of problems), \citet{foster2021statistical,foster2022complexity,foster2023tight} provide unified algorithm
design principles and measures of statistical complexity that are both
necessary and sufficient for low regret. However, minimax regret and
other notions of worst-case performance are inherently pessimistic,
and may not be sufficient to close the gap between theory and practice.
For example, recent work has shown that algorithms that are optimal in the worst-case can be arbitrarily suboptimal on ``easier'' instances \citep{wagenmaker2022beyond}.
To overcome these challenges and develop algorithms that perform well on \emph{every} instance, a promising approach is to develop algorithms that \emph{adapt} to the
      difficulty of the problem instance under consideration.

The performance of such adaptive algorithms can be quantified through
\emph{instance-dependent} regret bounds, which become smaller (leading to
low regret) when the underlying problem instance is favorable. Algorithms with such guarantees have been studied throughout the
literature on bandits and reinforcement learning; basic examples
include adapting to large gaps in value between
alternative actions \citep{lai1985asymptotically} or low
noise or variance in bandit problems \citep{allenberg2006,hazan2011better,foster2016learning,wei2018more,bubeck2018sparsity}, and
adapting to the difficulty of
reaching certain states in Markov Decision Processes
\citep{zanette2019tighter,simchowitz2019non,dann2021beyond,wagenmaker2022beyond}.

While
there are many notions of adaptivity and instance-dependence,
they are generally incomparable. A stronger notion of adaptivity is
\emph{instance-optimality}, which asserts that the performance of the algorithm on a problem instance of interest exceeds that of \emph{any consistent algorithm} (that is, any algorithm
with sublinear regret for all problem instances). Instance-optimality enjoys a rich theory originating with the work of \citet{lai1985asymptotically}
and \citet{graves1997asymptotically}, with a celebrated line of
research developing sharp guarantees for the special case of
finite-armed bandits
\citep{burnetas1996optimal,garivier2016explore,kaufmann2016complexity,lattimore2018refining,garivier2019explore}. Beyond the finite-armed
bandit setting, however, development has been largely \emph{asymptotic} in
nature, and existing algorithms either:
\begin{enumerate}
\item achieve instance-optimality only as $T\to{}\infty$ (or, to the
  extent that they are non-asymptotic, require $T$ to be exponentially
  large \arxiv{with respect to}\colt{in} problem-dependent parameters) \citep{graves1997asymptotically,komiyama2015regret,combes2017minimal,degenne2020structure,dong2022asymptotic}, or\loose
\item achieve non-asymptotic guarantees, but require restrictive
  modeling assumptions such as linear function approximation \citep{tirinzoni2020asymptotically,kirschner2021asymptotically}.
\end{enumerate}
Indeed, even for the simple problem of tabular (finite-state/action) reinforcement learning,
  existing algorithms do not attain instance-optimal performance until
  the number of rounds of interaction is \emph{doubly exponential} in
  the number of states \citep{ok2018exploration,dong2022asymptotic}. In this paper, we address these challenges, providing algorithms that i) accommodate flexible, general-purpose
  function approximation, and ii) attain instance-optimality in finite
  time, in a sense which is itself optimal.\loose

\paragraph{Contributions} We take the first steps toward building a
non-asymptotic theory of instance-optimal decision making. We observe
that asymptotic characterizations for \iftoggle{colt}{instance-optimality}{instance-optimal statistical complexity}:
\begin{enumerate}
\item reflect the regret incurred by an allocation of decisions
  designed to optimally distinguish the ground truth problem instance from a
  set of alternatives, but
\item do not capture the statistical complexity required to
  \emph{learn} such an allocation.
\end{enumerate}
To address this, we introduce a new complexity measure, the \CompText
(\CompShort), which aims to capture the statistical complexity of
learning an optimal Graves-Lai allocation. We provide a new
algorithm, \mainalg, which attains non-asymptotic instance-optimal regret
at a rate controlled by the \CompShort. We complement this result with lower bounds that show that under
certain technical conditions, boundedness of the \CompText is not just
sufficient, but \emph{necessary} to learn an instance-optimal
allocation in finite time.

Our algorithm is simple, and
can be applied to any hypothesis class in a generic
fashion. It recovers the best known
guarantees for standard problems such as finite-armed and linear bandits
and, when specialized to tabular reinforcement learning, achieves the first instance-optimal regret
bounds with polynomial dependence on all problem parameters.
We believe that our approach clarifies and elucidates many tradeoffs
and statistical considerations left implicit in prior work, and hope
that it will serve as a
foundation for further development of instance-optimal algorithms.

\subsection{Interactive Decision Making}
\label{sec:setup}

We adopt the \emph{\Framework} (\FrameworkShort) framework of
\citet{foster2021statistical}, which is a general setting for
interactive decision making that encompasses bandit problems
(structured, contextual, and so forth) and reinforcement learning with function
approximation. 

The \FrameworkShort framework is specified by a \emph{decision space} $\Pi$, \emph{reward space} $\cR \subseteq \bbR$, and \emph{observation space} $\cO$. The learner is given access to a (known) \emph{model class} $\cM \subset (\Pi \rightarrow \simplex_{\cR \times \cO})$, and it is assumed there exists some true model $\Mst \in \cM$, unknown to the learner, which represents the underlying environment. Formally, we make the following assumption.

\iftoggle{colt}{
\begin{assumption}[Realizability]
  \label{ass:realizability}
  We have that $\Mst \in \cM$.
\end{assumption}
}{
\begin{assumption}[Realizability]
  \label{ass:realizability}
  The learner has access to a model class $\cM$ containing the true model $\Mstar$.
\end{assumption}}
The learning protocol consists of $T$ rounds. For each round $t=1,\ldots,T$:
  \begin{enumerate}
  \item The \learner selects a \emph{decision} $\act\ind{t}\in\Act$.
  \item The learner receives a reward $r\ind{t}\in\cR$
    and observation $o\ind{t}\in\cO$ sampled \arxiv{via}
    $(r\ind{t},o\ind{t})\sim{}\Mstar(\pi\ind{t})$, and observes $(r\ind{t},o\ind{t})$.
    
  \end{enumerate}

We can think of the model class $\cM$ as representing the learner's prior knowledge about the
decision making problem, and it allows one to appeal to estimation and function
approximation. For structured bandit problems, for example, models correspond to
reward distributions, and $\cM$ encodes structure in the reward
landscape. For reinforcement learning problems, models correspond to
Markov decision processes (MDPs), and $\cM$ typically encodes structure in value functions or transition probabilities. \iftoggle{colt}{See \Cref{sec:upper_ex} and \Cref{sec:tabular_results} for concrete examples of how standard decision-making settings can be instantiated within the DMSO framework, and \citet{foster2021statistical}
for further background. 
}{
We refer the reader to \Cref{sec:upper_ex} and \Cref{sec:tabular_results} for concrete examples of how standard decision-making settings can be instantiated within the DMSO framework, and refer to \citet{foster2021statistical}
for further background. 
}
  \iftoggle{colt}{
For a model $M\in\cM$, $\Empi[M]\brk*{\cdot}$ denotes the expectation
  under the process $(r,\obs)\sim{}M(\pi)$,
  $\fm(\pi)\ldef{}\Empi[M]\brk*{r}$ denotes the mean reward function,
  and $\pim\ldef{}\argmax_{\act\in\Act}\fm(\act)$ denotes the optimal decision. 
  When the algorithm is clear from context,
 $\Exp\sups{M}[\cdot]$ and $\Pr\sups{M}[\cdot]$ refer to the
 expectation and probability measure, respectively, induced over
 histories under $M$. When the context is clear, we overload notation
 somewhat and use $\Pr^{\sss{M},\pi}[\cdot]$ to refer to the conditional
 density over $\cR\times\cO$ induced by playing $\pi$ on $M$.
 We make the following assumptions.
  \begin{assumption}[Bounded Reward Means]\label{asm:bounded_means}
  For  all $M \in \cM$, $\pi \in \Pi$, we have $\fm(\pi) \in [0,1]$.
  \end{assumption}
  \begin{assumption}[Unique Optimal Actions]\label{ass:unique_opt}
For all $M \in \cM$, the optimal action $\pim$ is unique. \awcomment{update}
\end{assumption}

  }{
  
  For a model $M\in\cM$, $\Empi[M]\brk*{\cdot}$ denotes the expectation
  under the process $(r,\obs)\sim{}M(\pi)$,
  $\fm(\pi)\ldef{}\Empi[M]\brk*{r}$ denotes the mean reward function,
  $\pim\in \argmax_{\act\in\Act}\fm(\act)$ denotes any
  optimal decision, and $\pibm = \argmax_{\pi \in \Pi} \fm(\pi)$
denotes the \emph{set} of all optimal decisions.
 Similarly, when the algorithm is clear from context,
 $\Exp\sups{M}[\cdot]$ and $\Pr\sups{M}[\cdot]$ refer to the
 expectation and probability measure, respectively, induced over
 histories, $\cH^t = (\pi^1,r^1,o^1),\ldots,(\pi^t,r^t,o^t)$, on $M$. We overload notation
 somewhat and use $\Pr^{\sss{M},\pi}[\cdot]$ to refer to the conditional
 density over $\cR\times\cO$ induced by $M$.
  While we do not assume the \emph{random} rewards are strictly
  bounded (allowing, for example, Gaussian rewards), we make the
  following assumption on the reward means.
  \begin{assumption}[Bounded Reward Means]\label{asm:bounded_means}
  For each $M \in \cM$ and $\pi \in \Pi$, we have $\fm(\pi) \in [0,1]$.
  \end{assumption}
    Throughout this work, we also make the following assumption on the uniqueness of the optimal action of $\Mst$.

\begin{assumption}[Unique Optimal Action]\label{ass:unique_opt}
  For the ground truth model $\Mst \in \cM$, the optimal action $\pim[\Mst]$ is unique.
\end{assumption}
  }

  Note that the latter assumption is standard in the literature on instance-optimality. We measure performance in terms of regret, which is given by
  \iftoggle{colt}{
  \begin{equation}
    \label{eq:regret}
   \textstyle \RegDM\ldef\sum_{t=1}^{T}\En_{\pi\ind{t}\sim{}p\ind{t}}\brk*{\fmstar(\pimstar)-\fmstar(\pi\ind{t})},
  \end{equation}
  }{
   \begin{equation}
    \label{eq:regret}
    \RegDM\ldef\sum_{t=1}^{T}\En_{\pi\ind{t}\sim{}p\ind{t}}\brk*{\fmstar(\pimstar)-\fmstar(\pi\ind{t})},
  \end{equation}
  }
  where $p\ind{t}$ is the learner's randomization distribution for
  round $t$. In addition, we define $\delm(\pi)=\fm(\pim)-\fm(\pi)$ as
  the \emph{suboptimality gap} function for model $M$ and decision $\pi$, and the \emph{minimum suboptimality gap} as
  \begin{align}\label{eq:mingap_def}
  \delminm := \begin{cases} \inf_{\pi\in\Pi: \delm(\pi)>0}\delm(\pi), & \pim \text{ is unique,} \\ 0, & \text{otherwise.} \end{cases} 
  \end{align}
Since by assumption $\pim[\Mst]$ is unique, we have $\delminm[\Mst] > 0$.
Throughout, we replace dependence on $\Mst$ with ``$^\star$'' when
  the meaning is clear from context, for example:
  $\delminst := \delminm[\Mst]$, $\fst(\pi) :=
  \fm[\Mst](\pi)$, or $\pistar=\pimstar$.  

  \paragraph{Further Notation}
  We let
$\cMall=\crl{M:\Pi\to\simplex_{\cR \times\cO}\mid{}\fm(\pi)\in\brk{0,1}}$ denote the space of
all possible models $M$ with rewards in $\cR$ and
$\fm(\pi)\in\brk{0,1}$.
We use $\simplex_\cX$ to refer to the
  set of probability distributions over any $\cX$. Throughout, we often abbreviate $\Exp_{\pi \sim p}[\cdot]$ with $\Exp_{p}[ \cdot ]$.

\subsection{Background: Asymptotic Instance-Optimality}

Our aim is to develop algorithms that are \emph{instance-optimal} in a strong sense:
for \emph{every model $\Mstar\in\cM$}, the regret of the
algorithm under $\Mstar$ is at least as good as that of any \emph{consistent}
algorithm; here, an algorithm is said to be ``consistent'' if it ensures that
$\En\sups{M}\brk{\RegDM}=o(T)$ for all $M\in\cM$. Instance-optimality
is a powerful notion of performance: no algorithm---even one designed specifically
with $\Mstar$ in mind---can achieve lower regret on $\Mstar$
without giving up consistency. For multi-armed bandits, a long line of work initiated
by \citet{lai1985asymptotically} characterizes the instance-optimal
regret as a function of the instance $\Mstar$, and provides efficient algorithms that
attain instance-optimality in finite time
\citep{garivier2016explore,kaufmann2016complexity,lattimore2018refining,garivier2019explore}. For
the general decision making setting we consider, the forward-looking work
of  \citet{graves1997asymptotically} (see \citet{dong2022asymptotic}
for a contemporary treatment) introduced a complexity measure we refer
to as the \emph{\GLText}, which asymptotically characterizes the
instance-optimal performance as a function of the instance $\Mstar$
and model class $\cM$. \arxiv{Define the
Kullback-Leibler divergence by
  \[
    \Dkl{\bbP}{\bbQ} =\left\{
    \begin{array}{ll}
\int\log\prn[\big]{
      \frac{\rmd \bbP}{\rmd \bbQ}
      }\rmd \bbP,\quad{}&\bbP\ll\bbQ,\\
      +\infty,\quad&\text{otherwise.}
    \end{array}\right.
\]}
For any class $\cM$ and model $M \in\cM$, the \GLText is
\iftoggle{colt}{defined as}{the value of the program}
\begin{align}
  \label{eq:glc}
  \glc(\cM,M)
  \ldef \inf_{\eta\in\bbR_{+}^{\Pi}}\crl*{
  \sum_{\pi\in\Pi}\eta_\pi \delm(\pi)
  \mid{} \forall{}M'\in\cMalt(M) :
  \sum_{\pi\in\Pi}\eta_\pi\Dkl{M(\pi)}{M'(\pi)} \geq{} 1
  },
\end{align}
where, for $M$ with unique optimal decision $\pim$, we define 
\iftoggle{colt}{
\awcomment{update}
$\cMalt(M)\ldef{}\crl*{M\in\cM\mid{}\pim \not\in \pibm}$
}{
\begin{align}
  \cMalt(M)\ldef{}\crl*{M' \in\cM\mid{}\pim \not\in \pibm[M']},
\end{align}}
the set of ``alternative'' models---the models $M' \in \cM$ that disagree with $M$ on
the optimal decision\colt{---and $\Dkl{\cdot}{\cdot}$ denotes the Kullback-Leibler divergence}\arxiv{\footnote{For $M \in \cM$ such that $\pim$ is not unique, we define $\cMalt(M)\ldef{}\crl*{M'\in\cM\mid{} \pibm \cap \pibm[M'] = \emptyset }$, and define $\glc(\cM,M)$ as in \eqref{eq:glc}, with respect to this $\cMalt(M)$. 
We also define $\cMalt(\pi) = \{ M \in \cM \mid \pi \not\in \pibm \}$.}}. When $\cM$ is clear from context, we will abbreviate $\gm:=\glc(\cM,M)$ (and $\gst := \glc(\cM,\Mst)$). We also denote any solution to \eqref{eq:glc} by $\etam$---note that this is not in general unique. The characterization of \citet{graves1997asymptotically} is as follows.
\begin{proposition}[\citet{graves1997asymptotically,dong2022asymptotic}]
  \label{prop:glc}
  For any model class $\cM$ with $\abs{\Pi}<\infty$, any algorithm that is consistent with
  respect to $\cM$ must have
  \begin{align}
    \label{eq:glc_lower}
    \En\sups{\Mstar}\brk*{\RegDM} \geq{}
    \glc(\cM,\Mstar)\cdot{}\log(T) - o(\log(T))
  \end{align}
  for any $\Mstar\in\cM$, and there exists an algorithm which
  achieves\colt{, for all $\Mst \in \cM$,}\footnote{To be precise, rather than scaling directly with $\glc(\cM,\Mstar)$, the upper bound given by \cite{dong2022asymptotic} scales with a quantity $\glc_T(\cM,\Mstar)$ such that $\glc_T(\cM,\Mstar) \rightarrow_T \glc(\cM,\Mstar)$.}
    \begin{align}
    \label{eq:glc_upper}
    \En\sups{\Mstar}\brk*{\RegDM} \leq
    \glc(\cM,\Mstar)\cdot{}\log(T) + o(\log(T))\colt{.}
    \end{align}
\arxiv{    for any $\Mstar\in\cM$ satisfying \Cref{ass:unique_opt}. }
  \end{proposition}
The interpretation of the \GLText of $\Mst$ with respect to $\cM$, $\glc(\cM,\Mst)$, is simple. It asks, if
$\Mstar$ is known to the learner (to be clear, $\Mstar$ is not known a-priori), what is the minimum regret
that must be incurred to gather enough information to rule out all
possible alternatives $M' \in \cM$ which do not have $\pimstar$ as an
optimal decision (i.e., $\pimstar \not\in \pibm[M']$)? In other words, it aims
to \emph{certify}
that $\pimstar$ is indeed the optimal decision while incurring the minimum regret possible.
\arxiv{This intuition is
explicitly encoded in \eqref{eq:glc}: the objective, $\sum_{\pi \in
  \Pi} \eta(\pi) \delm[\Mst](\pi)$ denotes the regret that an
allocation $\eta\in\Aspace$ will incur on the true instance $\Mst$,
while the constraint $\sum_{\pi \in \Pi} \eta(\pi)
\Dkl{\Mst(\pi)}{M'(\pi)} \ge 1$ ensures that $\Mst$ and $M$ are
statistically distinguishable under the allocation $\eta$, for all $M' \in \cMalt(\Mst)$. }

\arxiv{As a simple
example, for the finite-armed bandit problem where $\Pi=\brk{A}$ and
$\cM=\crl*{M(\pi) = \cN(f(\pi),1)\mid{}f\in\bbR^{A}}$, a
straightforward calculation shows that
\[
\glc(\cM,\Mstar) \propto \sum_{\pi\neq\pimstar}\frac{1}{\delmstar(\pi)}
\]
and any optimal allocation has $\etamstar(\pi)\propto \frac{1}{(\delmstar(\pi))^2}$ for
$\pi\neq\pimstar$. In this case, \pref{prop:glc} recovers the well-known
gap-dependent logarithmic regret bound
\[
  \En\sups{\Mstar}\brk*{\RegDM}
  \approxleq{} \sum_{\pi\neq\pimstar}\frac{\log(T)}{\delmstar(\pi)}
\]
of \citet{lai1985asymptotically}, and shows that it is
instance-optimal. Note that this bound can offer significant improvement over the minimax rate
$\bigtheta\prn[\big]{\sqrt{AT}}$ when $T$ is large.}

The \GLText characterization is appealing in its simplicity, but the
catch---at least when one moves beyond finite-armed bandits---is hiding in the
lower-order terms, particularly for the upper bound
\pref{eq:glc_upper}. For general model classes, the best known
finite-time regret bounds \citep{dong2022asymptotic} take the form
\begin{align}
  \label{eq:glc_regret}
    \En\sups{\Mstar}\brk*{\RegDM} \leq
      \glc(\cM,\Mstar)\cdot{}\log(T) +
      \poly(\abs{\Pi},(\delminmstar)^{-1})\cdot\log^{1-\expconst}(T)
    \end{align}
    where $\expconst>0$ is a universal constant. While this indeed
    leads to instance-optimality as $T\to\infty$, the ``lower-order''
    term in \eqref{eq:glc_regret} scales with the size of the decision space, which is
    intractably large for most problems of interest. As an example, consider the
    problem of tabular reinforcement learning in an episodic MDP with $S$
    states, $A$ actions, and horizon $H$. Here, we typically have
    $\glc(\cM,\Mstar)=\poly(S,A,H)$ \arxiv{(in particular, the \GLText can be
    bounded in terms of suboptimality gaps for the
    optimal value function and reachability for states of interest
    \citep{simchowitz2019non,dann2021beyond,wagenmaker2022beyond})}, yet $\abs{\Pi}=A^{HS}$. Consequently, the \GLText does
    not become the dominant term in \pref{eq:glc_regret} until
    $\geq\exp(\exp(S))$. That is, for realistic time horizons, asymptotic instance-optimality does not
      tell the full story.

\paragraph{Learning an optimal allocation}
Given knowledge of an optimal Graves-Lai allocation $\etamstar$
solving \eqref{eq:glc}, a
learner could simply take actions as specified by $\etamstar$, and
would achieve the instance-optimal rate given in
\Cref{prop:glc}. However, this is typically infeasible, as the optimal
allocation itself depends strongly upon the ground truth model $\Mst$,
and is therefore unknown to the learner. In light of this challenge, the
approach taken by essentially all existing algorithms \citep{burnetas1996optimal,graves1997asymptotically,magureanu2014lipschitz,komiyama2015regret,lattimore2017end,combes2017minimal,hao2019adaptive,hao2020adaptive,van2020optimal,degenne2020structure,tirinzoni2020asymptotically,kirschner2021asymptotically,dong2022asymptotic} is to first learn an estimate for a \alloc
$\etamstar$, and then take actions as specified by this
estimate. In addition to being natural, this approach is
\emph{necessary} in a certain (weak) sense: for any algorithm that achieves instance-optimality, the expected decision frequencies must converge to an approximately
optimal allocation as $T$ grows
(cf. \cref{lem:opt_regret_feasible}).\footnote{The connection
  between instance-optimal regret and learning an optimal allocation
  has many subtleties; we refer ahead to \cref{sec:lower} for extensive discussion.}

The presence of the lower-order term scaling with $\abs{\Pi}$ in
\eqref{eq:glc_upper} (and in similar regret bounds from most existing work) reflects
the sample complexity required to learn an optimal Graves-Lai
allocation through uniform exploration. Specifically, one can estimate an allocation by uniformly
exploring the decision space to gather data, and then solving an
empirical approximation to the Graves-Lai program \cref{eq:glc}.
Naive exploration of this type inevitably results in $\bigom(\abs{\Pi})$
sample complexity, and it is natural to ask whether a more
deliberate exploration strategy, perhaps by exploiting the structure of \arxiv{the class }$\cM$, could lead to better finite-time regret bounds. For the
setting of linear bandits, where $\Pi\subseteq\bbR^{d}$ and the mean
reward function $\pi\mapsto{}\fm(\pi)$ is linear, this is indeed the
case: a recent line of work
\citep{tirinzoni2020asymptotically,kirschner2021asymptotically}
provides regret bounds of the form
\[
  \En\sups{\Mstar}\brk*{\RegDM} \leq
      \glc(\cM,\Mstar)\cdot{}\log(T) +
      \poly(d,(\delminmstar)^{-1})\cdot\log^{1-\expconst}(T).
    \]
This bound replaces the size of the decision space in the lower-order
term by the dimension
$d$, reflecting the fact that there are only $d$ ``effective'' directions in
    which exploration is required. While this is an encouraging start,
    the techniques used in these works are specialized to linear bandits, and it is unclear how to generalize them beyond this setting. \arxiv{Recently, some progress has been made in understanding more general structured bandit problems \citep{jun2020crush}, yet a complete understanding of such problems remains lacking, and even less is known in complex
    decision making problems such as reinforcement learning. }

\iftoggle{colt}{
\subsection{A Motivating Example}\label{sec:motivating_example}
As discussed in the prequel, for finite-armed bandits and linear
bandits, it is possible to achieve
instance-optimal regret bounds where the lower-order terms scale
with the number of actions $A$, or dimension $d$, respectively.
Extrapolating, one might be tempted to ask whether we
can always learn a near-optimal allocation with
sample complexity no larger than, say, the minimax rate for $\cM$. The
starting point for our work is to recognize that in general, the
answer is no: existing notions
of statistical complexity are insufficient to capture the complexity of
learning the \alloc in finite time, as illustrated in the following simple example.
}{
\subsection{A Motivating Example}\label{sec:motivating_example}
As discussed in the prequel, for finite-armed bandits and linear
bandits, it is possible to achieve
instance-optimal regret bounds where the lower-order terms scale
with reasonable problem-dependent quantities of interest: the number
of actions $A$ for bandits, and the dimension $d$ for linear
bandits. 
Informally, these quantities reflect the amount of
exploration required to learn a near-optimal \alloc for the underlying
model $\Mstar$. One might be tempted to ask whether this phenomenon is
universal. That is, 
can we always learn a near-optimal allocation with
sample complexity no larger than, say, the minimax rate for $\cM$? The
starting point for our work is to recognize that in general, the
answer is no: existing notions
of statistical complexity---including those proposed in the minimax
framework \citep{foster2021statistical,foster2022complexity,foster2023tight} and
the \GLText itself---are insufficient to capture the complexity of
learning the \alloc in finite time. To highlight this, consider the following simple example.
}

  \begin{example}[Searching for an informative arm]
    \label{ex:revealing}
    Let $A,N\geq{}2$ and $\beta\in(0,1)$ be parameters, and consider the class
    $\cM$ of all models defined as
    follows. First, $\Pi=\brk{A}\cup\crl{\picirc_i}_{i\in\brk{N}}$; decisions in $\brk{A}$ are ``bandit
      arms'', and decisions in $\crl{\picirc_i}_{i\in\brk{N}}$ are
      ``informative'' (or, revealing) arms. Each model $M$ has a unique optimal decision $\pim$, and the following structure, with $\cO=\brk{A}\cup\crl{\perp}$.
    \begin{itemize}
    \item For each bandit arm $k\in\brk{A}$ we have
      $r\sim{}\cN(\fm(k),1)$ for $\fm\in\brk{0,1}$. There are no observations,
      i.e. $o=\perp$ almost surely.
    \item All informative arms $\picirc_k$ give $0$ reward almost surely.
      There exists a unique informative arm
      $\picirc\subs{M}\in\crl{\picirc_i}_{i\in\brk{N}}$ associated with $M$, so that if we
      play any $\picirc_k$, we receive an observation
      \begin{align*}
        o \sim\left\{
        \begin{array}{ll}
         \unif(\brk{A}),&\quad{}\picirc_k\neq{}\picirc\subs{M},\\
          \beta\indic_{\pim}+(1-\beta)\unif(\brk{A}),&\quad{}\picirc_k=\picirc\subs{M}.
        \end{array}
        \right.
      \end{align*}   
    \end{itemize}
We take $\cM$ to consist of all possible models with this structure. The interpretation here is as follows. \arxiv{Suppose for concreteness
    that $\beta=9/10$.}
    If one were to ignore the
revealing arms $\crl{\picirc_i}_{i\in\brk{N}}$, this would be a
standard finite-armed bandit problem. In particular, if we were to consider a
model $\Mstar$ with  $\fmstar(\pi)=\frac{1}{2} + \Delta\indic\crl{\pi=i}$ for
$i\in\brk{A}$, a standard calculation would yield
\iftoggle{colt}{
$\gmstar \propto \frac{A}{\Delta}$.
}{
\begin{align}
  \gmstar \propto \frac{A}{\Delta}.
  \label{eq:revealing_bandit}
\end{align}
}
However, the presence of the informative arms makes the problem
substantially easier. With $\beta=9/10$ (for concreteness), one can see that for the model $M$, pulling the informative arm
$\picirc\subs{\Mstar}$ will give $o=\pimstar$ with probability at least $9/10$, meaning that we
can identify that $\pimstar$ is optimal with high probability by pulling
$\picirc\subs{\Mstar}$ a constant number of times. 
It follows that the optimal allocation is to ignore the bandit arms
and set
$\eta\sups{\Mstar}(\pi)\propto\indic\crl{\pi=\picirc\subs{\Mstar}}$. This gives $\gmstar \le \bigoh(1)$, which is substantially better than $\gmstar \propto \frac{A}{\Delta}$ if $\Delta$ is small or $A$ is large.

If one only is only concerned with asymptotic rates, this is the end
of the story, but for non-asymptotic rates, we need
to consider the amount of exploration required to \emph{learn the optimal allocation}. In
particular, in order to identify the informative arm $\picirc\subs{\Mstar}$,
which is necessary to learn the optimal allocation, it is clear that
in the worst case, any algorithm needs to try all of the revealing
arms, leading to
\iftoggle{colt}{
$\En\brk*{\RegDM} = \bigom(N)$. While the complexity of learning the
optimal allocation is washed away by an asymptotic analysis with $T \rightarrow \infty$, it cannot be ignored for finite $T$. In addition, the 
$\bigom(N)$ factor cannot be explained away by standard complexity
measures. As we have seen, $\sup_{M\in\cM}\gm=\bigoh(1)$, so the
  \GLText is not sufficient to explain it. Furthermore, the minimax rate for this problem is always bounded by $\bigoh(\sqrt{AT})$, which does not scale with $N$; yet $\bigom(A)$ sample
  complexity does not suffice to learn an optimal allocation. In addition, it can be shown that existing complexity measures such as the Decision-Estimation Coefficient \citep{foster2021statistical}
  and information ratio
  \citep{russo2018learning} also do not scale with $N$.
   \end{example}
}{
\begin{align*}
\En\brk*{\RegDM}
  = \bigom(N).
\end{align*}
This shows that while the complexity of learning the optimal allocation is washed away
by an asymptotic analysis:
\[
\lim_{T\to\infty} \frac{\bigoh(1)\cdot{}\log(T) + N}{\log(T)} = \bigoh(1),
\]
it cannot be ignored for finite $T$. In addition, the 
$\bigom(N)$ factor cannot be explained away by standard complexity
measures:
\begin{itemize}
\item By the argument above, $\sup_{M\in\cM}\gm=\bigoh(1)$, so the
  \GLText---even in a worst-case sense---does not reflect the complexity
  of learning an optimal allocation $\etamstar$.
\item The minimax rate for this problem is always bounded by
  $\bigoh(\sqrt{AT})$ (since we can treat it as a multi-armed bandit
  instance), which does not scale with $N$. Yet, $\bigom(A)$ sample
  complexity does not suffice to learn an optimal allocation. As a
  result, existing complexity measures such as the
  Decision-Estimation Coefficient \citep{foster2021statistical}, 
  Eluder dimension \citep{russo2013eluder}, and information ratio
  \citep{russo2018learning}, which are tailored to the minimax setting, cannot explain away the $\bigom(N)$ factor
 above.\loose
\end{itemize}
  \end{example}
}

\iftoggle{colt}{
\pref{ex:revealing} shows that if we want to achieve
instance-optimality in finite time, new notions of problem complexity
for the class $\cM$ are required, motivating the following central questions:
}{
\pref{ex:revealing} shows that if we want to achieve
instance-optimality in finite time, new notions of problem complexity
for the class $\cM$ are required.  
This motivates the following central questions:
}
    \begin{enumerate}
    \item Can we develop algorithms for general model classes $\cM$ that achieve non-asymptotic instance-optimal regret
      bounds of the form
      \begin{align}
  \En\sups{\Mstar}\brk*{\RegDM} \leq
      \glc(\cM,\Mstar)\cdot{}\log(T) +
        \mathrm{comp}(\cM)\cdot\log^{1-\expconst}(T),
      \end{align}
 \iftoggle{colt}{
 where $\mathrm{comp}(\cM)$ is a complexity measure that reflects the
intrinsic difficulty of exploring in order to learn a Graves-Lai
optimal allocation for $\cM$?
}{
 where $\mathrm{comp}(\cM)$ is a complexity measure that reflects the
intrinsic difficulty of exploration for $\cM$---in particular, the
difficulty of exploring to learn a Graves-Lai optimal allocation for? Ideally, such a
complexity measure should be small
for standard classes of interest,
generalizing what is already known for finite-armed and linear bandits.}
    \item Can we understand when the presence of such lower-order terms is \emph{necessary}? 
    \end{enumerate}

 \colt{
 \subsection{Organization}
 The remainder of the paper is organized as follows. 
 In \Cref{sec:colt_results} we introduce a novel complexity measure,
 the \CompText, which captures the complexity of learning a Graves-Lai
 optimal allocation (\Cref{sec:colt_aec}), present our main upper and
 lower bounds (\Cref{sec:main_results_informal}), and instantiate our
 bounds on several concrete examples
 (\Cref{sec:colt_ex_overview}). In
 \Cref{sec:colt_alg_overview} we present an overview of our main
 algorithm, \mainalg. Due to space constraints, we present results in
 the main body in a somewhat informal form, with many details
 omitted---\emph{for full statements of our results, as well as additional explanation and examples, please see \Cref{part:overview_results} of the appendix.}
 }

\iftoggle{colt}{
\section{Overview of Results}\label{sec:colt_results}
}{
\subsection{The \CompText}
}

\iftoggle{colt}{
To
}{To answer these questions and}
capture the statistical complexity
  of learning an optimal \alloc in finite time, we provide a new complexity
  measure, the \emph{\CompText} (\CompShort). 
  \colt{\subsection{The \CompText}\label{sec:colt_aec}}
  To describe the \CompShort, let us introduce some additional
notation. For a model $M \in \cM$ and parameter $\veps\in\brk{0,1}$,
we define 
\begin{align}\label{eq:lambda_allocation_set}
\Lambda(M;\veps) = \crl*{ \lambda \in \simplex_\Pi \mid \exists \nsf \in \bbR_+ \text{ s.t. } \En_{\pi\sim{}\lam}\brk*{\delm(\pi)} \leq{}
    \frac{(1+\veps)\gm}{\nsf}, \inf_{M'\in\cMalt(M)}\En_{\pi\sim{}\lam}\brk*{\Dkl{M(\pi)}{M'(\pi)}} \geq{} \frac{1-\veps}{\nsf}}
\end{align}
the set of (normalized) allocations $\lambda \in \simplex_\Pi$ which are $\veps$-optimal for the Graves-Lai program $\glc(\cM,M)$ in
  \eqref{eq:glc}---both in terms of achieving the optimal objective value and satisfying the information 
  constraint.
In addition, for a distribution
  $\lambda\in\simplex_\Pi$, we define
\iftoggle{colt}{
  \begin{align}
  \begin{split}
    \label{lambda-set}
\cMgl(\lam)
    = \Big \{
    M\in\cM & \mid{} \exists  \nsf \in \bbR_+ \text{ s.t. } \En_{\pi\sim{}\lam}\brk*{\delm(\pi)} \leq{}
    \frac{(1+\veps)\gm}{\nsf}, \\
    & \inf_{M'\in\cMalt(M)}\En_{\pi\sim{}\lam}\brk*{\Dkl{M(\pi)}{M'(\pi)}} \geq{} \frac{1-\veps}{\nsf}
    \Big \} .
   \end{split}
  \end{align}
 }{
   \begin{align}
    \label{lambda-set}
\cMgl(\lam)
    = \crl*{
    M\in\cM \mid{} \lam \in \Lambda(M;\veps)
    }.
  \end{align}
  }
  Informally, $\cMgl(\lam)$ represents the set of models for
  which the (normalized allocation) $\lambda\in\simplex_\Pi$  is
  $\veps$-optimal for the Graves-Lai program $\glc(\cM,M)$.

  \iftoggle{colt}{
  For a
  \emph{reference model} $\Mbar:\Pi\to\simplex_{\cR\times\cO}$ (not
  necessarily in $\cM$) and $\cMsub \subseteq \cM$, the 
   \CompText
  is given by
  \begin{align}
    \label{eq:comp}
    \aecM{\veps}{\cM}(\cMsub,\Mbar)
    =\inf_{\lambda,\omega\in\simplex_\Pi}\sup_{M\in \cMsub \backslash \cM^{\textrm{gl}}_{\veps}(\lambda)}\crl*{
    \frac{1}{
    \En_{\pi\sim{}\omega}\brk*{\Dkl{\Mbar(\pi)}{M(\pi)}}
    }
    },
  \end{align}
  where we adopt the convention that the value is 0 if
  $\cM^{\textrm{gl}}_{\veps}(\lambda) = \cMsub$. Here $\cMsub$ denotes the
  set we take the supremum over, while $\cM$ denotes the set that $\cMgl(\lambda)$ is defined with respect to. When $\cMsub=\cM$, we abbreviate $\aec(\cM,\Mbar) = \aecM{\veps}{\cM}(\cM,\Mbar)$.
  In addition, letting $\conv(\cM)$ denote the convex hull for $\cM$,
  we define $\comp(\cM)\ldef{}\sup_{\Mbar\in\conv(\cM)}\aecM{\veps}{\cM}(\cM,\Mbar)$.
}{
      For a \emph{reference model}
      $\Mbar:\Pi\to\simplex_{\cR\times\cO}$ (not necessarily in $\cM$)
      and parameter $\veps>0$, the \CompText is given by
      \begin{align}
        \label{eq:comp}
        \aec(\cM,\Mbar)
        =\inf_{\lambda,\omega\in\simplex_\Pi}\sup_{M\in \cM \backslash \cM^{\textrm{gl}}_{\veps}(\lambda)}\crl*{
        \frac{1}{
        \En_{\pi\sim{}\omega}\brk*{\Dkl{\Mbar(\pi)}{M(\pi)}}
        }
        },
      \end{align}
      where we adopt the convention that the value is 0 if
      $\cM^{\textrm{gl}}_{\veps}(\lambda) = \cM$. In addition, letting $\conv(\cM)$ denote the convex hull for
      $\cM$, we define
      $\comp(\cM)\ldef{}\sup_{\Mbar\in\conv(\cM)}\aec(\cM,\Mbar)$.
  }

The \CompText is a game between a min-player choosing
$\lam,\om\in\simplex_\Pi$ and a max-player choosing a model $M\in\cM$
(with the restriction that $M\notin\cMgl(\lam)$). \iftoggle{colt}{The distribution $\lam\in\simplex_\Pi$ represents a normalized Graves-Lai
   allocation, while $\om\in\simplex_\Pi$ is an \emph{exploration}
   distribution used to gather information. The reference model $\Mbar$ should be interpreted as a guess for the
  underlying $\Mstar\in\cM$.
}{Let us
interpret the value. First:
\begin{itemize}
   \item The distribution $\lam\in\simplex_\Pi$ represents a normalized Graves-Lai
   allocation, while $\om\in\simplex_\Pi$ is an \emph{exploration}
   distribution used to gather information.
\item The reference model $\Mbar$ should be interpreted as a guess for the
  underlying model $\Mstar\in\cM$.
\end{itemize}
}
When $\lam\in\simplex_\Pi$ is fixed,  
\iftoggle{colt}{
$\inf_{\om\in\simplex_\Pi}\sup_{M\in \cM \backslash \cMgl(\lam)}(\En_{\pi\sim{}\om}\brk*{\Dkl{\Mbar(\pi)}{M(\pi)}})^{-1}$ 
}{the value
\[
\inf_{\om\in\simplex_\Pi}\sup_{M\in \cM \backslash \cMgl(\lam)}\crl*{
    \frac{1}{
    \En_{\pi\sim{}\om}\brk*{\Dkl{\Mbar(\pi)}{M(\pi)}}
    }
    }
\]
}
represents the time required to gather enough information to
distinguish between the reference model $\Mbar$ and all alternative
models $M\notin\cMgl(\lam)$ for which $\lam$ is not an
$\veps$-optimal Graves-Lai allocation---provided that we explore optimally by minimizing
over $\om\in\simplex_\Pi$. For intuition, consider the case when
$\Mbar \in \cM$. In this case, $\lambda$ \emph{must} be chosen so that
$\Mbar \in \cMgl(\lambda)$ (i.e., $\lambda$ must be a Graves-Lai
optimal allocation for $\Mbar$), as otherwise the value of the \CompShort will be infinite, since $\En_{\pi\sim{}\om}\brk*{\Dkl{\Mbar(\pi)}{\Mbar(\pi)}} = 0$\arxiv{ for all $\omega$}. 
\iftoggle{colt}{Therefore, in such cases, the \CompShort reflects the
  \emph{difficulty of distinguishing $\Mbar$ from models that have
    different Graves-Lai optimal allocations}. 
    Such models might have the \emph{same optimal decision $\pim$ as $\Mbar$} (cf. \cref{ex:revealing}) but, if our goal is to play a Graves-Lai optimal allocation for $\Mbar$, we must still distinguish $\Mbar$ from such models.
}{Therefore, in such cases, the \CompShort reflects the \emph{difficulty of distinguishing $\Mbar$ from models that have different Graves-Lai optimal allocations}. In particular, such models might have the \emph{same optimal decision as $\Mbar$} but, if our goal is to play a Graves-Lai optimal allocation for $\Mbar$, we must still distinguish $\Mbar$ from such models. The value $\comp(\cM,\Mbar)$ then reflects
the \emph{least possible time required} to distinguish such models, which is achieved by
minimizing over the best possible (normalized) Graves-Lai allocation, $\lam\in\simplex_\Pi$.
}

The \CompText plays a natural role for deriving both upper and lower bounds
on the time required to learn an optimal allocation. For lower bounds,
the significance of the \CompShort is somewhat immediate: it precisely
quantifies the time required to acquire enough
information to learn an $\veps$-optimal allocation for the \emph{best
  possible exploration strategy}, and thus leads to a lower bound on
time required to learn such an allocation for any algorithm. Notably,
the \CompShort serves as a lower bound for \emph{all possible model classes $\cM$}, and hence may be thought of as an intrinsic
structural property of the class $\cM$. \colt{For upper bounds, the \CompText acts as a mechanism to drive exploration; see \Cref{sec:colt_alg_overview} for further explanation.}

\arxiv{For upper bounds,
the \CompText acts as a mechanism to drive exploration. Given an
estimate $\Mhat\ind{t}$ for $\Mstar$, if we select the decision
$\pi\ind{t}\in\Pi$ using an appropriate mixture of the distributions
$(\lam,\omega)$ that attain the value of $\comp(\cM,\Mhat\ind{t})$, we are
guaranteed that one of two good outcomes occurs. Either:
\begin{enumerate}
\item the normalized allocation $\lam\in\simplex_\Pi$ is
  $\veps$-optimal for $\glc(\cM,\Mstar)$, so that by playing $\lambda$ the learner matches the performance of the optimal allocation (up to a tolerance $\veps$), thereby incurring the minimum amount of regret needed to gain information that allows it to distinguish $\Mst$ from $M \in \cMalt(\Mst)$, or
\item $\Mstar\notin\cMgl(\lambda)$ (that is, $\lambda$ is not an $\veps$-optimal Graves-Lai allocation), in which case---from the definition
  of the \CompText---playing $\om\in\simplex_\Pi$ will allow us to better distinguish $\Mst$ from $\Mhat^t$.
\end{enumerate}
As long as we can perform estimation in a consistent, online fashion,
this reasoning will allow us to argue that $\Mst \in \cMgl(\lambda)$ the majority of the time, and that
the total regret incurred through the learning process will scale with $\comp(\cM)$.}

\paragraph{Generalized \CompText}
For certain results, we make use of the following, slightly more
general variant of the \CompShort. For a reference model
$\Mbar:\Pi\to\simplex_{\cR\times\cO}$ and \emph{subset of models}
$\cMsub \subseteq \cM$, we define
  \begin{align}
    \label{eq:comp_generalized}
    \aecM{\veps}{\cM}(\cMsub,\Mbar)
    =\inf_{\lambda,\omega\in\simplex_\Pi}\sup_{M\in \cMsub \backslash \cM^{\textrm{gl}}_{\veps}(\lambda)}\crl*{
    \frac{1}{
    \En_{\pi\sim{}\omega}\brk*{\Dkl{\Mbar(\pi)}{M(\pi)}}
    }
    },
  \end{align}
  where we adopt the convention that the value is 0 if
  $\cMsub \backslash \cM^{\textrm{gl}}_{\veps}(\lambda) = \emptyset$. Here $\cMsub$ denotes the
  set we take the supremum over, while $\cM$ denotes the set that
  $\cMgl(\lambda)$ is defined with respect to (i.e., the set with
  respect to which the \alloc is defined). When $\cMsub=\cM$, we
recover the \CompShort as defined in \eqref{eq:comp}: $\aec(\cM,\Mbar) = \aecM{\veps}{\cM}(\cM,\Mbar)$.

\subsection{Main Results}\label{sec:main_results_informal}
Building on the intuition above, our main results show that
boundedness of the \CompText is sufficient to achieve
instance-optimality in finite time, and is also necessary in order to learn a
near-optimal allocation. Formal statements of our upper bounds are given in \Cref{sec:upper} and formal statements for our lower bounds in \Cref{sec:lower}.

\paragraph{Upper Bound}
Our upper bounds are based on a new algorithm, \mainalg
\emph{(\MainAlg)}, achieves instance-optimality by using the \CompText to drive exploration.

\begin{theorem}[Upper Bound---Informal Version of
  \Cref{thm:upper_main}]
\label{thm:upper_informal}
  For any model class $\cM$ satisfying certain regularity conditions,
  the \mainalg algorithm ensures that for all $\veps>0$, $\Mstar\in\cM$, and $T\in\bbN$:
  \begin{align}
    \label{eq:upper_informal}
    \En\sups{\Mstar}\brk*{\RegDM}
    \leq{} (1+\veps)\cdot{}\glc(\cM,\Mstar)\cdot{}\log(T)
    + \specialOt \big (\comp[\veps/12](\cM) + \comp[\veps/12]^{1/2} (\cM) \cdot \log^{1/2}(T) \big ),
  \end{align}
  where $\specialOt(\cdot)$ suppresses polynomial dependence on
  $\veps^{-1}$, the log-covering number of $\cM$, $\sup_{M \in \cM} 1/\delminm$, $\log \log T$, and several other measures of the regularity for the class $\cM$.
\end{theorem}

\Cref{thm:upper_informal} shows that
it is therefore possible to achieve instance-optimality in finite time with lower-order terms scaling (primarily) as the cost of learning the
optimal allocation, as captured by the \CompShort.  For multi-armed bandits with $\Pi=\brk{A}$, we have
$\comp(\cM)=\specialOt(\poly(A))$, and for linear bandits with
$\Pi\subseteq\bbR^{d}$, we have $\comp(\cM)=\specialOt(\poly(d))$, so that the
regret bound in \pref{eq:upper_informal} enjoys similar scaling as
existing non-asymptotic approaches
\citep{tirinzoni2020asymptotically,kirschner2021asymptotically}. For
tabular reinforcement learning, we have $\aec(\cM)=\specialOt(\poly(H,S,A))$,
which leads to exponential improvement over prior work
\citep{dong2022asymptotic}. Finally, for the instance in
\Cref{ex:revealing}, in cases when $N \gg A, 1/\delmin$, $\aec(\cM) = \bigoh(N)$, so $\aec(\cM)$
captures the intuitive difficulty of learning a Graves-Lai allocation
in this setting. 

\begin{remark}[Technical Conditions]
The technical conditions under which \Cref{thm:upper_informal} is
proven are relatively mild, and include certain smoothness of the KL
divergences, sub-Gaussian tail behavior for log-likelihood ratios, and
bounded covering number for $\cM$ with standard parametric growth
(note that $\cM$ may be
infinite), all of which can be shown to hold for standard
classes. In addition, we requires that the amount of information that
can be gained by playing the optimal decision for $\Mstar$ is bounded (see \Cref{sec:regularity} for precise statements of our
conditions). \end{remark}

\begin{remark}[Asymptotic Performance]
Asymptotically as $T \rightarrow \infty$, the regret bound given in
\Cref{thm:upper_informal} scales as $(1+\veps) \cdot \glc(\cM,\Mstar)
\cdot \log T$, which is a factor of $(1+\veps)$ off of the lower
bound. For all standard classes, $\aec[\veps/12](\cM)$ scales polynomially in $1/\veps$ so, to obtain an optimal leading-order constant, it suffices to choose $\veps = 1/\log^a T$, for small enough $a > 0$. 
\end{remark}

\paragraph{Adapting to Minimum Gap}
  Note that the lower-order term given in \Cref{thm:upper_informal}
  scales with $\sup_{M \in \cM} 1/\delminm$, the minimum gap of the
  entire model class. In \Cref{sec:no_mingap_asm}, we give a
  refinement of the \mainalg algorithm (\mainalgb) which attains an
  improved regret bound which replaces the term $\sup_{M \in \cM} 1/\delminm$ with
  $1/\delminst$, the minimum gap of the underlying model; notably
  \mainalgb requires no prior knowledge of $\delminst$ (i.e.,
it is able to adapt to the minimum gap of the underlying model). In
addition, rather than scaling with $\comp[\veps/12](\cM)$, the
lower-order term now scales with $\aecM{\veps/12}{\cM}(\cMst)$ for a
subset $\cMst\subset\cM$ which, informally, restricts to models in
$\cM$ for which the minimum gap is at least $\delminst$.
\begin{theorem}[Upper Bound---Informal Version of
  \Cref{thm:upper_main_no_mingap}]
\label{thm:upper_informal2}
  For any model class $\cM$ satisfying certain regularity conditions,
  the \mainalgb algorithm ensures that
  for all $\veps>0$, $\Mstar\in\cM$, and $T\in\bbN$:
  \begin{align}
    \label{eq:upper_informal2}
    \En\sups{\Mstar}\brk*{\RegDM}
    \leq{} (1+\veps)\cdot{}\glc(\cM,\Mstar)\cdot{}\log(T)
    + \specialOt \big ((\aecM{\veps/12}{\cM}(\cMst))^3 + \log^{6/7}(T) \big ),
  \end{align}
  where $\specialOt(\cdot)$ suppresses polynomial dependence on
  $\veps^{-1}$, the log-covering number of $\cM$, $1/\delminst$, $\log \log T$, and several other measures of the regularity for the class $\cM$.
\end{theorem}

\paragraph{Lower Bound}
To provide lower bounds, we adopt a novel minimax framework which
asks, for the model class $\cM$ under consideration, what
is the least value of $T\in\bbN$ for which it is possible to learn an
$\veps$-optimal \alloc for any model in
$\cM$. 
To state our result, we introduce the following notation,
defined with respect to any $\Mbar \in \cMp$:
\begin{align*}
  \cMopt(\Mbar) = \crl*{M\in\cM\mid{} \pibm\subseteq\pibmbar,\;\;
  \Dkl{\Mbar(\pi)}{M(\pi)}=0\;\; \forall{}\pi\in\pibmbar}.
\end{align*}
The set $\cMopt(\Mbar)$ represents the set of models where 1) the optimal decision
coincides with that of $\Mbar$ and 2) $\Mbar$ and $M \in \cM$ cannot
be distinguished by playing the optimal decision.

Our main lower bound provides a sort of converse to the upper bound in \eqref{eq:upper_informal}.

\begin{theorem}[Lower Bound---Informal Version of \Cref{thm:lower_logt}]
\label{thm:lower_informal}
For any model class $\cM$ and $\veps>0$, it holds that unless
\begin{equation}
  \label{eq:lower_informal}
    \log(T) \geq \sup_{\Mbar \in \cMp} \specialOmt \big (\aecM{\veps}{\cM}(\cMopt(\Mbar),\Mbar) \big ),
  \end{equation}
  no algorithm can simultaneously achieve the following for all instances $M\in\cM$:
  \begin{enumerate}
    \item attain Graves-Lai optimality on $M$ within a constant factor
      (i.e., ensure
      $\Exp\sups{M}[\RegDM] \le 2 \cdot \glc(\cM,M) \log (T)$).
    \item discover an $\veps$-optimal allocation for $M$ (i.e.,
      find $\lambda$ with
    $M \in\cMgl(\lambda)$) with probability greater than $\specialOmt(1)$.\loose
  \end{enumerate}
  Here, $\specialOmt(\cdot)$ hides polynomial dependence on regularity
    parameters of $\cM$.

\end{theorem}

Observe that the \GLText becomes the dominant term in the upper bound
\pref{eq:upper_informal} as soon as
$\log(T)\geq\specialOmt(\aec[\veps/12](\cM))$. The lower bound
\pref{eq:lower_informal} shows that for any algorithm that aims to
estimate the \alloc (in particular, \mainalg), such scaling is
necessary, and therefore the lower-order term in
\Cref{thm:upper_informal} is in some sense unimprovable. To the best
of our knowledge, this is the first general approach to quantifying the lower-order terms necessary in order to achieve instance-optimality.
We make several remarks on the lower bound.

\begin{remark}[Scaling in \CompShort]
Our upper and lower bounds scale with a slightly different version of the \CompShort, as the lower bound restricts the \CompShort to $\cMopt(\Mbar)$. In \Cref{sec:lower}, we show an additional lower bound that scales directly with $\aec(\cM)$, matching our upper bound, but which only provides a lower bound on $T$ rather than $\log(T)$ (see \Cref{thm:lower_t}).
\end{remark}
\begin{remark}[Asymptotic Optimality and Learning Optimal Allocations]
\Cref{thm:lower_informal} gives a lower bound on the time needed to
learn a near-optimal Graves-Lai allocation, but does not directly
imply that it is \emph{necessary} that an asymptotically optimal
algorithm learn such an allocation. As we have noted, the allocations
played by any asymptotically optimal algorithm must converge to an
optimal allocation \emph{in expectation}. However, showing that this
convergence is necessary with even constant probability (the condition
under which \Cref{thm:lower_informal} is proved) is rather subtle. As
we show in \cref{sec:lower} (\Cref{thm:lb_aec_to_regret}), if one
assumes that, in addition to being asymptotically optimal in
expectation, the algorithm under consideration also has regret with appropriately bounded second moment, then if $\Exp\sups{M}[\RegDM] \le (1+\veps) \cdot \glc(\cM,M) \log (T)$ for all $M \in \cM$, it is indeed necessary that a burn-in time analogous to \eqref{eq:lower_informal} is satisfied.
A detailed discussion of this point is given in \Cref{sec:gaps}.
\end{remark}

\arxiv{Together, our upper and lower
bounds represent an initial step toward building a sharp
non-asymptotic theory of instance-optimality, and lead to a number of
new conceptual insights. Our results open the door for
further-development, and to this end we highlight a number of opportunities for improvement (\pref{sec:gaps}), as well as open
problems (\pref{sec:discussion}).}

\subsection{Concrete Examples}\label{sec:colt_ex_overview}
We next present several examples illustrating our upper and lower bounds. 
All results in this section are informal---see \Cref{sec:main,sec:tabular_results,sec:upper_ex,sec:examples_lower} for formal results and additional examples.

\begin{example}[Searching for an Informative Arm (revisited)]
We return to the example of \Cref{sec:motivating_example}. Some
calculation shows that, for the choice of $\cM$ in
\Cref{ex:revealing}, as long as $\beta$ is constant and $N \ge A/\Delta^2$, we have
\begin{align*}
\bigom(N) \le \sup_{\Mbar \in \cMp} \aecM{\veps}{\cM}(\cMopt(\Mbar),\Mbar) \quad \text{and} \quad \aec(\cM) \le \bigoh(N).
\end{align*}
\Cref{thm:upper_informal} then implies that \mainalg has regret on \Cref{ex:revealing} of
\begin{align*}
\En\sups{\Mstar}\brk*{\RegDM} \leq{} (1+\veps)\cdot{}\gst \log(T) + N \cdot \poly(A,\tfrac{1}{\Delta}, \tfrac{1}{\veps}, \log N, \log\log T) \cdot \log^{1/2} (T).
\end{align*}
Furthermore, \Cref{thm:lower_informal} shows that a scaling of $\log
(T) \ge \specialOmt(N)$ is necessary for any algorithm to learn a
Graves-Lai optimal allocation. It follows that, on this example, the
\CompShort reflects a notion of problem difficulty not captured by
any existing complexity measure, matching our intuitive understanding
of what the correct scaling should be. We remark that the scaling $\log(T) \ge \specialOmt(N)$ is natural (as compared to $T \ge \specialOmt(N)$) since, if an algorithm is instance-optimal as required by \Cref{thm:lower_informal}, it can allocate at most $\specialO(\log (T))$ pulls to suboptimal decisions. To pull every informative arm (each of which is suboptimal) while achieving instance-optimality, it follows that we must have $\log(T) \ge \specialOmt(N)$. 
\end{example}

\begin{example}[Tabular Reinforcement Learning]
Consider the setting of tabular reinforcement learning. Here we take
$M$ to be a (tabular) episodic Markov Decision Processes (MDP) with
$S$ states, $A$ actions, horizon $H$, probability transition kernels
$\{ \Pm_h \}_{h=1}^H$, and Gaussian rewards; see
\Cref{sec:tabular_results} for a full definition of this setting. Let $\cM$ denote the set
of all such tabular MDPs which, for each state-action-state triple
$(s,a,s')$ and $h \in [H]$, have $\Pm_h(s' \mid s,a) \ge \Pmin > 0$;
that is, each transition can occur with some minimum probability.
Then it can be shown that:
\begin{align*}
 \aecM{\veps}{\cM}(\cMst) \le \poly \prn*{ S, A, H, \tfrac{1}{\veps}, \tfrac{1}{\delminst}, \log \tfrac{1}{\Pmin} }.
\end{align*}
This
implies that for any tabular MDP in $\cM$, the \mainalgb algorithm has regret bounded as:
\begin{align*}
\En\sups{\Mstar}\brk*{\RegDM} \leq{} (1+\veps)\cdot{}\gst \log(T) + \poly(S,A,H,\tfrac{1}{\delminst}, \tfrac{1}{\veps}, \log \tfrac{1}{\pmin}, \log\log T) \cdot \log^{1/2} (T).
\end{align*}
To the best of our knowledge, this is the first regret bound in the
setting of tabular reinforcement learning which is instance-optimal
with lower-order terms scaling only polynomially in problem
parameters, an exponential improvement over past work
\citep{ok2018exploration,dong2022asymptotic}. Furthermore, one can
also show that $\sup_{\Mbar \in \cMp}
\aecM{\veps}{\cM}(\cMopt(\Mbar),\Mbar)\geq\bigomt \prn*{
  \tfrac{1}{\veps^2} \cdot \tfrac{SA}{(\delminst)^2}}$, so that
our lower
bound, \Cref{thm:lower_informal}, implies that a burn-in time scaling
polynomially in $S,A, \frac{1}{\veps}$, and $\frac{1}{\delminst}$ is
necessary to learn an $\veps$-optimal Graves-Lai allocation for every
model in $\cM$.

We remark that the prior work of \cite{dong2022asymptotic} does not
require that $\Pm_h(s' \mid s,a) \ge \Pmin$ as we do,
yet their bound scales \emph{polynomially} in the inverse probability of observing the trajectory that occurs with minimum non-zero probability
 (the work of
\cite{ok2018exploration} only holds for ergodic MDPs, itself a very
strong assumption). Our finite-time results are therefore, in general, significantly stronger, scaling only logarithmically in $\Pmin$. Removing the $\Pmin$ assumption while still achieving reasonable lower-order terms is an interesting direction for future work.
\end{example}

\begin{example}[Linear Bandits]
Consider the setting of linear bandits in $d$ dimensions with
unit-variance Gaussian noise. Let $\cM$ denote the set of all linear
bandit models defined with respect to some arm set $\cX \subseteq
\bbR^d$ and parameter set $\Theta \subseteq \R^d$. Concretely, each
model $M\in\cM$ takes the form
\[
M(\pi) = \cN(\tri{\theta,x_\pi},1),
\]
for some $\theta\in\Theta$,
where $x_\pi\in\cX$ is an embedding of $\pi$. 
Let $\delminst$ denote the minimum gap of $\Mst$ (which is unknown to the algorithm). Then it can be shown that
\begin{align*}
\aecM{\veps}{\cM}(\cMst) \le \poly \prn*{ d, \tfrac{1}{\veps}, \tfrac{1}{\delminst} }
\end{align*}
which implies that the \mainalgb algorithm
has regret bounded as
\begin{align}\label{eq:colt_linear_upper}
\En\sups{\Mstar}\brk*{\RegDM} \leq{} (1+\veps)\cdot{}\gst \log(T) + \poly(d,\tfrac{1}{\delminst}, \tfrac{1}{\veps}, \log\log T) \cdot \log^{6/7}(T).
\end{align}
We remark that the scaling of \eqref{eq:colt_linear_upper} matches the
state-of-the-art instance-optimal bounds for linear bandits (in that
all have polynomial lower-order terms---our polynomial dependence on
$d$ is slightly worse as our upper bound on the \CompShort is somewhat coarse)
\citep{tirinzoni2020asymptotically,kirschner2021asymptotically}. Notably,
it is a simple corollary of a much more general result, while
prior work relies on specialized algorithms tailored to linear
bandits.
\end{example}

In \Cref{sec:main,sec:tabular_results,sec:upper_ex,sec:examples_lower}
we formalize these examples and present additional examples,
including structured bandits with bounded eluder dimension and
finite-action contextual bandits. In all cases, we obtain lower-order terms scaling only polynomially with problem parameters, and in each setting either match the best-known existing bound, or are the first to provide any meaningful finite-time bounds.

\arxiv{
\subsection{Organization}
\pref{sec:upper} presents our algorithm and main upper bounds, as well
as examples. \pref{sec:lower} presents complementary lower bounds. In \Cref{sec:related} we review additional related work. We
conclude with discussion of open problems and future directions in
\pref{sec:discussion}. Proofs are deferred to the appendix.
\paragraph{Additional notation} For an integer $n\in\bbN$, we let $[n]$ denote the set
  $\{1,\dots,n\}$. 
        We adopt standard
        big-oh notation, and write $f=\bigoht(g)$ to denote that $f =
        \bigoh(g\cdot{}\max\crl*{1,\mathrm{polylog}(g)})$. We let $\specialOt(\cdot)$ and $\specialOmt(\cdot)$ additionally suppress problem-dependent and regularity terms, as, for example, in \Cref{thm:upper_informal} or \Cref{thm:lower_informal}.
        We use
        $\approxleq$ only in informal statements to emphasize the most
        notable elements of an inequality. We will let $\linear(\cdot)$ denote a function multi-linear and poly-logarithmic in its arguments. For a decision $\pi\in\Pi$, we use $\indic_{\pi}\in\simplex_\Pi$ to
denote the delta distribution which places probability mass $1$ on $\pi$. $\bbR_+$ denotes the set of nonnegative real numbers.

        }

 \colt{
\section{Algorithm Overview}\label{sec:colt_alg_overview}

 \begin{algorithm}[tp]
\caption{\MainAlg (\mainalg)}
\begin{algorithmic}[1]
\State \textbf{input:} optimality tolerance $\veps$, model class $\cM$.
\State Initialize $s \leftarrow 1$ and $ q \leftarrow \frac{4 + \veps \zeta \delmin}{4 + 2\veps \zeta \delmin}$ for $\zeta$ a class-dependent constant. 
\For{$t = 1,2,3,\ldots $}
\If{$\exists \pihat \in \pibm[\Mhat^s]$ s.t. $\forall M \in \cMalt(\pihat)$, $\sum_{i=1}^{s-1} \Exp_{\Mhat \sim \xi^i} \Big [ \log \frac{\Prm{\Mhat}{\pi^i}(r^i, o^i)}{\Prm{M}{\pi^i}(r^i, o^i )} \Big ]  \ge \log(t \log t)$} \label{line:test_informal}
\State Play $\pihat$. \hfill \algcommentlight{Exploit} \label{line:exploit_informal}
\Else \hfill \algcommentlight{Explore}
\State Set $p^s \leftarrow  q \lam^s +  (1-q) \omega^s$ for
\begin{align}\label{eq:alg_alloc_comp_informal}
 \lam^s, \omega^s & \leftarrow \argmin_{\lam \in \simplex_\Pi^\zeta, \omega \in \simplex_\Pi}  \sup_{M \in \cM\setminus\cMgl[\veps/6](\lam)}  \frac{1}{\Exp_{\Mhat \sim \xi^{s}}\big [\Exp_{\pi \sim \omega} \big [\Dklbig{\Mhat(\pi)}{M(\pi)} \big ] \big ]} .
\end{align}\label{line:aec_informal}
\State Draw $\pi^s \sim p^s$ and observe reward $r^s$ and observation $o^s$.
\State Compute $\Mhat^{s+1} = \Exp_{\Mhat \sim \xi^{s+1}}[\Mhat]$ for $\xi^{s+1} \leftarrow \AlgEstKL(\{ (\pi^i, r^i, o^i) \}_{i=1}^{s})$, $s \leftarrow s+1$. \label{line:update_informal}
\EndIf
\EndFor
\end{algorithmic}
\label{alg:gl_alg_main_informal}
\end{algorithm}

Finally, we present our algorithm, \mainalg, in
\Cref{alg:gl_alg_main_informal}. \mainalg relies on an \emph{online
  estimation oracle}, denoted by $\AlgEstKL$, which at every step $s$,
given access to data $\crl*{(\pi^{i},r^{i},o^{i})}_{i=1}^{s-1}$ with
$\pi^{i}\sim{}p^{i}$ and $(r^{i},o^{i})\sim{}\Mstar(\pi^{i})$ returns
a randomized estimate $\xi^{i}=\AlgEstKL\prn[\big]{\crl{(\pi^{i},r^{i},o^{i})}_{i=1}^{s-1}}\in\simplex_\cM$
with the goal of
approximating $\Mst$
\citep{foster2020beyond,foster2021statistical}. The
estimates produced by $\AlgEstKL$ must ensure the total KL estimation
error is bounded:
\begin{align}\label{eq:est_body}
\textstyle \EstKL(s) := \sum_{i=1}^s \Exp_{\Mhat \sim \xi^i}[\Exp_{\pi \sim p^i}[\Dklbig{\Mhat(\pi)}{\Mst(\pi)}]]\lesssim \bigoh(\log s).
\end{align}
We show that, under the regularity conditions required by \Cref{thm:upper_informal}, such a guarantee can be achieved, with $\bigoh(\cdot)$ hiding the log-covering number of $\cM$.

\mainalg alternates between \emph{exploit} steps and \emph{explore} steps, tracking the number of explore steps that have been performed with a counter $s\in\bbN$. For each step $t\in\bbN$, the algorithm makes use of an estimator $\Mhat^{s}=\En_{\Mhat\sim\xi^{s}}\brk{\Mhat}$, where $\xi^{s}=\AlgEstKL\big ( \big \{(\act^{i},
    r^{i},\obs^{i}) \big \}_{i=1}^{s-1} \big )$ is computed by calling
    the estimation oracle with data gathered at previous explore
    steps. Given the estimator, the algorithm first checks whether it
    has enough information to guarantee that the greedy decision is optimal, in which case it exploits (\Cref{line:exploit_informal}); otherwise it explores. In explore steps, the key component is the choice of the exploration distributions $\lam^s$ and $\omega^s$ in \eqref{eq:alg_alloc_comp_informal}, which mimics the \CompText program.\footnote{Note that we take $\lambda \in \simplex_\Pi^\zeta$ here, where $\simplex_\Pi^\zeta$ denotes the restriction of the simplex to distributions which place no more than mass $1-\zeta$ on any single point. See \Cref{sec:algorithm} for further discussion of this choice.}
To understand the role of the \CompText here, we consider two cases. In the first case, if $\lambda^s$ is an $\veps$-optimal Graves-Lai allocation for $\Mst$ (that is, $\Mst \in \cMgl(\lam^s)$), then playing $\lambda^s$ will optimize the tradeoff between minimizing regret and collecting information, and will therefore match the optimal performance prescribed by the Graves-Lai Coefficient. 

In the second case, if $\lambda^s$ is not an $\veps$-optimal Graves-Lai allocation for $\Mst$, we have $\Mst \not\in \cMgl(\lam^s)$, so $\omega^s$ will place mass on actions that ensure $\Exp_{\Mhat \sim \xi^s}[\Exp_{\pi \sim \omega}  [\Dklbig{\Mhat(\pi)}{\Mst(\pi)} ] ]$ is large. 
Since $p^s$ plays $\omega^s$ with constant probability, the quantity $\Exp_{\Mhat \sim \xi^s}[\Exp_{\pi \sim p^s}[\Dklbig{\Mhat(\pi)}{\Mst(\pi)}]]$ will also be large. 
If our estimator is consistent and \eqref{eq:est_body} holds, this can only happen a small number of times without violating \eqref{eq:est_body}; at most logarithmic in the number of exploration rounds.
As such, we can show that $\lambda^s$ must be a near-optimal Graves-Lai allocation for $\Mst$ on all but a logarithmic number of exploration rounds, and that \mainalg achieves the optimal rate on such rounds, yielding the optimal performance rate of \Cref{thm:upper_informal}. Critically, rather than exploring in a naive fashion (e.g., by sampling decisions uniformly), \mainalg explores specifically with the goal of learning a Graves-Lai optimal allocation for $\Mst$ and adapts to the structure of $\cM$ to perform this exploration efficiently. \loose

We emphasize the simplicity of \mainalg. While most existing instance-optimal
algorithms are quite complicated even for
basic settings, \mainalg relies on a few simple components
yet is far
more general than existing approaches and performs comparably or better.
See \Cref{sec:algorithm} for a full description.\loose
\awcomment{test} \dfcomment{test}
 }


\section{The \mainalg Algorithm: Regret Bounds and Examples}
\label{sec:upper}

This section presents our main algorithm and regret bounds. We begin by introducing the most basic variant of our algorithm, \mainalg, and using it to provide instance-optimal regret bounds for simple settings (\cref{sec:algorithm,sec:main}); with preliminaries in \cref{sec:regularity}. 
We then give a refined variant of the algorithm, \mainalgb, which adapts to the minimum gap $\delminst$ and leads to regret bounds under relaxed regularity conditions (\cref{sec:aest_alg} and \cref{sec:no_mingap_asm}). We use this variant to provide applications to structured and contextual bandits (\cref{sec:upper_ex}) and tabular reinforcement learning (\cref{sec:tabular_results}). We conclude with an overview of our analysis in \Cref{sec:proof_sketch}.
For all results in this section, we assume that \Cref{ass:realizability,asm:bounded_means,ass:unique_opt} hold.

\subsection{Regularity Conditions}
\label{sec:regularity}

To present our algorithm and results, we first introduce several regularity conditions for the model class $\cM$.

\paragraph{Likelihood ratios}
We next make two assumptions concerning smoothness of KL divergences and behavior of log-likelihood ratios. 

\begin{assumption}[Smooth KL]\label{asm:smooth_kl_kl}
There exists $\LKL>0$ such that for all $M,M',M'' \in \cM$ and $\pi \in \Pi$,
\begin{align*}
\left |\kl{M(\pi)}{M''(\pi)} - \kl{M'(\pi)}{M''(\pi)} \right | \le \LKL \sqrt{\Dkl{M(\pi)}{M'(\pi)}}.
\end{align*}
\end{assumption}

\begin{assumption}[Sub-Gaussian Log-Likelihood]\label{asm:bounded_likelihood}
  There exists $\VM>0$ such that for all $M, M', M'' \in\cM$ and $\pi \in \Pi$, 
\begin{align*}
\Pr_{(r,o) \sim M(\pi)} \left [ \left | \log \frac{\Prm{M'}{\pi}(r,o)}{\Prm{M''}{\pi}(r,o)} - \Exp_{(r',o') \sim M(\pi)} \left [ \log \frac{\Prm{M'}{\pi}(r',o')}{\Prm{M''}{\pi}(r',o')} \right ]  \right | \ge x \right ] \le 2 \exp(-x^2/\VM^2)
\end{align*}
for all $x\geq{}0$.
\end{assumption}
\Cref{asm:smooth_kl_kl} and \Cref{asm:bounded_likelihood} facilitate finite-sample estimation guarantees with respect to the KL divergence. Both assumptions are met by standard problem classes, including general structured bandit problems with Gaussian noise. Existing works that consider general model classes make similar assumptions \citep{dong2022asymptotic}.

\paragraph{Estimation}
To provide estimation guarantees that accommodate infinite classes $\cM$, we assume certain covering properties. We will consider the following notion of a cover.
\begin{definition}[$(\rho,\mu)$-Cover]\label{def:cover}
We say that a set $\cMcov \subseteq \cM$ is a $(\rho,\mu)$-cover of $\cM$ if there exists some event $\cE$ such that:\footnote{Note that we require that $\cMcov\subseteq\cM$, i.e. that $\cMcov$ is a \emph{proper} cover.}
\begin{enumerate}
\item $\sup_{M \in \cM} \sup_{\pi \in \Pi} \Prm{M}{\pi}(\cE^c) \le \mu$.
\item For each $M \in \cM$, there exists some $M'\in \cMcov$ such that
\begin{align*}
 \left | \log \Prm{M}{\pi}(r,o) - \log \Prm{M'}{\pi}(r,o) \right | \le \rho
\end{align*}
for all $(r,o) \in \cR \times \cO$ with $\sup_{M'' \in \cM} \Prm{M''}{\pi}(r,o \mid \cE) > 0$.
\end{enumerate}
We denote the size of the smallest such cover by $\Ncov(\cM,\rho, \mu)$. 
\end{definition}

\Cref{def:cover} states that the log-likelihoods are ``covered'' under some good event $\cE$ which occurs with high probability: for any model in the class, we can find some model in the cover with log-likelihoods that are ``close'' on $\cE$.
We assume that the covering number for the model class $\cM$ is bounded, and has reasonable (``parametric'') growth.
\begin{assumption}[Bounded Covering Number]\label{asm:covering}
For some parameters $\dcov \ge 1, \Ccov \ge 1$, we have
\begin{align*}
\log \Ncov(\cM,\rho,\mu) \le \dcov \cdot \log \left ( \frac{\Ccov}{\rho \mu} \right ).
\end{align*}
\end{assumption}
Note that the rate of growth of the covering number required by \Cref{asm:covering} is the standard rate of growth for parametric (e.g., linear) classes. Our results easily extend to accommodate general growth rates, but we adopt \cref{asm:covering} because it suffices for all of the examples we will consider, and simplifies presentation.

\paragraph{Information content of optimal decisions}
As noted in the introduction, the Graves-Lai Coefficient $\gst=\glc(\cM,\Mstar)$ can be thought of as the minimal regret needed to distinguish $\Mst$ from all possible models with different optimal decisions. As playing the optimal decision, $\pist$, incurs no regret, any allocation $\eta$ which is optimal for the Graves-Lai program, \eqref{eq:glc}, will still be optimal if we increase the number of plays of $\pist$ arbitrarily. As we are interested in finite-time behavior in this work, it is undesirable to consider allocations for which the number of pulls of optimal decisions are arbitrarily large. Instead, we would like to consider allocations which play optimal decisions only as long as they still provides useful information about models in the alternate set. The following definition gives a formal quantification of this.

\begin{definition}[Information Content of Optimal Decision]\label{def:inf_content_opt}
Fix $\epsilon \in (0,1/2]$. For a model $M \in \cM$, we define $\nmeps > 0$ as the minimum value such that, for any allocation $\eta \in \R_+^\Pi$ satisfying
\begin{align*}
(1+\veps) \gm \ge \sum_{\pi \in \Pi} \eta(\pi) \delm(\pi)  \quad \text{and} \quad \inf_{M' \in \cMalt(M)} \sum_{\pi \in \Pi} \eta(\pi) \Dkl{M(\pi)}{M'(\pi)} \ge 1 - \veps,
\end{align*}
we have
\begin{align*}
 \inf_{M' \in \cMalt(M)} \sum_{\pi \in \Pi, \pi \not\in \pibm} \eta(\pi) \Dkl{M(\pi)}{M'(\pi)} + \sum_{\pi \in \pibm} \nmeps \Dkl{M(\pi)}{M'(\pi)} \ge 1 - 2\veps.
\end{align*}
We denote $\ncMeps := \sup_{M \in \cM} \nmeps$.
\end{definition}
Intuitively, any allocation which is $\veps$-optimal for the Graves-Lai program \eqref{eq:glc} need not play any optimal decision $\pim \in \pibm$ more than $\nmeps$ times. Therefore, for model $M$, $\nmeps$ can be thought of as a quantification of the extent to which playing optimal decisions provides useful information---no additional useful information can be acquired on models in the alternate set $\cMalt(M)$ by playing optimal decisions more than $\nmeps$ times. As we will see, $\nmeps$ is bounded polynomially in problem parameters for many classes of interest.

\paragraph{Uniformly regular classes}
We refer to a class as \emph{uniformly regular} if $\ncMeps < \infty$, and the following assumption on the minimum gaps holds.

\begin{assumption}[Lower-Bounded Minimum Gap]\label{asm:mingap}
We have $\inf_{M \in \cM} \delminm > 0$. We denote by $\delmin > 0$ a (known) lower bound on $\inf_{M \in \cM} \delminm$.
\end{assumption}

Note that \Cref{asm:mingap} implies that for all $M \in \cM$, $\pim$ is unique.
For the results concerning the most basic version of our algorithm, \mainalg (\Cref{sec:algorithm,sec:main}), we assume for expositional purposes that the class $\cM$ is uniformly regular. Our more general algorithm, \mainalgb (\Cref{sec:no_mingap_asm}), achieves guarantees similar to those of \mainalg, but without uniform regularity. In particular, \mainalgb replaces dependence on $\delmin$ with the minimum gap  $\delminst\ldef\delminm[\Mstar]$ for the true model, and replaces dependence on $\ncMeps$ with  $\nepsst := \nmeps[\Mst]$. We note, however, that in cases where a lower bound $\Delta$ on the minimum gap $\delminst$ of the true model is known a-priori, \Cref{asm:mingap} can be satisfied by restricting the model class to models with minimum gap at least $\Delta$.

\subsection{The \mainalg Algorithm}
\label{sec:algorithm}
We now present the most basic variant of our main algorithm, \mainalg (\cref{alg:gl_alg_main}). This will serve as the starting point for the most general version of our algorithm, \mainalgb (\cref{sec:no_mingap_asm}). To describe the algorithm, we first introduce the primitive of an \emph{online estimation oracle} \citep{foster2020beyond,foster2021statistical}.

\paragraph{Estimation oracles}
\cref{alg:gl_alg_main} makes use of an \emph{online estimation oracle}, denoted by
$\AlgEstKL$, which is an algorithm that, given knowledge of the class $\cM$, estimates the underlying model $\Mstar\in\cM$ from
data in a sequential fashion. When invoked at step $s\in\bbN$ with the data
$(\pi^1,r^1,o^1),\ldots,(\pi^{s-1},r^{s-1},o^{s-1})$
observed so far, the estimation oracle builds an estimate
\[
\Mhat^{s}  = \AlgEstKL\prn*{ \crl*{(\act^{i},
r^{i},\obs^{i})}_{i=1}^{s-1} }
\]
which aims to approximate the true model $\Mstar$. Following \citep{foster2021statistical,chen2022unified,foster2022note}, we make use of \emph{randomized} estimation oracles that, at each step produces $\xi^{s}=\AlgEstKL\big ( \crl*{(\act^{i},
    r^{i},\obs^{i})}_{i=1}^{s-1} \big )$, where $\xi^{s}\in\simplex_\cM$ is a randomization distribution, and draw $\Mhat\sim\xi^{s}$. We measure the oracle's performance in terms of cumulative estimation error, defined as follows.
\begin{definition}[Cumulative Estimation Error]
  \label{def:est_error}
Consider the process where, for each round $i\in\bbN$, given $(\pi^{1},r^{1},o^{1}),\ldots,(\pi^{i-1},r^{i-1},o^{i-1})$ with $\pi^{i}\sim{}p^{i}$ and $(r^{i},o^{i})\sim{}\Mstar(\pi^{i})$, the estimation oracle returns $\xi^{i}=\AlgEstKL\prn*{ (\pi^{1},r^{1},o^{1}),\ldots,(\pi^{i-1},r^{i-1},o^{i-1})}$. For any $s\in\bbN$, we define the oracle's cumulative KL estimation error under this process as:
\begin{align*}
\EstKL(s) := \sum_{i=1}^s \Exp_{\Mhat \sim \xi^i}\brk*{\Exp_{\pi \sim p^i}\brk*{\Dklbig{\Mst(\pi)}{\Mhat(\pi)}}}.
\end{align*}
\end{definition}
\cref{alg:gl_alg_main} can be invoked with any off-the-shelf algorithm for estimation, but our main results make use of the fact that under \Cref{asm:bounded_likelihood} and \Cref{asm:covering}, there exists an estimation oracle $\AlgEstKL$ (\cref{alg:estimator} in \cref{sec:upper_est_proofs}) which ensures that with probability at least $1-\delta$, for all $s\in\bbN$:
\begin{align}\label{eq:est_oracle_bound}
\EstKL(s) \lesssim \VM \cdot\dcov \cdot \log^{3/2} \prn*{\frac{\Ccov \cdot s}{\delta} }.
\end{align}
That is, the estimation oracle ensures that the KL divergence between the true model $\Mst$ and the estimates returned scales at most poly-logarithmically in the exploration horizon. Note that on its own, this guarantee does not necessarily imply that $\Mhat^s = \Exp_{M \sim \xi^s}[M] \rightarrow \Mst$---low online estimation error only requires that $\Mhat^s$ is on average close to $\Mst$ on the decisions we have actually played.

\begin{algorithm}[tp]
\caption{\MainAlg (\mainalg)}
\begin{algorithmic}[1]
\State \textbf{input:} optimality tolerance $\veps$, model class $\cM$.
\State Initialize $s \leftarrow 1$, $\nmax \leftarrow \nmax(\cM,\veps/6)$, and $ q \leftarrow \frac{4\nmax + \veps \guncM}{4\nmax + 2 \veps \guncM}$ for $\guncM := \inf_{M \in \cM : \gm > 0} \gm$. 
\State Compute $\xi^1 \leftarrow \AlgEstKL(\{ \emptyset \})$ and $\Mhat^1 \leftarrow \Exp_{M \sim \xi^1}[M]$.
\For{$t = 1,2,3,\ldots $}
\If{$\exists \pim[\Mhat^s] \in \pibm[\Mhat^s]$ s.t. $\forall M \in \cMalt(\pim[\Mhat^s])$, $\sum_{i=1}^{s-1} \Exp_{\Mhat \sim \xi^i} \Big [ \log \frac{\Prm{\Mhat}{\pi^i}(r^i, o^i)}{\Prm{M}{\pi^i}(r^i, o^i )} \Big ]  \ge \log(t \log t)$} \arxiv{\hfill \algcommentlight{Exploit}} \label{line:test}
\State Play $\piMhats$. \colt{\hfill \algcommentlight{Exploit}}\label{line:exploit}
\Else \hfill \algcommentlight{Explore}
\State Set $p^s \leftarrow  q \lam^s +  (1-q) \omega^s$ for
\begin{align}\label{eq:alg_alloc_comp}
 \lam^s, \omega^s & \leftarrow \argmin_{\lam, \omega \in \simplex_\Pi}  \sup_{M \in \cM\setminus\cMgl[\veps/6](\lam; \nmax)}  \frac{1}{\Exp_{\Mhat \sim \xi^{s}}\big [\Exp_{\pi \sim \omega} \big [\Dklbig{\Mhat(\pi)}{M(\pi)} \big ] \big ]} .
\end{align}\label{line:aec}
\State Draw $\pi^s \sim p^s$ and observe reward $r^s$ and observation $o^s$.\label{line:sample}
\State Compute estimate $\xi^{s+1} \leftarrow \AlgEstKL(\{ (\pi^i, r^i, o^i) \}_{i=1}^{s})$ and $\Mhat^{s+1} = \Exp_{\Mhat \sim \xi^{s+1}}[\Mhat]$ .\label{line:update}
\State $s \leftarrow s + 1$.
\EndIf
\EndFor
\end{algorithmic}
\label{alg:gl_alg_main}
\end{algorithm}

\paragraph{Algorithm overview}
\iftoggle{colt}{We now present a formal overview of our algorithm \mainalg. We restate it here for convenience in \Cref{alg:gl_alg_main}.
}{We are now ready to give an overview of \cref{alg:gl_alg_main}.}
The algorithm alternates between \emph{exploit} steps and \emph{explore} steps, tracking the number of explore steps that have been performed with a counter $s\in\bbN$. For each step $t\in\bbN$, the algorithm makes use of an estimator $\Mhat^{s}=\En_{\Mhat\sim\xi^{s}}\brk{\Mhat}$, where $\xi^{s}=\AlgEstKL\big ( \big \{(\act^{i},
    r^{i},\obs^{i}) \big \}_{i=1}^{s-1} \big )$ is computed by calling the estimation oracle with data gathered at previous explore steps. Given the estimator, \mainalg performs a test based on likelihood ratios (\cref{line:test}) to check whether it has collected enough information to rule out all models for which $\piMhats\in\pibm[\Mhat\ind{s}]$ is not an optimal decision. If so, it \emph{exploits}, and plays $\piMhats$ (\cref{line:exploit}), as in this case $\piMhats = \pist$ with high probability. 
If the test fails, the algorithm must gather more information to eliminate alternatives, and it \emph{explores} (\cref{line:aec}). The key component of the explore phase is the choice of the exploration distribution in \eqref{eq:alg_alloc_comp}, which is based on the \CompText program, but incorporates some small modifications: 1) First, $\Mhat$ is randomized according to the distribution $\xi^{s}$, 2) Second, the set $\Mst \in \cMgl(\lam)$ is replaced with a smaller set $\Mst \in \cMgl(\lam;\nmax)$, which requires that $\lambda$ obeys certain normalization constraints; this is detailed below. Using the distributions $\lambda^{s}$ (representing a normalized allocation) and $\omega^{s}$ (representing an exploration distribution) returned in \eqref{eq:alg_alloc_comp}, the algorithm computes a mixture $p^{s}=q\lambda^{s}+(1-q)\om^{s}$, where $q\in(0,1)$ is a carefully chosen parameter, and plays $\pi^{s}\sim{}p^{s}$ (\cref{line:sample}). The reward and observation $(r^{s},o^{s})$ that result from playing $\pi^{s}$ are then used to update the estimation oracle for subsequent rounds (\cref{line:update}).

To understand the intuition behind the explore phase and why the \CompText plays a useful role here, we can consider two cases. In the first case, if $\lambda^s$ is an $\veps$-optimal Graves-Lai allocation for $\Mst$ (that is, $\Mst \in \cMgl[\veps/6](\lam^s;\nmax)$), then playing $\lambda^s$ will optimize the tradeoff between minimizing regret on $\Mst$ and collecting information that allows one to distinguish $\Mst$ from $M \in \cMalt(\Mst)$, and will therefore match the optimal performance prescribed by the Graves-Lai Coefficient, incurring regret scaling as $\gst$. 

In the second case, if $\lambda^s$ is not an $\veps$-optimal Graves-Lai allocation for $\Mst$, we have $\Mst \not\in \cMgl[\veps/6](\lam^s;\nmax)$, so by the definition of the \CompShort (\eqref{eq:alg_alloc_comp}), $\omega^s$ will place mass on actions that ensure $\Exp_{\Mhat \sim \xi^s}[\Exp_{\pi \sim \omega}  [\Dklbig{\Mhat(\pi)}{\Mst(\pi)} ] ]$ is large; exactly how large this quantity is will be quantified by the value of the \CompShort. Since $p^s$ plays $\omega^s$ with constant probability, the quantity
\begin{align*}
\Exp_{M \sim \xi^s}[\Exp_{\pi \sim p^s}[\Dkl{\Mst(\pi)}{M(\pi)}]] 
\end{align*}
will also be large, but if the estimation oracle is consistent in the sense of \cref{def:est_error}, this can only happen a small number of times. In particular, if \eqref{eq:est_oracle_bound} holds, the number of times in which we encounter this second case is at most \emph{logarithmic} in the number of exploration rounds. As such, we can show that $\lambda^s$ must be a near-optimal Graves-Lai allocation for $\Mst$ on all but a logarithmic number of exploration rounds, and that \mainalg achieves the optimal rate on such rounds.

Critically, rather than exploring in a naive fashion (e.g., by sampling decisions uniformly), \mainalg explores only to the extent necessary to learn a Graves-Lai allocation for $\Mst$. There may exist instances $M \neq \Mst$ which differ significantly from $\Mst$ but have a similar Graves-Lai allocations---\mainalg will make no effort to distinguish such instances since, as long as it knows that one of these instances is correct, it can simply play their shared Graves-Lai allocation. This notion of exploration, which is targeted toward distinguishing instances that have different Graves-Lai allocations, is precisely the notion captured by the \CompText.

\paragraph{Normalization factor for allocations}
As noted in \Cref{sec:regularity}, while an optimal Graves-Lai allocation may place an arbitrarily large number of pulls on an optimal decision, for finite-time guarantees it is useful to restrict to allocations which place only finite mass on optimal decisions. To this end, \mainalg restricts the optimization problem based on the \CompShort in \eqref{eq:alg_alloc_comp} to only consider normalized allocations $\lambda$ for which the \emph{normalization factor} is at most 
\begin{align}
  \nmax(\cM,\veps) := \frac{64}{\delmin^2} \cdot \prn*{\frac{1}{\veps} +  \VM \ncMeps } \cdot \max_{M \in \cM} \gm .
  \label{eq:nmax}
\end{align}
where the normalization factor refers to the value $\nsf$ in the definition of $\Lambda(M;\veps)$ (see \eqref{eq:lambda_allocation_set}). 
In particular, to enforce this restriction, the optimization problem in \eqref{eq:alg_alloc_comp} restricts the max-player to  $M\in\cM \backslash \cMgl[\veps/6](\lam;\nmax)$, where $\nmax\ldef\nmax(\cM,\veps/6)$ and $\cMgl[\veps/6](\lam;\nmax)$ is defined identically to $\cMgl[\veps/6](\lam)$ in \eqref{lambda-set}, but with $\nsf$ restricted to $\nsf \le \nmax$. As we show in \Cref{lem:nM_aec_bound}, for $\nmax(\cM,\veps)$ defined as in \eqref{eq:nmax}, the optimal value in \eqref{eq:alg_alloc_comp} can be bounded by the \CompShort.

\paragraph{Computational efficiency}
The primary computational burden in \mainalg lies in solving the optimization problem \cref{eq:alg_alloc_comp} to compute the exploration distributions. In general there is little hope of solving this efficiently (i.e., in time sublinear in $\abs{\Pi}$ and $\abs{\cM}$)---indeed, in some cases it may be that to even determine whether $M \in \cMgl(\lambda)$ will require enumerating the model class $\cM$. However, for nicely structured problems, we anticipate that this program can be solved, or at least approximated, efficiently. As the focus of our work is primarily statistical, we leave further exploration as to when the algorithm can be implemented efficiently to future work. 

\paragraph{Simplicity}
We emphasize the simplicity of \mainalg. Most existing algorithms which achieve instance-optimality are quite complex, even in specialized settings such as linear bandits. In contrast, \mainalg is very simple and intuitive, and relies only on three basic components: an explore-exploit test, an estimation oracle, and a single optimization to compute the exploration distributions. Despite its simplicity, as we show, \mainalg obtains comparable or better performance over existing approaches.

\paragraph{Relation to existing approaches}
At a very high level, \mainalg bears some similarity to the $\mathsf{E2D}$ algorithm of \cite{foster2021statistical}, which achieves the \emph{minimax} optimal rate for general classes $\cM$ in the DMSO framework. Both algorithms rely on online estimation algorithms, and both solve min-max programs based on the output of the estimator to determine which allocations to play. However, the algorithm design and analysis principles for the two algorithms, and in particular the motivation for the min-max programs they solve, differ significantly.

\subsection{\mainalg Algorithm: Regret Bound for Uniformly Regular Classes}
\label{sec:main}
We present upper bounds for \mainalg in the setting where our class $\cM$ is uniformly regular: \Cref{asm:mingap} holds and $\ncMeps < \infty$; these assumptions are relaxed by the \mainalgb algorithm in the sequel.
To state the regret bound for \mainalg in the tightest form possible, we introduce the following variant of the \CompText, which incorporates randomized estimators $\xi\in \simplex_\cM$:
\begin{align}\label{eq:aecflip_def}
\aecflipM{\veps}{\cM}(\cMsub,\xi) := \inf_{\lambda,\omega \in \simplex_\Pi} \sup_{M \in \cM \backslash \cMgl(\lambda)} \frac{1}{\Exp_{\Mbar \sim \xi}[\Exp_{\pi \sim \omega}[\Dkl{\Mbar(\pi)}{M(\pi)}]]},
\end{align}
with $\aecflip(\cM,\xi) := \aecflipM{\veps}{\cM}(\cM,\xi)$ and $\aecflip(\cM) := \sup_{\xi \in \simplex_\cM} \aecflip(\cM,\xi)$. Note that one can always bound $\aecflip(\cM) \le \aec(\cM)$ due to the convexity of the KL divergence. In fact, these definitions are equivalent up to dependence on problem-dependent parameters in \cref{sec:regularity} (indeed, our lower bounds in \cref{sec:lower} scale with the latter quantity), but the former can be simpler to bound for some of the examples we consider.

Our main theorem concerning the performance of \mainalg is as follows.
\begin{theorem}[Regret Bound for \mainalg]\label{thm:upper_main}
For any $\veps\in (0, 1/2]$, there exists a choice for the estimation oracle $\AlgEstKL$ such that for all $T\in\bbN$, under \Cref{asm:smooth_kl_kl,asm:bounded_likelihood,asm:covering,asm:mingap} and if $\gst > 0$,
the expected regret of \mainalg is bounded by
\begin{align}
  \Exp\sups{\Mst}[\RegDM] & \le (1+\veps) \cst \cdot \log (T) + \aecflip[\veps/12](\cM)\cdot\Caec \cdot \log^{3/2}( \log T )  +  \Clo \cdot \log^{1/2} (T),
      \label{eq:upper_main}
\end{align}
where
\begin{align*}
\Caec & := c \cdot \frac{\VM^2 \dcov \log(\Ccov) \cdot \max_{M \in \cM} \gm}{\veps \delmin^3} \cdot \prn*{\veps^{-1} + \VM \ncMeps[\veps/6]} \cdot  \log(\Clo),
\end{align*}
for a universal numerical constant $c>0$, and $\Clo$ is a lower-order constant given by
\begin{align*}
\Clo & :=\linear \left ( \max_{M \in \cM} \gm, \aec[\veps/12]^{1/2}(\cM), \tfrac{1}{\veps^2}, \tfrac{1}{\delmin^3}, \ncMeps[\veps/6] , \LKL^2, \VM^{13/2}, \dcov,\log (\Ccov),  \log \log T \right ),
\end{align*}
where $\linear (\cdot )$ denotes a function multi-linear and poly-logarithmic in its arguments.
\end{theorem}
We prove \Cref{thm:upper_main} in \Cref{sec:regret_delmin_proofs}, and give a proof sketch in \cref{sec:proof_sketch}.
\Cref{thm:upper_main} shows that \mainalg achieves the asymptotically optimal Graves-Lai rate for $\Mstar$, as given in \Cref{prop:glc}, up to a $(1+\veps)$ approximation factor. In more detail, if we label the terms in \eqref{eq:upper_main} as
\newcommand{\I}{(\mathrm{I})}
\newcommand{\II}{(\mathrm{II})}
\newcommand{\III}{(\mathrm{III})}
\[
  \Exp\sups{\Mst}[\RegDM]\leq   \underbrace{(1+\veps) \cst \cdot \log (T)}_{\I} + \underbrace{\aecflip[\veps/12](\cM)\cdot\Caec \cdot \log^{3/2}( \log T )}_{\II}  +  \underbrace{\Clo \cdot \log^{1/2} (T)}_{\III},
\]
the regret bound can be seen to consist of:
\begin{itemize}
\item The \emph{leading-order term} $\I = (1+\veps) \cst \cdot \log(T)$. This is the only term that scales linearly with $\log(T)$, as a consequence we have $\lim_{T\to\infty}\frac{\Exp\sups{\Mst}[\RegDM]}{\log(T)}\leq{}(1+\veps)\cst$, which matches the instance-optimal rate given in \Cref{prop:glc} up to a factor of $(1+\veps)$.
\item A lower-order term $\II = \aecflip[\veps/12](\cM)\cdot\Caec \cdot \log^{3/2}( \log T )$, which is polylogarithmic in $\log(T)$, and
  scales with $\aecflip[\veps/12](\cM)$, as well as regularity parameters from \cref{sec:regularity}.
\item A second lower-order term $\III=\Clo \cdot \log^{1/2} T$. This term scales with $\log^{1/2}(T)=o(\log(T))$ and, like the term $\II$, scales with the \CompShort and regularity parameters from \cref{sec:regularity}. Compared to $\II$, this term has worse dependence on $\log(T)$, but enjoys sublinear $\aecflip[\veps/12]^{1/2}(\cM)$ scaling with the \CompShort.
\end{itemize}
Critically, both of the $o(\log T)$ lower-order terms above do not scale with (often exponentially large) terms such as $|\Pi|$ or $|\cM|$ found in prior work, and instead scale principally with $\aecflip(\cM)$, which, as we will show in \Cref{sec:lower}, is unavoidable in a certain sense. In particular, note that once $T$ is large enough that 
\begin{align*}
\log(T) \ge \specialOmt(\aecflip[\veps/12](\cM)),
\end{align*}
the leading-order term $\I = (1+\veps) \gst \cdot \log(T)$ term in \cref{thm:upper_main} will dominate the regret. This is precisely the time horizon given by the lower bound in \cref{thm:lower_informal}, which is necessary for an algorithm to learn a near-optimal allocation for the Graves-Lai program. We offer a more thorough comparison of \Cref{thm:upper_main} with our lower bounds in \Cref{sec:gaps}. Below, we discuss the lower-order terms and asymptotic performance in greater detail.

\begin{remark}[Additional Lower-Order Terms]
The lower-order terms in \Cref{thm:upper_main} depend on the model class $\cM$ through the regularity, covering, and smoothness assumptions, as well as the minimum gap (\cref{sec:regularity}). For many of the examples we consider, the \CompText will dominate these other terms, yet there may exist classes where this is not the case. Resolving the optimal dependence on these problem-dependent parameters in the lower-order terms, as well as understanding when these parameters are necessary, remains an interesting direction for future work. 

In addition, let us mention that while both lower-order terms scale with $o(\log T)$ (note that the scaling is no larger than $\bigoh(\sqrt{\log T} \cdot \polylog \log T)$), it is not clear what the optimal dependence on $T$ should be for the lower-order terms. For example, one might hope to replace the dependence on $\log^{1/2}(T)$ with $\log^{a}(T)$ for some constant $a<1/2$, or even with $\polylog(\log(T))$. Precisely characterizing the optimal $\log(T)$ scaling for lower-order terms remains an interesting open question. To this end, we remark that \cite{jun2020crush} show that in some cases, an $\Omega(\log \log T)$ term is indeed necessary.
\end{remark}

\begin{remark}[Asymptotic Performance]
  Asymptotically, as $T \rightarrow \infty$, the regret of \mainalg scales with $(1+\veps) \cst \cdot \log(T)$, which is a factor of $(1+\veps)$ off from the asymptotic lower bound in \cref{prop:glc}. For any fixed $T\in\bbN$ of interest, as long as $\aec(\cM)=\poly(\veps^{-1})$, one can obtain an asymptotic constant of $1$ by choosing $\veps = 1/\log^{a}(T)$ for a sufficiently small constant $a>0$. For example, when $\abs{\Pi}<\infty$, it is always possible to bound $\aec(\cM) \lesssim \poly(\abs{\Pi})/\veps^4$ (see \Cref{prop:aec_to_Cexp} below), so choosing $\veps$ as above ensures that the lower-order terms scale $o(\log T)$, while the leading-order term scales as $\cst \cdot \log(T)$ asymptotically.
\end{remark}

\subsubsection{Example: Searching for an Informative Arm}
We next provide an example of a uniformly regular class in order to illustrate a case where \Cref{thm:upper_main} holds. In particular, we revisit the \emph{informative arm} setting described in the introduction (\cref{ex:revealing}). Recall that we exhibited a model class for which the complexity of learning the \alloc is not governed by existing complexity measures, and can be larger than the minimax optimal rate for learning with $\cM$. In what follows, we show that on this example the \CompText correctly adapts to the complexity of this model class.
We emphasize that the main applications of our results, which take advantage of the more general \mainalgb algorithm, will be given in \Cref{sec:upper_ex} and \Cref{sec:tabular_results}, and we also present additional examples of uniformly regular classes in \Cref{sec:discrete_structured_bandit}.

\begin{example}[Searching for an Informative Arm (revisited)]\label{ex:informative_arm_upper}
Let $\cM$ denote the model class constructed in \Cref{ex:revealing}, with parameters $A,N \ge 5$ and $\beta \in [4/A,9/10]$. We additionally discretize the space so that, for each $M \in \cM$ and $\pi \in [A]$, we have $\fm(\pi) \in \{0,\delmin,2\delmin,\ldots, \lfloor \frac{1}{\delmin} \rfloor \delmin \}$,\footnote{The discretization is required to satisfy the uniform regularity assumption---we include it here to provide a concrete example of a uniformly regular class. However, it can be shown that, without this discretization assumption, \Cref{thm:upper_main_no_mingap} applies to the original formulation given in \Cref{ex:revealing} with the \CompShort again scaling as $\bigoh(N)$.} and furthermore restrict $\cM$ so that it does not include instances with multiple optimal arms.
For this class, one can show that \Cref{asm:smooth_kl_kl,asm:bounded_likelihood,asm:covering,asm:mingap} hold with $\LKL,\VM \le \bigoh(\log A)$ and $\dcov = \bigoh(A), \Ccov = \bigoh(N)$ (see \Cref{sec:inf_arm_proofs}). Furthermore, we can bound $\ncMeps \le \frac{2}{\delmin^2}$, and
\begin{align*}
\aec(\cM) \le \frac{64N}{\beta^2} + \frac{16A}{\delmin^2}.
\end{align*}
As a result, for this class, \mainalg has expected regret bounded as
\begin{align*}
\Exp\sups{\Mst}[\RegDM] & \le (1+\veps) \cst \cdot \log(T) + N \cdot \poly(A,\tfrac{1}{\veps},\tfrac{1}{\delmin}, \log N,\log\log T) \cdot \log^{1/2} (T) .
\end{align*}
Note that here the only term that scales linearly with $N$ is the \CompText---every other class-dependent term appearing in the regret bound scales at most logarithmically in $N$. 
We are particularly interested in situations where the cost of finding the correct informative arm is much larger than any existing complexity measures for the problem: that is, when $\beta$ is constant and $N \gg A, \frac{1}{\delmin}$. In this case, we have $\aec(\cM) \le \bigoh(N)$, and the \CompText correctly captures the intuitive complexity of learning the optimal allocation. In particular, the dependence on $N$ reflects the fact that we need to test each informative arm at least once. Furthermore, as we show in \Cref{ex:revealing_revisited}, we can lower bound $\aec(\cM) \ge \Omega(N)$ as well, so in the regime where $N \gg A, \frac{1}{\delmin}$, the \CompShort is the dominant lower-order term. 
\end{example}

See \Cref{sec:gauss_bandits_proof} for the proof of this example.

\subsection{The \mainalgb Algorithm}\label{sec:aest_alg}

\begin{algorithm}[h]
\caption{Adaptive Exploration for Allocation Estimation for classes
  without uniform regularity (\mainalgb)}
\begin{algorithmic}[1]
\State \textbf{input:} Optimality tolerance $\veps$, estimation oracle $\AlgEstKL$, growth parameters $\alpha_q$, $\alpha_n$, $\alpha_\cM\geq0$.
\State $s \leftarrow 1$, $\ell \leftarrow 1$, $\delta \leftarrow \frac{\veps}{4 + 2 \veps}$, $q^s \leftarrow 1- s^{-\alphaq}$, $\nsf^s \leftarrow s^{\alphan}$.
\State Compute $\xi^1 \leftarrow \AlgEstKL(\{ \emptyset \})$ and $\Mhat^1 \leftarrow \Exp_{M \sim \xi^1}[M]$.
\For{$t = 1,2,3,\ldots $}
\If{$s \ge 2^\ell$} \hfill \algcommentlight{Form active set and cover}
\State $\ell \leftarrow \ell + 1$.
\State $\Delta^\ell \leftarrow \argmin_{\Delta \ge 0} \Delta \quad \text{s.t.} \quad \aecM{\delta/2}{\cM}(\cM_{\Delta,\frac{1}{\Delta}}) \le s^{\alphaM}$. \label{line:cMell_aec_bound_main}
\State $\cM^\ell \leftarrow \cM_{\Delta^\ell, \frac{1}{\Delta^\ell}} \cap \big \{ M \in \cM \ : \ \nmepsb + \frac{1}{\delminm} + \frac{4 \gm}{\delminm} + \frac{2 \nmepsb}{\gm} + \frac{4}{\delminm \delta} \le \sqrt{\nsf^s} \big \}$. \label{line:cMell_defn_main}
\State $\cMcov^\ell \leftarrow$ $(\rho_\ell,\mu_\ell)$-cover of $\cM^\ell$ for $\rho_\ell \leftarrow 2^{-\ell}, \mu_\ell \leftarrow 2^{-5\ell}$, $\frakD^\ell \leftarrow \emptyset$.
\EndIf
\If{$\exists \pim[\Mhat^s] \in \pibm[\Mhat^s]$ s.t. $\forall M \in \cMalt(\pim[\Mhat^s])$, $\sum_{i=1}^{s-1} \Exp_{\Mhat \sim \xi^i} \Big [ \log \frac{\Prm{\Mhat}{\pi^i}(r^i,o^i )}{\Prm{M}{\pi^i}(r^i,o^i )}\Big ] \ge \log(t \log t)$} \arxiv{\hfill \algcommentlight{Exploit}}
\State Play $\piMhats$. \colt{\hfill \algcommentlight{Exploit}}
\Else \hfill \algcommentlight{Explore}
\State Set $p^s \leftarrow  q^s \lam^s +  (1-q^s) \omega^s$ for
\begin{align}\label{eq:alg_alloc_comp2_main} 
 \lam^s, \omega^s & \leftarrow \argmin_{\lam, \omega \in \simplex_\Pi} \sup_{M \in \cM^\ell \backslash \cMgl(\lam; \nsf^s)} \frac{1}{\Exp_{\Mhat \sim \xi^{s}}[\Exp_{\pi \sim \omega}[\Dklbig{\Mhat(\pi)}{M(\pi)}]]}.
\end{align}
\State Draw $\pi^s \sim p^s$, observe $r^s, o^s$, set $\frakD^\ell \leftarrow \frakD^\ell \cup \{ (\pi^s,r^s,o^s) \}$.
\State Compute estimate $\xi^{s+1} \leftarrow \AlgEstKL(\frakD^\ell,\cMcov^\ell)$ and $\Mhat^{s+1} = \Exp_{M \sim \xi^{s+1}}[M]$.
\State $s \leftarrow s + 1$.
\EndIf
\EndFor
\end{algorithmic}
\label{alg:gl_alg_main_infinite_nomingap_main}
\end{algorithm}

While it may be reasonable to assume that the minimum gap of $\Mst$, $\delminst$, is bounded away from 0, and that the amount of useful information playing $\pist$ provides is also bounded on $\Mst$, assuming that this is true for every model in the model class (as in the prequel) is a significantly stronger assumption. For example, if we let $\cM$ denote the space of all multi-armed bandits with means in $[0,1]$, the only possible value of $\delmin$ is 0, as we can always find some instance with minimum gap arbitrarily close to 0. In this section, we dispense with the uniform regularity assumption:  we relax \Cref{asm:mingap}, and additionally prove that it suffices if only $\nepsst := \nmeps[\Mst]$ (as opposed to $\ncMeps$) is bounded.

Our main algorithm for this section, \mainalgb, is given in \cref{alg:gl_alg_main_infinite_nomingap_main}. It is very similar to \mainalg but to remove the requirement of uniform regularity, the algorithm avoids solving \eqref{eq:alg_alloc_comp} over the entire model class $\cM$, and instead solves it over a carefully restricted model class. For $x,y>0$, define
\begin{align}
  \label{eq:cMxy_def}
\cM_{x,y} := \{ M \in \cM \ : \ \delminm \ge x, \nmepsb \le y \}.
\end{align}
\mainalgb breaks its explore rounds into doubling epochs. For each epoch $\ell$, \eqref{eq:alg_alloc_comp2_main} in \Cref{alg:gl_alg_main_infinite_nomingap_main} solves an \CompShort-like optimization problem over a restricted class $\cM^\ell\subseteq\cM_{\Delta^\ell, \frac{1}{\Delta^\ell}}$, which is chosen in \Cref{line:cMell_defn} to explicitly ensure that the value of the optimization problem in \eqref{eq:alg_alloc_comp2_main}, is bounded; this is guaranteed by the definition of $\Delta^\ell$ in \Cref{line:cMell_aec_bound}. 
Similar to \mainalg, the value of the optimization in \eqref{eq:alg_alloc_comp2_main} quantifies how much information we are gaining about the Graves-Lai allocation of $\Mst$, and the regret of the explore phase can be bounded in terms of the value of this optimization. By restricting $\cM^\ell$ so that the value of \eqref{eq:alg_alloc_comp2_main} is always bounded, we can therefore ensure that the regret during the exploration phase is bounded. 

Intuitively, this restriction of $\cM^\ell$ reduces the space of models we must distinguish $\Mst$ from in order to identify its Graves-Lai allocation: rather than distinguishing $\Mst$ from all models in $\cM$, we must only distinguish it from models in $\cM^\ell$, which could be significantly easier. The caveat is that, since we do not know the value of $\nepsst$ or $\delminst$, $\Mst$ may not always be in $\cM^\ell$. In such cases, little can be said about the exploration phase---we are not able to provide any meaningful guarantees on how much information $\lam^s$ and $\omega^s$ acquire about $\Mst$. To mitigate this, as $s$ increases we gradually relax the criteria for inclusion in $\cM^\ell$, ensuring that for large enough $s$, $\Mst$ will be in $\cM^\ell$. In particular, one can show that the number of exploration rounds needed to guarantee $\Mst \in \cM^\ell$ scales with $\aecM{\veps}{\cM}(\cMst)$, for
\begin{align}\label{eq:cMst_def}
\cMLdelst := \big \{ M \in \cM \ : \ \delminm \ge \DelLst, \nmepsc{\veps/36}{M} \le 1/\DelLst \big \} \quad \text{for} \quad \DelLst := \min \{ \delminst, 1/\nsf^\star_{\veps/36} \}.
\end{align}
That is, $\cMLdelst$ is the restriction of $\cM$ to models with gap at least $\min \{ \delminst, 1/\nsf^\star_{\veps/36} \}$ (implying all models in $\Mst$ have a unique optimal decision), and for which the information content of the optimal decision is at most $\max \{ 1/\delminst, \nsf^\star_{\veps/36} \}$.

\paragraph{Estimation oracle}
While \mainalg simply requires that the estimation oracle $\AlgEstKL$ returns randomized estimators supported on $\simplex_{\cM}$, for \mainalgb, we wish to ensure that the estimators produced are instead only supported on $\cM^\ell$. To this end, we restrict the estimator to $\cMcov^\ell$, the $(\rho_\ell,\mu_\ell)$-cover of $\cM^\ell$. We denote the resulting estimation oracle with $\AlgEstKL(\frakD^\ell,\cMcov^\ell)$, where the first argument represents the set of available observations, and the second argument the set over which the estimation oracle must return an estimate.

\paragraph{Computational efficiency of \mainalgb}
Similar to \mainalg, it is not clear how to solve the main optimization required by \mainalgb, \eqref{eq:alg_alloc_comp2_main}, in general. In addition, unlike \mainalg, \mainalgb maintains a version space of models, $\cM^\ell$, which could increase the computational burden further. We emphasize that the focus of this work is primarily statistical, and leave addressing the computational challenges for future work.

\subsection{\mainalgb Algorithm: Regret Bound without Uniform Regularity}\label{sec:no_mingap_asm}

The following theorem provides the main guarantee for \mainalgb.

\begin{theorem}[Regret Bound for \mainalgb]\label{thm:upper_main_no_mingap}
  For any $\veps \in (0,1/2]$, if \Cref{asm:smooth_kl_kl,asm:bounded_likelihood,asm:covering} hold and $\gst > 0$, \mainalgb (\cref{alg:gl_alg_main_infinite_nomingap_main}) ensures that for all $T\in\bbN$, the expected regret is bounded as
\begin{align*}
\Exp\sups{\Mst}[\RegDM] & \le (1+\veps) \cst \cdot \log(T) + \Big ( \aecflipM{\veps/12}{\cM}(\cMLdelst)\Big )^3 \cdot \Caec \cdot \log^{3/2}( \log T )  +  \Clo \cdot \log^{6/7} (T),
\end{align*}
where
\begin{align*}
\Caec & := \bigoht \prn*{\frac{\VM^3(\VM+\LKL) \cdot \dcov \log(\Ccov)}{\veps \delminst} },
\end{align*}
and $\Clo$ is a lower-order constant given by
\begin{align*}
\Clo & :=\poly \left ( \cst, \tfrac{1}{\delminst}, \nsf_{\veps/6}^\star,  \tfrac{1}{\veps}, \VM, \LKL, \dcov, \log \Ccov, \log \log T \right ).
\end{align*}
\end{theorem}

The proof of \Cref{thm:upper_main_no_mingap} is given in \cref{sec:upper_proofs_no_mingap}. As \Cref{thm:upper_main_no_mingap} illustrates, at the expense of a slightly larger polynomial dependence on the \CompText, and slightly larger lower-order terms, we can obtain near instance-optimal regret---matching the instance-optimal lower bound given in \Cref{prop:glc} up to a factor of $(1+\veps)$---without requiring any assumption on the minimum gap, or boundedness of $\ncMeps$. Rather than scaling with the minimum gap for the entire class, $\delmin$,  \Cref{thm:upper_main_no_mingap} scales only with the minimum gap of the ground truth model, $\delminst$, which could be substantially larger than $\delmin$. An additional advantage of \Cref{thm:upper_main_no_mingap} is that it scales with $\aecM{\veps}{\cM}(\cMst)$ as opposed to $\aec(\cM)$; for the examples we consider in the sequel, the former quantity enjoys better dependence on problem-dependent parameters. For example, we show in \Cref{sec:lower} that for standard classes, $\aec(\cM)$ can scale with the minimum gap amongst all models in the class. On the other hand $\aecM{\veps}{\cM}(\cMst)$ typically scales with $\delminst$.

We emphasize that \mainalgb \emph{requires no prior knowledge of $\delminst$ or $\nepsst$}---it is able to adapt to the minimum gap and regularity of the underlying model.

\begin{remark}[Dependence on $\nepsst$]
As we show in the following examples, $\nepsst$ is typically bounded polynomially in standard problem parameters, though in practice this needs to be verified for each problem instance. We remark that some scaling in terms of $\nepsst$ seems unavoidable---if there is a significant amount of information to be gained playing the optimal decision, any algorithm which is nearly instance-optimal will play the optimal decision at least $\nepsst$ times, and therefore the ``effective horizon'' to eliminate alternate instances scales with $\nepsst$. As we are the first to formalize this notion of how informative the optimal decision is, we believe more research in this direction is required.
\end{remark}

\subsection{Application: Structured and Contextual Bandits}
\label{sec:upper_ex}
We now instantiate \cref{thm:upper_main_no_mingap} to give regret bounds for \mainalgb in standard settings of interest, bounding the \CompText for each setting. We begin by focusing on structured bandit settings and contextual bandits, then turn to tabular reinforcement learning in the sequel. We recall that, to map bandit problems to the DMSO framework, we take the decision space $\Pi$ to be the set of ``arms'', the observation space $\cO = \{ \emptyset \}$, and the reward space $\cR$ to be the rewards from the bandit (while we do not explicitly include the rewards in the observation space, we assume they are observed). We defer proofs for all examples to \Cref{sec:ex_proofs}.

\subsubsection{The Uniform Exploration Coefficient}
For the main examples we consider, we proceed by first bounding the \CompText in terms of another, somewhat simpler parameter we refer to as the \emph{uniform exploration coefficient}.
\begin{definition}[Uniform Exploration Coefficient]\label{def:uniform_exp}
\iftoggle{colt}{
For a randomized estimator $\xi \in \simplex_\cM$, we define the uniform exploration coefficient with respect to $\xi$ at scale $\veps>0$, $\Cexpxi(\veps)$, as the value of the following program: 
\begin{align*}
 \min_{C \in \R_+,p \in \simplex_\Pi} \crl*{ C \ \Big | \ \forall M,M' \in \cM \ : \  \begin{matrix} \max_{M'' \in \{M,M' \}} \Exp_{\Mbar \sim \xi} [\Exp_{\pi \sim p} [ \Dkl{\Mbar(\pi)}{M''(\pi)}]] \le 1/C \\
 \implies \max_{p' \in \simplex_\Pi} \Exp_{\pi \sim p'}[\Dkl{M(\pi)}{M'(\pi)}] \le \veps \end{matrix}}.
\end{align*}}{
For a randomized estimator $\xi \in \simplex_\cM$, we define the uniform exploration coefficient with respect to $\xi$ at scale $\veps>0$ as the value of the following program: 
\begin{align*}
\Cexpxi(\veps) := \min_{C \in \R_+,p \in \simplex_\Pi} \crl*{ C \ \Big | \ \forall M,M' \in \cM \ : \  \begin{matrix} \max_{M'' \in \{M,M' \}} \Exp_{\Mbar \sim \xi} [\Exp_{\pi \sim p} [ \Dkl{\Mbar(\pi)}{M''(\pi)}]] \le 1/C \\
 \implies \max_{p' \in \simplex_\Pi} \Exp_{\pi \sim p'}[\Dkl{M(\pi)}{M'(\pi)}] \le \veps \end{matrix}}.
\end{align*}}
We define $\pexpxi(\veps)$ as any minimizing distribution for this program, and let
\begin{align*}
\Cexp(\cM, \veps) := \sup_{\xi \in \simplex_\cM} \Cexpxi(\veps)
\end{align*}
denote the uniform exploration constant for class $\cM$.
\end{definition}

Intuitively, the uniform exploration coefficient characterizes the extent to which it is possible to explore by uniformly covering the decision space. In particular, one can always choose $p$ to be uniform over $\Pi$, which gives $\Cexp(\cM,\veps) \approxleq |\Pi|/\veps$, but in cases where information is shared between actions, the parameter is significantly smaller, as we will show for familiar examples below. For example, in the case of linear bandits with dimension $d$, we have $\Cexp(\cM,\veps) \le \bigoht(\frac{d \cdot \log 1/\veps}{\veps})$.

The following result shows that the \CompText can be bounded in terms of the uniform exploration coefficient. 

\begin{proposition}[Informal]\label{prop:aec_to_Cexp}
For $\cMsub \subseteq \cM$, we can bound $\aecflipM{\veps}{\cM}(\cMsub) \le \Cexp(\cMsub,\delta)$ for any
\begin{align*}
\sqrt{\delta} \leq \min_{M \in \cMsub} \min \crl*{ \min \left \{ \frac{1}{81 \LKL}, \frac{\delminm}{34 \VM} \right \} \cdot \frac{\veps}{2 \gm/\delminm + \nmepsc{\veps/36}{M}}, \frac{\delminm}{3}}.
\end{align*}
\end{proposition}

The full statement of \Cref{prop:aec_to_Cexp} is given in \Cref{lem:aec_Cexp_bound}.
Using \cref{prop:aec_to_Cexp}, we obtain guarantees for \mainalgb on several familiar classes, beginning with several bandit settings. We remark that \Cref{prop:aec_to_Cexp} is not in general tight---it simply shows that the \CompText is bounded by a simple, general, and interpretable notion of how easily a class can be explored. As, such the bounds on the \CompText in the following examples can almost certainly be improved using more specialized tools.

\subsubsection{Finite-Armed Bandits}
 We first consider the simplest bandit setting: multi-armed bandits with finite arms.

\begin{example}[Finite-Armed Bandit]
  \label{ex:mab_upper}
 Fix $A > 0$, and consider the class of finite-armed bandits with $A$ arms and unit-variance Gaussian noise:
  \begin{align*}
    \cM = \crl*{M(\pi)=\cN(\fm(\pi),1)\mid{}\fm\in [0,1]^{A} }.
  \end{align*}
  It is straightforward to verify that \Cref{asm:smooth_kl_kl,asm:bounded_likelihood,asm:covering,asm:covering} hold with $\LKL,\VM \le 4$ and $\dcov = \bigoh(A), \Ccov = \bigoh(1)$, and it can also be shown that, as long as $\fm[\Mst](\pist) < 1$, we can bound $\nst_{\veps} \le c \cdot \frac{A^2}{\veps (\delminst)^4}$.
 In addition, we can bound $\Cexp(\cMst,\veps) \le 4A/\veps$, so \Cref{prop:aec_to_Cexp} gives the following result.
 \begin{proposition}
 \label{prop:aec_bound_mab}
 For the finite-armed bandit problem with $A$ actions, there exists a universal constant $c>0$ such that
\begin{align}\label{eq:mab_aec_bound}
\aecflipM{\veps}{\cM}(\cMst) \le  c \cdot \frac{A^{15}}{\veps^8 (\delminst)^{24}}.
\end{align}
\end{proposition}
We immediately obtain the following corollary to \Cref{thm:upper_main_no_mingap}.

\begin{corollary}\label{cor:mab}
For finite-armed bandits with Gaussian noise, as long as $\fm[\Mst](\pist) < 1$, \mainalgb has regret bounded by
\begin{align*}
\Exp\sups{\Mst}[\RegDM] & \le (1+\veps) \cst \cdot \log(T) + \poly(A,\tfrac{1}{\veps},\tfrac{1}{\delminst},\log\log T) \cdot \log^{6/7}(T).
\end{align*}
\end{corollary}
With more refined analyses, various works have achieved instance-optimal regret bounds for finite-armed bandits with tighter lower-order terms than \Cref{cor:mab}  \citep{garivier2016explore,kaufmann2016complexity,lattimore2018refining,garivier2019explore}.
We emphasize that \Cref{cor:mab} is a special case of a much more general result. In particular, we proved the bound on the \CompText, \eqref{eq:mab_aec_bound}, using tools which hold for general classes (e.g. \Cref{prop:aec_to_Cexp}). An analysis of \mainalgb specialized to finite-armed bandits would likely yield a tighter result. 
\end{example}

\subsubsection{Structured Bandits}
Many bandit problems exhibit richer structure than the multi-armed bandit setting, and the study of these settings has been the focus of much of the recent work on instance-optimal learning. We next consider one such setting, that of structured bandits with bounded \emph{eluder dimension} \citep{russo2013eluder}. 

\begin{example}[Structured Bandits with Bounded Eluder Dimension]\label{ex:eluder_bandits}
Consider a bandit problem with unit-variance Gaussian noise but where the means are now given by a \emph{general function class} $\cF$ mapping from $\Pi$ to $[0,1]$:
  \begin{align}
    \cM = \crl*{M(\pi)=\cN(f(\pi),1)\mid{}f \in \cF}.
    \label{eq:structured_bandit}
  \end{align}
For such general settings, we might hope to capture the complexity of learning in terms of generalized notions of dimension for $\cF$. We consider one such notion here: the eluder dimension.

\begin{definition}[Eluder Dimension \citep{russo2013eluder}]
\label{def:eluder}
Let $\check{\dE}(\cF,\veps')$ denote the length of the longest sequence of actions $\{ \pi_1, \ldots, \pi_d \}$ such that, for each $n \le d$, there exist functions $f,f' \in \cF$ with $\sqrt{\sum_{i=1}^{n-1} (f(\pi_i) - f'(\pi_i))^2} \le \veps'$ but $f(\pi_n) - f'(\pi_n) > \veps'$. The \emph{eluder dimension} of function class $\cF$ at scale $\veps$ is then defined as $\dE(\cF,\veps) = \sup_{\veps' \ge \veps} \check{\dE}(\cF,\veps') \vee 1$. 
\end{definition}

The eluder dimension can be thought of as quantifying how easily a function class can be ``explored'': evaluating a pair of functions on $\dE(\cF,\veps)$ points allows one to determine whether they are nearly identical over the entire space. It is known to be bounded for many standard classes---for example, for linear function classes with dimension $d$, $\dE(\cF,\veps) \le \bigoht(d \cdot \log 1/\veps)$---and is also closely related to the \emph{disagreement coefficient} \citep{foster2020instance}. Furthermore, it can be shown to be a sufficient condition for bounded (worst-case) regret in general bandit problems \citep{russo2013eluder}. The following result shows that the eluder dimension bounds the \CompText.

\begin{proposition}\label{prop:aec_bound_eluder}
For the structured bandit class $\cM$ considered in \cref{eq:structured_bandit}, we have
\begin{align*}
\Cexp(\cMst,\delta) \le \frac{16\dE(\cF,\sqrt{\delta}/2)}{\delta}.
\end{align*}
This implies that
\begin{align*}
\aecflipM{\veps}{\cM}(\cMst) \le \frac{16\dE(\cF, \sqrt{\delta}/2)}{\delta }  \quad \text{for scale} \quad \delta = c \cdot \frac{\veps^2 \delLst^8}{\dE(\cF,\tfrac{1}{2} \delLst)^2} \quad \text{and} \quad \delLst := \min \crl*{ \delminst, 1/\nst_{\veps/36} },
\end{align*}
where $c > 0$ is a universal constant.
\end{proposition}
\Cref{prop:aec_bound_eluder} highlights the ability of the \CompText to adapt to the inherent complexity of ``exploring'' for the model class under consideration. We henceforth abbreviate $\dE := \dE(\cF, \sqrt{\delta}/2)$.

It is straightforward to show that \Cref{asm:smooth_kl_kl,asm:bounded_likelihood} are met in this setting with $\LKL, \VM \le 4$ (see \Cref{sec:gauss_bandits_proof}). Furthermore, \Cref{asm:covering} can be shown to hold with $\dcov$ scaling with the covering number of $\cF$ in the distance $d(f,f') = \sup_{\pi \in \Pi} | f(\pi) - f'(\pi)|$ and $\Ccov = \bigoh(1)$. In general, it must be shown that $\nepsst$ is bounded for each $\Mst$ and class of interest; as we show in the following examples, it is bounded for standard structured bandit settings such as linear bandits.
We have the following corollary to \Cref{thm:upper_main_no_mingap}.
\begin{corollary}\label{cor:bounded_eluder}
In the structured bandit setting with bounded eluder dimension considered above, \mainalgb has regret bounded as
\begin{align*}
\Exp\sups{\Mst}[\RegDM] & \le (1+\veps) \cst \cdot \log(T) + \poly(\dE, \dcov, \tfrac{1}{\veps}, \tfrac{1}{\delminst}, \nst_{\veps/36}, \log\log T) \cdot \log^{6/7}(T).
\end{align*}
\end{corollary}
To the best of our knowledge, \Cref{cor:bounded_eluder} is the first result to show that it is possible to obtain the instance-optimal rate in classes with bounded eluder dimension, with lower-order terms scaling only polynomially in the eluder dimension. 
More generally, \cref{cor:bounded_eluder} illustrates that \mainalgb can adapt to the structural properties of the given model class, and achieve regret scaling with existing notions of intrinsic dimension.  
\end{example}

We next consider two examples of structured bandits where it is known that the eluder dimension is bounded: linear bandits and generalized linear models. While these results are immediate given \Cref{cor:bounded_eluder}, the additional structure present in these settings allows us to obtain somewhat more explicit results.

\begin{example}[Linear Bandits]\label{ex:linear_bandits}
Consider the class of linear bandits with unit-variance Gaussian noise defined as
\begin{align*}
\cM = \{ M(\pi) = \cN(\langle \theta, x_\pi \rangle,1) \mid \theta \in \Theta \},
\end{align*}
where $\Theta \subseteq \R^d$ is some convex set with $\ell_2$ diameter $\bigoh(1)$ and $\cX := \{ x_\pi \ : \ \pi \in \Pi \} \subseteq \R^d$ is the arm set, which we assume has $\| x_\pi \|_2 \le 1$ for all $\pi \in \Pi$. As in \Cref{ex:eluder_bandits}, \Cref{asm:smooth_kl_kl,asm:bounded_likelihood} are met in this setting with $\LKL, \VM \le 4$; furthermore, \Cref{asm:covering} is also met with $\dcov = \bigoh(d)$ and $\Ccov = \bigoh(1)$.
We then have the following bound on the \CompShort.
\begin{proposition}\label{prop:linear_bandit}
For the linear bandit class $\cM$ defined above, we have
\begin{align*}
\aecflipM{\veps}{\cM}(\cMst) \le c \cdot \frac{d^3}{\veps^2} \cdot \prn*{ \frac{1}{\delminst} + \nst_{\veps/36}}^8 \cdot \polylog \prn*{d,\tfrac{1}{\veps},\tfrac{1}{\delminst},\nst_{\veps/36}}
\end{align*}
for a universal constant $c > 0$. 
\end{proposition}
Using \Cref{prop:linear_bandit}, we obtain the following corollary to \Cref{thm:upper_main_no_mingap}.
\begin{corollary}\label{cor:linear_bandit}
In the linear bandit setting defined above, \mainalgb has regret bounded as
\begin{align*}
\Exp\sups{\Mst}[\RegDM] & \le (1+\veps) \cst \cdot \log(T) + \poly(d,\tfrac{1}{\veps},\tfrac{1}{\delminst}, \nst_{\veps/36}, \log \log T) \cdot \log^{6/7} (T).
\end{align*}
\end{corollary}
\Cref{cor:linear_bandit} has lower-order terms scaling similarly to the best known lower-order terms in the linear bandit setting \citep{tirinzoni2020asymptotically,kirschner2021asymptotically}. These works, however, develop algorithms which are specialized to the linear bandit setting, while \Cref{cor:linear_bandit} is the instantiation of a more general result designed for arbitrary general decision-making settings.

The following result shows that under general conditions, we can bound the parameter $\nst_{\veps}$ for linear bandits.
\begin{proposition}[Informal]\label{prop:lin_bandit_nm_bound}
For linear bandits satisfying certain regularity conditions, $\nepsst$ is bounded by a polynomial function of problem parameters and a geometry-dependent term scaling with the structure of $\cX$ and $\Theta$. 
\end{proposition}
The full statement of \Cref{prop:lin_bandit_nm_bound} is given in \Cref{prop:lin_bandit_nm_bound_full}. The regularity condition for \Cref{prop:lin_bandit_nm_bound} requires primarily that $\thetast$ lies sufficiently far within the interior of $\Theta$ (see \Cref{sec:linear_bandit_proofs} for further details). We remark that the guarantees given in both \cite{tirinzoni2020asymptotically} and \cite{kirschner2021asymptotically} scale with geometric parameters very similar to $\nepsst$.

\end{example}

\begin{example}[Generalized Linear Models]
\label{ex:glm}
In the generalized linear model setting, we take the model class to be
\begin{align*}
\cM = \{ M(\pi) = \cN(\link(\langle \theta, x_\pi \rangle),1) \mid \theta \in \Theta \},
\end{align*}
where $\Theta$ and $\cX$ are as in \Cref{ex:linear_bandits}, and $\link(\cdot)$ is a known link function which is increasing and Lipschitz, but potentially nonlinear. Let $\sigmamax$ and $\sigmamin$ denote upper and lower bounds on the derivative of $g$, respectively:
\begin{align*}
\sigmamax := \max_{\theta \in \Theta, x \in \conv(\cX)} \link'(\inner{\theta}{x}) \quad \text{and} \quad \sigmamin := \min_{\theta \in \Theta, x \in \conv(\cX)} \link'(\inner{\theta}{x}).
\end{align*}
As in the linear bandit setting, we can show that \Cref{asm:smooth_kl_kl,asm:bounded_likelihood} are both met with $\LKL, \VM \le 4$, and that \Cref{asm:covering} is also met with $\dcov = \bigoh(d)$ and $\Ccov = \bigoh(\sigmamax)$. Furthermore, under the same conditions as for linear bandits, $\nepsst$ can be bounded for generalized linear models exactly as for linear bandits,  but with an additional scaling of $(\frac{\sigmamax}{\sigmamin})^2$. We then have the following.
\begin{proposition}\label{prop:glm_aec}
For the generalized linear model class $\cM$ defined above, we have
\begin{align*}
\aecflipM{\veps}{\cM}(\cMst) \le c \cdot \frac{d^3 \sigmamax^3}{\veps^2 \sigmamin^3} \cdot \prn*{ \frac{1}{\delminst} + \nst_{\veps/36}}^8 \cdot \polylog \prn*{d,\tfrac{1}{\veps},\tfrac{1}{\delminst},\nst_{\veps/36}}
\end{align*}
for a universal constant $c > 0$. 
\end{proposition}
Using \Cref{prop:linear_bandit}, we obtain the following corollary to \Cref{thm:upper_main_no_mingap}.
\begin{corollary}\label{cor:linear_bandit}
In the generalized linear model setting defined above, \mainalgb has regret bounded as
\begin{align*}
\Exp\sups{\Mst}[\RegDM] & \le (1+\veps) \cst \cdot \log(T) + \poly(d,\tfrac{\sigmamax}{\sigmamin}, \tfrac{1}{\veps},\tfrac{1}{\delminst},\nst_{\veps/36},\log \log T) \cdot \log^{6/7} (T).
\end{align*}
\end{corollary}
To the best of our knowledge, this is the first result to obtain finite-time instance-optimality for generalized linear models with lower-order terms polynomial in problem parameters.
\end{example}

\subsubsection{Contextual Bandits}
The previous examples illustrate that \mainalgb is able to learn efficiently in a variety of structured bandit settings. We now show that it leads to new guarantees for finite-action contextual bandits with general function approximation. 

\begin{example}[Contextual Bandits with Finitely Many Arms]
\label{ex:cb}
Consider the contextual bandit setting with context set $\cX$ (which could be arbitrarily large) and action set $\cA$ such that $A := |\cA| < \infty$. Let $\pX$ denote the context distribution, which we assume is known to the learner. The learning protocol is then, for step $t = 1,2,3,\ldots$:
\begin{enumerate}
\item Environment samples context $x^t \sim \pX$.
\item Learner chooses action $a^t \in \cA$, receives reward $r_t$.
\end{enumerate}
We assume that $r^t = \fst(x^t,a^t) + w^t$ for $w^t \sim \cN(0,1)$, for some $\fst : \cX \times \cA \rightarrow [0,1]$. We assume as well that the learner is given access to a set of functions $\cF$ such that $\fst \in \cF$.

To view this setting as a special case of the \FrameworkShort framework, we take the decision space to be the set $\Pi=(\cX\to\cA)$ of all policies mapping from $\cX$ to $\cA$, and take $\cO=\cX$ as the observation space. The learner's decision at round $t$ is a policy $\pi\ind{t}$, and they receive a reward-observation pair $(r\ind{t},o\ind{t})=(r\ind{t},x\ind{t})$ under the process $x\ind{t}\sim{}\pX$, $r\ind{t}\sim\cN(\fst(x\ind{t},\pi\ind{t}(x\ind{t})),1)$. The model class $\cM$ is the set of all instances of this form for $\fstar\in\cF$.

The following result shows that the \CompText is be bounded by the number of actions $A$, and is \emph{independent} of the size of the context space. See \Cref{sec:contextual_proofs} for a proof.

\begin{proposition}\label{prop:con_band_aec_bound}
For the contextual bandit setting, we can bound
\begin{align*}
\Cexp(\cMst,\veps) \le \frac{4A}{\veps},
\end{align*}
which implies that
\begin{align*}
\aecflipM{\veps}{\cM}(\cMst) \le c \cdot \frac{A^3}{\veps^2} \cdot \prn*{\frac{1}{\delminst} + \nst_{\veps/36}}^8,
\end{align*}
for a universal constant $c > 0$. 
\end{proposition}

As in the cases of bandits with bounded eluder dimension, $\nepsst$ must be bounded for each $\Mst$ and class $\cF$ of interest. It is straightforward to show, however, that  \Cref{asm:smooth_kl_kl,asm:bounded_likelihood} are met in this setting with $\LKL, \VM \le 4$, and, furthermore, that \Cref{asm:covering} is also met with $\dcov$ scaling as the covering number of $\cF$ in the distance $d(f,f') = \sup_{x \in \cX, a \in \cA} | f(x,a) - f'(x,a)|$, and $\Ccov = \bigoh(1)$. We then have the following corollary.

\begin{corollary}\label{cor:contextual}
In the finit-action contextual bandit setting considered above, \mainalgb has regret bounded as
\begin{align*}
\Exp\sups{\Mst}[\RegDM] & \le (1+\veps) \cst \cdot \log(T) + \poly(A, \dcov, \tfrac{1}{\veps}, \tfrac{1}{\delminst},\nst_{\veps/36},\log\log T) \cdot \log^{6/7}(T).
\end{align*}
\end{corollary}

To the best of our knowledge, \Cref{cor:contextual} is the first instance-optimal guarantee in the contextual bandit setting with general function approximation that obtains lower-order term scaling polynomially in problem parameters. Notably, the lower-order term scales independently of the size of the context space, $|\cX|$. We anticipate that extending this result to contextual bandit settings that have large action spaces, but which exhibit additional structure allowing for efficient exploration (e.g., linearity), will be straightforward. 
\end{example}

\subsection{Application: Tabular Reinforcement Learning}\label{sec:tabular_results}
As a final application of our results, we turn to the setting of episodic tabular reinforcement learning. 

\paragraph{Episodic Markov decision processes}
Recall that episodic reinforcement learning is a special case of the \FrameworkShort framework in which each model $M\in\cM$ is an episodic Markov Decision Process (MDP) given by the tuple $M = ( \cS, \cA, H, \{ \Pm_h \}_{h=1}^H, \{ \Rm_h \}_{h=1}^H, s_1 )$. Here $\cS$ is a set of states, $\cA$ a set of actions, $H$ the horizon, $\Pm_h : \cS \times \cA \rightarrow \simplex_\cS$ the probability transition kernel at step $h$, $\Rm_h : \cS \times \cA \rightarrow \simplex_{\bbR}$ the reward distribution at step $h$, and $s_1$ a deterministic initial state, which we take to be fixed across models. We assume that $\Rm_h(s,a)$ is unit-variance Gaussian, and that $\Exp_{r_h \sim \Rm_h(s,a)}[r_h] \in [0,1/H]$.\footnote{This assumption only serves to ensure that $\fm(\pi)\in\brk{0,1}$, in line with the convention for the rest of the paper. Our results continue to hold up to $\poly(H)$ factors if $\Exp_{r_h \sim \Rm_h(s,a)}[r_h] \in [0,1]$.}

The decision space $\Pi$ consists of non-stationary policies $\pi = (\pi_1,\ldots,\pi_H)$, where $\pi_h:\cS\to\cA$. For a fixed policy $\pi$, an episode proceeds in an MDP $M$ proceeds as follows. First, beginning from the initial state $s_1$, we take action $a_1 \sim \pi_1(s_1)$, receive reward $r_1 \sim \Rm_1(s_1,a_1)$, and transition to $s_2 \sim \Pm_1( \cdot \mid s_1, a_1)$. This continues for $H$ steps at which point the episode terminates and the process repeats. We define $\fm(\pi) := \Exp\sups{M,\pi}\brk[\big]{\sum_{h=1}^H r_h}$ as the expected reward achieved over the entire episode under this process.

Each round $t\in\brk{T}$ in the \FrameworkShort framework corresponds to an episode in the underlying MDP $\Mstar$. At each round, the learner selects a policy $\pi^{t}$, and receives reward $r^t = \sum_{h=1}^H r_h^t$ and $o^t = (s^t_1,a^t_1,r^t_1,\ldots,s^t_H,a^t_H,r^t_H)$, where $(s^t_1,a^t_1,r^t_1,\ldots,s^t_H,a^t_H,r^t_H) $ is the trajectory that results from executing $\pi^{t}$ in $\Mstar$ for a single episode.

\paragraph{Tabular model class}

In the tabular RL setting, it is assumed that $S := |\cS|$ and $A := | \cA|$ are both finite, and we take $\Pi$ to be the set of all deterministic policies. In addition to assuming that $\cM$ consists of tabular MDPs, we restrict to the following subclass:
\iftoggle{colt}{\begin{align*}
\cMtab(\pmin) := \bigg \{ & M = ( \cS, \cA, H, \{ \Pm_h \}_{h=1}^H, \{ \Rm_h \}_{h=1}^H, s_1 ) \\
& : \  \min_{s,a,s',h} \Pm_h(s'\mid s,a) \ge \pmin \bigg \}.
\end{align*}}
{\begin{align}
\cMtab(\pmin) := \crl*{M = ( \cS, \cA, H, \{ \Pm_h \}_{h=1}^H, \{ \Rm_h \}_{h=1}^H, s_1 ) \ : \  \min_{s,a,s',h} \Pm_h(s'\mid s,a) \ge \pmin }.\label{eq:tabular_class}
\end{align}}
While the assumption that $\min_{s,a,s',h} \Pm_h(s'\mid s,a) \ge \pmin$ may be seen as restrictive, the guarantees we provide scale only with $\log \frac{1}{\pmin}$, so $\pmin$ can be taken to be extremely small without affecting the result significantly.

Note that when our results are specialized to reinforcement learning, $\delminst$ denotes the gap between the performance of the optimal policy, and the next-best \emph{deterministic} policy. This quantity can be lower bounded in terms of other standard quantities including gaps in the rewards at each state and the transition probabilities.

Toward instantiating \Cref{thm:upper_main_no_mingap} in this tabular RL setting, we first provide a bound on the \CompText, which we establish by first bounding the Uniform Exploration Coefficient.
\begin{proposition}\label{prop:tabular_aec_bound}
  For $\cM \leftarrow \cMtab(\pmin)$, we can bound\footnote{Here $\Cexp^{\mathsf{H}}$ denote the uniform exploration coefficient as defined in \Cref{def:uniform_exp}, but with $\Dkl{\cdot}{\cdot}$ replaced with $\Dhels{\cdot}{\cdot}$. To prove \cref{prop:tabular_aec_bound}, we show that a variant of \Cref{prop:aec_to_Cexp} still holds with this alternate definition of $\Cexp(\cM,\veps)$. }
\begin{align*}
\Cexp^{\mathsf{H}}(\cMst,\veps) \le c \cdot \frac{SAH^2 \cdot \log^2 H}{\veps^2}
\end{align*}
for a universal constant $c>0$, which implies that
\begin{align*}
\aecflipM{\veps}{\cM}(\cMst) \le c \cdot \frac{S^5 A^5 H^{14} \cdot \log^{10} H}{\veps^4} \cdot \prn*{ \frac{1}{\delminst} + \nst_{\veps/36}}^{24} \cdot \log^4 \frac{1}{\pmin}
\end{align*}
for a universal constant $c>0$. 
\end{proposition}
Next, it can be shown that \Cref{asm:smooth_kl_kl,asm:bounded_likelihood,asm:covering} hold for $\cM \leftarrow \cMtab(\pmin)$ with constants (see \Cref{sec:tabular_proofs}):
\iftoggle{colt}{\begin{align*} \LKL = \VM = \bigoh(H \cdot \log 1/\pmin), \\
\dcov = \bigoh(S^2 AH), & \quad \Ccov = \bigoh(H/\Pmin).
\end{align*}
}{
\begin{align*}
 \LKL = \VM = \bigoh(H \cdot \log 1/\pmin), \quad \dcov = \bigoh(S^2 AH), \quad \Ccov = \bigoh(H/\Pmin).
\end{align*}
}

We then obtain the following corollary to \Cref{thm:upper_main_no_mingap}.
\begin{corollary}\label{cor:tabular_rl}
For $\cM \leftarrow \cMtab(\pmin)$ and $\veps$, \mainalgb has regret bounded by
\begin{align*}
\Exp\sups{\Mst}[\RegDM] & \le (1+\veps) \cst \cdot \log(T) + \poly(S,A,H, \tfrac{1}{\veps},\tfrac{1}{\delminst}, \log \tfrac{1}{\Pmin},\nst_{\veps/36}, \log\log(T)) \cdot \log^{6/7} (T) .
\end{align*}
\end{corollary}

To our knowledge, \Cref{cor:tabular_rl} is the first guarantee for tabular RL that achieves the instance-optimal rate while obtaining lower-order terms that scale only polynomially in problem parameters. As noted previously, existing approaches to instance-optimal regret in tabular RL \citep{ok2018exploration,dong2022asymptotic} have lower-order terms that scale exponentially in problem parameters, and as a result are truly asymptotic in nature.

While \Cref{cor:tabular_rl} is stated in terms of $\nepsst$ for the sake of generality, as we show in \Cref{sec:tabular_proofs}, if $\Mst$ has rewards that are sufficiently small (in particular, if it satisfies $\Exp_{r_h \sim \Rm[\Mst]_h(s,a)}[r_h] < 1/H^2$ for all $(s,a,h)$), then we can bound
\begin{align*}
\nepsst \le c \cdot \frac{\gst}{\delminst} \cdot \prn*{1 + \frac{\gst}{\veps (\delminst)^2}}.
\end{align*}
In this case, \Cref{cor:tabular_rl} scales polynomially in all standard problem parameters.

Let us remark that the prior work of \cite{dong2022asymptotic} does not
require that $\Pm_h(s' \mid s,a) \ge \Pmin$ as we do (the work of
\cite{ok2018exploration} only holds for ergodic MDPs, itself a very
strong assumption). However, the lower-order term obtained in \cite{dong2022asymptotic} scales polynomially in the inverse probability of observing the trajectory that occurs with minimum (non-zero) probability. In general, this will scale exponentially in $H$, and inversely with the probability of the transition with minimum (non-zero) probability occurring, that is $\min_{s,a,s',h : \Pm_h(s' | s,a) > 0} \Pm_h(s' | s,a)$. Thus, while we must impose the stronger condition that all transitions occur with some probability $\Pmin$, our bounds only scale logarithmically in this quantity, and polynomially in $S,A,$ and $H$, a significant improvement over \cite{dong2022asymptotic}.
Understanding whether it is possible to remove the additional restrictions we impose while still obtaining reasonable finite-time performance is an interesting direction for future work. 

As far as we are aware, there is no prior work on instance-optimal algorithms for RL settings with general function approximation. While we have only instantiated \Cref{thm:upper_main_no_mingap} and \mainalgb in the tabular RL setting, the tools we have developed can also be applied to RL with general model classes. Exploring the application of \mainalgb to, for example, bilinear classes \citep{du2021bilinear} is an exciting avenue for future work.

\subsection{Overview of Analysis}\label{sec:proof_sketch}
To close this section we briefly sketch the proof of the regret bound for \mainalg (\Cref{thm:upper_main}); the proof of the regret bound for \mainalgb (\cref{thm:upper_main_no_mingap}) builds on these ideas, but is slightly more involved. See \Cref{sec:upper_proofs} for full proofs. 

Let us refer to the \emph{exploit phase} as the subset of rounds $t$ in
which \cref{line:exploit_gen} of \mainalg is reached, and refer to the
\emph{explore phase} as the subset of rounds in which
\cref{line:explore_gen} is reached. We focus on bounding the regret in the explore phase---it can be shown (\cref{lem:exploit_regret}) that in the exploit phase, where the if statement on \Cref{line:exploit} is true, $\piMhats = \pist$ for all but $\bigoh(\log\log T)$ rounds, so that the regret incurred in this phase is at most $\bigoh(\log\log T)$.

Let $s_T$ denote the total number of rounds in the explore phase up to time $T$. Fix an explore round $s\in\brk{s_T}$. We bound the regret $\Delst(p^s)$ by considering three cases.

\paragraph{Case 1: $\Mst \in \cM \backslash \cMgl[\veps/6](\lam^s; \nmax)$}
In this case, $\lam^s$ is not an optimal (normalized) allocation for $\Mstar$, but we can use the \CompShort to argue that the information gained by the algorithm is large. In particular, since $p^s$ plays $\omega^s$ with probability at least $1-q$, we can bound
\begin{align*}
\frac{1}{\Exp_{\Mhat \sim \xi^s}[\Exp_{\pi \sim p^s}[\Dklbig{\Mhat(\pi)}{\Mst(\pi)}]]} & \le \frac{1}{1 - q} \cdot \frac{1}{\Exp_{\Mhat \sim \xi^s}[\Exp_{\pi \sim \omega^s}[\Dklbig{\Mhat(\pi)}{\Mst(\pi)}]]} \\
& \overset{(a)}{\le} \frac{1}{1 - q} \cdot \min_{\lambda, \omega \in \simplex_\Pi} \sup_{M \in \cM \backslash \cMgl[\veps/6](\lambda;\nmax)} \frac{1}{\Exp_{\Mhat \sim \xi^s}[\Exp_{\pi \sim \omega^s}[\Dklbig{\Mhat(\pi)}{M(\pi)}]]} \\
& \lesssim \frac{1}{1-q} \cdot \aecflip[\veps/12](\cM) 
\end{align*}
as long as $\nmax$ is chosen appropriately; here, $(a)$ follows because $\Mst \in \cM \backslash \cMgl(\lam^s; \nmax)$ by assumption in this case, and by the choice of $\lam^s$ and $\omega^s$ given in \eqref{eq:alg_alloc_comp}. 
Rearranging this gives
\begin{align*}
1 \lesssim  \frac{1}{1-q} \cdot \aecflip[\veps/12](\cM) \cdot \Exp_{\Mhat \sim \xi^s}[\Exp_{\pi \sim p^s}[\Dklbig{\Mhat(\pi)}{\Mst(\pi)}]].
\end{align*}
This reflects that, when $\Mst \in \cM \backslash \cMgl[\veps/6](\lam^s; \nmax)$, our choice of $p^s$ ensures that $\Mhat \sim \xi^s$ and $\Mst$ can be distinguished, with the amount of information gained lower bounded by $\bigoh(\aecflip[\veps/12](\cM)^{-1})$

Adding and subtracting $\frac{2}{1-q} \cdot \aecflip[\veps/12](\cM) \cdot \Exp_{\Mhat \sim \xi^s}[\Exp_{\pi \sim p^s}[\Dklbig{\Mhat(\pi)}{\Mst(\pi)}]]$ to $\Delst(p^s)$, and using that $\Delst(p^s) \le 1$ always, we then have that the instantaneous regret in this case is bounded by
\begin{align*}
  \Delst(p^s) & = \Delst(p^s) - \frac{2}{1-q} \cdot \aecflip[\veps/12](\cM) \cdot \Exp_{\Mhat \sim \xi^s}[\Exp_{\pi \sim p^s}[\Dklbig{\Mhat(\pi)}{\Mst(\pi)}]] \\
  &\qquad+\frac{2}{1-q} \cdot \aecflip[\veps/12](\cM) \cdot \Exp_{\Mhat \sim \xi^s}[\Exp_{\pi \sim p^s}[\Dklbig{\Mhat(\pi)}{\Mst(\pi)}]]\\
& \le -1 + \frac{2}{1-q} \cdot \aecflip[\veps/12](\cM) \cdot \Exp_{\Mhat \sim \xi^s}[\Exp_{\pi \sim p^s}[\Dklbig{\Mhat(\pi)}{\Mst(\pi)}]] .
\end{align*}
Summing over $s$, it follows that the total regret in this case is bounded by
\begin{align*}
\sum_{s = 1}^{s_T} \Delst(p^s) \cdot \bbI \{ s \text{ in Case 1} \} & \le \frac{2}{1-q} \cdot \aecflip[\veps/12](\cM) \cdot \EstKL(s_T) - \sum_{s=1}^{s_T} \bbI \{ s \text{ in Case 1} \}.
\end{align*}
Furthermore, since regret is always non-negative---that is, $\Delst(p) \ge 0$ for all $p$---rearranging this inequality leads to a bound on the total number of times this case can occur:
\begin{align*}
\sum_{s=1}^{s_T} \bbI \{ s \text{ in Case 1} \} \le \frac{2}{1-q} \cdot \aecflip[\veps/12](\cM) \cdot \EstKL(s_T).
\end{align*}
Critically, as given in \Cref{eq:est_oracle_bound}, $\EstKL(s_T)$ scales at most poly-logarithmically in $s_T$. Thus, as long as $s_T$ is at most $\bigoh(\log T)$, the total regret incurred in this case (as well as the total number of times this case can occur), will be at most $\bigoh(\aecflip[\veps/12](\cM) \cdot \log \log T)$.

\paragraph{Case 2: $\Mst \in \cMgl[\veps/6](\lam^s; \nmax)$ and $\pist \in \pibm[\Mhats]$}
In this case, we have that $\lambda^s$ is a Graves-Lai optimal allocation for $\Mst$. Thus, it follows that
\begin{align*}
\Delst(\lam^s) \le (1+\veps/6) \gst / \betast \quad \text{and} \quad \inf_{M \in \cMalt(\Mst)} \Exp_{\pi \sim \lam^s}[\Dkl{\Mst(\pi)}{M(\pi)}] \ge (1-\veps/6) / \betast
\end{align*}
for some $\betast \le \nmax$.
This implies that, for any $M \in \cMalt(\Mst)$, we can bound
\begin{align*}
  \Delst(p^s) & = \Delst(p^s) - (1+\veps) \gst \Exp_{\pi \sim p^s}[\Dkl{\Mst(\pi)}{M(\pi)}] + (1+\veps) \gst \Exp_{\pi \sim p^s}[\Dkl{\Mst(\pi)}{M(\pi)}]\\
& \lesssim  (1+\veps/6) \gst / \betast - (1+\veps) (1-\veps/6) \gst / \betast  + (1+\veps) \gst \Exp_{\pi \sim p^s}[\Dkl{\Mst(\pi)}{M(\pi)}] \\
& \lesssim -\veps \gst/\betast + (1+\veps) \gst \Exp_{\pi \sim p^s}[\Dkl{\Mst(\pi)}{M(\pi)}] \\
& \le -\veps \gst/\nmax + (1+\veps) \gst \Exp_{\pi \sim p^s}[\Dkl{\Mst(\pi)}{M(\pi)}]
\end{align*}
Since this bound holds uniformly for all $M\in\cMalt(\Mst)$, it follows that the total regret in this case can be bounded as
\begin{align*}
  \sum_{s=1}^{s_T} \Delst(p^s) \cdot \bbI \{s \text{ in Case 2} \} & \lesssim (1+\veps) \gst \cdot \sum_{s=1}^{s_T} \inf_{M\in\cMalt(\Mstar)}\Exp_{\pi \sim p^s}[\Dkl{\Mst(\pi)}{M(\pi)}] \cdot \bbI \{ \pist \in \pibm[\Mhats] \} \\
  &\qquad- \frac{\veps \gst}{\nmax} \sum_{s=1}^{s_T} \bbI \{ s \text{ in Case 2}\}.
\end{align*} 
To bound this, the key observation is that, if we explore at round $s$, then it must be the case that, for all $\pim[\Mhat^s] \in \pibm[\Mhat^s]$, there exists some $M \in \cMalt(\pim[\Mhat^s])$ such that $\sum_{i=1}^{s-1} \Exp_{\Mhat \sim \xi^i}  [ \log \frac{\Prm{\Mhat}{\pi^i}(r^i, o^i)}{\Prm{M}{\pi^i}(r^i, o^i )}  ]  \le \log(T \log T)$. Using \Cref{asm:smooth_kl_kl} and \Cref{asm:bounded_likelihood} to move from $\Mhat \sim \xi^i$ to $\Mst$ and to relate the observed log-likelihood ratios to the KL divergence, we can furthermore show that
\begin{align*}
  &\inf_{M\in\cMalt(\Mstar)}\sum_{s=1}^{s_T} \Exp_{\pi \sim p^s}[\Dkl{\Mst(\pi)}{M(\pi)}] \cdot \bbI \{ \pist \in \pibm[\Mhats] \} \\
  & \lesssim \inf_{M\in\cMalt(\Mstar)}\sum_{s=1}^{s_T} \Exp_{\Mhat \sim \xi^s} \Big  [ \log \frac{\Prm{\Mhat}{\pi^s}(r^s, o^s)}{\Prm{M}{\pi^s}(r^s, o^s )} \Big ] + \sqrt{s_T} \cdot \EstKL(s_T) 
 \lesssim \log(T \log T) + \sqrt{s_T} \cdot \EstKL(s_T).
\end{align*}
This allows us to bound
\begin{align*}
(1+\veps) \gst \cdot \sum_{s=1}^{s_T}\inf_{M\in\cMalt(\Mstar)} \Exp_{\pi \sim p^s}[\Dkl{\Mst(\pi)}{M(\pi)}] \cdot \bbI \{ \pist \in \pibm[\Mhats] \} \lesssim (1+\veps) \gst \cdot \log T + \gst \sqrt{s_T} \cdot \EstKL(s_T).
\end{align*}
Thus, as long as $s_T = \bigoh(\log T)$, we can bound the total regret incurred in Case by $(1+\veps) \gst \cdot \log T + o(\log T)$. Using that regret is always lower bounded by 0 in the same fashion as Case 1, we can further use this to bound the total number of times that Case 2 occurs by 
\begin{align*}
\sum_{s=1}^{s_T} \bbI \{ s \text{ in Case 2} \} \leq \frac{\nmax}{\veps \gst} \Big ( \gst \cdot \log T + \gst \sqrt{s_T} \cdot \EstKL(s_T) \Big ) .
\end{align*}
The intuition for this case is that, since we are playing a Graves-Lai allocation for $\Mst$, the regret will scale with $\gst$, the instance-optimal rate, and, furthermore, the allocation will allow us to distinguish $\Mst$ from alternatives $M \in \cMalt(\Mst)$. Using that the total estimation error is bounded, and that we only enter the explore phase if there exists some $M \in \cMalt(\Mst)$ that we cannot distinguish from $\Mst$, this ultimately implies that the total number of times this phase occurs, and therefore the total regret incurred by this phase, is bounded. 

\paragraph{Case 3: $\Mst \in \cMgl[\veps/6](\lam^s; \nmax)$ and $\pist \not\in \pibm[\Mhats]$}
In this case, we bound $\Delst(p^s)$ by adding and subtracting $\Exp_{\Mhat \sim \xi^s}[\Exp_{\pi \sim p^s}[\Dklbig{\Mhat(\pi)}{\Mst(\pi)}]]$ in the same fashion as Case 1. Since $\pist \not\in \pibm[\Mhats]$, it can be shown that
$\xi^s$ must place $\Omega(\delmin)$ probability mass on $M \in \cMalt(\Mst)$, allowing us to lower bound $\Exp_{\Mhat \sim \xi^s}[\Exp_{\pi \sim p^s}[\Dklbig{\Mhat(\pi)}{\Mst(\pi)}]]$ using the same reasoning as in Case 2. In total, we can show that the regret in this case is bounded by $\bigoh( \frac{\gst}{\delmin} \cdot \EstKL(s_T))$, and the number of times this case can occur is at most $\bigoh( \frac{\nmax}{\veps \delmin} \cdot \EstKL(s_T))$.

\paragraph{Concluding the Proof}
Combining all three cases, we have shown that the regret of the explore phase is bounded by
\begin{align*}
 (1+\veps) \gst \cdot \log T + \bigoht \prn*{ \frac{1}{1-q} \cdot \aecflip[\veps/12](\cM) \cdot \EstKL(s_T) + \gst \sqrt{s_T} \cdot \EstKL(s_T) + \frac{\gst}{\delmin} \cdot \EstKL(s_T)}.
\end{align*}
Furthermore, using our bounds on the number of times each case can occur, one can show that $s_T = \bigoh(\log T)$. Since $\EstKL(s_T)$ is at most polylogarithmic in $s_T$, it follows that the regret is bounded as
\begin{align*}
 (1+\veps) \gst \cdot \log T + \bigoht \prn*{ \aecflip[\veps/12](\cM) + \aecflip[\veps/12]^{1/2}(\cM) \cdot \log^{1/2} T},
\end{align*}
as stated in \Cref{thm:upper_main}.


\section{Lower Bounds for Learning the Optimal Allocation}
\label{sec:lower}

The \mainalg algorithm achieves instance-optimal regret by explicitly learning an
$\veps$-optimal \alloc for the underlying model
$\Mstar$. In this section, we introduce an abstract formulation for
the problem of learning an optimal allocation
(\Cref{sec:lower_bound_learning_opt}), and provide lower bounds which show that the
\CompText is a fundamental limit for this task
(\Cref{sec:lower_bound_main_result}). We then present several examples
illustrating lower bounds on the \CompText
(\Cref{sec:examples_lower}), and discuss how our lower bounds relate
to the problem of minimizing regret (\Cref{sec:gaps}).

\paragraph{Additional Notation}
Throughout this section we will also make use of the following
definition:
\iftoggle{colt}{
\begin{align}\label{eq:lambda_set_nmax}
  &\Lambda(M,\veps)\\
  &:= \crl*{ \lambda \in \simplex_\Pi \ : \ \exists \nsf \in \bbR_+ \text{ s.t. } \delm(\lambda) \le \frac{(1+\veps) \gm}{\nsf}, \inf_{M'\in\cMalt(M)}\En_{\pi\sim{}\lam}\brk*{\Dkl{M(\pi)}{M'(\pi)}} \ge \frac{1-\veps}{\nsf} }.\notag
\end{align}}{
\begin{align}
  &\Lambda(M;\veps,\nmax) \label{eq:lambda_set_nmax} \\
  &:= \crl*{ \lambda \in \simplex_\Pi \ : \ \exists \nsf \in (0,\nmax] \text{ s.t. } \delm(\lambda) \le \frac{(1+\veps) \gm}{\nsf}, \inf_{M'\in\cMalt(M)}\En_{\pi\sim{}\lam}\brk*{\Dkl{M(\pi)}{M'(\pi)}} \ge \frac{1-\veps}{\nsf} }. \nonumber
\end{align}}
That is, $\Lambda(M;\veps,\nmax)$ denotes the set of normalized
allocations that are Graves-Lai optimal for $M$ with tolerance
$\veps$, and have normalization factor at most $\nmax$. Unless
  otherwise stated, the results in this section do not make use of
  \cref{ass:unique_opt} or \cref{asm:mingap}.

\subsection{Learning the Optimal Allocation: Minimax Formulation}\label{sec:lower_bound_learning_opt}
We consider the following protocol, which captures the task of
learning an optimal \alloc for an unknown model $\Mstar$.
\begin{itemize}
\item For $t=1,\ldots,T$, sample $\pi\ind{t}\sim{}p\ind{t}$ and
  observe $(r\ind{t},o\ind{t})$.
\item Based on the entire history $\hist\ind{T}=(\pi\ind{1},r\ind{1},o\ind{1}),\ldots,(\pi\ind{T},r\ind{T},o\ind{T})$, output a normalized
  allocation $\lambdahat\in\simplex_\Pi$. The allocation may be
  randomized according to a distribution $q\in\simplex_{\simplex_\Pi}$.
\end{itemize}
We formalize an \emph{algorithm} for this task as a pair $\Alg = (p, q)$, where
$q(\cdot\mid{}\hist\ind{T})$ is the distribution over $\lambdahat$ given
the history, and $p=\crl*{p\ind{t}}_{t\in\brk{T}}$ is a sequence of
exploration distributions of the form
$p\ind{t}(\cdot\mid\hist\ind{t-1})$. We let $\Pma\prn{\cdot}$ denote the law of
$\hist\ind{T}$ when $M$ is the underlying model and $\Alg$ is the
algorithm, and let $\Ema\brk{\cdot}$ denote the corresponding
expectation. The goal of the algorithm is to
ensure that $\lambdahat$ is an $\veps$-optimal allocation for $\Mstar$ with
high probability, i.e. 
\begin{align*}
  \bbP^{\sss{\smash{\Mstar}\hspace{-2pt},\,\bbA}}\prn*{\lambdahat\in\Lambda(\Mstar;\veps,\nmax)} \geq{} 1-\delta
\end{align*}
for some failure probability $\delta>0$ and normalization factor $\nmax>0$

Intuitively, learning an optimal \alloc is closely related to
achieving instance-optimal regret, but there are some subtle technical
differences which we discuss in detail in the sequel. We study the
former task because we find it to be more amenable to non-asymptotic
lower bounds, and because it captures the behavior of ``natural''
algorithms such as \mainalg and essentially every existing asymptotically optimal algorithm we are aware of.

\paragraph{Minimax framework}
To provide lower bounds on the complexity of learning an optimal
\alloc, we consider a minimax
framework.
Our main quantity of interest will be:
\begin{equation}
  \label{eq:tgl_basic}
\Tgl(\cM;\veps,\nmax,\delta)
  =\inf_{\Alg}\inf\crl*{T\in\bbN \mid{}
    \Pma\prn*{\lambdahat\in\Lambda(M;\veps,\nmax)}
\geq{} 1-\delta,\;\forall{}M\in\cM
  }.
\end{equation}
This represents the earliest time $T\in\bbN$ for which there exists an
algorithm that learns an $\veps$-optimal allocation with probability
at least $1-\delta$, and does so uniformly for all $M\in\cM$. Recall that for our upper bounds (\pref{thm:upper_main}),
  the \CompText gives a bound on the lower-order terms in the regret (reflecting the time required to learn an
  $\veps$-optimal allocation) that holds \emph{uniformly} for all
  models in the class $\cM$.
To
understand the
optimality of uniform bounds of this type, a minimax framework is
natural. This framework also naturally complements recent
non-asymptotic algorithms for linear models such as
\citep{tirinzoni2020asymptotically,kirschner2021asymptotically},
where the complexity of exploration is captured by problem-dependent
quantities such as the feature dimension, which are bounded uniformly for all
models in the class. Nonetheless, exploring other notions of
optimality (for example, instance-dependent complexity) for learning
the \alloc is an interesting direction for future research.

Note that the quantity \pref{eq:tgl_basic} does not place any
constraint on the regret of the algorithm under consideration. It will
also be useful to consider the notion
\begin{align}
  \label{eq:tgl_regret}
  &\Tgl(\cM;\veps,\nmax,\delta,R)\notag\\
  &=\inf_{\Alg}\inf\crl*{T\in\bbN \mid{}
\Pma\prn*{\lambdahat\in\Lambda(M;\veps,\nmax)}
  \geq{} 1-\delta,
  \; \Ema\brk*{\RegDM}\leq{}R\cdot\log(T),\;
  \;\forall{}M\in\cM
  },
\end{align}
which captures the minimax complexity of learning the \alloc, subject
to the constraint that the algorithm achieves logarithmic regret
throughout the learning process. This notion is particularly well
suited to complement the upper bounds achieved by algorithms such as \mainalg.

We remark that the restriction to allocations with normalization
factor no more than $\nmax$ in the definitions above is natural for
several reasons. First, without such a restriction, the returned
allocation can place an arbitrarily small amount of mass on
informative actions, and an arbitrarily large amount of mass on the
optimal action, since any Graves-Lai optimal allocation is still
Graves-Lai optimal if the amount of mass on the optimal action is
increased arbitrarily. While technically Graves-Lai optimal, such
allocations do not reflect an allocation one could play over a finite
time horizon in order to certify the optimal decision $\pistar$---as any algorithm with finite-time guarantees must do---and therefore do not reflect the cost such an algorithm must pay to learn an allocation. 
Second, given knowledge of the optimal action, any Graves-Lai optimal
allocation which has a normalization factor larger than $\nmax$ can be
transformed into a Graves-Lai optimal allocation with normalization
factor $\nmax$ by adjusting the mass on the optimal action, assuming
$\nmax$ is taken to be sufficiently large (in particular,
as large as $\nmax(\cM,\veps)$; cf. \eqref{eq:nmax}). 
Therefore, if we restrict our attention to the task to learning the
Graves-Lai allocation for a subclass of models which agree on the optimal action (as we do in \Cref{thm:lower_logt}), any lower bound for learning an allocation with normalization $\nmax$ also applies to learning an unrestricted allocation, for large enough $\nmax$. Finally, \mainalg itself plays allocations with bounded normalization, which as we note in \Cref{sec:algorithm}, does not affect the optimality of its performance---allocations with bounded normalization are always sufficient.

\subsection{Main Result}\label{sec:lower_bound_main_result}

We state two lower bounds. The first scales with the version of
the \CompText appearing in our upper bound (\pref{thm:upper_main}),
but leads to a lower bound on $\Tgl$, while the second lower bound
scales with the \CompShort for a restriction $\cM$, but is
exponentially stronger in the sense that it provides a similar lower bound on $\log(\Tgl)$. We remark briefly that the regularity conditions of \Cref{sec:regularity} are not required to hold here, unless otherwise stated. 

\begin{restatable}[Main lower bound---weak variant]{theorem}{lowert}
  \label{thm:lower_t}
  Let $\veps>0$, $\nmax > 0$, and $\cMsub \subseteq \cM$ be given, and set $\delta := \frac{\veps}{2} \cdot \min \{ 1, \inf_{M \in \cMsub} \gm/\nmax \}$.
  Unless
\[
  T > \frac{\delta}{8}\cdot \sup_{\Mbar \in \cMp} \aecM{2\veps}{\cM}(\cMsub,\Mbar),
\]
any algorithm must have, for some $M \in \cMsub$:
\[
  \Pma\brk*{\lambdahat\notin\Lambda(M;\veps,\nmax)} \geq \frac{\delta}{6}.
\]
\end{restatable}
Stated equivalently, \pref{thm:lower_t} implies that for any
$\veps>0$, if we set $\delta =
c\cdot{}\veps \cdot{}\inf_{M\in\cMsub}\gm/\nmax$ for sufficiently
small numerical constant $c$, then
\[
\Tgl(\cM;\veps,\nmax,\delta) \geq{} \delta\cdot{} \sup_{\Mbar \in \cMp} \comp[2\veps](\cMsub,\Mbar).
\]
\begin{remark}[Choice of $\cMsub$]
\Cref{thm:lower_t} is stated with respect to an arbitrary subset, $\cMsub$, of $\cM$. While one could simply choose $\cMsub \leftarrow \cM$, in some cases it is advantageous to choose $\cMsub \subsetneq \cM$. In particular, note that our lower bound scale with $\inf_{M \in \cMsub} \gm$. For classes $\cM$ where $\inf_{M \in \cM} \gm = 0$---which will be the case, for example, if $\cM$ corresponds to a model class where reward means form a compact set---it is advantageous to restrict $\cMsub$ so that it corresponds only to instances with $\gm > 0$. 

We remark as well that the restriction of the lower bound to $\cMsub$
is somewhat analogous to our upper bound
\Cref{thm:upper_main_no_mingap}, which provides guarantees in terms of
the \CompShort of a restriction of $\cM$, $\cMst$. We may therefore
choose $\cMsub \leftarrow \cMst$ to obtain lower bounds matching the
scaling of \Cref{thm:upper_main_no_mingap}. In the following section
we provide several examples of how $\cMsub$ can be chosen to yield
intuitive lower bounds.

Lastly, let us mention that $\cMsub$ may also be chosen to yield lower bounds that have a more instance-dependent flavor. For example, given some instance $\Mst$, we could choose $\cMsub$ to correspond to all instances identical to $\Mst$ up to a permutation of the decisions, in which case \Cref{thm:lower_t} is will yield lower bounds on the performance of any algorithm on a permutation of $\Mst$, rather than over the entire class. 
\end{remark}
\begin{remark}
  The lower bound in \pref{thm:lower_t}, for $\cM_0=\cM$, scales with the quantity
\[\sup_{\Mbar\in\cMall}\aec(\cM,\Mbar)\geq{}\sup_{\Mbar\in\conv(\cM)}\aec(\cM,\Mbar) \ge \sup_{\xi \in \simplex_\cM} \aecflip(\cM,\xi) = \aecflip(\cM),\]
  which at first glance might appear to be larger than the version of
  the \CompShort appearing in our upper bounds. However, there is no
  contradiction, because these quantities can be shown to be
  equivalent (up to problem-dependent parameters) under the assumptions with which \pref{thm:upper_main} is proven.
\end{remark}

Our second lower bound yields a lower bound on $\log(\Tgl)$ as opposed to
$\Tgl$---a significantly stronger result---but scales with the
\CompShort for a restricted class. To state the result, for $\Mbar\in\cMall$ and $\cMsub \subseteq \cM$, define
\begin{align*}
  \cMoptsub(\Mbar) = \crl*{M\in\cMsub\mid{} \pibm\subseteq\pibmbar,\;\;
  \Dkl{\Mbar(\pi)}{M(\pi)}=0\;\; \forall{}\pi\in\pibmbar}.
\end{align*}
This represents the set of models $M$ where 1) the optimal decisions
for $M$
are also optimal for $\Mbar$ and 2) playing the an optimal decision
reveals no information that can distinguish $M$ and $\Mbar$. Our second lower bound scales with the
\CompShort for $\cMoptsub(\Mbar)$, and is restricted to algorithms with
low regret.~\\
\begin{restatable}[Main lower bound---strong variant]{theorem}{lowerlogt}
  \label{thm:lower_logt}
      Let $\veps>0$, $\nmax > 0$, and $\cMsub \subseteq \cM$ be given, and define
  $\delta=\frac{\veps}{2} \cdot \min \{ 1, \inf_{M \in \cMsub} \gm/\nmax \}$. Unless
  \begin{align*}
\sup_{M\in\cMsub}\frac{\gm}{\delminm}\cdot    \log(T) \geq
    \bigom(\delta^2)\cdot\sup_{\Mbar\in\cMall}\aecM{2\veps}{\cM}(\cMoptsub(\Mbar),\Mbar),
  \end{align*}
  there is no algorithm that simultaneously ensures that
  \begin{enumerate}
  \item $\Ema\brk*{\RegDM}\leq{}2\cdot\gm\log(T),\;\;\forall{}M\in\cMsub$.
  \item $\Pma\brk*{\lambdahat\notin\Lambda(M;\veps,\nmax)}\leq{} \frac{\delta}{12},\;\;\forall{}M\in\cMsub$.
  \end{enumerate}
\end{restatable}
Equivalently, \pref{thm:lower_logt} implies that for any
$\veps>0$, if we set $\delta =
c\cdot{}\veps \cdot{}\inf_{M\in\cMsub}\gm/\nmax$ for a sufficiently
small numerical constant $c$, then for
$R\ldef{}2\sup_{M\in\cMsub}\gm$, 
\[
\log(\Tgl(\cM;\veps,\nmax,\delta,R)) \geq{} \frac{\delta^2 \cdot \inf_{M \in \cMsub} \delminm}{R}\cdot{} \sup_{\Mbar\in\cMall}\aecM{2\veps}{\cM}(\cMoptsub(\Mbar),\Mbar).
\]

      \subsection{Examples}
      \label{sec:examples_lower}

As we discuss in the sequel, the dependence on the \CompText in our
lower bounds (particularly \pref{thm:lower_logt}) qualitatively matches the dependence on the
\CompShort in our main upper bounds, \pref{thm:upper_main,thm:upper_main_no_mingap}. We defer a
detailed discussion comparing our upper and lower bounds for a moment, and make matters concrete by
considering three examples: the informative arm example from the introduction (\pref{ex:revealing}), multi-armed bandits, and tabular reinforcement
learning. We defer proofs for all examples to \Cref{sec:lb_ex_proofs}.

\begin{example}[Searching for an Informative Arm (revisited)]
  \label{ex:revealing_revisited}
  Let $\cM$ be the class constructed in \pref{ex:revealing}
  with parameters $N,A\in\bbN$ and $\beta\in(0,1/2)$. Let $\cMsub$ denote the restriction of $\cM$ to models with
  $\delminm\geq{}\Delta$ for some $\Delta\in(0,1/6)$ (so that $\pim$
  is unique), and $\fm(\pim) < 1$. As long as
  $N\geq{}16$ and $\beta\log(1+\beta{}A)\geq{}2\Delta/(A-1)$, we have
  \begin{align}
    \sup_{\Mbar\in\cMall}\aecM{\veps}{\cM}(\cMoptsub(\Mbar),\Mbar) \geq{} \frac{N}{4\beta}.
  \end{align}
  For this construction, we have
$\delminm\geq\Delta$ and $\Omega( 1) \le \gm \le \bigoh(\beta^{-1})$ for all
$M\in\cMsub$. Thus, \pref{thm:lower_logt} implies that any algorithm with
$\Ema\brk{\RegDM}\leq{}2\gm\log(T)$ for all $M\in\cMsub$ must 
fail to learn an $\veps$-optimal allocation with probability
$\delta=\bigom(\veps/\nmax)$ unless
\[
\sup_{M\in\cM}\frac{\gm}{\delminm}\cdot{}\log(T) \approxgeq{} \frac{\veps^2}{\nmax^2} \cdot{}\frac{N}{\beta},
\]
which implies that $\log(\Tgl(\cMsub;\veps,\nmax,\delta,R))
\approxgeq{}\veps^2/\nmax^2 \cdot{}N$ for
$R=2\sup_{M\in\cM}\gm = \bigoh(\beta^{-1})$. Any Graves-Lai allocation need only take at most $\bigoh(\beta^{-1})$ pulls to eliminate alternative instances, so an appropriate choice of $\nmax$ is $\bigoh(\beta^{-1})$, yielding
$\log(\Tgl(\cMsub;\veps,\nmax,\delta,R)) \approxgeq{} \beta^2 \veps^2 \cdot{}N$.
While the dependence on the parameter $\nmax,
\veps, \Delta>0$ here is certainly loose, this formalizes the
intuition sketched in the introduction: any algorithm that learns an
optimal allocation must explore $\bigom\prn*{N}$ times, yet any algorithm
that achieves near-instance-optimal regret
$\Ema\brk*{\RegDM}\leq{}2\gm\log(T)\approxleq\beta^{-1}\log(T)$ can
play a sub-optimal decision no more than roughly $\beta^{-1}\log(T)/\Delta$ times,
leading to the constraint that
\[
\log(T) \approxgeq N.
\]
We can also apply \pref{thm:lower_t} to
show that any algorithm must fail to learn an $\veps$-optimal allocation with
probability $\delta=\bigom(\veps/\nmax)$ unless
$T \approxgeq \veps/\nmax \cdot\frac{N}{\beta}$.
\end{example}

\begin{example}[Finite-Armed Bandit]
  \label{ex:mab_lower}
  Let $A\geq{}6$ and $\Delta\in(0,1/2)$ be given and set $\Pi=\brk{A}$. Let $\cM$ be the set of all
  multi-armed bandit instances with Gaussian noise:
  \begin{align}
    \cM = \crl*{M(\pi)=\cN(\fm(\pi),1/2)\mid{}\fm\in\brk*{0,1}^{A}}
  \end{align}
  and let $\cMsub = \{ M \in \cM \ : \ |\pibm| = 1, \delm(\pi) \in
  [\Delta,2\Delta] \text{ for } \pi \neq \pim, \fm(\pim) < 1 \}$
  denote the set of all bandit instances where suboptimal arms have gaps on order $\Delta$. 
  Then for all $\veps\in(0,1/32)$,
  \begin{align}
    \label{eq:lower_logt}
    \sup_{\Mbar\in\cM}\aecM{\veps}{\cM}(\cMoptsub(\Mbar),\Mbar) \geq{} c\cdot\veps^{-2}\cdot\frac{A}{\Delta^2},
  \end{align}
  where $c>0$ is an absolute constant.

It can be shown that, by the construction of $\cMsub$, $\gm = \Theta(A/\Delta)$ for all $M \in \cMsub$, so $\delta \propto \veps \cdot \min \{ 1, \frac{A}{\Delta \nmax} \}$. \pref{thm:lower_t} then implies that any algorithm must fail to learn an $\veps$-optimal allocation with probability
$\delta=\bigom(\veps \cdot \min \{ 1, \frac{A}{\Delta \nmax} \})$ unless
\begin{align*}
  T \approxgeq \min \crl*{ 1, \frac{A}{\Delta \nmax} } \cdot\frac{A}{\veps \Delta^2}.
\end{align*}
Any Graves-Lai allocation need only take $\bigoh(\frac{A}{\Delta^2})$ pulls to eliminate all alternate instances, so a reasonable choice of $\nmax$ is therefore $\bigoh(\frac{A}{\Delta^2})$. With this choice of $\nmax$, we have $\Tgl(\cM;\veps,\nmax,\delta)\approxgeq \frac{A}{\veps \Delta}$. We remark that the scaling on all parameters here is natural. Intuitively, we would expect that we need to pull each arm at least once to learn a near-optimal allocation, yielding an $\Omega(A)$ scaling. In addition, note that for multi-armed bandits, the optimal allocation places mass $\propto \frac{1}{\Delta^2}$ on arms with gap $\Delta$. To correctly estimate this proportion requires an accurate estimate of $\Delta$, which becomes increasingly difficult as $\Delta$ becomes smaller, yielding an $\Omega(\frac{1}{\Delta})$ scaling. Finally, as we decrease $\veps$, we require that the returned allocation becomes closer to a truly optimal allocation, and we would therefore expect an $\Omega(\frac{1}{\veps})$ scaling.

Note that the result derived by applying \Cref{thm:lower_t} above only gives a lower bound on $T$. To obtain a lower bound on $\log (T)$, we combine
\pref{thm:lower_logt} and
\eqref{eq:lower_logt} with $\nmax = \bigoh(\frac{A}{\Delta^2})$ as above, which implies that any algorithm with
$\Ema\brk{\RegDM}\leq{}2\gm\log(T)$ for all $M\in\cMsub$ must 
fail to learn an $\veps$-optimal allocation with probability
$\delta=\bigom(\veps \cdot \min \{ 1, \frac{A}{\Delta \nmax} \})$ unless
\[
\sup_{M\in\cMsub}\frac{\gm}{\delminm}\cdot{}\log(T) \approxgeq{} A,
\]
or equivalently $\log(\Tgl(\cM;\veps,\nmax,\delta,R))
\approxgeq{} \frac{A}{\sup_{M\in\cMsub}\gm/\delminm}$ for
$R=2\sup_{M\in\cM}\gm$. To see why such scaling is natural, note that
any algorithm which has $\Ema\brk*{\RegDM}\leq{}2\gm\log(T)$ for each
instance $M$ can afford to explore (that is, play a suboptimal
decision) at most $2\frac{\gm}{\delminm}\log(T)$ times, or their regret could exceed $2 \gm \log (T)$. However, no algorithm has any hope of learning an optimal allocation unless they play every arm at least once, so taking at least $A$ pulls of suboptimal arms seems unavoidable, and we therefore would expect that we must have $2\frac{\gm}{\delminm}\log(T) \gtrsim A$, which is precisely the necessary scaling shown here.

We make two remarks on this $\log (T)$ lower bound. First, note that due to the presence of the $\delta^2$ term in
\eqref{eq:lower_logt} (which we believe to be loose), the lower bound
we derive by applying \cref{thm:lower_logt} does not
scale with the parameter $\veps^{-1}$, as one might hope. Second, as noted, for $M \in \cMsub$ for $\cMsub$ chosen as in \Cref{ex:mab_lower}, we have $\gm = \Omega(A/\Delta)$, in which cases the dependence on $A$ cancels, and the lower bound becomes trivial. Note that this is somewhat to be expected. \Cref{thm:lower_logt} will give a trivial lower bound whenever $\sup_{M \in \cMsub} \gm$ is much larger than the \CompShort. Recall that in our upper bound, \Cref{thm:upper_main_no_mingap}, the leading order $\gst \cdot \log T$ term will dominate the lower-order terms once
\begin{align*}
\gst \cdot \log T \gtrsim \Omega(\aec[\veps](\cM)).
\end{align*}
Therefore, if $\gst$ is much larger than the
\CompShort, \Cref{thm:upper_main_no_mingap} simply gives an upper bound scaling
as approximately $\gst \cdot \log T$, as long as $\log T =
\Omega(1)$. The interpretation in this setting is that the complexity
of learning the Graves-Lai allocation is dominated by the regret
incurred by \emph{playing} the Graves-Lai allocation, which we know is
necessary from \Cref{prop:glc}, and therefore we would not expect
lower-order terms to be significant components of the regret, as
reflected by \Cref{thm:upper_main_no_mingap}. However, as we have already shown,
In settings such as \cref{ex:revealing_revisited} where this is not the
case and the \CompShort is much larger than $\gst$, \Cref{thm:lower_logt}
will give a non-trivial lower bound which reflects the difficulty of
learning the optimal allocation.
\end{example}

\begin{example}[Tabular Reinforcement Learning]
  \label{ex:tabular_lower}
  Let $S,A,H\in\bbN$ and $\Delta\in(0,1/2)$ be given, and assume that
  $SA\geq{}24$ and $H\geq{}\log_2(S/2)$. Let $\cM$ be the set of all
  tabular MDPs with 1) $\abs{\cS}=S$, $\abs{\cA}=A$ and horizon $H$, and 2) Gaussian rewards with variance
  $\sigma^2=1/2$ (cf. \pref{sec:tabular_results}). Let $\cMsub$ be the
  result of restricting $\cM$ in the same fashion as \Cref{ex:mab_lower}.
  Then for all $\veps\in(0,1/32)$,
  \begin{align}
    \label{eq:tabular_aec_lower}
    \sup_{\Mbar\in\cM}\aecM{\veps}{\cM}(\cMoptsub(\Mbar),\Mbar) \geq{} c\cdot\veps^{-2}\cdot\frac{SA}{\Delta^2},
  \end{align}
  where $c>0$ is an absolute constant.

It can be shown that, by the construction of $\cMsub$, $\gm \ge \Omega(SA/\Delta)$ for all $M \in \cMsub$. 
Thus, analogous to the multi-armed bandit example,
\pref{thm:lower_t} and \eqref{eq:tabular_aec_lower} imply that
any algorithm must fail to learn an $\veps$-optimal allocation with
probability $\delta=\bigom(\veps \cdot \min \{ 1, \frac{SA}{\Delta \nmax} \})$ unless
\begin{align*}
  T \approxgeq \min \crl*{ 1, \frac{SA}{\Delta \nmax}} \cdot\frac{SA}{\veps \Delta^2}.
\end{align*}
Choosing $\nmax = \bigoh(\frac{SA}{\Delta^2})$ gives $\Tgl(\cM;\veps,\nmax,\delta)\approxgeq \frac{SA}{\veps \Delta}$. With this same choice of $\nmax$, 
 \pref{thm:lower_logt} implies that any algorithm with
$\Ema\brk{\RegDM}\leq{}2\gm\log(T)$ for all $M\in\cMsub$ must 
fail to learn an $\veps$-optimal allocation with probability
$\delta=\bigom(\veps \Delta)$ unless
\[
\sup_{M\in\cMsub}\frac{\gm}{\delminm}\cdot{}\log(T) \approxgeq{} SA,
\]
or equivalently $\log(\Tgl(\cM;\veps,\nmax,\delta,R))
\approxgeq{}\frac{SA}{\sup_{M\in\cMsub}\gm/\delminm}$ for
$R=2\sup_{M\in\cMsub}\gm$.  
  
\end{example}

\subsection{Discussion and Interpretation}
    \label{sec:gaps}

Our lower bounds show that the \CompText serves as a fundamental limit
on the sample complexity required to learn an approximate
\alloc. In particular, they capture phenomena such as the necessity of searching for
an informative arm in \pref{ex:revealing} that are missed by purely
asymptotic analyses. To the best of our knowledge, our lower bounds
represent the first attempt to systematically understand the sample complexity
of learning the \alloc in a general decision making framework. As
such, they are somewhat coarse (in particular, the dependence on
parameters such as $\veps$, $\nmax$, and $\delminm$ is almost
certainly loose), and they are best thought of as a starting point for
further research. In what follows, we provide additional
interpretation of the results, and highlight some of the most
interesting remaining questions.     

\paragraph{Regret versus learning the optimal allocation}
\pref{thm:lower_t,thm:lower_logt} lower bound the sample complexity
required to learn an $\veps$-optimal \alloc. Intuitively, this task is
closely related to achieving instance-optimal regret. Our analysis of \mainalg shows that it is \emph{sufficient}, and many
prior works aim to directly estimate the optimal allocation as well. However, it is unclear to what extent learning the
optimal allocation is \emph{necessary} to achieve instance-optimal regret.

In more detail, it is quite straightforward to show that if an algorithm achieves
instance-optimal regret, its empirical frequencies act as an optimal
allocation \emph{in expectation}.

    \begin{restatable}{lemma}{optregretfeasible}
  \label{lem:opt_regret_feasible}
  Let $\veps\in(0,2)$, and suppose that \pref{asm:mingap} holds. Fix
  $T\in\bbN$ and consider an 
  algorithm $\mathbb{A}$ such that for all $M\in\cM$,
  \[
    \Ema\brk*{\RegDM}
    \leq{} (1+\veps)\gm\cdot{}\log(T).
  \]
  For each $M\in\cM$, define $\etam\in\Aspace$ via
  $\etam(\pi) = \Ema\brk*{\frac{T(\pi)}{\log(T)}}$, where $T(\pi)$
  denotes the number of pulls of decision $\pi$, and define $\lambdam=\etam/\nrm*{\etam}_1$.
  Then if
  \[
\log(T) \geq{} \frac{6}{\veps}\log\prn*{\sup_{M\in\cM}\frac{2\gm}{\delminm}\cdot\log(T)},
\]
we have that for all $M\in\cM$,
\begin{equation}
  \label{eq:feasible}
  \lambdam \in \Lambdam.
\end{equation}
\end{restatable}
This result gives a guarantee on the expected frequencies of any
instance-optimal algorithm, but does not give any guarantee for the
realized frequencies. As such, without further assumptions on the algorithm under
consideration, it is unclear whether instance-optimal regret implies
that it is possible to learn an optimal allocation with high or even
\emph{constant} probability. We cannot currently rule out the
existence of pathological algorithms for
which $\etam$ is optimal in expectation, yet the empirical arm
frequencies deviate from the mean with moderate probability. Nonetheless, if one is
willing to make stronger assumptions on the algorithm under
consideration---in particular, that the \emph{second moment} of regret
is controlled---then it is possible to derive lower bounds on regret
directly.

\begin{theorem}[Simplified version of \pref{prop:regret_lower}]\label{thm:lb_aec_to_regret}
        Let the time horizon $T\in\bbN$ and $\veps\in(0,1/2)$ be
        given, and suppose that \pref{asm:mingap,asm:bounded_likelihood} hold. Suppose there exists an algorithm
      $\Alg$ with the property that for all $M\in\cM$:
1) $\Ema\brk*{\RegDM}\leq{}(1+\veps)\gm[M]\log(T)$, 2)
$\sqrt{\Ema\brk*{(\RegDM)^2}}\leq{}2\gm\log(T)$, and 3) for all $\pi\in\Pi$, if $\Ema\brk*{T(\pi)}\neq{}0$, then $\Ema\brk*{T(\pi)}\geq{}1$.
Then if we define $\delta=\veps \cdot \min \{ 1, \inf_{M\in\cMsub} \frac{\gm}{3 \gm/\delminm + \ncMeps} \}$, it
must be the case that
\begin{align*}
\log^3(T)
  \geq{} \frac{\delta^2}{C} \cdot\sup_{\Mbar\in\cMall}\aecM{4\veps}{\cM}(\cMopt(\Mbar),\Mbar).
\end{align*}
for $C\leq{} \bigoh\prn*{(\sup_{M\in\cM}\frac{\gm}{\delminm})^4\cdot\frac{\VM^2\log(\delta^{-1})}{\veps^2}}$.
\end{theorem}
See \pref{app:lower_regret} for a full statement and details. The
idea behind the proof is to 1) show (via robust mean estimation) that any instance-optimal
algorithm with well-behaved tails can be used to estimate the optimal
allocation with high probability (with a small blowup in time horizon), and then 2) appeal to \pref{thm:lower_logt}.
More work is required to understand whether 1) we can prove lower
bounds on regret directly, and 2) whether it is possible to show that
low regret and learning the optimal allocation are equivalent in a
stronger sense.

\paragraph{Comparing upper and lower bounds}
Keeping the differences between regret minimization and learning the
optimal allocation in mind, let us highlight that the lower bound on
$T$ provided by \pref{thm:lower_logt} seems to qualitatively match the
upper bound from \pref{thm:upper_main,thm:upper_main_no_mingap}. In particular, ignoring
problem-dependent parameters and polylogarithmic factors, the upper bound \pref{thm:upper_main}
scales, for every model $M\in\cM$, as 
\begin{align*}
  \En\sups{M}\brk*{\RegDM}
  \leq{} (1+\veps)\gm\log(T) + \specialOt(\aec(\cM)).
\end{align*}
In order for this bound to simplify to, say,
\begin{align*}
  \En\sups{M}\brk*{\RegDM}
  \leq{} (1+2\veps)\gm\log(T),
\end{align*}
we need
\begin{align*}
  \gm\cdot{}\log(T) \geq \specialOmt(1)\cdot\frac{\aec(\cM)}{\veps}, 
\end{align*}
which has similar scaling to the lower bound
\begin{align*}
  \log(T) \approxgeq{}
  \specialOmt(1)\cdot\sup_{\Mbar\in\cMall}\aecM{\veps}{\cM}(\cMoptsub(\Mbar),\Mbar)
\end{align*}
from \pref{thm:lower_logt}. As discussed in the prequel,
the former result is concerned with regret, while the latter considers
the task of learning the optimal allocation, but the scaling $\log(T)
\approxgeq \aec(\cM)$ seems to be fundamental for both. Of course, beyond the gap
between regret and learning the allocation, there is still much room
to improve the dependence on problem-dependent parameters in both results.

    \paragraph{Comparing \cref{thm:lower_t} and \cref{thm:lower_logt}}
    \pref{thm:lower_t} and \pref{thm:lower_logt} exhibit an
    interesting dichotomy: \pref{thm:lower_t} places no constraints on
    the regret of the algorithm under consideration, and gives a lower
    of the form $T \approxgeq{}
    \specialOmt(1)\cdot\sup_{\Mbar\in\cMall}\aecM{\veps}{\cM}(\cMsub,\Mbar)$, while \pref{thm:lower_logt}
     gives a lower bound of the
    form
    $\log(T) \approxgeq{}
    \specialOmt(1)\cdot\sup_{\Mbar\in\cMall}\aecM{\veps}{\cM}(\cMoptsub(\Mbar),\Mbar)$, or
    equivalently $T\approxgeq{}
    \exp ( \specialOmt(1)\cdot\sup_{\Mbar\in\cMall}\aec(\cMoptsub(\Mbar),\Mbar))$;
    the latter lower bound is exponentially stronger, with the caveat
    that 1) the class $\cMsub$ is replaced with the subclass
    $\cMoptsub(\Mbar)$ and 2) \pref{thm:lower_logt} assumes that the algorithm achieve
    nearly-instance optimal regret for every model in $\cMsub$
    ($\Ema\brk*{\Reg}\leq2\cdot\gm\log(T)$). In what follows, we argue that
    this tradeoff is fundamental.
    \begin{itemize}
    \item First, let us consider the role of the assumption
      $\Ema\brk*{\Reg}\leq2\cdot\gm\log(T)$. Without this assumption,
      the lower bound from \pref{thm:lower_t} is qualitatively tight: if
      the algorithm explores optimally for every round $t\in\brk{T}$,
      it gains roughly
      $(\sup_{\Mbar\in\cMall}\aecM{\veps}{\cM}(\cMsub,\Mbar))^{-1}$ units of
      information per round, which is sufficient to identify an
      optimal allocation as soon as $T\approxgeq{}
      \sup_{\Mbar\in\cMall}\aecM{\veps}{\cM}(\cMsub,\Mbar)$. 
    \item On the other hand, if the require that
      $\Ema\brk*{\Reg}\leq2\cdot\gm\log(T)$, then for each $M\in\cMsub$,
      the algorithm can afford to explore (i.e., play a non-optimal
      action) at most $2\cdot\frac{\gm}{\delminm}\log(T)$ times. This
      changes the ``effective'' time horizon for exploration to
      $T'=2\cdot\frac{\gm}{\delminm}\log(T)$, but only if we restrict to
      models for which playing an optimal decision gives no
      information. This is precisely what the subclass
      $\cMoptsub(\Mbar)$ captures: models $M\in\cMsub$ for which decisions that are
      optimal for $\Mbar$ lead to no information. Combining these
      insights leads to the lower bound
      $T'\approx\frac{\gm}{\delminm}\log(T)\approxgeq{}
      \bigom(1)\cdot\sup_{\Mbar\in\cMall}\aecM{\veps}{\cM}(\cMoptsub(\Mbar),\Mbar)$
      in \pref{thm:lower_logt}.
    \end{itemize}
We remark in passing that the definition of $\cMoptsub(\Mbar)$, which
places the constraint that
$\Dkl{\Mbar(\pi)}{M(\pi)}=0,\;\;\forall{}\pi\in\pibmbar$ is somewhat
coarse. We expect that both \pref{thm:lower_logt} and
\pref{thm:upper_main}/\cref{thm:upper_main_no_mingap} can be improved to scale with
\[
  \cMoptsub(\Mbar;\alpha) = \crl*{M\in\cMsub\mid{} \pibm\subseteq\pibmbar,\;\;
  \Dkl{\Mbar(\pi)}{M(\pi)}\leq\alpha^2\;\; \forall{}\pi\in\pibmbar}
\]
for $\alpha\approx{}1/\sqrt{T}$; the intuition is that we get
$\bigom(T)$ rounds worth of information on $\pibm$ for ``free'', which
facilitates accurate estimation.

    \paragraph{Minimax versus instance-dependent lower bounds} As
    mentioned in the prequel, our lower bounds have a (constrained)
    minimax flavor. Specifically, \pref{thm:lower_logt} shows that if
    $T$ is not sufficiently large, then for any algorithm, there must
    exist a ``worst-case'' model $M\in\cM$ for which the algorithm
    either 1) fails to achieve (approximately) instance-optimal regret
    or 2) fails to learn an $\veps$-optimal \alloc. While this is
    quite different from a classical minimax analysis, and certainly
    is closely connected to instance-optimality, an interesting
    direction for future work is to develop a fully instance-dependent
    understanding of the complexity of learning Graves-Lai allocations.


\section{Additional Related Work}
\label{sec:related}

In this section, we discuss further related work not already covered
in detail.

\paragraph{Asymptotic guarantees for general decision making}
For the general decision making framework we consider, which allows
for arbitrary model classes and subsumes structured bandits and
reinforcement, the only prior works we are aware of that achieve the
instance-optimal lower bound from \citet{graves1997asymptotically} are
\citet{komiyama2015regret}, which restricts to finite observation
spaces, and \citet{dong2022asymptotic}, which restricts to finite
decision spaces; these works do not provide non-asymptotic guarantees.

Many works provide purely asymptotic instance-optimality guarantees for more
specialized settings, including multi-armed bandits
\citep{lai1985asymptotically,garivier2016explore,lattimore2018refining,garivier2019explore},
linear bandits
\citep{lattimore2017end,hao2019adaptive,hao2020adaptive}, and general
structured bandits
\citep{burnetas1996optimal,magureanu2014lipschitz,combes2017minimal,van2020optimal,degenne2020structure}. Some
  of these works do provide non-asymptotic bounds on regret, but these
  results generally have lower-order terms that scale linearly in $\abs{\Pi}$, which renders them
  vacuous until $\log(T)\approxgeq\abs{\Pi}$; we consider such results
to be asymptotic in spirit. Along these lines, it is worth discussing \citet{jun2020crush}, which provides non-asymptotic
guarantees for structured bandits in which the lower-order terms scale
with a quantity $K_{\psi}$ that aims to capture the number of ``effective
arms''. While this quantity can improve over $\abs{\Pi}$ in certain
situations, it is not clear whether it is well behaved for standard
classes of interest (e.g., linear bandits).

\paragraph{Non-asymptotic guarantees for linear bandits}
For linear bandits, a number of recent works provide non-asymptotic
instance-optimal regret bounds in which lower order terms scale only
with the dimension $d$ rather than the number of decisions $\abs{\Pi}$
\citep{tirinzoni2020asymptotically,kirschner2021asymptotically}. These
results take advantage of the specialized geometric structure of the
linear bandit setting (e.g., existence of optimal design) for exploration, and cannot
be directly adapted to general function approximation, but our results can
be viewed as generalizing these guarantees.
\paragraph{Reinforcement learning}
For reinforcement learning, a number of works---mostly focusing on
tabular settings or linear function approximation---provides
non-asymptotic guarantees that are instance-dependent, but not necessarily
instance-optimal
\citep{simchowitz2019non,al2021adaptive,dann2021beyond,al2021navigating,wagenmaker2022beyond,wagenmaker2022instance,wagenmaker2022leveraging}. For
instance-optimality, the results we are aware of are the classical work of \cite{agrawal1988asymptotically} and very recent work
\cite{dong2022asymptotic}, which provides asymptotic guarantees for
finite-horizon tabular RL, \citet{ok2018exploration}, which provides
asymptotic guarantees for an infinite-horizon setting under ergodicity
assumptions, and \citet{tirinzoni2022near} which provides PAC
guarantees for deterministic MDPs (though we note that the guarantee of \citet{tirinzoni2022near} is also achieved, up to $H$ factors, by \cite{wagenmaker2022instance}).

\paragraph{Instance-optimal PAC guarantees}
Our discussion has largely centered on regret, which is the focus
of our work. For the PAC setting, where the goal is to identify the optimal decision (or a near-optimal decision) as quickly as possible, a number of recent works have
employed similar techniques to derive instance-optimal algorithms for
settings such as multi-armed and structured bandits
\citep{kaufmann2016complexity,garivier2016optimal,russo2016simple,degenne2019pure,degenne2019non,degenne2020gamification}. While many of these works are asymptotic in nature, in specialized settings such as multi-armed bandits \citep{jamieson2014lil}, linear bandits \citep{fiez2019sequential,katz2020empirical}, and linear dynamical systems \citep{wagenmaker2021task}, recent work has shown that the optimal rates are achievable in finite-time.

\paragraph{Complexity of learning the \alloc}
The \CompText aims to capture the sample complexity required to learn
an $\veps$-optimal \alloc. To the best of our knowledge, our work is
the first to study the complexity of learning the allocation with
general function approximation, but a small body of work has studied
the complexity in simple settings such as top-$k$ bandits
\citep{simchowitz2017simulator,chen2017nearly}, and graph bandits
\citep{marinov2022stochastic,marinov2022open}.

\paragraph{Minimax regret}
While the focus of this work has been on instance-optimality, a large body of work exists on \emph{minimax} optimality, where the goal is to perform optimally on the \emph{hardest} instance within a class. This line of work has established worst-case optimal (or nearly optimal) rates in settings such as multi-armed bandits \citep{auer2002finite,audibert2009minimax}, linear bandits \citep{dani2008stochastic,abbasi2011improved}, tabular reinforcement learning \citep{dann2019policy,zhang2021reinforcement}, and reinforcement learning with function approximation \citep{zhou2021nearly,du2021bilinear}. The recent line of work \citet{foster2021statistical,foster2022complexity,foster2023tight} 
shows that, in the interactive decision-making setting considered in
this work, the minimax-optimal rates are governed by a quantity known
as the \emph{Decision-Estimation Coefficient}. While our work takes
inspiration and bears some similarity with this work, we remark that
the techniques necessary to establish instance-optimality are
significantly more intricate. It is also worth stating that it is always possible to bound the DEC by the \CompShort; the converse, however, is not true.


\section{Discussion}
\label{sec:discussion}

Our work initiates the systematic study of non-asymptotic
instance-optimality in interactive decision making. We close by
highlighting a number of interesting open problems and future
directions raised by our work. On the technical side:
\begin{itemize}
\item Our upper bounds depend on a number of different
  problem-dependent parameters, such as the minimum gap in the
  lower-order terms. Can we improve the dependence on these
  parameters, 
  or understand to what extend they are necessary?
\item Our lower bounds concern the problem of learning a near-optimal
  allocation. Like the upper bounds, these results are likely loose in
  terms of dependence on various problem parameters, and new
  techniques will be required to tighten them. Furthermore, it remains to develop a complete understanding of the connections between
  this problem and the problem of minimizing regret. While our results show that ``well-behaved'' algorithms which achieve instance-optimal regret must pay a burn-in proportional to the cost of learning a near-optimal allocation, it remains unclear if this is truly necessary for algorithms which only have optimal expected regret (but, for example, could exhibit heavy-tailed behavior). 
\item We show that boundedness of the \CompText is necessary to learn
  the optimal allocation in a minimax sense, but a natural question
  for future research is to derive instance-dependent guarantees for
  learning optimal allocations. The tools used to prove the Graves-Lai
  lower bound can be applied similarly to prove an instance-dependent
  lower-bound on the cost of learning the optimal allocation, but it is not clear that such lower bounds capture the true complexity of learning (since, as with the Graves-Lai lower bound, realizing such a lower bound would itself require knowledge of the ground truth instance). Is a minimax-style lower bound always necessary in order to capture the performance of algorithms which do not assume initial knowledge of the ground truth instance?
\item While our algorithm achieves the instance-optimal rate, its regret could scale linearly over shorter time horizons, until it has learned a near-optimal allocation. 
Can we develop ``best-of-both-worlds'' algorithms that achieve the same
  instance-optimal guarantees of \mainalg, yet also achieves the
  minimax-optimal rate (for example, a $\cO(\sqrt{T})$-style guarantee) over shorter time horizons?
  \item The focus of this work is primarily on regret minimization, yet the challenge of learning the optimal allocation also arises in the PAC setting. Does the \CompShort extend to the PAC setting, and can algorithms be developed in the PAC setting which achieve the instance-optimal rate in the leading-order term, while scaling with an \CompShort-like quantity in the lower-order term?
\end{itemize}
More broadly, it will be interesting to explore whether our framework
and algorithm design ideas can be used to develop practical and computationally efficient algorithms.

\arxiv{
\subsection*{Acknowledgements}
The authors would like to thank Johannes Kirschner for helpful
discussions. The work of AW was supported in part by NSF TRIPODS
62-2945 and NSF HDR 62-0221. A portion of this work was completed
while AW was an intern at Microsoft Research, and while visiting the Simons Institute for the Theory of Computing. 
}

\clearpage

\bibliography{refs} 

\begin{thebibliography}{81}
\providecommand{\natexlab}[1]{#1}
\providecommand{\url}[1]{\texttt{#1}}
\expandafter\ifx\csname urlstyle\endcsname\relax
  \providecommand{\doi}[1]{doi: #1}\else
  \providecommand{\doi}{doi: \begingroup \urlstyle{rm}\Url}\fi

\bibitem[Abbasi-Yadkori et~al.(2011)Abbasi-Yadkori, P{\'a}l, and
  Szepesv{\'a}ri]{abbasi2011improved}
Yasin Abbasi-Yadkori, D{\'a}vid P{\'a}l, and Csaba Szepesv{\'a}ri.
\newblock Improved algorithms for linear stochastic bandits.
\newblock In \emph{Advances in Neural Information Processing Systems}, 2011.

\bibitem[Agrawal et~al.(1988)Agrawal, Teneketzis, and
  Anantharam]{agrawal1988asymptotically}
Rajeev Agrawal, Demosthenis Teneketzis, and Venkatachalam Anantharam.
\newblock Asymptotically efficient adaptive allocation schemes for controlled
  markov chains: Finite parameter space.
\newblock Technical report, MICHIGAN UNIV ANN ARBOR COMMUNICATIONS AND SIGNAL
  PROCESSING LAB, 1988.

\bibitem[Al~Marjani and Proutiere(2021)]{al2021adaptive}
Aymen Al~Marjani and Alexandre Proutiere.
\newblock Adaptive sampling for best policy identification in markov decision
  processes.
\newblock In \emph{International Conference on Machine Learning}, pages
  7459--7468. PMLR, 2021.

\bibitem[Al~Marjani et~al.(2021)Al~Marjani, Garivier, and
  Proutiere]{al2021navigating}
Aymen Al~Marjani, Aur{\'e}lien Garivier, and Alexandre Proutiere.
\newblock Navigating to the best policy in markov decision processes.
\newblock \emph{Advances in Neural Information Processing Systems},
  34:\penalty0 25852--25864, 2021.

\bibitem[Allenberg et~al.(2006)Allenberg, Auer, Gy{\"o}rfi, and
  Ottucs{\'a}k]{allenberg2006}
Chamy Allenberg, Peter Auer, L{\'a}szl{\'o} Gy{\"o}rfi, and Gy{\"o}rgy
  Ottucs{\'a}k.
\newblock \emph{Hannan Consistency in On-Line Learning in Case of Unbounded
  Losses Under Partial Monitoring}, pages 229--243.
\newblock Springer Berlin Heidelberg, Berlin, Heidelberg, 2006.
\newblock ISBN 978-3-540-46650-5.
\newblock \doi{10.1007/11894841_20}.
\newblock URL \url{http://dx.doi.org/10.1007/11894841_20}.

\bibitem[Audibert and Bubeck(2009)]{audibert2009minimax}
Jean-Yves Audibert and S{\'e}bastien Bubeck.
\newblock Minimax policies for adversarial and stochastic bandits.
\newblock In \emph{COLT}, volume~7, pages 1--122, 2009.

\bibitem[Auer et~al.(2002)Auer, Cesa-Bianchi, and Fischer]{auer2002finite}
Peter Auer, Nicolo Cesa-Bianchi, and Paul Fischer.
\newblock Finite-time analysis of the multiarmed bandit problem.
\newblock \emph{Machine learning}, 47\penalty0 (2-3):\penalty0 235--256, 2002.

\bibitem[Bubeck et~al.(2018)Bubeck, Cohen, and Li]{bubeck2018sparsity}
S{\'e}bastien Bubeck, Michael Cohen, and Yuanzhi Li.
\newblock Sparsity, variance and curvature in multi-armed bandits.
\newblock In \emph{Algorithmic Learning Theory}, pages 111--127. PMLR, 2018.

\bibitem[Burnetas and Katehakis(1996)]{burnetas1996optimal}
Apostolos~N Burnetas and Michael~N Katehakis.
\newblock Optimal adaptive policies for sequential allocation problems.
\newblock \emph{Advances in Applied Mathematics}, 17\penalty0 (2):\penalty0
  122--142, 1996.

\bibitem[Chen et~al.(2022)Chen, Mei, and Bai]{chen2022unified}
Fan Chen, Song Mei, and Yu~Bai.
\newblock Unified algorithms for {RL} with decision-estimation coefficients:
  No-regret, {PAC}, and reward-free learning.
\newblock \emph{arXiv preprint arXiv:2209.11745}, 2022.

\bibitem[Chen et~al.(2017)Chen, Li, and Qiao]{chen2017nearly}
Lijie Chen, Jian Li, and Mingda Qiao.
\newblock Nearly instance optimal sample complexity bounds for top-k arm
  selection.
\newblock In \emph{Artificial Intelligence and Statistics}, pages 101--110.
  PMLR, 2017.

\bibitem[Combes et~al.(2017)Combes, Magureanu, and
  Proutiere]{combes2017minimal}
Richard Combes, Stefan Magureanu, and Alexandre Proutiere.
\newblock Minimal exploration in structured stochastic bandits.
\newblock In \emph{Proceedings of the 31st International Conference on Neural
  Information Processing Systems}, pages 1761--1769, 2017.

\bibitem[Dani et~al.(2008)Dani, Hayes, and Kakade]{dani2008stochastic}
Varsha Dani, Thomas~P Hayes, and Sham~M Kakade.
\newblock Stochastic linear optimization under bandit feedback.
\newblock In \emph{Conference on Learning Theory (COLT)}, 2008.

\bibitem[Dann et~al.(2017)Dann, Lattimore, and Brunskill]{dann2017unifying}
Christoph Dann, Tor Lattimore, and Emma Brunskill.
\newblock Unifying pac and regret: Uniform pac bounds for episodic
  reinforcement learning.
\newblock \emph{Advances in Neural Information Processing Systems}, 30, 2017.

\bibitem[Dann et~al.(2019)Dann, Li, Wei, and Brunskill]{dann2019policy}
Christoph Dann, Lihong Li, Wei Wei, and Emma Brunskill.
\newblock Policy certificates: Towards accountable reinforcement learning.
\newblock In \emph{International Conference on Machine Learning}, pages
  1507--1516. PMLR, 2019.

\bibitem[Dann et~al.(2021)Dann, Marinov, Mohri, and Zimmert]{dann2021beyond}
Christoph Dann, Teodor~Vanislavov Marinov, Mehryar Mohri, and Julian Zimmert.
\newblock Beyond value-function gaps: Improved instance-dependent regret bounds
  for episodic reinforcement learning.
\newblock \emph{Advances in Neural Information Processing Systems},
  34:\penalty0 1--12, 2021.

\bibitem[Degenne and Koolen(2019)]{degenne2019pure}
R{\'e}my Degenne and Wouter~M Koolen.
\newblock Pure exploration with multiple correct answers.
\newblock \emph{Advances in Neural Information Processing Systems}, 32, 2019.

\bibitem[Degenne et~al.(2019)Degenne, Koolen, and M{\'e}nard]{degenne2019non}
R{\'e}my Degenne, Wouter~M Koolen, and Pierre M{\'e}nard.
\newblock Non-asymptotic pure exploration by solving games.
\newblock \emph{Advances in Neural Information Processing Systems}, 32, 2019.

\bibitem[Degenne et~al.(2020{\natexlab{a}})Degenne, M{\'e}nard, Shang, and
  Valko]{degenne2020gamification}
R{\'e}my Degenne, Pierre M{\'e}nard, Xuedong Shang, and Michal Valko.
\newblock Gamification of pure exploration for linear bandits.
\newblock In \emph{International Conference on Machine Learning}, pages
  2432--2442. PMLR, 2020{\natexlab{a}}.

\bibitem[Degenne et~al.(2020{\natexlab{b}})Degenne, Shao, and
  Koolen]{degenne2020structure}
R{\'e}my Degenne, Han Shao, and Wouter Koolen.
\newblock Structure adaptive algorithms for stochastic bandits.
\newblock In \emph{International Conference on Machine Learning}, pages
  2443--2452. PMLR, 2020{\natexlab{b}}.

\bibitem[Domingues et~al.(2021)Domingues, M{\'e}nard, Kaufmann, and
  Valko]{domingues2021episodic}
Omar~Darwiche Domingues, Pierre M{\'e}nard, Emilie Kaufmann, and Michal Valko.
\newblock Episodic reinforcement learning in finite mdps: Minimax lower bounds
  revisited.
\newblock In \emph{Algorithmic Learning Theory}, pages 578--598. PMLR, 2021.

\bibitem[Dong and Ma(2022)]{dong2022asymptotic}
Kefan Dong and Tengyu Ma.
\newblock Asymptotic instance-optimal algorithms for interactive decision
  making.
\newblock \emph{arXiv preprint arXiv:2206.02326}, 2022.

\bibitem[Du et~al.(2021)Du, Kakade, Lee, Lovett, Mahajan, Sun, and
  Wang]{du2021bilinear}
Simon~S Du, Sham~M Kakade, Jason~D Lee, Shachar Lovett, Gaurav Mahajan, Wen
  Sun, and Ruosong Wang.
\newblock Bilinear classes: A structural framework for provable generalization
  in {RL}.
\newblock \emph{International Conference on Machine Learning}, 2021.

\bibitem[Fiez et~al.(2019)Fiez, Jain, Jamieson, and
  Ratliff]{fiez2019sequential}
Tanner Fiez, Lalit Jain, Kevin~G Jamieson, and Lillian Ratliff.
\newblock Sequential experimental design for transductive linear bandits.
\newblock \emph{Advances in neural information processing systems}, 32, 2019.

\bibitem[Foster and Rakhlin(2020)]{foster2020beyond}
Dylan~J Foster and Alexander Rakhlin.
\newblock Beyond {UCB}: Optimal and efficient contextual bandits with
  regression oracles.
\newblock \emph{International Conference on Machine Learning (ICML)}, 2020.

\bibitem[Foster et~al.(2016)Foster, Li, Lykouris, Sridharan, and
  Tardos]{foster2016learning}
Dylan~J Foster, Zhiyuan Li, Thodoris Lykouris, Karthik Sridharan, and Eva
  Tardos.
\newblock Learning in games: Robustness of fast convergence.
\newblock \emph{Advances in Neural Information Processing Systems}, 29, 2016.

\bibitem[Foster et~al.(2020)Foster, Rakhlin, Simchi-Levi, and
  Xu]{foster2020instance}
Dylan~J Foster, Alexander Rakhlin, David Simchi-Levi, and Yunzong Xu.
\newblock Instance-dependent complexity of contextual bandits and reinforcement
  learning: A disagreement-based perspective.
\newblock \emph{Conference on Learning Theory (COLT)}, 2020.

\bibitem[Foster et~al.(2021)Foster, Kakade, Qian, and
  Rakhlin]{foster2021statistical}
Dylan~J Foster, Sham~M Kakade, Jian Qian, and Alexander Rakhlin.
\newblock The statistical complexity of interactive decision making.
\newblock \emph{arXiv preprint arXiv:2112.13487}, 2021.

\bibitem[Foster et~al.(2022{\natexlab{a}})Foster, Golowich, Qian, Rakhlin, and
  Sekhari]{foster2022note}
Dylan~J Foster, Noah Golowich, Jian Qian, Alexander Rakhlin, and Ayush Sekhari.
\newblock A note on model-free reinforcement learning with the
  decision-estimation coefficient.
\newblock \emph{arXiv preprint arXiv:2211.14250}, 2022{\natexlab{a}}.

\bibitem[Foster et~al.(2022{\natexlab{b}})Foster, Rakhlin, Sekhari, and
  Sridharan]{foster2022complexity}
Dylan~J Foster, Alexander Rakhlin, Ayush Sekhari, and Karthik Sridharan.
\newblock On the complexity of adversarial decision making.
\newblock \emph{arXiv preprint arXiv:2206.13063}, 2022{\natexlab{b}}.

\bibitem[Foster et~al.(2023)Foster, Golowich, and Han]{foster2023tight}
Dylan~J. Foster, Noah Golowich, and Yanjun Han.
\newblock Tight guarantees for interactive decision making with the
  decision-estimation coefficient.
\newblock \emph{arXiv preprint arXiv:2301.08215}, 2023.

\bibitem[Garivier and Kaufmann(2016)]{garivier2016optimal}
Aur{\'e}lien Garivier and Emilie Kaufmann.
\newblock Optimal best arm identification with fixed confidence.
\newblock In \emph{Conference on Learning Theory}, pages 998--1027, 2016.

\bibitem[Garivier et~al.(2016)Garivier, Lattimore, and
  Kaufmann]{garivier2016explore}
Aur{\'e}lien Garivier, Tor Lattimore, and Emilie Kaufmann.
\newblock On explore-then-commit strategies.
\newblock In \emph{Advances in Neural Information Processing Systems}, pages
  784--792, 2016.

\bibitem[Garivier et~al.(2019)Garivier, M{\'e}nard, and
  Stoltz]{garivier2019explore}
Aur{\'e}lien Garivier, Pierre M{\'e}nard, and Gilles Stoltz.
\newblock Explore first, exploit next: The true shape of regret in bandit
  problems.
\newblock \emph{Mathematics of Operations Research}, 44\penalty0 (2):\penalty0
  377--399, 2019.

\bibitem[Graves and Lai(1997)]{graves1997asymptotically}
Todd~L Graves and Tze~Leung Lai.
\newblock Asymptotically efficient adaptive choice of control laws in
  controlled {Markov} chains.
\newblock \emph{SIAM journal on control and optimization}, 35\penalty0
  (3):\penalty0 715--743, 1997.

\bibitem[Hao et~al.(2019)Hao, Lattimore, and Szepesvari]{hao2019adaptive}
Botao Hao, Tor Lattimore, and Csaba Szepesvari.
\newblock Adaptive exploration in linear contextual bandit.
\newblock \emph{arXiv preprint arXiv:1910.06996}, 2019.

\bibitem[Hao et~al.(2020)Hao, Lattimore, and Szepesvari]{hao2020adaptive}
Botao Hao, Tor Lattimore, and Csaba Szepesvari.
\newblock Adaptive exploration in linear contextual bandit.
\newblock In \emph{International Conference on Artificial Intelligence and
  Statistics}, pages 3536--3545. PMLR, 2020.

\bibitem[Hazan and Kale(2011)]{hazan2011better}
Elad Hazan and Satyen Kale.
\newblock Better algorithms for benign bandits.
\newblock \emph{Journal of Machine Learning Research}, 12\penalty0 (4), 2011.

\bibitem[Jamieson et~al.(2014)Jamieson, Malloy, Nowak, and
  Bubeck]{jamieson2014lil}
Kevin Jamieson, Matthew Malloy, Robert Nowak, and S{\'e}bastien Bubeck.
\newblock lil’ucb: An optimal exploration algorithm for multi-armed bandits.
\newblock In \emph{Conference on Learning Theory}, pages 423--439. PMLR, 2014.

\bibitem[Jiang et~al.(2017)Jiang, Krishnamurthy, Agarwal, Langford, and
  Schapire]{jiang2017contextual}
Nan Jiang, Akshay Krishnamurthy, Alekh Agarwal, John Langford, and Robert~E
  Schapire.
\newblock Contextual decision processes with low {Bellman} rank are
  {PAC}-learnable.
\newblock In \emph{International Conference on Machine Learning}, pages
  1704--1713, 2017.

\bibitem[Jin et~al.(2021)Jin, Liu, and Miryoosefi]{jin2021bellman}
Chi Jin, Qinghua Liu, and Sobhan Miryoosefi.
\newblock Bellman eluder dimension: New rich classes of {RL} problems, and
  sample-efficient algorithms.
\newblock \emph{Neural Information Processing Systems}, 2021.

\bibitem[Jun and Zhang(2020)]{jun2020crush}
Kwang-Sung Jun and Chicheng Zhang.
\newblock Crush optimism with pessimism: Structured bandits beyond asymptotic
  optimality.
\newblock \emph{Advances in Neural Information Processing Systems},
  33:\penalty0 6366--6376, 2020.

\bibitem[Kakade(2003)]{kakade2003sample}
Sham~Machandranath Kakade.
\newblock \emph{On the sample complexity of reinforcement learning}.
\newblock University of London, University College London (United Kingdom),
  2003.

\bibitem[Katz-Samuels et~al.(2020)Katz-Samuels, Jain, Jamieson,
  et~al.]{katz2020empirical}
Julian Katz-Samuels, Lalit Jain, Kevin~G Jamieson, et~al.
\newblock An empirical process approach to the union bound: Practical
  algorithms for combinatorial and linear bandits.
\newblock \emph{Advances in Neural Information Processing Systems},
  33:\penalty0 10371--10382, 2020.

\bibitem[Kaufmann et~al.(2016)Kaufmann, Capp{\'e}, and
  Garivier]{kaufmann2016complexity}
Emilie Kaufmann, Olivier Capp{\'e}, and Aur{\'e}lien Garivier.
\newblock On the complexity of best-arm identification in multi-armed bandit
  models.
\newblock \emph{The Journal of Machine Learning Research}, 17\penalty0
  (1):\penalty0 1--42, 2016.

\bibitem[Kirschner et~al.(2021)Kirschner, Lattimore, Vernade, and
  Szepesv{\'a}ri]{kirschner2021asymptotically}
Johannes Kirschner, Tor Lattimore, Claire Vernade, and Csaba Szepesv{\'a}ri.
\newblock Asymptotically optimal information-directed sampling.
\newblock In \emph{Conference on Learning Theory}, pages 2777--2821. PMLR,
  2021.

\bibitem[Komiyama et~al.(2015)Komiyama, Honda, and
  Nakagawa]{komiyama2015regret}
Junpei Komiyama, Junya Honda, and Hiroshi Nakagawa.
\newblock Regret lower bound and optimal algorithm in finite stochastic partial
  monitoring.
\newblock \emph{Advances in Neural Information Processing Systems}, 28, 2015.

\bibitem[Lai and Robbins(1985)]{lai1985asymptotically}
Tze~Leung Lai and Herbert Robbins.
\newblock Asymptotically efficient adaptive allocation rules.
\newblock \emph{Advances in Applied Mathematics}, 6\penalty0 (1):\penalty0
  4--22, 1985.

\bibitem[Lattimore(2018)]{lattimore2018refining}
Tor Lattimore.
\newblock Refining the confidence level for optimistic bandit strategies.
\newblock \emph{The Journal of Machine Learning Research}, 19\penalty0
  (1):\penalty0 765--796, 2018.

\bibitem[Lattimore and Szepesvari(2017)]{lattimore2017end}
Tor Lattimore and Csaba Szepesvari.
\newblock The end of optimism? an asymptotic analysis of finite-armed linear
  bandits.
\newblock In \emph{Artificial Intelligence and Statistics}, pages 728--737.
  PMLR, 2017.

\bibitem[Lillicrap et~al.(2015)Lillicrap, Hunt, Pritzel, Heess, Erez, Tassa,
  Silver, and Wierstra]{lillicrap2015continuous}
Timothy~P Lillicrap, Jonathan~J Hunt, Alexander Pritzel, Nicolas Heess, Tom
  Erez, Yuval Tassa, David Silver, and Daan Wierstra.
\newblock Continuous control with deep reinforcement learning.
\newblock \emph{arXiv preprint arXiv:1509.02971}, 2015.

\bibitem[Lugosi and Mendelson(2019)]{lugosi2019mean}
G{\'a}bor Lugosi and Shahar Mendelson.
\newblock Mean estimation and regression under heavy-tailed distributions: A
  survey.
\newblock \emph{Foundations of Computational Mathematics}, 19\penalty0
  (5):\penalty0 1145--1190, 2019.

\bibitem[Magureanu et~al.(2014)Magureanu, Combes, and
  Proutiere]{magureanu2014lipschitz}
Stefan Magureanu, Richard Combes, and Alexandre Proutiere.
\newblock Lipschitz bandits: Regret lower bound and optimal algorithms.
\newblock In \emph{Conference on Learning Theory}, pages 975--999. PMLR, 2014.

\bibitem[Marinov et~al.(2022{\natexlab{a}})Marinov, Mohri, and
  Zimmert]{marinov2022stochastic}
Teodor~V Marinov, Mehryar Mohri, and Julian Zimmert.
\newblock Stochastic online learning with feedback graphs: Finite-time and
  asymptotic optimality.
\newblock \emph{arXiv preprint arXiv:2206.10022}, 2022{\natexlab{a}}.

\bibitem[Marinov et~al.(2022{\natexlab{b}})Marinov, Mohri, and
  Zimmert]{marinov2022open}
Teodor~Vanislavov Marinov, Mehryar Mohri, and Julian Zimmert.
\newblock Open problem: Finite-time instance dependent optimality for
  stochastic online learning with feedback graphs.
\newblock In \emph{Conference on Learning Theory}, pages 5644--5649. PMLR,
  2022{\natexlab{b}}.

\bibitem[Mnih et~al.(2015)Mnih, Kavukcuoglu, Silver, Rusu, Veness, Bellemare,
  Graves, Riedmiller, Fidjeland, Ostrovski, et~al.]{mnih2015human}
Volodymyr Mnih, Koray Kavukcuoglu, David Silver, Andrei~A Rusu, Joel Veness,
  Marc~G Bellemare, Alex Graves, Martin Riedmiller, Andreas~K Fidjeland, Georg
  Ostrovski, et~al.
\newblock Human-level control through deep reinforcement learning.
\newblock \emph{Nature}, 518\penalty0 (7540):\penalty0 529, 2015.

\bibitem[Ok et~al.(2018)Ok, Proutiere, and Tranos]{ok2018exploration}
Jungseul Ok, Alexandre Proutiere, and Damianos Tranos.
\newblock Exploration in structured reinforcement learning.
\newblock \emph{Advances in Neural Information Processing Systems}, 31, 2018.

\bibitem[Osband and Van~Roy(2016)]{osband2016lower}
Ian Osband and Benjamin Van~Roy.
\newblock On lower bounds for regret in reinforcement learning.
\newblock \emph{arXiv preprint arXiv:1608.02732}, 2016.

\bibitem[Russo(2016)]{russo2016simple}
Daniel Russo.
\newblock Simple bayesian algorithms for best arm identification.
\newblock In \emph{Conference on Learning Theory}, pages 1417--1418, 2016.

\bibitem[Russo and Van~Roy(2013)]{russo2013eluder}
Daniel Russo and Benjamin Van~Roy.
\newblock Eluder dimension and the sample complexity of optimistic exploration.
\newblock In \emph{Advances in Neural Information Processing Systems}, pages
  2256--2264, 2013.

\bibitem[Russo and Van~Roy(2018)]{russo2018learning}
Daniel Russo and Benjamin Van~Roy.
\newblock Learning to optimize via information-directed sampling.
\newblock \emph{Operations Research}, 66\penalty0 (1):\penalty0 230--252, 2018.

\bibitem[Shamir(2011)]{shamir2011variant}
Ohad Shamir.
\newblock A variant of azuma's inequality for martingales with subgaussian
  tails.
\newblock \emph{arXiv preprint arXiv:1110.2392}, 2011.

\bibitem[Silver et~al.(2016)Silver, Huang, Maddison, Guez, Sifre, Van
  Den~Driessche, Schrittwieser, Antonoglou, Panneershelvam, Lanctot,
  et~al.]{silver2016mastering}
David Silver, Aja Huang, Chris~J Maddison, Arthur Guez, Laurent Sifre, George
  Van Den~Driessche, Julian Schrittwieser, Ioannis Antonoglou, Veda
  Panneershelvam, Marc Lanctot, et~al.
\newblock Mastering the game of go with deep neural networks and tree search.
\newblock \emph{nature}, 529\penalty0 (7587):\penalty0 484, 2016.

\bibitem[Simchowitz and Jamieson(2019)]{simchowitz2019non}
Max Simchowitz and Kevin~G Jamieson.
\newblock Non-asymptotic gap-dependent regret bounds for tabular mdps.
\newblock \emph{Advances in Neural Information Processing Systems}, 32, 2019.

\bibitem[Simchowitz et~al.(2017)Simchowitz, Jamieson, and
  Recht]{simchowitz2017simulator}
Max Simchowitz, Kevin Jamieson, and Benjamin Recht.
\newblock The simulator: Understanding adaptive sampling in the
  moderate-confidence regime.
\newblock In \emph{Conference on Learning Theory}, pages 1794--1834. PMLR,
  2017.

\bibitem[Sun et~al.(2019)Sun, Jiang, Krishnamurthy, Agarwal, and
  Langford]{sun2019model}
Wen Sun, Nan Jiang, Akshay Krishnamurthy, Alekh Agarwal, and John Langford.
\newblock Model-based {RL} in contextual decision processes: {PAC} bounds and
  exponential improvements over model-free approaches.
\newblock In \emph{Conference on learning theory}, pages 2898--2933. PMLR,
  2019.

\bibitem[Tirinzoni et~al.(2020)Tirinzoni, Pirotta, Restelli, and
  Lazaric]{tirinzoni2020asymptotically}
Andrea Tirinzoni, Matteo Pirotta, Marcello Restelli, and Alessandro Lazaric.
\newblock An asymptotically optimal primal-dual incremental algorithm for
  contextual linear bandits.
\newblock \emph{Advances in Neural Information Processing Systems},
  33:\penalty0 1417--1427, 2020.

\bibitem[Tirinzoni et~al.(2021)Tirinzoni, Pirotta, and
  Lazaric]{tirinzoni2021fully}
Andrea Tirinzoni, Matteo Pirotta, and Alessandro Lazaric.
\newblock A fully problem-dependent regret lower bound for finite-horizon mdps.
\newblock \emph{arXiv preprint arXiv:2106.13013}, 2021.

\bibitem[Tirinzoni et~al.(2022)Tirinzoni, Al-Marjani, and
  Kaufmann]{tirinzoni2022near}
Andrea Tirinzoni, Aymen Al-Marjani, and Emilie Kaufmann.
\newblock Near instance-optimal pac reinforcement learning for deterministic
  mdps.
\newblock \emph{arXiv preprint arXiv:2203.09251}, 2022.

\bibitem[Van~Parys and Golrezaei(2020)]{van2020optimal}
Bart~PG Van~Parys and Negin Golrezaei.
\newblock Optimal learning for structured bandits.
\newblock \emph{arXiv preprint arXiv:2007.07302}, 2020.

\bibitem[Wagenmaker and Jamieson(2022)]{wagenmaker2022instance}
Andrew Wagenmaker and Kevin Jamieson.
\newblock Instance-dependent near-optimal policy identification in linear mdps
  via online experiment design.
\newblock \emph{arXiv preprint arXiv:2207.02575}, 2022.

\bibitem[Wagenmaker and Pacchiano(2022)]{wagenmaker2022leveraging}
Andrew Wagenmaker and Aldo Pacchiano.
\newblock Leveraging offline data in online reinforcement learning.
\newblock \emph{arXiv preprint arXiv:2211.04974}, 2022.

\bibitem[Wagenmaker et~al.(2021)Wagenmaker, Simchowitz, and
  Jamieson]{wagenmaker2021task}
Andrew~J Wagenmaker, Max Simchowitz, and Kevin Jamieson.
\newblock Task-optimal exploration in linear dynamical systems.
\newblock In \emph{International Conference on Machine Learning}, pages
  10641--10652. PMLR, 2021.

\bibitem[Wagenmaker et~al.(2022{\natexlab{a}})Wagenmaker, Chen, Simchowitz, Du,
  and Jamieson]{wagenmaker2022reward}
Andrew~J Wagenmaker, Yifang Chen, Max Simchowitz, Simon Du, and Kevin Jamieson.
\newblock Reward-free rl is no harder than reward-aware rl in linear markov
  decision processes.
\newblock In \emph{International Conference on Machine Learning}, pages
  22430--22456. PMLR, 2022{\natexlab{a}}.

\bibitem[Wagenmaker et~al.(2022{\natexlab{b}})Wagenmaker, Simchowitz, and
  Jamieson]{wagenmaker2022beyond}
Andrew~J Wagenmaker, Max Simchowitz, and Kevin Jamieson.
\newblock Beyond no regret: Instance-dependent pac reinforcement learning.
\newblock In \emph{Conference on Learning Theory}, pages 358--418. PMLR,
  2022{\natexlab{b}}.

\bibitem[Wang et~al.(2020)Wang, Salakhutdinov, and Yang]{wang2020provably}
Ruosong Wang, Russ~R Salakhutdinov, and Lin Yang.
\newblock Reinforcement learning with general value function approximation:
  Provably efficient approach via bounded eluder dimension.
\newblock \emph{Advances in Neural Information Processing Systems}, 33, 2020.

\bibitem[Wei and Luo(2018)]{wei2018more}
Chen-Yu Wei and Haipeng Luo.
\newblock More adaptive algorithms for adversarial bandits.
\newblock In \emph{Conference On Learning Theory}, pages 1263--1291. PMLR,
  2018.

\bibitem[Yang and Barron(1998)]{yang1998asymptotic}
Yuhong Yang and Andrew~R Barron.
\newblock An asymptotic property of model selection criteria.
\newblock \emph{IEEE Transactions on Information Theory}, 44\penalty0
  (1):\penalty0 95--116, 1998.

\bibitem[Zanette and Brunskill(2019)]{zanette2019tighter}
Andrea Zanette and Emma Brunskill.
\newblock Tighter problem-dependent regret bounds in reinforcement learning
  without domain knowledge using value function bounds.
\newblock In \emph{International Conference on Machine Learning}, pages
  7304--7312. PMLR, 2019.

\bibitem[Zhang et~al.(2021)Zhang, Ji, and Du]{zhang2021reinforcement}
Zihan Zhang, Xiangyang Ji, and Simon Du.
\newblock Is reinforcement learning more difficult than bandits? a near-optimal
  algorithm escaping the curse of horizon.
\newblock In \emph{Conference on Learning Theory}, pages 4528--4531. PMLR,
  2021.

\bibitem[Zhou et~al.(2021)Zhou, Gu, and Szepesvari]{zhou2021nearly}
Dongruo Zhou, Quanquan Gu, and Csaba Szepesvari.
\newblock Nearly minimax optimal reinforcement learning for linear mixture
  markov decision processes.
\newblock In \emph{Conference on Learning Theory}, pages 4532--4576. PMLR,
  2021.

\end{thebibliography}

\clearpage

\appendix

\section{Additional Notation}

\newcommand{\specialcell}[2][l]{\begin{tabular}[#1]{@{}l@{}}#2\end{tabular}}

\begin{center}
\begin{tabular}{ |c | l |  }
\hline
\textbf{Mathematical Notation} & \textbf{Definition} \\ 
\hline
$\Dkl{\cdot}{\cdot}$ & KL divergence \\
$\Dhel{\cdot}{\cdot}$ & Hellinger distance \\
$\Dtv{\cdot}{\cdot}$ & Total variation distance \\
$D(\cdot \parallel \cdot)$ & General divergence \\
$\simplex_\cX$ & Set of probability distributions over $\cX$ \\
\hline
\textbf{DMSO Notation} &  \\
\hline
$M$ & Model \\
$\Mst$ & Ground truth model \\
$\cM$ & Set of models \\
$\cMsub$ & Arbitrary subset of $\cM$ \\
$\pi, \Pi$ & Decision $\pi$, set of all decisions $\Pi$ \\
$r, \cR$ & Reward $r$, set of all rewards $\cR$ \\
$o, \cO$ & Observation $o$, set of all observations $\cO$ \\
$\Empi[M]\brk*{\cdot}, \Prm{M}{\pi}\brk*{\cdot}$ & Expectation and distribution of $(r,o) \sim M(\pi)$ \\
$\Exp\sups{M}\brk*{\cdot}, \bbP\sups{M}\brk*{\cdot}$ & Expectation and distribution induced over histories on $M$ \\
$\fm(\pi)$ & Expected reward playing $\pi$ on $M$, $\fm(\pi) = \Empi[M]\brk*{r}$ \\
$\pim$ & Optimal decision of model $M$, $\pim \in \argmax_{\pi \in \Pi} \fm(\pi)$ \\
$\pibm$ & Set of optimal decisions of model $M$ \\
$\delm(\pi)$ & Gap of decision $\pi$ on model $M$, $\delm(\pi) = \fm(\pim) - \fm(\pi)$ \\
$\delminm$ & Minimum gap on model $M$ (see \eqref{eq:mingap_def}) \\
$\cMall$ & Set of all possible models, $\cMall=\crl{M:\Pi\to\simplex_{\cR \times\cO}\mid{}\fm(\pi)\in\brk{0,1}}$ \\
$ \RegDM$ & Regret after $T$ rounds (see \eqref{eq:regret}) \\
$\LKL$ & Lipschitz constant of KL divergence (see \Cref{asm:smooth_kl_kl}) \\
$\VM$ & Sub-gaussian parameter of log-likehood ratio (see \Cref{asm:bounded_likelihood}) \\
$\Ncov(\cM,\rho,\mu)$ & $(\rho,\mu)$ covering number of $\cM$ (see \Cref{def:cover}) \\
$\dcov, \Ccov$ & Bounds on covering number (see \Cref{asm:covering}) \\
$\nmeps$ & \specialcell{Information content of optimal decision on $M$ with \\ \qquad tolerance $\veps$ (see \Cref{def:inf_content_opt})} \\
$\ncMeps$ & Maximum information content of optimal decision on $\cM$, $\ncMeps = \sup_{M \in \cM} \nmeps$ \\
$\cM_{x,y}$ & $\cM_{x,y} = \{ M \in \cM \ : \ \delminm \ge x, \nmepsb \le y \}$ \\
$\DelLst$ &  $\DelLst = \min \{ \delminst, 1/\nsf^\star_{\veps/36} \}$ \\
$\cMLdelst$ & Restriction of $\cM$ induced by $\Mst$ (see \eqref{eq:cMst_def}) \\
\hline 
\textbf{Graves-Lai Notation} &  \\
\hline
$ \glc(\cM,M)$ & Graves-Lai Coefficient for class $\cM$, model $M$ (see \eqref{eq:glc}) \\
$\gm, \gst$ & Graves-Lai Coefficient for model $M$, $\Mst$; $\gm = \glc(\cM,M)$, $\gst = \glc(\cM,\Mst)$ \\ 
$\guncM$ & Minimum non-zero Graves-Lai Coefficient on $\cM$, $\guncM := \min_{M \in \cM : \gm > 0} \gm$ \\
$\cMalt(M)$ & Alternate set for model $M$, $\cMalt(M) = \{ M' \in \cM \mid \pibm \cap \pibm[M'] = \emptyset \}$ \\
$\cMaltst$ & Alternate set for $\Mst$, $\cMaltst = \cMalt(\Mst)$ \\
$\eta$ & Allocation, $\eta \in \bbR_+^\Pi$ \\
$\lambda$ & Normalized allocation, $\lambda \in \simplex_\Pi$ \\
$\omega$ & Exploration distribution, $\omega \in \simplex_\Pi$ \\
$\Lambda(M;\veps)$ & Set of $\veps$-optimal normalized Graves-Lai allocations for model $M$ (see \eqref{eq:lambda_allocation_set}) \\
$\Lambda(M;\veps,\nmax)$ & \specialcell{Set of $\veps$-optimal normalized Graves-Lai allocations for model $M$ with \\ \qquad normalization factor at most $\nmax$ (see \eqref{eq:lambda_set_nmax})} \\
$\cMgl(\lambda)$ & Models for which $\lambda$ is an $\veps$-optimal Graves-Lai allocation (see \eqref{lambda-set}) \\
$\cMgl[\veps](\lam; \nmax)$ & \specialcell{Models for which $\lambda$ is an $\veps$-optimal Graves-Lai allocation with \\ \qquad normalization factor at most $\nmax$ (see \eqref{eq:cMgl_nmax_def})} \\
$\Im(\eta;\cM)$ & Information content of $\eta$ on $M$ with respect to $\cM$ (see \eqref{eq:information}) \\
$\Im(\eta)$ & $\Im(\eta) = \Im(\eta;\cM)$ \\
$\Tgl(\cMsub;\veps,\nmax,\delta)$ & Minimum time to learn Graves-Lai allocation over $\cMsub$ (see \eqref{eq:tgl_basic}) \\
\hline
\end{tabular}
\end{center}

\begin{center}
\begin{tabular}{ |c | l |  }
\hline
\textbf{\CompShort Notation} & \\
\hline
$\aecM{\veps}{\cM}(\cMsub,\Mbar)$ & \CompShort with tolerance $\veps$, model set $\cMsub$, reference model $\Mbar$ (see \eqref{eq:comp}) \\
$\aec(\cM,\Mbar)$ & $\aec(\cM,\Mbar) = \aecM{\veps}{\cM}(\cM,\Mbar)$ \\
$\aec(\cM)$ & $\aec(\cM) = \sup_{\Mbar \in \conv(\cM)} \aec(\cM,\Mbar)$ \\
$\aecflipM{\veps}{\cM}(\cMsub,\xi)$ & \CompShort with randomized estimator $\xi$ (see \eqref{eq:aecflip_def}) \\
$\aecflip(\cM,\xi)$ & $\aecflip(\cM,\xi) = \aecflipM{\veps}{\cM}(\cM,\xi)$ \\
$\aecflip(\cM)$ & $\aecflip(\cM) = \sup_{\xi \in \simplex_\cM} \aecflip(\cM,\xi)$ \\
$\aecD(\cM,\Mbar)$ & \CompShort defined with respect to general divergence (see \eqref{eq:aecD_def}) \\
$\aecflipD(\cM,\xi)$ & \CompShort with randomized estimator, general divergence (see \eqref{eq:aecflipD_def}) \\
\hline
\textbf{Uniform Exploration} &  \\ 
\textbf{Notation} & \\
\hline
$\Cexpxi(\veps)$ & Uniform exploration coefficient with respect to $\xi$ at scale $\veps$ (see \Cref{def:uniform_exp}) \\
$\pexpxi(\veps)$ & Uniform exploration distribution with respect to $\xi$ at scale $\veps$ \\
$\Cexp(\cM, \veps)$ & Uniform exploration coefficient for class $\cM$ at scale $\veps$ \\
$\Cexp^{\mathsf{D},\xi}(\veps)$ & Uniform exploration coefficient, general divergence (see \Cref{def:uniform_exp_general_D}) \\
$\pexp^{\mathsf{D},\xi}(\veps)$ & Uniform exploration distribution, general divergence \\
$\Cexp^{\mathsf{D}}(\cM,\veps)$ &  Uniform exploration coefficient for class $\cM$, general divergence \\
\hline
\textbf{Estimation Notation} & \\
\hline
$\AlgEstKL$ & Estimation oracle \\
$\EstKL(s)$ & Cumulative KL estimation error (see \Cref{def:est_error}) \\
$\EstD(s)$ & Cumulative estimation error with respect to divergence $D$ (see \eqref{eq:EstD_def}) \\
$\EstDbar(s)$ & Cumulative estimation error with arguments flipped (see \eqref{eq:EstDbar_def}) \\
\hline
\end{tabular}
\end{center}

  \paragraph{Divergences}
    For probability distributions $\bbP$ and $\bbQ$ over a measurable space
  $(\Omega,\filt)$ with a common dominating measure, we define the total variation distance as 
  \[
    \Dtv{\bbP}{\bbQ}=\sup_{A\in\filt}\abs{\bbP(A)-\bbQ(A)}
    = \frac{1}{2}\int\abs{d\bbP-d\bbQ}
  \]
  and (squared) Hellinger distance as
  \[
    \Dhels{\bbP}{\bbQ}=\int\prn*{\sqrt{d\bbP}-\sqrt{d\bbQ}}^{2}.
\]

\paragraph{Interactive decision making}
We formalize probability spaces in the same fashion as
\citet{foster2021statistical,foster2022complexity}. Decisions are
associated with a measurable space $(\Act,\Asig)$, rewards are
associated with the space $(\Rspace,\Rsig)$, and observations are
associated with the space $(\Obs,\Osig)$.  The history after round $t$ is denoted by  $\hist\ind{t}
=(\act\ind{1},r\ind{1},\obs\ind{1}),\ldots,(\act\ind{t},r\ind{t},\obs\ind{t})$. We define
\[
\Hspace\ind{t}=\prod_{i=1}^{t}(\Act\times\Rspace\times\Obs),\mathand\Hsig\ind{t}=\bigotimes_{i=1}^{t}(\Asig\otimes{}\Rsig\otimes\Osig)
\]
so that $\hist\ind{t}$ is associated with the space
$(\Hspace\ind{t},\Hsig\ind{t})$.

When the algorithm is clear from context, we let $\bbP\sups{M}$ denote
the law it induces on $\hist\ind{T}$ when $M:\Pi\to\simplex_{\cR\times\cO}$ is the underlying model,
and let $\En\sups{M}\brk{\cdot}$ denote the corresponding expectation.
We will also overload notation somewhat and let $\bbP\sups{M,\pi}$ the
density of $(r,o)\sim{}M(\pi)$.

\paragraph{Notation for complexity measures and allocations}
We let $T(\pi)$ denote the number of times decision $\pi$ is taken up
to time $T$.  For $\eta\in\Rplus^{\Pi}$, we define
\begin{align}
  \label{eq:information}
  \Im(\eta\midsem\cM) = \inf_{M'\in\cMalt(M)}\sum_{\pi\in\Pi}\eta(\pi)\Dkl{M(\pi)}{M'(\pi)},
\end{align}
so that we can write $\glc(\cM,\Mstar) = \inf_{\eta\in\bbR_{+}^{\Pi}}\crl*{
  \sum_{\pi\in\Pi}\eta(\pi)\delmstar(\pi)
  \mid{} \Imstar(\eta;\cM) \geq{} 1
  }$. We abbreviate $\Im(\eta) = \Im(\eta; \cM)$ whenever the class $\cM$ is clear from context. We will occasionally overload notation and write
  $\delm(\eta)=\sum_{\pi\in\Pi}\eta(\pi)\delm(\pi)$ for $\eta\in\Aspace$. We also let $\cMaltst :=\cMalt(\Mst)$ and
  \begin{align}\label{eq:cMgl_nmax_def}
  \cMgl[\veps](\lam; \nmax) := \{ M \in \cM \ : \ \lambda \in \Lambda(M;\veps,\nmax) \}.
  \end{align}
  
Recall the definition
\begin{align*}
\Lambda(M;\veps) = \crl*{ \lambda \in \simplex_\Pi \mid \exists \nsf \in \bbR_+ \text{ s.t. } \En_{\pi\sim{}\lam}\brk*{\delm(\pi)} \leq{}
    \frac{(1+\veps)\gm}{\nsf}, \inf_{M'\in\cMalt(M)}\En_{\pi\sim{}\lam}\brk*{\Dkl{M(\pi)}{M'(\pi)}} \geq{} \frac{1-\veps}{\nsf}}.
\end{align*}
For a given $\lambda \in \Lambda(M;\veps)$, we refer to the $\nsf \in \bbR_+$ which realizes
\begin{align*}
\En_{\pi\sim{}\lam}\brk*{\delm(\pi)} \leq{}
    \frac{(1+\veps)\gm}{\nsf} \quad \text{and} \quad \inf_{M'\in\cMalt(M)}\En_{\pi\sim{}\lam}\brk*{\Dkl{M(\pi)}{M'(\pi)}} \geq{} \frac{1-\veps}{\nsf}
\end{align*}
as the \emph{normalization factor} of $\lambda$.

  \paragraph{General divergences}
  While in the main text we have focused on results that hold for the KL divergence, throughout the appendix we will consider more general divergences $\D{\cdot}{\cdot}$. 
 In particular, rather than
fixing the divergence in \eqref{eq:alg_alloc_comp} to the KL
divergence, we utilize divergence $D$. We also perform online
estimation with respect to $D$ rather than with respect to the KL divergence. While $D$ could be an arbitrary non-negative function, we make the following assumptions on it.

First, we replace \Cref{asm:smooth_kl_kl} with the
following more general assumption.
\begin{assumption}\label{asm:smooth_kl}
For all $M,M',M'' \in \cM$, and $\pi \in \Pi$, there exists some $\LKL$ such that
\begin{align*}
\left |\kl{M(\pi)}{M''(\pi)} - \kl{M'(\pi)}{M''(\pi)}] \right | \le \LKL \sqrt{\D{M(\pi)}{M'(\pi)}}.
\end{align*}
\end{assumption}
Note that, by Jensen's inequality, \Cref{asm:smooth_kl} immediately implies that, for $\xi \in \simplex(\cM)$,
\begin{align*}
\left |\kl{M(\pi)}{M''(\pi)} - \Exp_{\Mbar \sim \xi}[\kl{\Mbar(\pi)}{M''(\pi)}] \right | \le \LKL \sqrt{\Exp_{\Mbar \sim \xi}[\D{\Mbar(\pi)}{M(\pi)}]}.
\end{align*}

We in addition make the following assumption, which we note is met for the KL divergence. 
\begin{assumption}\label{asm:D_to_hel}
For all $M,M' \in \conv(\cM)$, and $\pi$, we have 
\begin{align*}
\Dhels{M(\pi)}{M'(\pi)} \le \D{M(\pi)}{M'(\pi)} .
\end{align*}
Furthermore, $\D{\cdot}{\cdot}$ is convex in its second argument.
\end{assumption}
A direct consequence of this assumption is that, when rewards are
observed and bounded in $\brk*{0,1}$, we have
\begin{align*}
|\fm(\pi) - \fmp(\pi)| \le \sqrt{\D{M(\pi)}{M'(\pi)}}.
\end{align*}
Throughout the appendix, we assume that \Cref{asm:smooth_kl} and \Cref{asm:D_to_hel} hold for our divergence $D$.

We also generalize the definition of the \CompText to account for general divergences as
\begin{align}\label{eq:aecD_def}
\aecD(\cM,\Mbar) := \inf_{\lambda,\omega \in \simplex_\Pi} \sup_{M
  \in \cM \backslash \cMgl(\lambda)} \frac{1}{\Exp_{\pi \sim
  \omega}[\D{\Mbar(\pi)}{M(\pi)}]}
\end{align}
and
\begin{align}\label{eq:aecflipD_def}
\aecflipD(\cM;\xi) := \inf_{\lambda,\omega \in \simplex_\Pi} \sup_{M
  \in \cM \backslash \cMgl(\lambda)} \frac{1}{\Exp_{\Mbar \sim \xi}[\Exp_{\pi \sim
  \omega}[\D{\Mbar(\pi)}{M(\pi)}]]}.
\end{align}
Other variants of the \CompShort are generalized similarly with a superscript $\mathsf{D}$.
Our guarantees will also depend on a notion of estimation error for
general divergences, given by
\begin{align}\label{eq:EstD_def}
\EstD(s) := \sum_{i=1}^s \Exp_{\Mhat \sim \xi^i} [ \Exp_{\pi \sim p^i} [ \D{\Mst(\pi)}{\Mhat(\pi)} ] ] .
\end{align}
It will also be convenient to work with the following notion of estimation error:
\begin{align}\label{eq:EstDbar_def}
\EstDbar(s) := \sum_{i=1}^s \Exp_{\Mhat \sim \xi^i}[\Exp_{\pi \sim p^i} [ \D{\Mhat(\pi)}{\Mst(\pi)} ] ].
\end{align}


\section{Technical Tools}


\subsection{Online Learning}

In this section, we state online estimation guarantees for variants
of the Tempered Aggregation algorithm of \cite{chen2022unified}. Throughout the section we abbreviate $\Exp_{t-1}[\cdot] = \Exp[\cdot
\mid \cH^{t-1}, p^t, \xi^t]$. We recall that $\Prm{M}{\pi}(r,o)$
denotes the density over rewards and observations $(r,o)$ under $M(\pi)$

\begin{algorithm}[h]
\caption{Tempered Aggregation}
\begin{algorithmic}[1]
\State \textbf{input:} Finite class $\cM$.
\State Initialize $\xi^1 \leftarrow \unif(\cM)$.
\For{$t=1,2,3,\ldots$}
	\State Receive $(\pi^t,r^t,o^t)$.
	\State Update estimator:
	\begin{align*}
	\xi^{t+1}(M) \propto \xi^t(M) \cdot \exp \left ( \frac{1}{2} \log \Prm{M}{\pi^t}(r^t, o^t) \right )\quad\forall{}M\in\cM.
	\end{align*}
\EndFor
\end{algorithmic}
\label{alg:tempered_aggregation}
\end{algorithm}

\begin{proposition}[Tempered Aggregation, Finite Class Setting]\label{prop:temp_agg_finite}
Assume $|\cM| \le \infty$ and $\Mstar\in\cM$. Then \Cref{alg:tempered_aggregation} produces estimates $(\xi^t)_{t=1}^T$ which satisfy, with probability at least $1-\delta$,
\begin{align*}
\sum_{t=1}^T \Exp_{M \sim \xi^t} [ \Exp_{\pi \sim p^t} [ \Dhels{\Mst(\pi)}{M(\pi)} ] ] \le 2 \log \frac{|\cM|}{\delta} .
\end{align*}
\end{proposition}
\begin{proof}[Proof of \Cref{prop:temp_agg_finite}]
We follow closely the proof of Theorem C.1 of \cite{chen2022unified}. Define the random variable
\begin{align*}
A^t := - \log \Exp_{M \sim \xi^t} \brk*{\exp \prn*{\beta \log \frac{\Prm{M}{\pi^t}(r^t, o^t)}{\Prm{\Mst}{\pi^t}(r^t, o^t)}}}.
\end{align*}
We have
\begin{align*}
\Exp_{t-1}[\exp(-A^t)] & = \Exp_{t-1} \brk*{\Exp_{M \sim \xi^t}\brk*{\exp \prn*{\frac{1}{2} \log \frac{\Prm{M}{\pi^t}(r^t, o^t)}{\Prm{\Mst}{\pi^t}(r^t, o^t)}}}} \\
& = \sum_{M \in \cM} \xi^t(M) \Exp_{t-1} \brk*{\exp \prn*{\frac{1}{2} \log \frac{\Prm{M}{\pi^t}(r^t, o^t)}{\Prm{\Mst}{\pi^t}(r^t, o^t)}}} \\
& = \sum_{M \in \cM} \xi^t(M) \Exp_{t-1} \brk*{\Exp_{o \sim \Mst(\pi^t)} \brk*{\sqrt{ \frac{\Prm{M}{\pi^t}(r^t, o^t)}{\Prm{\Mst}{\pi^t}(r^t, o^t)}}}} \\
& = \sum_{M \in \cM} \xi^t(M) \cdot \prn*{1 - \frac{1}{2} \Exp_{\pi \sim p^t}[\Dhels{\Mst(\pi)}{M(\pi)}]}
\end{align*}
where the last equality holds by the definition of the Hellinger distance. This implies that
\begin{align*}
1 - \Exp_{t-1}[\exp(-A^t)]  = \frac{1}{2} \Exp_{M \sim \xi^t}\brk*{\Exp_{\pi \sim p^t}[\Dhels{\Mst(\pi)}{M(\pi)}]}.
\end{align*}
By Lemma A.4 of \cite{foster2021statistical}, we have that with probability at least $1-\delta$, 
\begin{align*}
\sum_{t=1}^T A^t + \log \frac{1}{\delta} & \ge \sum_{t=1}^T -\log \Exp_{t-1}[\exp(-A^t)] \\
& \ge \sum_{t=1}^T \prn*{1 - \Exp_{t-1}[\exp(-A^t)]} \\
& = \frac{1}{2} \sum_{t=1}^T\Exp_{M \sim \xi^t}\brk*{\Exp_{\pi \sim p^t}[\Dhels{\Mst(\pi)}{M(\pi)}]}.
\end{align*}
We turn to upper bounding $\sum_{t=1}^T A^t$. Following the proof of Theorem C.1 of \cite{chen2022unified}, we have
\begin{align*}
\sum_{t=1}^T A^t & = -\log \prn*{\sum_{M \in \cM} \xi^1(M) \exp \prn*{\sum_{t=1}^T \frac{1}{2} \log \frac{\Prm{M}{\pi^t}(r^t, o^t)}{\Prm{\Mst}{\pi^t}(r^t, o^t)}}} .
\end{align*}
Since $\Mst \in \cM$, we can then bound
\begin{align*}
\sum_{t=1}^T A^t  & \le -\log \prn*{\xi^1(\Mst) \exp \prn*{\sum_{t=1}^T \frac{1}{2} \log \frac{\Prm{\Mst}{\pi^t}(r^t, o^t)}{\Prm{\Mst}{\pi^t}(r^t, o^t)}}}  = \log |\cM|.
\end{align*}
Combining these expressions gives the result.
\end{proof}

\begin{proposition}[Tempered Aggregation, Infinite Class Setting]\label{prop:temp_agg_infinite}
Let $\cMcov$ denote a $(\rho,\mu)$-cover of $\cM$ with covering number
$\Ncov(\cM,\rho, \mu)$. If we apply \Cref{alg:tempered_aggregation} to
$\cMcov$, we have that whenever $\Mstar\in\cM$, with probability at least $1-\delta - T \mu$,
\begin{align*}
\sum_{t=1}^T \Exp_{M \sim \xi^t} [ \Exp_{\pi \sim p^t} [ \Dhels{\Mst(\pi)}{M(\pi)} ] ] \le 2 \log \frac{\Ncov(\cM,\rho, \mu)}{\delta} + T\rho.
\end{align*}
\end{proposition}
\begin{proof}[Proof of \Cref{prop:temp_agg_infinite}]
Defining $A^t$ as in \Cref{prop:temp_agg_finite}, the first part of the proof is identical to that of \Cref{prop:temp_agg_finite}, and we conclude that, with probability at least $1-\delta$,
\begin{align*}
\sum_{t=1}^T A^t + \log \frac{1}{\delta} \ge \frac{1}{2} \sum_{t=1}^T \Exp_{M \sim \xi^t} [ \Exp_{\pi \sim p^t} [ \Dhels{\Mst(\pi)}{M(\pi)} ] ] 
\end{align*}
and 
\begin{align*}
\sum_{t=1}^T A^t & = -\log \prn*{\sum_{M \in \cMcov} \xi^1(M) \exp \prn*{\sum_{t=1}^T \frac{1}{2} \log \frac{\Prm{M}{\pi^t}(r^t, o^t)}{\Prm{\Mst}{\pi^t}(r^t, o^t)}}} .
\end{align*}
Let $\cE$ denote the event associated with $\cMcov$, as defined in \Cref{def:cover}, which satisfies $\sup_{M' \in \cM} \sup_{\pi} \Prmb[M'](\cE \mid \pi) \le \mu$.
Let $\Mtil \in \cMcov$ denote the element in the cover which satisfies
\begin{align*}
\left | \log \frac{\Prm{\Mst}{\pi}(r,o)}{\Prm{\Mtil}{\pi}(r,o)} \right | = \left | \log \Prm{\Mst}{\pi}(r,o) - \log \Prm{\Mtil}{\pi}(r,o) \right | \le \rho,
\end{align*}
for all $(r,o,\pi)$ with $\sup_{M' \in \cM} \Prm{M'}{\pi}(r,o \mid \cE) > 0$.
We then have
\begin{align*}
\sum_{t=1}^T \log \frac{\Prm{M}{\pi^t}(r^t, o^t)}{\Prm{\Mst}{\pi^t}(r^t, o^t)} & = \sum_{t=1}^T \prn*{\log \frac{\Prm{M}{\pi^t}(r^t, o^t)}{\Prm{\Mtil}{\pi^t}(r^t, o^t)} + \log \frac{\Prm{\Mtil}{\pi^t}(r^t, o^t)}{\Prm{\Mst}{\pi^t}(r^t, o^t)}} 
\end{align*}
so we can bound
\begin{align*}
\sum_{t=1}^T A^t \le \log |\cMcov| + \frac{1}{2} \sum_{t=1}^T \log \frac{\Prm{\Mst}{\pi^t}(r^t, o^t)}{\Prm{\Mtil}{\pi^t}(r^t, o^t)}
\end{align*}
which gives that with probability at least $1-\delta$,
\begin{align*}
\sum_{t=1}^T \Exp_{M \sim \xi^t} [ \Exp_{\pi \sim p^t} [ \Dhels{\Mst(\pi)}{M(\pi)} ] ]  \le 2 \log |\cMcov| + 2 \log \frac{1}{\delta} + \sum_{t=1}^T \log \frac{\Prm{\Mst}{\pi^t}(r^t, o^t)}{\Prm{\Mtil}{\pi^t}(r^t, o^t)}.
\end{align*}
Let $\cE_t$ denote the event that $\cE$ occurs at step $t$. Denote the event $\cE_1 := \cap_{t=1}^T \cE_t$ and
\begin{align*}
\cE_2 := \left \{  \sum_{t=1}^T \Exp_{M \sim \xi^t} [ \Exp_{\pi \sim p^t} [ \Dhels{\Mst(\pi)}{M(\pi)} ] ]  \le 2 \log |\cMcov| + 2 \log \frac{1}{\delta} + T \rho \right \} .
\end{align*}
Then 
\begin{align*}
\Pr\sups{\Mst}[\cE_2] & \le \Pr\sups{\Mst}[\cE_2 \cap \cE_1] + \Pr\sups{\Mst}[\cE_1^c].
\end{align*}
By definition of $\cE_t$ and a union bound we have $\Pr\sups{\Mst}[\cE_1^c] \le T \mu$. Furthermore, on the event $\cE_1$ we can bound, for each $t \le T$,
\begin{align*}
\log \frac{\Prm{\Mst}{\pi^t}(r^t, o^t)}{\Prm{\Mtil}{\pi^t}(r^t, o^t)} \le \rho.
\end{align*}
Thus, it follows that on $\cE_1$, $\sum_{t=1}^T \log \frac{\Prm{\Mst}{\pi^t}(r^t, o^t)}{\Prm{\Mtil}{\pi^t}(r^t, o^t)} \le T \rho$, which implies that $\Pr\sups{\Mst}[\cE_2 \cap \cE_1] \le \delta$. The result follows.

\end{proof}

\subsection{Properties of Graves-Lai Program}

In this section, we establish some basic properties of the Graves-Lai
coefficient $\gm\equiv{}\glc(\cM,M)$. Through out the section, we omit dependence on the class
$\cM$ for various quantities of interest whenever it is clear from
context. Throughout, we will use the fact that whenever
  $\delminm>0$, $\pim$ is unique.

\subsubsection{Basic Properties of Graves-Lai Program}

\begin{lemma}\label{lem:gl_program_lb}
  For any $M \in \cM$ and $\nsf > 0$, we have
\begin{align*}
\frac{\gm}{\nsf} \le \inf_{\lambda \in \simplex_{\Pi}}\crl*{ \delm(\lambda) : \Im(\lambda) \ge \frac{1}{\nsf}}.
\end{align*}
\end{lemma}
\begin{proof}[\pfref{lem:gl_program_lb}]
Assume the contrary. Then there exists some $\lamtil$ such that
\begin{align*}
\delm(\lamtil) < \gm/\nsf \quad \text{and} \quad \Im(\lamtil) \ge 1/\nsf.
\end{align*}
However, since both $\delm(\lambda)$ and $\Im(\lambda)$ are linear in
rescaling of $\lambda$, this implies that
\begin{align*}
\delm(\nsf \lamtil) < \gm \quad \text{and} \quad \Im(\nsf \lamtil) \ge 1.
\end{align*}
By definition we have
\begin{align*}
\gm = \inf_{\eta \in \bbR_+^\Pi} \delm(\eta) \quad \text{s.t.} \quad \Im(\eta) \ge 1.
\end{align*}
This is a contradiction, so the desired conclusion follows. 
\end{proof}

\begin{lemma}\label{lem:gl_const_Cexp_bound}
For any $M \in \cMp$ with $\delminm > 0$, we can bound
\begin{align*}
\gm \le \Cexp^{\xi}(\tfrac{1}{4} (\delminm)^2 )
\end{align*}
for $\xi = \bbI_{M}$.
\end{lemma}
\begin{proof}[Proof of \Cref{lem:gl_const_Cexp_bound}]
By definition we have
\begin{align*}
\gm & = \inf_{\eta \in \R_+^\Pi} \DelM(\eta) \quad \text{s.t.} \quad \inf_{M' \in \cMalt(M)} \sum_\pi \eta(\pi) \kl{M(\pi)}{M'(\pi)} \ge 1 \\
& \le \inf_{\eta \in \R_+^\Pi} \| \eta \|_1 \quad \text{s.t.} \quad \inf_{M' \in \cMalt(M)} \sum_\pi \eta(\pi) \Dkl{M(\pi)}{M'(\pi)} \ge 1.
\end{align*}
Let $\xi \in \simplex_\cM$ denote the distribution with $\xi(M) = 1$. Let
and let $\pexp := \pexp^\xi(\veps)$ denote the uniform exploration distribution defined with respect to $\xi$, and $\Cexp^{\xi}(\veps)$ the corresponding uniform exploration coefficient, for some $\veps$ to be chosen.

Consider some $M' \in \cMalt(M)$. Since $\Exp_{\Mbar \sim \xi}[\Exp_{\pi \sim
  p}[\Dkl{\Mbar(\pi)}{M''(\pi)}]] = \Exp_{\pi \sim
  p}[\Dkl{M(\pi)}{M''(\pi)}]$ for all $M''$ and $\Dkl{M(\pi)}{M(\pi)} = 0$ for all
$\pi$, it follows from the definition of the uniform exploration
coefficient that 
\begin{align*}
\Exp_{\pexp}[\Dkl{M(\pi)}{M'(\pi)}] \le 1/\Cexp^{\xi}(\veps) \implies \sup_{p \in \simplex_\Pi} \Exp_{p}[\Dkl{M(\pi)}{M'(\pi)}] \le \veps
\end{align*}
or, alternatively,
\begin{align*}
\sup_{p \in \simplex_\Pi} \Exp_{p}[\Dkl{M(\pi)}{M'(\pi)}] > \veps \implies \Exp_{\pexp}[\Dkl{M(\pi)}{M'(\pi)}] > 1/\Cexp^{\xi}(\veps).
\end{align*}
If $M' \in \cMalt(M)$, then it follows that $\pim \not\in \pibm[M']$. Take some $\pim[M'] \in \pibm[M']$. 
By definition we have $\fm[M](\pim) \ge \fm[M](\pim[M']) + \delminm$ and $\fm[M'](\pim[M']) \ge f\sups{M'}(\pim)$. Thus,
\begin{align*}
\delminm & \le f\sups{M}(\pim) -  f\sups{M}(\pim[M']) + f\sups{M'}(\pim[M']) - f\sups{M'}(\pim) \\
& \le | f\sups{M}(\pim) - f\sups{M'}(\pim)| + |f\sups{M'}(\pim[M']) - f\sups{M}(\pim[M'])| \\
& \le \sqrt{\D{M(\pim)}{M'(\pim)}} + \sqrt{\D{M(\pim[M'])}{M'(\pim[M'])}} .
\end{align*}
This implies that there exists some $\pi$ such that $\D{M(\pi)}{M'(\pi)} \ge (\delminm/2)^2$, so we can lower bound
\begin{align*}
\sup_{p \in \simplex_\Pi} \Exp_{p}[\Dkl{M(\pi)}{M'(\pi)}] \ge \frac{1}{4} (\delminm)^2.
\end{align*}
Thus, setting $\veps = \frac{1}{4} (\delminm)^2$, we have that 
\begin{align*}
\Exp_{\pexp}[\Dkl{M(\pi)}{M'(\pi)}] >  1/\Cexp^{\xi}(\tfrac{1}{4} (\delminm)^2 ).
\end{align*}
It follows that the allocation $\eta = \Cexp^{\xi}(\tfrac{1}{4} (\delminm)^2 ) \cdot \pexp$ realizes
\begin{align*}
\inf_{M' \in \cMalt(M)} \sum_\pi \eta(\pi) \Dkl{M(\pi)}{M'(\pi)} = \inf_{M' \in \cMalt(M)}  \Cexp^{\xi}(\tfrac{1}{4} (\delminm)^2 ) \cdot \Exp_{\pexp\sups{M}}[\Dkl{M(\pi)}{M'(\pi)}] \ge 1,
\end{align*}
which proves the result.
\end{proof}

\begin{lemma}\label{lem:gm_lb}
Assume $\gm > 0, \delminm > 0,$ and $\nsf\sups{M}_{1/4} < \infty$. Then it must be the case that
\begin{align*}
\gm \ge \delminm \cdot \frac{1}{\max_{M' \in \cM, \pi \in \Pi} \Dkl{M(\pi)}{M'(\pi)}}.
\end{align*}
\end{lemma}
\begin{proof}[Proof of \Cref{lem:gm_lb}]
Recall that
\begin{align*}
\gm & = \inf_{\eta \in \R_+^\Pi} \DelM(\eta) \quad \text{s.t.} \quad \inf_{M' \in \cMalt(M)} \sum_\pi \eta(\pi) \kl{M(\pi)}{M'(\pi)} \ge 1.
\end{align*}
By the definition of $\nsf\sups{M}_{1/4}$, for any allocation $\eta \in \bbR_+^\Pi$ satisfying $\inf_{M' \in \cMalt(M)} \sum_\pi \eta(\pi) \kl{M(\pi)}{M'(\pi)} \ge 3/4$, the allocation $\etatil$ defined as $\etatil(\pi) = \eta(\pi)$ for $\pi \neq \pim$, and $\etatil(\pim) = \nsf\sups{M}_{1/2}$ satisfies 
$$\inf_{M' \in \cMalt(M)} \sum_\pi \etatil(\pi) \kl{M(\pi)}{M'(\pi)} \ge 1/2,$$
and furthermore $\delm(\eta) = \delm(\etatil)$. It follows that
\begin{align*}
\gm & \ge \inf_{\eta \in \R_+^\Pi} \DelM(\eta) \quad \text{s.t.} \quad \inf_{M' \in \cMalt(M)} \sum_\pi \eta(\pi) \kl{M(\pi)}{M'(\pi)} \ge 3/4 \\
& \ge \inf_{\eta \in \R_+^\Pi} \DelM(\eta) \quad \text{s.t.} \quad \inf_{M' \in \cMalt(M)} \sum_\pi \eta(\pi) \kl{M(\pi)}{M'(\pi)} \ge 1/2, \ \eta(\pim) \le \nsf\sups{M}_{1/4}.
\end{align*}
If for all $M' \in \cMalt(M)$ we have $\Dkl{M(\pim)}{M'(\pim)} > 0$, this implies that we can distinguish $M$ from $M'$ by playing only $\pim$. Furthermore, by what we have just shown, this can be achieved by playing $\pim$ at most $2\nsf\sups{M}_{1/4}$ times. It follows that, if $\Dkl{M(\pim)}{M'(\pim)} > 0$ for all $M' \in \cMalt(M)$, then $\gm = 0$. Thus, if $\gm > 0$, there must exist some $M' \in \cMalt(M)$ such that $\Dkl{M(\pim)}{M'(\pim)} = 0$.

We then have
\begin{align*}
\gm & \ge \inf_{\eta \in \R_+^\Pi} \DelM(\eta) \quad \text{s.t.} \quad  \sum_{\pi \neq \pim} \eta(\pi) \kl{M(\pi)}{M'(\pi)} \ge 1 \\
& =  \inf_{\eta \in \R_+^\Pi, \eta(\pim) = 0} \DelM(\eta) \quad \text{s.t.} \quad  \sum_{\pi \neq \pim} \eta(\pi) \kl{M(\pi)}{M'(\pi)} \ge 1 \\
& \ge \inf_{\eta \in \R_+^\Pi, \eta(\pim) = 0} \delminm \cdot \| \eta \|_1 \quad \text{s.t.} \quad  \max_{\pi} \kl{M(\pi)}{M'(\pi)}  \cdot \| \eta \|_1 \ge 1 \\
& = \delminm \cdot \frac{1}{\max_{\pi} \kl{M(\pi)}{M'(\pi)} }.
\end{align*}
The result follows. 
\end{proof}

\subsubsection{Properties of the Information Content of Optimal Decisions}

\begin{lemma}\label{lem:nM_aec_bound}
Let $\veps \in [0,1/2)$ and $\nbar>0$ be given. We can bound, for any function $g(\omega,M)$ and $\cMsub \subseteq \cM$ with $\inf_{M \in \cMsub} \delminm > 0$,
\begin{align*}
\inf_{\omega,\lambda \in \simplex_\Pi} \sup_{M \in \cMsub \backslash \cMgl[2\veps](\lambda; \nbar)} g(\omega,M) \le \inf_{\omega,\lambda \in \simplex_\Pi} \sup_{M \in \cMsub \backslash \cMgl(\lambda)} g(\omega,M)
\end{align*}
as long as
\begin{align*}
\nbar \ge \max_{M \in \cMsub} \max \crl*{ \nmeps,  \frac{4 \gm}{\delminm}, \frac{2 \gm}{\zeta \delminm} },
\end{align*}
where
\begin{align*}
\zeta \ldef \min_{M \in \cMsub :  \gm > 0} \min \crl*{ \frac{\gm}{ \gm +2 \nmeps}, \frac{\delminm \veps}{4} }.
\end{align*}
\end{lemma}
\begin{proof}[Proof of \Cref{lem:nM_aec_bound}]
Note that each of these expressions only depend on $\lambda$ through $\cMgl[2\veps](\lambda; \nbar)$ and $\cMgl(\lambda)$, respectively. 
To prove the result, it therefore suffices to show that, for every $\lambda \in \simplex_\Pi$, there exists $\lambda' \in \simplex_\Pi$ such that $\cMsub \cap \cMgl(\lambda) \subseteq \cMsub \cap \cMgl[2\veps](\lambda';\nbar)$.

Fix $\lambda \in \simplex_\Pi$. 
Consider $M \in \cMsub \cap \cMgl(\lambda)$. By definition we have that there exists some $\nsf > 0$ such that
\begin{align*}
\delm(\lambda) \le (1+\veps) \gm / \nsf \quad \text{and} \quad \Im(\lambda) \ge (1-\veps)/\nsf,
\end{align*}
where here $\Im(\lambda) = \Im(\lambda; \cM)$.
We consider two cases. 

\paragraph{Case 1: $\lambda = \bbI_\pi$ for some $\pi \in \Pi$}
First, suppose that $\pi \neq \pim$. Then we have
\begin{align*}
\delminm \le \delm(\lambda) \le (1+\veps) \gm / \nsf \implies \nsf \le (1+\veps) \gm/\delminm.
\end{align*}
It follows that as long as $\nbar \ge (1+\veps) \gm/\delminm$, then $M \in \cMgl(\lambda;\nbar)$.

Now, suppose that $\pi = \pim$. In this case, by the definition of $\nmeps$, we immediately have that it suffices to take $\nsf = \nmeps$. Thus, for $\lambda = \bbI_\pi$, we have $\cMgl[2\veps](\lambda) = \cMgl(\lambda;\nbar)$ as long as
\begin{align*}
\nbar \ge \max \{ \nmeps,  (1+\veps) \gm/\delminm\}.
\end{align*}

\paragraph{Case 2a: $\lambda \neq \bbI_\pi$ for any $\pi \in \Pi$}
Fix some $\zeta \in (0,1/2)$ to be chosen. Suppose that there exists some $\pi'$ such that $\lambda(\pi') \ge 1 - \zeta$, and note that there can exist at most one such $\pi'$. Define $\lam'$ as
\begin{align*}
\lam'(\pi') = 1-\zeta, \quad \lam'(\pi) = \frac{\zeta}{1-\lam(\pi')} \lam(\pi), \quad\forall\pi \neq \pi' 
\end{align*}
and note that $\lam' \in \simplex_\Pi$. Our goal is to show that $\cMgl(\lambda) \subseteq \cMgl[2\veps](\lambda';\nbar)$.

Suppose that $\pim = \pi'$. Then
\begin{align*}
\delm(\lambda') = \frac{\zeta}{1-\lambda(\pi')} \cdot \delm(\lambda) \le \frac{\zeta}{1-\lambda(\pi')} \cdot \frac{(1+\veps) \gm}{\nsf}. 
\end{align*}
Denote $\nsf' := (\frac{\zeta}{1-\lambda(\pi')} \cdot \frac{1}{\nsf})^{-1}$. Since $\Im(\lambda) \ge (1-\veps) /\nsf$, we have $\Im( \nsf \lambda) \ge 1 - \veps$. Then, by the definition of $\nmeps$, we have
\begin{align*}
 \inf_{M' \in \cMalt(M)} \sum_{\pi \neq \pim}  \nsf \lam(\pi) \Dkl{M(\pi)}{M'(\pi)} + \nmeps \Dkl{M(\pim)}{M'(\pim)} \ge 1 - 2\veps.
\end{align*}
However, note that
\begin{align*}
\Im(\nsf'  \lam') & = \inf_{M' \in \cMalt(M)} \sum_{\pi \neq \pim}  \frac{\zeta \nsf'}{1-\lambda(\pim)} \lam(\pi) \Dkl{M(\pi)}{M'(\pi)} + \nsf' (1-\zeta) \Dkl{M(\pim)}{M'(\pim)} \\
& = \inf_{M' \in \cMalt(M)} \sum_{\pi \neq \pim} \nsf \lam(\pi) \Dkl{M(\pi)}{M'(\pi)} + \frac{(1-\zeta) (1-\lambda(\pi')) \nsf}{\zeta}  \cdot  \Dkl{M(\pim)}{M'(\pim)} \\
& \overset{(a)}{\ge} \inf_{M' \in \cMalt(M)} \sum_{\pi \neq \pim} \nsf \lam(\pi) \Dkl{M(\pi)}{M'(\pi)} +  \nmeps \Dkl{M(\pim)}{M'(\pim)} \\
& \overset{(b)}{\ge} 1 - 2\veps
\end{align*}
where $(a)$ follows as long as
\begin{align}\label{eq:nm_aec_zeta}
\frac{(1-\zeta) (1-\lambda(\pi')) \nsf}{\zeta}  \ge \nmeps,
\end{align}
and $(b)$ follows from what we have just shown. Rearranging, we have that \eqref{eq:nm_aec_zeta} is equivalent to
\begin{align*}
\frac{(1-\lambda(\pi')) \nsf}{(1-\lambda(\pi')) \nsf + \nmeps} \ge \zeta. 
\end{align*}
Note that by \Cref{lem:gl_program_lb} and the definition of $\lambda$, we have
\begin{align*}
\frac{(1-\veps) \gm}{\nsf} \le \inf_{\lamtil \in \simplex_{\Pi}}\crl*{ \delm(\lamtil) : \Im(\lamtil) \ge \frac{1-\veps}{\nsf}} \le \delm(\lambda),
\end{align*}
which implies that
\begin{align*}
(1-\veps) \gm \le \delm(\lambda) \cdot \nsf \le (1 - \lambda(\pim)) \cdot \nsf.
\end{align*}
As the function $\frac{x}{x + \nmeps}$ is increasing in $x$, a sufficient choice of $\zeta$ for this $M$ is then
\begin{align*}
\min \{ \frac{\gm}{ \gm + 2\nmeps}, 3/8 \} \ge \zeta
\end{align*}
Thus, for such a $\zeta$, we have that $M \in \cMgl(\lambda'; \nsf')$, which implies that $M \in \cMgl[2\veps](\lambda'; \nsf')$. Note that $(1-\lam(\pi')) \nsf \le (1+\veps) \gm/\delminm$ in this case, so we have that $M \in \cMgl[2\veps](\lambda'; \nbar)$ as long as
\begin{align*}
\nbar \ge \frac{2 \gm}{\zeta \delminm}.
\end{align*}

Consider now the case where $\pim \neq \pi'$. In this case, defining
$\lambda'$ as before, we can bound
\begin{align*}
\delm(\lam') \le \zeta + (1-\zeta) \delm(\pi') \le \zeta + \lam(\pi') \delm(\pi') \le \zeta + \delm(\lambda) \le \zeta + (1+\veps) \gm/\nsf. 
\end{align*}
Furthermore,
\begin{align*}
\Im(\lamGL') & = \inf_{M' \in \cMalt(M)} \sum_{\pi \neq \pi'} \frac{\zeta}{1-\lamGL(\pi')} \lamGL(\pi) \kl{M(\pi)}{M'(\pi)} + (1-\zeta)  \kl{M(\pi')}{M'(\pi')} \\
& \ge \inf_{M' \in \cMalt(M)} \sum_{\pi \neq \pi'} (1-\zeta) \lamGL(\pi) \kl{M(\pi)}{M'(\pi)} + (1-\zeta) \lamGL(\pi') \kl{M(\pi')}{M'(\pi')} \\
& = (1-\zeta) \Im(\lamGL) \\
& \ge (1-\zeta) (1-\veps)/\nsf. 
\end{align*}
Since $\pi' \neq \pim$, we can lower bound $\delm(\lamGL) \ge (1-\zeta) \delminm$, so
\begin{align*}
(1-\zeta) \delminm \le \delm(\lamGL) \le (1+\veps) \gm/\nsf \implies \nsf \le \frac{(1+\veps) \gm}{(1-\zeta) \delminm} \le \frac{4 \gm}{\delminm}
\end{align*}
We can therefore bound $\zeta \le \veps \gm/\nsf$
as long as
\begin{align*}
\zeta \le \frac{\delminm \cdot \veps}{4}.
\end{align*}
Consider $\zeta$ that satisfies this inequality. Then $\delm(\lamGL') \le (1+2\veps) \gm/\nsf$. We can also lower bound $(1-\zeta) (1-\veps) \ge (1-2\veps)$ as long as $\zeta \le \frac{\veps}{1 - \veps}\le \veps$. 
Thus, as long as
\begin{align*}
\zeta \le \min \{ 1, \frac{\delminm}{4} \} \cdot \veps \quad \text{and} \quad \nbar \ge \frac{4 \gm}{\delminm},
\end{align*}
we have that $M \in \cMgl[2\veps](\lam';\nbar)$. 

\paragraph{Case 2b: $\lambda \neq \bbI_\pi$ for any $\pi \in \Pi$}
Finally, it remains to handle the case then there does not exist
$\pi'$ such that $\lam(\pi') \ge 1-\zeta$. In this case, we can always lower bound
\begin{align*}
\zeta \delminm \le \delm(\lam) \le (1+\veps) \gm/\nsf \implies \nsf \le \frac{(1+\veps) \gm}{\zeta \delminm},
\end{align*}
so as long as $\nbar \ge \frac{(1+\veps) \gm}{\zeta \delminm}$, we have $M \in \cMgl[2\veps](\lam;\nbar)$.

Since we are in the regime where $\lambda \neq \bbI_\pi$, it must be the case that if $M \in \cMsub \cap \cMgl(\lambda)$, then $\gm > 0$. Thus, a sufficient choice of $\zeta$ is
\begin{align*}
\zeta = \min_{M \in \cMsub :  \gm > 0} \min \crl*{ \frac{\gm}{ \gm +2 \nmeps}, \frac{\delminm \veps}{4} }.
\end{align*}

\paragraph{Concluding the Proof}
To show the result, we need that $\nbar$ is large enough for each $M
\in \cMsub$. The argument above shows that it suffices to take
\begin{align*}
\nbar \ge \max_{M \in \cMsub} \max \crl*{ \nmeps,  \frac{4 \gm}{\delminm}, \frac{2 \gm}{\zeta \delminm} }
\end{align*}
for 
\begin{align*}
\zeta \ldef \min_{M \in \cMsub :  \gm > 0} \min \crl*{ \frac{\gm}{ \gm +2 \nmeps}, \frac{\delminm \veps}{4} }.
\end{align*}
This proves the result.
\end{proof}

\begin{lemma}\label{lem:nsf_bound}
For every $M \in \cM$ with $\delminm > 0$, there exists some $\lambda \in
\Lambda(M;\veps)$ with normalization factor $\nsf$ satisfying
\begin{align*}
\nsf \le  \gm/\delminm + \nmeps,
\end{align*}
i.e., we have
$M\in\cMgl(\lambda; \nsf)$.
\end{lemma}
\begin{proof}[Proof of \Cref{lem:nsf_bound}]
Consider some allocation $\eta \in \R_+^\Pi$ such that
\begin{align}\label{eq:nsf_upper_eq1}
\delm(\eta) \le  \gm \quad \text{and} \quad \Im(\eta) \ge 1.
\end{align}
Let $\eta'$ denote the allocation satisfying $\eta'(\pi) = \eta(\pi)$ for $\pi \neq \pim$, and $\eta'(\pim) = 0$. Note that $\delm(\eta) \ge \delminm \| \eta' \|_1$, which implies that
\begin{align*}
\| \eta' \|_1 \le  \gm / \delminm.
\end{align*}
Let $\eta''$ denote the allocation satisfying $\eta''(\pi) = \eta(\pi)$ for $\pi \neq \pim$ and $\eta''(\pim) = \nmeps$. Then by definition of $\nmeps$, and since $\eta$ satisfies \eqref{eq:nsf_upper_eq1}, we have $\Im(\eta'') \ge 1 - \veps$. Furthermore, it is straightforward to see that $\delm(\eta'') \le  \gm$. This implies that $\eta'' / \| \eta'' \|_1 \in \Lambda(M;\veps)$ with normalization factor $\| \eta'' \|_1$. However, we can bound
\begin{align*}
\| \eta'' \|_1 = \| \eta' \|_1 + \nmeps \le  \gm/\delminm + \nmeps.
\end{align*}
This proves the result.
\end{proof}

\subsubsection{Bounding \CompText via Uniform Exploration Coefficient}

In this section we prove a generalized version of \Cref{prop:aec_to_Cexp}.
In particular, rather than specializing to the KL divergence, we consider a general divergence $D$. We define the uniform exploration coefficient with respect to $D$ as follows.

\begin{definition}[Uniform Exploration Coefficient, General Divergences]\label{def:uniform_exp_general_D}
For a randomized estimator $\xi \in \simplex_\cM$ and divergence $\D{\cdot}{\cdot}$, we define the uniform exploration coefficient with respect to $\xi$ at scale $\veps>0$ as the value of the following program: 
\begin{align*}
\Cexp^{\mathsf{D},\xi}(\veps) := \min_{C \in \R_+,p \in \simplex_\Pi} \crl*{ C \ \Big | \ \forall M,M' \in \cM \ : \  \begin{matrix} \max_{M'' \in \{M,M' \}} \Exp_{\Mbar \sim \xi} [\Exp_{\pi \sim p} [ \D{\Mbar(\pi)}{M''(\pi)}]] \le 1/C \\
 \implies \max_{p' \in \simplex_\Pi} \Exp_{\pi \sim p'}[\D{M(\pi)}{M'(\pi)}] \le \veps \end{matrix}}.
\end{align*}
We define $\pexp^{\mathsf{D},\xi}(\veps)$ as the minimizing distribution for this program, and let
\begin{align*}
\CexpD(\cM, \veps) := \sup_{\xi \in \simplex_\cM} \Cexp^{\mathsf{D},\xi}(\veps)
\end{align*}
denote the uniform exploration constant for class $\cM$.
\end{definition}

\begin{lemma}[Formal version of \Cref{prop:aec_to_Cexp}]\label{lem:aec_Cexp_bound}
Let $\veps \in [0,1/2)$ and $\cMsub \subseteq \cM$ be given, and
assume that $\inf_{M \in \cMsub} \gm > 0, \inf_{M \in \cMsub} \delminm > 0$, $\sup_{M \in \cMsub} \nsf\sups{M}_{1/4} < \infty$, and
\Cref{asm:bounded_likelihood,asm:smooth_kl,asm:D_to_hel} hold. Then for any $\xi \in \simplex_{\cMsub}$, we can bound
\begin{align*}
\min_{\lamGL, \omega \in \simplex_{\Pi}} \sup_{M \in \cMsub \backslash \cMgl(\lambda)} \frac{1}{\Exp_{\Mbar \sim \xi}[\Exp_{\omega}[\D{\Mbar(\pi)}{M(\pi)}]]} \le  \CexpD(\cMsub,\delta)
\end{align*}
for any $\delta > 0$ satisfying
\begin{align*}
\sqrt{\delta} \le \min_{M \in \cMsub} \min \crl*{ \min \left \{ \frac{1}{81 \LKL}, \frac{\delminm}{34 \VM} \right \} \cdot \frac{\veps}{2 \gm/\delminm + \nmepsc{\veps/36}{M}}, \frac{\delminm}{3}}.
\end{align*}
\end{lemma}

\begin{proof}[\pfref{lem:aec_Cexp_bound}]
Let $\delta > 0$ be some tolerance to be chosen. Let $\Mtil$ denote
some $M \in \cMsub$ such that $\Exp_{\Mbar \sim
  \xi}[\Exp_{\pexp}[\D{\Mbar(\pi)}{M(\pi)}]] \le 1/\Cexp$, where we
abbreviate $\Cexp := \CexpD(\cMsub,\delta)$ and $\pexp := \pexp^{\mathsf{D},\xi}(\delta)$ is
a distribution that achieves the value of $\CexpD(\cMsub,\delta)$ for
$\xi$; if such an $\Mtil$ does not exist, we let $\Mtil = \argmin_{M \in \cMsub} \Exp_{\Mbar \sim \xi}[\Exp_{\pexp}[\D{\Mbar(\pi)}{M(\pi)}]]$. 
Let $\veps' > 0$ be some value to be chosen, and let $\lamtil \in \Lambda(\Mtil;\veps')$ denote the allocation in $\Lambda(\Mtil;\veps')$ with smallest normalizing factor $\nsf$. Let $\ntil$ denote the value of this normalizing factor, then:
\begin{align*}
\delm[\Mtil](\lamtil) \le (1+\veps') \gm[\Mtil]/\ntil \quad \text{and} \quad \Im[\Mtil](\lamtil) \ge (1-\veps')/\ntil.
\end{align*}
We can bound:
\begin{align}
 \min_{\lamGL \in \simplex_{\Pi}} & \min_{\omega \in \simplex_{\Pi}} \sup_{M \in \cMgl(\lambda)^c \cap \cMsub} \frac{1}{\Exp_{\Mbar \sim \xi}[\Exp_{\omega}[\D{\Mbar(\pi)}{M(\pi)}]]}   \le \sup_{M \in \cMgl(\lamtil)^c \cap \cMsub} \frac{1}{\Exp_{\Mbar \sim \xi}[\Exp_{\pexp}[ \D{\Mbar(\pi)}{M(\pi)}]]} \nonumber \\
 & \le \sup_{M \in \cMgl(\lamtil)^c \cap \cMsub} \frac{\bbI \{ \Exp_{\Mbar \sim \xi}[\Exp_{\pexp}[ \D{\Mbar(\pi)}{M(\pi)}]] \le 1/\Cexp \}}{\Exp_{\Mbar \sim \xi}[\Exp_{\pexp}[ \D{\Mbar(\pi)}{M(\pi)}]]}  + \Cexp, \label{eq:aec_cexp_decomp}
\end{align}
where here we take $\frac{0}{0} = 0$. If $\Exp_{\Mbar \sim \xi}[\Exp_{\pexp}[ \D{\Mbar(\pi)}{\Mtil(\pi)}]] > 1/\Cexp$, then by definition of $\Mtil$, for all $M \in \cMsub$ we have 
$$\Exp_{\Mbar \sim \xi}[\Exp_{\pexp}[ \D{\Mbar(\pi)}{M(\pi)}]] > 1/\Cexp,$$ 
so we can simply bound \eqref{eq:aec_cexp_decomp} $\le \Cexp$. Going
forward, we assume this is not the case, so that $\Exp_{\Mbar \sim \xi}[\Exp_{\pexp}[ \D{\Mbar(\pi)}{\Mtil(\pi)}]] \le 1/\Cexp$.
Our goal is to show that, for small enough $\delta$, $\lamtil \in \Lambda(M;\veps)$ for every other $M \in \cMsub$ with $\Exp_{\Mbar \sim \xi}[\Exp_{\pexp}[ \D{\Mbar(\pi)}{\Mtil(\pi)}]] \le 1/\Cexp$, so that $M \in \cMgl(\lamtil)$. This will imply that there does not exist $M \in \cMgl(\lamtil)^c \cap \cMsub$ with $\Exp_{\Mbar \sim \xi}[\Exp_{\pexp}[ \D{\Mbar(\pi)}{\Mtil(\pi)}]] \le 1/\Cexp$, which further implies that \eqref{eq:aec_cexp_decomp} $\le \Cexp$.

Fix any $M\in\cM_0$. We note that by the definition of $\pexp$, if
$\Exp_{\Mbar \sim \xi}[\Exp_{\pexp}[ \D{\Mbar(\pi)}{M(\pi)}]] \le
1/\Cexp$, then 
\begin{align*}
\sup_{p \in \simplex_\Pi} \Exp_{p}[\D{\Mtil(\pi)}{M(\pi)}] \le \delta,
\end{align*}
This implies in particular that, for each $\pi$, $\D{\Mtil(\pi)}{M(\pi)} \le \delta$.

\paragraph{Step 1: $\Exp_{\protect\Mbar \sim \xi}[\Exp_{\pexp}[ \D{\protect\Mbar(\pi)}{M(\pi)}]] \le 1/\Cexp$ implies $\pim = \pim[\Mtil]$}
As noted, if 
$$\Exp_{\Mbar \sim \xi}[\Exp_{\pexp}[ \D{\Mbar(\pi)}{M(\pi)}]] \le 1/\Cexp,$$ 
we have $\D{\Mtil(\pi)}{M(\pi)} \le \delta$ for all $\pi$. Assume that $\pim \neq \pim[\Mtil]$ (note that since $\inf_{M \in \cMsub} \delminm > 0$ by assumption, all $M \in \cMsub$ have unique optimal). By definition we have $\fm[\Mtil](\pim[\Mtil]) \ge \fm[\Mtil](\pim) + \Delmin\sups{\Mtil}$ and $f\sups{M}(\pim[M]) \ge f\sups{M}(\pim[\Mtil])$. Thus,
\begin{align*}
\Delmin\sups{\Mtil} & \le f\sups{\Mtil}(\pim[\Mtil]) -  f\sups{\Mtil}(\pim) + f\sups{M}(\pim[M]) - f\sups{M}(\pim[\Mtil]) \\
& \le | f\sups{\Mtil}(\pim[\Mtil]) - f\sups{M}(\pim[\Mtil])| + |f\sups{M}(\pim) - f\sups{\Mtil}(\pim)| \\
& \le \sqrt{\D{\Mtil(\pim[\Mtil])}{M(\pim[\Mtil])}} + \sqrt{\D{\Mtil(\pim)}{M(\pim)}} .
\end{align*}
This implies that there exists some $\pi$ such that $\D{\Mtil(\pi)}{M(\pi)} \ge (\delminm[\Mtil]/2)^2$. Assuming
\begin{align}\label{eq:aec_cexp_delta_cond1}
\delta \le \min_{M \in \cMsub} (\delminm[\Mtil]/3)^2,
\end{align} 
this is a contradiction. Thus, it must be the case that $\pim = \pim[\Mtil]$, as long as \eqref{eq:aec_cexp_delta_cond1} is satisfied.

\paragraph{Step 2: $\Exp_{\protect\Mbar \sim \xi}[\Exp_{\pexp}[ \D{\protect\Mbar(\pi)}{M(\pi)}]] \le 1/\Cexp$ implies $\lamtil \in \Lambda(M;\epsilon)$}
Under \Cref{asm:D_to_hel}, we can bound, $\forall \pi \in \Pi$, 
\begin{align*}
| \fm[\Mtil](\pi) - \fm(\pi)| & \le \sqrt{\D{\Mtil(\pi)}{M(\pi)}} \le \sqrt{\delta}.
\end{align*}
This implies that, for any $\lambda \in \simplex_\Pi$, 
\begin{align*}
|\delm[\Mtil](\lambda) - \delm(\lambda)| \le | f\sups{\Mtil}(\piM) - f\sups{M}(\piM)| + \sum_{\pi} \lambda_\pi | f\sups{\Mtil}(\pi) - f\sups{M}(\pi)| \le 4 \sqrt{\delta}.
\end{align*}
In addition, under \Cref{asm:smooth_kl}, we have
\begin{align*}
\Dkl{M(\pi)}{M'(\pi)} & \ge \Dklbig{\Mtil(\pi)}{M'(\pi)} - \LKL \sqrt{\D{\Mtil(\pi)}{M(\pi)}}  \\
& \ge \Dklbig{\Mtil(\pi)}{M'(\pi)}  - 2 \LKL \sqrt{\delta}.
\end{align*}
This implies, for any $\lambda \in \simplex_\Pi$,
\begin{align*}
\Im(\lambda) & = \inf_{M' \in \cMalt(M)} \sum_\pi \lambda(\pi) \Dkl{M(\pi)}{M'(\pi)} \\
& \ge  \inf_{M' \in \cMalt(M)} \sum_\pi \lambda(\pi) \Dklbig{\Mtil(\pi)}{M'(\pi)} - 2 \LKL \sqrt{\delta} \\
& = \Im[\Mtil](\lambda) - 2 \LKL \sqrt{\delta}
\end{align*}
where the final equality uses that, given what we have already shown, $\pim = \pim[\Mtil]$, so that $\cMalt(M) = \cMalt(\Mtil)$. 
Repeating the calculation in the other direction, we get that $|\Im(\lambda) - \Im[\Mtil](\lambda)| \le 2\LKL \sqrt{\delta}$.

We next relate $\gm$ to $\gm[\Mtil]$. By definition we have
\begin{align*}
(1+\veps') \gm/\ntil \ge \inf_{\lambda \in \simplex_\Pi} \delm(\lambda) \quad \text{s.t.} \quad \Im(\lambda) \ge (1-\veps')/\ntil.
\end{align*}
Applying our perturbation bounds we can lower bound this as
\begin{align*}
& \ge \inf_{\lambda \in \simplex_\Pi} \delm[\Mtil](\lambda) - 4 \sqrt{\delta} \quad \text{s.t.} \quad \Im[\Mtil](\lambda) \ge (1-\veps')/\ntil - 2\LKL \sqrt{\delta} \\
& \ge \frac{\gm[\Mtil]}{((1-\veps')/\ntil - 2\LKL \sqrt{\delta})^{-1}} - 4 \sqrt{\delta}
\end{align*}
where the last inequality follows from \Cref{lem:gl_program_lb}.
This implies that
\begin{align}\label{eq:aec_to_cexp_gm1}
\gm[\Mtil] \le ((1-\veps')/\ntil - 2\LKL \sqrt{\delta})^{-1} (1+\veps') \cdot \frac{\gm}{\ntil} + 4 ((1-\veps')/\ntil - 2\LKL \sqrt{\delta})^{-1} \sqrt{\delta}.
\end{align}
Assuming that
\begin{align*}
\sqrt{\delta} \le \frac{\veps' - 2 (\veps')^2}{2(1+2\veps') \LKL \ntil},
\end{align*}
some algebra shows that 
\begin{align*}
\text{\eqref{eq:aec_to_cexp_gm1}} \le (1+2\veps')^2 \gm + 4 (1+2\veps') \ntil \sqrt{\delta}.
\end{align*}

Now we can bound
\begin{align*}
\Delta\sups{M}(\lamtil) & \le \delm[\Mtil](\lamtil) + 4\sqrt{\delta} \\
& \le (1+\veps') \gm[\Mtil]/\ntil + 4\sqrt{\delta} \\
& \le (1+2\veps')^3 \gm/\ntil + 4(1+2\veps')^2 \sqrt{\delta} + 4 \sqrt{\delta}
\end{align*}
and
\begin{align*}
I\sups{M}(\lamtil) & \ge I\sups{\Mtil}(\lamtil) - 2\LKL \sqrt{\delta}  \ge (1-\veps')/\ntil - 2 \LKL \sqrt{\delta}.
\end{align*}
If $\veps'$ and $\delta$ are small enough so that
\begin{align*}
(1+2\veps')^3 \le 1+\veps/2 \quad \text{and} \quad 4(1+2\veps')^2 \sqrt{\delta} + 4\sqrt{\delta} \le \veps \gm/2\ntil
\end{align*}
and
\begin{align*}
1-\veps' \ge 1-\veps/2 \quad \text{and} \quad 2\LKL \sqrt{\delta} \le \veps/2\ntil,
\end{align*}
then $\Delta\sups{M}(\lamtil) \le (1+\veps) \gm/\ntil$ and $I\sups{M}(\lamtil) \ge (1-\veps)/\ntil$, which implies that $\lamtil \in \Lambda(M;\veps)$ with scaling factor $\ntil$.

\paragraph{Step 3: Condition on $\delta$}
Altogether, we have assumed that $\delta$ satisfies
\eqref{eq:aec_cexp_delta_cond1} and, that for some $M\in\cM_0$ with
$\Exp_{\protect\Mbar \sim \xi}[\Exp_{\pexp}[
\D{\protect\Mbar(\pi)}{M(\pi)}]] \le 1/\Cexp$, we have
\begin{align}\label{eq:aec_cexp_delta_cond2}
\sqrt{\delta} \le \frac{\veps' - 2 (\veps')^2}{2(1+2\veps') \LKL \ntil}, \quad 4(1+2\veps')^2 \sqrt{\delta} + 4\sqrt{\delta} \le \veps \gm/2\ntil, \quad 2\LKL \sqrt{\delta} \le \veps/2\ntil
\end{align}
and
\begin{align*}
(1+2\veps')^3 \le 1+\veps/2, \quad 1-\veps' \ge 1-\veps/2.
\end{align*}
Some algebra shows that, as long as $\veps \le 1$, it suffices to take
$\veps' = \veps/36$ to satisfy the latter two conditions. Furthermore, some calculation shows that a sufficient condition for \eqref{eq:aec_cexp_delta_cond2} to be met is that
\begin{align*}
\sqrt{\delta} \le \min \left \{ \frac{1}{81 \LKL}, \frac{\gm}{17} \right \} \cdot \frac{\veps}{\ntil}.
\end{align*}
By \Cref{lem:nsf_bound} and our choice of $\ntil$, we can bound
\begin{align*}
\ntil \le 2\gm[\Mtil]/\delminm[\Mtil] + \nmepsc{\veps/36}{\Mtil}
\end{align*}
so it suffices that we take
\begin{align*}
\sqrt{\delta} \le \min \left \{ \frac{1}{81 \LKL}, \frac{\gm}{17} \right \} \cdot \veps \prn*{\frac{2 \gm[\Mtil]}{\delminm[\Mtil]} + \nmeps[\Mtil]}^{-1}.
\end{align*}
As $\Mtil$ was chosen to an arbitrary model in $\cM_0$ with $\Exp_{\Mbar \sim \xi}[\Exp_{\pexp}[ \D{\Mbar(\pi)}{\Mtil(\pi)}] ] \le 1/\Cexp$, we take it to minimize $\gm[\Mtil]$ over this constraint. It suffices then that
\begin{align*}
\sqrt{\delta} \le \min \left \{ \frac{1}{81 \LKL}, \frac{\gm[\Mtil]}{17} \right \} \cdot \veps \prn*{\frac{2 \gm[\Mtil]}{\delminm[\Mtil]} +\nmepsc{\veps/36}{\Mtil}}^{-1}.
\end{align*}
Finally, by \Cref{lem:kl_bound} (under \Cref{asm:bounded_likelihood}) and \Cref{lem:gm_lb}, we can lower bound $\gm[\Mtil] \ge \delminm[\Mtil]/2\VM$.
Combining this condition with \eqref{eq:aec_cexp_delta_cond1} gives the result.

\end{proof}


\section{Proofs from \creftitle{sec:upper}}\label{sec:upper_proofs}


\paragraph{Organization of \Cref{sec:upper_proofs}}
In this section we prove the main results from \Cref{sec:upper}. We consider a slightly generalized version of the setting in \Cref{sec:upper}, where we allow for divergences other than just the KL divergence, as described below. This section is organized as follows.

\begin{itemize}
\item First, in \Cref{sec:regret_delmin_proofs}, we give the proof of
  our main result, \Cref{thm:upper_main}. We break this proof into two
  principle components: bounding the regret of \mainalg in the exploit
  phase (Section \ref{sec:exploit_regret_proofs}), and explore phase
  (Section \ref{sec:explore_regret_proofs}). The key results in this
  section are \Cref{lem:exploration_regret_docile_mingap}, which
  formalizes the key algorithm intuition given in \Cref{sec:upper},
  showing that exploring via the \CompShort yields low regret, and
  \Cref{lem:info_gain_bound}, which shows that, to enter the explore
  phase, the total ``information gain'' must be bounded as
  $\bigoh(\log T)$, which ultimately yields the optimal leading-order
  scaling. We combine these results with our estimation guarantees in
  \iftoggle{colt}{Appendix \ref{sec:upper_completing_proof}}{\Cref{sec:upper_completing_proof}}, where we give the proof of
  \Cref{thm:upper_main}. 

\item  In \Cref{sec:upper_proofs_no_mingap}, we extend \mainalg and
  \Cref{thm:upper_main} to the case where we assume no lower bound on
  the minimum gap over the model class, first presenting our main
  algorithm in this setting, \Cref{alg:gl_alg_main_infinite_nomingap},
  and then giving a proof of \Cref{thm:upper_main_no_mingap}. The
  structure of this section is similar to
  \Cref{sec:regret_delmin_proofs}---the primary difference being a
  slightly different argument to handle the need to adapt to the
  minimum gap of the ground truth instance.

\item  Finally, in \Cref{sec:upper_est_proofs} we present our estimation
  routine with covering, and prove that it achieves low estimation
  error, and in \Cref{sec:upper_misc_results} we provide the proofs of
  miscellaneous results used throughout \Cref{sec:upper_proofs}.
\end{itemize}

\subsection{Regret Bound for Uniformly Regular
  Classes (\creftitle{thm:upper_main})}\label{sec:regret_delmin_proofs}

In this section we prove \cref{thm:regret_bound_mingap2}, which
generalizes \cref{thm:upper_main}. To do so, we analyze
\cref{alg:gl_alg_main_general_D}, wich generalizes
\cref{alg:gl_alg_main} to allow for general divergences.

Throughout, we define
\begin{align*}
\guncM := \min_{M \in \cM : \gm > 0} \gm.
\end{align*}

\begin{algorithm}[h]
\caption{\MainAlg (\mainalg, general divergences)}
\begin{algorithmic}[1]
\State \textbf{input:} Optimality tolerance $\delta$, model class $\cM$, estimation oracle $\AlgEstD$. 
\State Initialize $s \leftarrow 1$, $\veps \leftarrow \frac{\delta}{4+2\delta}$, $\nmax \leftarrow \nmax(\cM,\veps), q \leftarrow \frac{4 \nmax + \delta \guncM}{4 \nmax + 2\delta \guncM}$.
\State Compute $\xi^1 \leftarrow \AlgEstD(\{ \emptyset \})$ and $\Mhat^1 \leftarrow \Exp_{M \sim \xi^1}[M]$.
\For{$t = 1,2,3,\ldots $}
\If{$\exists \pim[\Mhat^s] \in \pibm[\Mhat^s]$ s.t. $\forall M \in \cMalt(\pim[\Mhat^s])$, $\sum_{i=1}^{s-1} \Exp_{\Mhat \sim \xi^i}\Big [ \log \frac{\Prm{\Mhat}{\pi^i}(r^i, o^i)}{\Prm{M}{\pi^i}(r^i, o^i )} \Big ] \ge \log(t \log t)$} \arxiv{\hfill \algcommentlight{Exploit}}
\State Play $\piMhats$. \colt{\hfill \algcommentlight{Exploit}}\label{line:exploit_gen}
\Else \hfill \algcommentlight{Explore}
\State Set $p^s \leftarrow  q \lam^s +  (1-q) \omega^s$ for\label{line:explore_gen}
\begin{align}
 \lam^s, \omega^s & \leftarrow \argmin_{\lam, \omega \in \simplex_\Pi}  \sup_{M \in \cM \backslash \cMgl[\veps](\lam; \nmax)}  \frac{1}{\Exp_{\Mhat \sim \xi^s} \big [\Exp_{\pi \sim \omega} \big [\D{\Mhat(\pi)}{M(\pi)} \big ] \big ]}.
\end{align}
\State Draw $\pi^s \sim p^s$, observe $r^s, o^s$.
\State Compute estimate $\xi^{s+1} \leftarrow \AlgEstD(\{ (\pi^i, r^i,
o^i) \}_{i=1}^{s})$ and let $\Mhat^{s} = \Exp_{M \sim \xi^{s}}[M]$.
\State $s \leftarrow s + 1$.
\EndIf
\EndFor
\end{algorithmic}
\label{alg:gl_alg_main_general_D}
\end{algorithm}

\subsubsection{Bounding Regret of Exploit
  Phase}\label{sec:exploit_regret_proofs}
We refer to the \emph{exploit phase} as the subset of rounds $t$ in
which \cref{line:exploit_gen} is reached, and refer to the
\emph{explore phase} as the subset of rounds in which
\cref{line:explore_gen} is reached.

\begin{lemma}\label{lem:exploit_regret}
The total expected regret incurred by the exploit phase of \Cref{alg:gl_alg_main_general_D} is bounded by $2 \log \log T + 3$. 
\end{lemma}
\begin{proof}[Proof of \Cref{lem:exploit_regret}]
Let $\cE_t$ denote the event that we exploit at round $t$, and that $\piMhats \neq \pist$. Then, since we can incur suboptimality of at most 1 at each round, the total expected regret incurred by the exploit phase is bounded by
\begin{align*}
\sum_{t=1}^T \Exp\sups{\Mst}[\bbI \{ \cE_t \}].
\end{align*}
Let $\cEtil_t$ denote the event
\begin{align*}
\cEtil_t := \crl*{ \forall s \ge 1 \ : \  \sum_{i=1}^s \Exp_{\Mhat \sim \xi^i} \brk*{ \log \frac{\Prm{\Mhat}{\pi^i}(r^i,o^i)}{\Prm{\Mst}{\pi^i}(r^i,o^i)}} < \log(t \log t) }.
\end{align*}
By \Cref{lem:lrt_confidence}, we have $\Pr\sups{\Mst}[\bbI \{ \cEtil_t^c \} ] \le \frac{1}{t \log t}$, and we can bound
\begin{align*}
\Exp\sups{\Mst}[\bbI \{ \cE_t \}] \le \Exp\sups{\Mst}[\bbI \{ \cE_t \cap \cEtil_t \}] + \Exp\sups{\Mst}[\bbI \{ \cEtil_t^c \}] \le \Exp\sups{\Mst}[\bbI \{ \cE_t \cap \cEtil_t \}] + \frac{1}{t \log t}.
\end{align*}
Let $s_t$ denote the exploration round at round $t$. If we exploit at round $t$, this implies that for all $M \in \cMalt(\piMhats)$, we have
\begin{align}\label{eq:lrt_exploit_Mst}
\sum_{i=1}^{s_t - 1} \Exp_{\Mhat \sim \xi^i} \brk*{ \log \frac{\Prm{\Mhat}{\pi^i}(r^i,o^i)}{\Prm{M}{\pi^i}(r^i,o^i)}} \ge \log(t \log t).
\end{align}
If $\piMhats \neq \pist$, then $\Mst \in \cMalt(\piMhats)$, so \eqref{eq:lrt_exploit_Mst} must hold for $M \leftarrow \Mst$. This contradicts $\cEtil_t$, however, so $\Exp\sups{\Mst}[\bbI \{ \cE_t \cap \cEtil_t \}] = 0$.
Thus, $\Exp\sups{\Mst}[\bbI \{ \cE_t \}] \le \frac{1}{t \log t}$, so
\begin{align*}
\sum_{t=1}^T \Exp\sups{\Mst}[\bbI \{ \cE_t \} ] \le 3 + \sum_{t=3}^T \frac{1}{t \log t} \le 3 + 2\int_{e}^T \frac{1}{t \log t} \rmd t = 3 + 2 \log \log T.
\end{align*}
\end{proof}

\begin{lemma}\label{lem:lrt_confidence}
For $\{ (r^i,o^i,\xi^i) \}_{i=1}^s$ generated as in
\Cref{alg:gl_alg_main_general_D}, we have that
\begin{align*}
\Pr\sups{\Mst} \left [ \exists s \ge 1 \ : \  \sum_{i=1}^s \Exp_{\Mhat \sim \xi^i} \brk*{ \log \frac{\Prm{\Mhat}{\pi^i}(r^i,o^i)}{\Prm{\Mst}{\pi^i}(r^i,o^i)}} \ge \veps \right ] \le e^{-\veps}.
\end{align*}
\end{lemma}
\begin{proof}[Proof of \Cref{lem:lrt_confidence}]
Denote
\begin{align*}
X_s := \exp \prn*{\sum_{i=1}^s \Exp_{\Mhat \sim \xi^i} \brk*{ \log \frac{\Prm{\Mhat}{\pi^i}(r^i,o^i)}{\Prm{\Mst}{\pi^i}(r^i,o^i)}}}.
\end{align*}
We first show that $X_s$ is a supermartingale. Letting $\cF_{s-1}$ denote the filtration up to $s-1$, we have
\begin{align*}
\Exp\sups{\Mst}[X_s \mid \cF_{s-1}] & = \exp \prn*{\sum_{i=1}^{s-1} \Exp_{\Mhat \sim \xi^i} \brk*{ \log \frac{\Prm{\Mhat}{\pi^i}(r^i,o^i)}{\Prm{\Mst}{\pi^i}(r^i,o^i)}}} \cdot \Exp\sups{\Mst} \brk*{\exp \prn*{ \Exp_{\Mhat \sim \xi^s} \brk*{ \log \frac{\Prm{\Mhat}{\pi^s}(r^s,o^s)}{\Prm{\Mst}{\pi^s}(r^s,o^s)}}} \mid \cF_{s-1} } \\
& = X_{s-1} \cdot \Exp\sups{\Mst} \brk*{\exp \prn*{ \Exp_{\Mhat \sim \xi^s} \brk*{ \log \frac{\Prm{\Mhat}{\pi^s}(r^s,o^s)}{\Prm{\Mst}{\pi^s}(r^s,o^s)}}} \mid \cF_{s-1} } \\
& \overset{(a)}{\le} X_{s-1} \cdot \Exp\sups{\Mst} \brk*{ \Exp_{\Mhat \sim \xi^s} \brk*{ \exp \prn*{ \log \frac{\Prm{\Mhat}{\pi^s}(r^s,o^s)}{\Prm{\Mst}{\pi^s}(r^s,o^s)}}} \mid \cF_{s-1} } \\
& = X_{s-1} \cdot \Exp\sups{\Mst} \brk*{ \frac{ \Exp_{\Mhat \sim \xi^s}[\Prm{\Mhat}{\pi^s}(r^s,o^s)]}{\Prm{\Mst}{\pi^s}(r^s,o^s)} \mid \cF_{s-1} } \\
& = X_{s-1}
\end{align*}
where $(a)$ holds by Jensen's inequality, and the final equality holds since $\Exp_{\Mhat \sim \xi^s}[\Prm{\Mhat}{\pi}(\cdot,\cdot)]$ is a valid distribution over $\cR \times \cO$. Thus, $X_s$ is a supermartingale. Ville's Maximal Inequality then immediately gives that
\begin{align*}
\Pr\sups{\Mst} \brk*{ \exists s \ge 1 \ : \ X_s \ge e^{\veps} } \le \frac{\Exp\sups{\Mst}[X_1]}{e^{\veps}}.
\end{align*}
To complete the proof, using the same calculation as above, we bound
\begin{align*}
\Exp\sups{\Mst}[X_1] & = \Exp\sups{\Mst} \brk*{\exp \prn*{\Exp_{\Mhat \sim \xi^1} \brk*{ \log \frac{\Prm{\Mhat}{\pi^1}(r^1,o^1)}{\Prm{\Mst}{\pi^1}(r^1,o^1)}}}}  \le \Exp\sups{\Mst} \brk*{\Exp_{\Mhat \sim \xi^1} \brk*{ \exp \prn*{ \log \frac{\Prm{\Mhat}{\pi^1}(r^1,o^1)}{\Prm{\Mst}{\pi^1}(r^1,o^1)}}}} \le 1.
\end{align*}
\end{proof}

\subsubsection{Bounding Regret of Explore Phase}\label{sec:explore_regret_proofs}
\begin{lemma}[Main Explore Phase Regret Bound]\label{lem:exploration_regret_docile_mingap}
Let $s_T$ denote the total number of exploration rounds (which is a
random variable), and assume that $\delta \in [0,1/2)$. Then running \Cref{alg:gl_alg_main_general_D}, if $\gst > 0$, we can bound
\begin{align*}
\Exp[s_T] & \le   \frac{24 \nmax^2 + 8 \nmax \guncM}{(\delta \guncM)^2}  \cdot  \aecflipD[\veps/2](\cM) \cdot \Exp[\EstDbar(s_T)] +  \frac{12 \nmax}{\delta \delmin} \cdot \Exp[\EstKL(s_T)]  \\
& \qquad    + \frac{6 \nmax}{\delta} \cdot \Exp \bigg [  \sum_{s = 1}^{s_T} \inf_{M \in \cMalt(\Mst)} \Exp_{p^s}[\kl{\Mst(\pi)}{M(\pi)}] \cdot \bbI \{ \pist \in \pibm[\Mhat^s] \}  \bigg ] 
\end{align*}
and the regret during exploration rounds is bounded as
\begin{align*}
\Exp \bigg [ \sum_{s=1}^{s_T} \Delst(\pi^s) \bigg ] & \le \frac{8 \nmax + 2 \guncM}{\delta \guncM}  \cdot  \aecflipD[\veps/2](\cM) \cdot \Exp[\EstDbar(s_T)] + \frac{2(1+\delta) \cst}{\delmin} \cdot \Exp[\EstKL(s_T)]  \\
& \qquad    + (1+\delta) \cst \cdot \Exp \bigg [ \sum_{s = 1}^{s_T} \inf_{M \in \cMalt(\Mst)} \Exp_{p^s}[\kl{\Mst(\pi)}{M(\pi)}]  \cdot \bbI \{ \pist \in \pibm[\Mhat^s] \}  \bigg ] .
\end{align*}
\end{lemma}
\begin{proof}[Proof of \Cref{lem:exploration_regret_docile_mingap}]
The expected regret during exploration can be written as $\Exp [\sum_{s =1}^{s_T} \Delst(\pi^s)  ] = \Exp  [ \sum_{s = 1}^{s_T} \Exp_{p^s}[\Delst(\pi)]  ]$.
By definition, for exploration rounds, we have $p^s \leftarrow q\lam^s + (1-q) \omega^s$. 
For each $s\leq{}s_T$, we consider three cases to bound the instantaneous expected regret, $\Exp_{p^s}[\Delst(\pi)] = \Delst(p^s)$.

\paragraph{Case 1: $\Mst \in \cM \backslash \cMgl(\lam^s; \nmax)$}
Denote such rounds as $\cSGLest^1$. 
Write
\begin{align*}
\Delst(p^s) = \left [ \Delst(p^s) - \gamma^s \Exp_{\Mhat \sim \xi^s}\brk*{ \Exp_{p^s}[\D{\Mhat(\pi)}{\Mst(\pi)}]} \right ] + \gamma^s \Exp_{\Mhat \sim \xi^s}\brk*{\Exp_{p^s}[\D{\Mhat(\pi)}{\Mst(\pi)}]}
\end{align*}
for 
\begin{align}
\label{eq:gamma_def}
\gamma^s \ldef \frac{1+\delta}{1-q} \cdot  \frac{1}{\Exp_{\Mhat \sim \xi^s}[\Exp_{\omega^s}[\D{\Mhat(\pi)}{\Mst(\pi)}]]}.
\end{align}
In this case we have that
 \begin{align*}
\gamma^s \Exp_{\Mhat \sim \xi^s}[\Exp_{p^s}[ \D{\Mhat(\pi)}{\Mst(\pi)}]] & = \frac{1+\delta}{1-q} \cdot  \frac{1}{\Exp_{\omega^s}[\Exp_{\Mhat \sim \xi^s}[\D{\Mhat(\pi)}{\Mst(\pi)}]]}\Exp_{\Mhat \sim \xi^s}[\Exp_{p^s}[ \D{\Mhat(\pi)}{\Mst(\pi)}]] \\
& \ge  \frac{1+\delta}{\Exp_{\Mhat \sim \xi^s}[\Exp_{p^s}[ \D{\Mhat(\pi)}{\Mst(\pi)}]]}\Exp_{\Mhat \sim \xi^s}[\Exp_{p^s}[ \D{\Mhat(\pi)}{\Mst(\pi)}]] \\
& = 1+\delta.
 \end{align*}
 Thus, since the suboptimality gap is always bounded by 1, we can bound
 \begin{align*}
\Delst(p^s) - \gamma^s \Exp_{\Mhat \sim \xi^s}\brk*{ \Exp_{p^s}[\D{\Mhat(\pi)}{\Mst(\pi)}]}  \le 1 - \gamma^s \Exp_{\Mhat \sim \xi^s}\brk*{ \Exp_{p^s}[\D{\Mhat(\pi)}{\Mst(\pi)}]}  \le -\delta. 
 \end{align*}
So,
 \begin{align*}
\Delst(p^s) \le -\delta + \gamma^s \Exp_{\Mhat \sim \xi^s}\brk*{\Exp_{p^s}[\D{\Mhat(\pi)}{\Mst(\pi)}]}.
\end{align*}

\paragraph{Case 2: $\Mst \in \cMgl(\lam^s; \nmax), \pist \in \pibm[\Mhat^s]$}
Denote such rounds as $\cSGLest^2$, and write 
\begin{align*}
\Delst(p^s) & = \left [ \Delst(p^s) - (1+\delta) \gstar \cdot \Exp_{p^s}[\kl{\Mst(\pi)}{M(\pi)}] \right ] + (1+\delta) \gstar \cdot \Exp_{p^s}[\kl{\Mst(\pi)}{M(\pi)}]
\end{align*}
for any $M \in \cMaltst$.
In this case, since $\Mst \in \cMgl(\lam^s; \nmax)$, we have that $\lam^s \in \Lambda(\Mst;\veps)$. This then implies that
\begin{align*}
\Delst(\lamGL^s) \le (1 + \veps) \gstar/\betast \quad \text{and} \quad \inf_{M \in \cMaltst} \Exp_{\lam^s}[\kl{\Mst(\pi)}{M(\pi)}] \ge (1 - \veps) / \betast
\end{align*}
for some $\betast \le \nmax$.  Since $M \in \cMaltst$, it follows that $\Exp_{\lam^s}[\kl{\Mst(\pi)}{M(\pi)}] \ge (1-\veps)/\betast$. Thus,
\begin{align*}
\Delst(p^s) - (1 + \delta) \cst \Exp_{p^s}[\kl{\Mst(\pi)}{M(\pi)}] & \le  q \left [ \Delst(\lamGL^s) - (1+\delta) \cst \Exp_{\lamGL^s}[\kl{\Mst(\pi)}{M(\pi)}] \right ] + 1 - q  \\
& \le  q \left [ (1+\veps) \cst/\betast - (1+\delta) (1 - \veps) \cst/\betast \right ] + 1 - q \\
& =  q  \left [ 2 \veps  - \delta (1-\veps)  \right ] \cdot \frac{\cst}{\betast} + 1 - q  \\
& \overset{(a)}{=} - \frac{q\delta}{2} \frac{\gst}{\nst} + 1-q   \\
& \le - \frac{q\delta}{2} \frac{\gst}{\nmax} + 1-q \\
& \overset{(b)}{\le} -\frac{\delta}{4} \frac{\gst}{\nmax},
\end{align*}
where $(a)$ follows from our choice of $\veps = \frac{\delta/2}{2 + \delta}$ and setting of $\betast$, and $(b)$ follows from our setting of $q = \frac{4 \nmax + \delta \guncM}{4 \nmax + 2\delta \guncM}$, since $\gst > 0$, and some algebra.
Thus,
\begin{align*}
\Delst(p^s) & \le  (1+\delta) \cst \Exp_{p^s}[\kl{\Mst(\pi)}{M(\pi)}] -\frac{\delta}{4} \frac{\gst}{\nmax}    .
\end{align*}
As this holds for every $M \in \cMaltst$, we therefore have
\begin{align*}
\Delst(p^s) & \le \inf_{M \in \cMaltst} (1+\delta) \cst \Exp_{p^s}[\kl{\Mst(\pi)}{M(\pi)}] -\frac{\delta}{4} \frac{\gst}{\nmax}    .
\end{align*}

\paragraph{Case 3: $\Mst \in \cMgl(\lam^s; \nmax), \pist \not\in \pibm[\Mhat^s]$}
Denote such rounds as $\cSGLest^3$, and write
\begin{align*}
\Delst(p^s) & = \left [ \Delst(p^s) - \frac{2 (1+\delta) \gstar}{\delmin} \cdot \Exp_{\Mhat \sim \xi^s}[\Exp_{p^s}[\Dklbig{\Mst(\pi)}{\Mhat(\pi)}]] \right ] \\
& \qquad + \frac{2 (1+\delta) \gstar}{\delmin} \cdot \Exp_{\Mhat \sim \xi^s}[\Exp_{p^s}[\Dklbig{\Mst(\pi)}{\Mhat(\pi)}]].
\end{align*}
Since $\Mst \in \cMgl(\lam^s;\nmax)$, we have that $\lam^s \in \Lambda(\Mst;\veps)$. This then implies that for any $M \in \cMaltst$:
\begin{align*}
\Delst(\lamGL^s) \le (1 + \veps) \gstar/\betast \quad \text{and} \quad \inf_{M \in \cMaltst} \Exp_{\lam^s}[\kl{\Mst(\pi)}{M(\pi)}] \ge (1 - \veps) / \betast
\end{align*}
for some $\betast \le \nmax$. By \Cref{lem:Mhat_alt_set_lb}, since $\pist \not\in \pibm[\Mhat^s]$, we can lower bound $\Exp_{\Mhat \sim \xi^s}[\bbI \{ \Mhat \in \cMaltst \}] \ge \frac{1}{2} \delmin$. Thus, we have
\begin{align*}
\frac{2 (1+\delta) \gstar}{\delmin} \Exp_{\Mhat \sim \xi^s}[\Exp_{\lambda^s}[\Dklbig{\Mst(\pi)}{\Mhat(\pi)}]] & \ge \frac{2 (1+\delta) \gstar}{\delmin}  \Exp_{\Mhat \sim \xi^s}[\Exp_{\lambda^s}[\Dklbig{\Mst(\pi)}{\Mhat(\pi)} \cdot \bbI \{ \Mhat \in \cMaltst \}]] \\
& \ge \frac{ (1+\delta) (1-\veps) \gstar}{\betast}.
\end{align*}
This implies that
\begin{align*}
\Delst(p^s) - \frac{2(1+\delta) \gstar}{\delmin} \cdot \Exp_{\Mhat \sim \xi^s}[\Exp_{p^s}[\Dklbig{\Mst(\pi)}{\Mhat(\pi)}]]  & \le q[ (1+\veps) \gstar/\betast - (1+\delta)(1-\veps) \gstar/\betast] + 1 - q \\
& \le -\frac{\delta}{4} \frac{\gst}{\nmax},
\end{align*}
where the final inequality follows by the same argument as in Case 2. Thus, 
\begin{align*}
\Delst(p^s) & \le  \frac{2 (1+\delta) \gstar}{\delmin} \cdot \Exp_{\Mhat \sim \xi^s}[\Exp_{p^s}[\Dklbig{\Mst(\pi)}{\Mhat(\pi)}]] - \frac{\delta}{4} \frac{\gst}{\nmax}  .
\end{align*}

\paragraph{Completing the Proof}
In total we have
\begin{align*}
\Exp \bigg [ \sum_{s=1}^{s_T} \Delst(p^s) \bigg ] & \le  \Exp \bigg [  -\delta | \cSGLest^1 | - \frac{\delta \gst}{4 \nmax} |\cSGLest^2 \cup \cSGLest^3| + \sum_{s \in \cSGLest^1} \gamma^s \Exp_{\Mhat \sim \xi^s}[\Exp_{p^s}[\D{\Mhat(\pi)}{\Mst(\pi)}]] \\
& \qquad    + (1+\delta) \cst \sum_{s \in \cSGLest^2} \inf_{M \in \cMaltst} \Exp_{p^s}[\kl{\Mst(\pi)}{M(\pi)}] \\
& \qquad + \frac{2 (1+\delta) \gstar}{\delmin} \sum_{s \in \cSGLest^3} \Exp_{\Mhat \sim \xi^s}[\Exp_{p^s}[\Dklbig{\Mst(\pi)}{\Mhat(\pi)}]] \bigg ] \\
& \le  \Exp \bigg [   - \frac{\delta \gst}{4 \nmax} \cdot s_T + \frac{8 \nmax + 2 \guncM}{\delta \guncM} \cdot  \aecflipD[\veps/2](\cM) \cdot \EstDbar(s_T)  + \frac{2(1+\delta) \cst}{\delmin} \cdot \EstKL(s_T) \\
& \qquad    + (1+\delta) \cst \sum_{s = 1}^{s_T} \inf_{M \in \cMaltst} \Exp_{p^s}[\kl{\Mst(\pi)}{M(\pi)}] \cdot \bbI \{ \pist \in \pibm[\Mhat^s] \}  \bigg ] 
\end{align*}
where the last inequality follows from \Cref{lem:gammas_bound}, which bounds
\begin{align*}
\gamma^s \le \frac{8 \nmax + 2 \guncM}{\delta \guncM} \cdot \aecflipD[\veps/2](\cM),
\end{align*}
and also using that for any $s \in \cSGLest^2$ we have $\pist \in \pibm[\Mhat^s]$.
Upper bounding $ - \frac{\delta \gst}{4 \nmax} \cdot s_T  \le 0$
proves the second claim in the lemma statement.

For the first claim, as regret is always nonnegative, it follows that
\begin{align*}
0 & \le \Exp \bigg [   - \frac{\delta \gst}{4 \nmax} \cdot s_T  + \frac{8 \nmax + 2 \guncM}{\delta \guncM}  \cdot  \aecflipD[\veps/2](\cM) \cdot \EstDbar(s_T) +  \frac{2(1+\delta) \cst}{\delmin} \cdot \EstKL(s_T)  \\
& \qquad    + (1+\delta) \cst \cdot \sum_{s = 1}^{s_T} \inf_{M \in \cMaltst}  \Exp_{p^s}[\kl{\Mst(\pi)}{M(\pi)}] \cdot \bbI \{ \pist \in \pibm[\Mhat^s] \} \bigg ] 
\end{align*}
which implies
\begin{align*}
\Exp[s_T] & \le \frac{4 \nmax}{\delta \gst} \cdot \Exp \bigg [   \frac{8 \nmax + 2 \guncM}{\delta \guncM}  \cdot  \aecflipD[\veps/2](\cM) \cdot \EstDbar(s_T) +  \frac{2(1+\delta) \cst}{\delmin} \cdot \EstKL(s_T)  \\
& \qquad    + (1+\delta) \cst \cdot  \sum_{s = 1}^{s_T} \inf_{M \in \cMaltst} \Exp_{p^s}[\kl{\Mst(\pi)}{M(\pi)}] \cdot \bbI \{ \pist \in \pibm[\Mhat^s] \}  \bigg ] .
\end{align*}
\end{proof}

\begin{lemma}\label{lem:info_gain_bound}
When running both \Cref{alg:gl_alg_main_general_D} and \Cref{alg:gl_alg_main_infinite_nomingap}, for all $\alpha \in [0,1)$, we have
\begin{align*}
\Exp \bigg [ &   \sum_{s = 1}^{s_T} \inf_{M \in \cMalt(\Mst)} s^\alpha \cdot \Exp_{p^s}[\kl{\Mst(\pi)}{M(\pi)}] \cdot \bbI \{ \pist \in \pibm[\Mhat^s] \}  \bigg ]  \le\Exp[ s_T^{\alpha}] \log T  + \Exp \bigg [ \VM s_T^{1/2+\alpha} \sqrt{1344 \dcov \cdot \log \prn*{128 \Ccov s_T}}   \\
& \qquad + (\VM + \LKL  ) \bigg (4 \frac{s_T^{1/2+\alpha/2}}{1-\alpha} +  \sum_{s=1}^{s_T} s^{1/2+\alpha/2}  \cdot \Exp_{\Mhat \sim \xi^s}[\Exp_{\pi \sim p^s}[\D{\Mhat(\pi)}{\Mst(\pi)} ]] \bigg )   \bigg ] + \Exp[ s_T^{\alpha}] \log \log T  + 7\VM .
\end{align*}
\end{lemma}
\begin{proof}[Proof of \Cref{lem:info_gain_bound}]
Let $\stil_T$ denote the final exploration round for which $\pist \in \pibm[\Mhat^s]$. Then, upper bounding the KL divergence by $2\VM$ via \Cref{lem:kl_bound}, we have
\begin{align*}
\sum_{s = 1}^{s_T} \inf_{M \in \cMaltst} s^\alpha \Exp_{p^s}[\kl{\Mst(\pi)}{M(\pi)}] \cdot \bbI \{ \pist \in \pibm[\Mhat^s] \} \le \stil_T^{\alpha} \sum_{s = 1}^{\stil_T-1} \inf_{M \in \cMaltst} \Exp_{p^s}[\kl{\Mst(\pi)}{M(\pi)}]  + 2\VM s_T^{\alpha}.
\end{align*}
Under \Cref{asm:smooth_kl} and via Jensen's inequality and AM-GM, we have, for any $\alpha_s > 0$ and $M \in \cM$, 
\begin{align*}
\Exp_{p^s}[& \kl{\Mst(\pi)}{M(\pi)}]  \le \Exp_{\Mhat \sim \xi^s}[\Exp_{p^s}[\Dklbig{\Mhat(\pi)}{M(\pi)}]] + \LKL \sqrt{\Exp_{\Mhat \sim \xi^s}[\Exp_{p^s}[\D{\Mhat(\pi)}{\Mst(\pi)}]] } \\
& \le \Exp_{\Mhat \sim \xi^s}[\Exp_{p^s}[\Dklbig{\Mhat(\pi)}{M(\pi)}]]  + \LKL \left ( \frac{1}{\alpha_s} + \alpha_s \Exp_{\Mhat\sim \xi^s}[\Exp_{p^s}[\D{\Mhat(\pi)}{\Mst(\pi)}]] \right ),
\end{align*}
We now wish to bound
\begin{align*}
\Exp \bigg [  \stil_T^{\alpha} \sum_{s = 1}^{\stil_T-1} \inf_{M \in \cMaltst} \Exp_{\Mhat \sim \xi^s}[\Exp_{p^s}[\Dklbig{\Mhat(\pi)}{M(\pi)}]] \bigg ] .
\end{align*}
Let $\cMcov^j$ denote an $(\rho_j,\mu_j)$-cover of $\cMaltst$ and $\cE_j^s$ the corresponding event at step $s \le 2^j$, for some $\rho_j$ and $\mu_j$ to be chosen. Let $\cE_j := \cap_{s \le 2^j} \cE_j^s$. By definition $\Pr\subs{\Mst}[\cE_j^s] \ge 1 - \mu_j$, so $\Pr\subs{\Mst}[\cE_j] \ge 1-2^j \mu_j$.
Define an event
\begin{align*}
A_j & := \Bigg \{  \forall s \le 2^j,  M \in \cMcov^j \ : \ (s+1)^{\alpha} \sum_{i=1}^{s} \Exp_{\Mhat \sim \xi^i}[ \Exp_{p^i}[\Dklbig{\Mhat(\pi)}{M(\pi)}]]  \\
& \quad\qquad \le (s+1)^{\alpha} \sum_{i=1}^{s} \Exp_{\Mhat \sim \xi^i}\brk*{\log \frac{\Prm{\Mhat}{\pi^i}(r^i,o^i)}{\Prm{M}{\pi^i}(r^i,o^i)}} + \VM (s+1)^{\frac{1}{2} + \alpha} \sqrt{56\log \frac{2^j \Ncov(\cMaltst,\rho_j,\mu_j)}{\delta_j}} \\
&  \quad\qquad + \VM \cdot  \left ( 4\frac{(s+1)^{\frac{1+\alpha}{2}}}{1 - \alpha} + \sum_{i=1}^s i^{\frac{1+\alpha}{2}} \cdot \Exp_{\Mhat \sim \xi^i}[\Exp_{\pi \sim p^i}[\D{\Mhat(\pi)}{\Mst(\pi)} ] \right ) \Bigg \}
\end{align*}
sfor some $\delta_j$ to be chosen. By invoking
\Cref{lem:information_gain_lrt_bound} with $\beta_i = i^{1/2 +
  \alpha/2}/(s+1)^{\alpha}$ and a union bound, we have $\Pr[A_j] \ge 1
- \delta_j$ (while \Cref{lem:information_gain_lrt_bound} does not
contain the $(s+1)^\alpha$ term, the bound in the expression for $A_j$
simply gives the bound from \Cref{lem:information_gain_lrt_bound}
multiplied through with $(s+1)^{\alpha}$, which is non-random, so this
is admissible). 
We can decompose
\begin{align*}
& \Exp \bigg [ \stil_T^\alpha \sum_{s = 1}^{\stil_T-1} \inf_{M \in \cMaltst} \Exp_{\Mhat \sim \xi^s}[\Exp_{p^s}[\Dklbig{\Mhat(\pi)}{M(\pi)}  ]] \bigg ]  \\
& \qquad \le \sum_{j=1}^{\lceil \log T \rceil } \Exp \bigg [ \stil_T^\alpha \cdot \inf_{M \in \cMaltst} \sum_{s = 1}^{\stil_T-1} \Exp_{\Mhat \sim \xi^s}[\Exp_{p^s}[\Dklbig{\Mhat(\pi)}{M(\pi)}]] \cdot \bbI \{ \stil_T \in [2^{j-1},2^j), A_j \cap \cE_j \} \bigg ]  \\
& \qquad \qquad + \sum_{j=1}^{\lceil \log T \rceil } \Exp \bigg [\stil_T^\alpha \cdot \inf_{M \in \cMaltst} \sum_{s = 1}^{\stil_T-1}  \Exp_{\Mhat \sim \xi^s}[\Exp_{p^s}[\Dklbig{\Mhat(\pi)}{M(\pi)}]] \cdot \bbI \{ \stil_T \in [2^{j-1},2^j), A_j^c \cup \cE_j^c \} \bigg ] .
\end{align*}
We bound these terms separately. We can bound the second term as
\begin{align*}
\sum_{j=1}^{\lceil \log T \rceil } \Exp \bigg [ \stil_T^\alpha \cdot \inf_{M \in \cMaltst} & \sum_{s = 1}^{\stil_T-1} \Exp_{\Mhat \sim \xi^s}[\Exp_{p^s}[\Dklbig{\Mhat(\pi)}{M(\pi)}]] \cdot \bbI \{ \stil_T \in [2^{j-1},2^j), A_j^c \cup \cE_j^c \} \bigg ] \\
& \le \sum_{j=1}^{\lceil \log T \rceil} 2\VM 2^{2j} (\Pr[A_j^c] + \Pr[\cE_j^c]) \\
& \le  \sum_{j=1}^{\lceil \log T \rceil} 2\VM 2^{2j} (\delta_j + 2^j \mu_j)
\end{align*}
where the first inequality follows by bounding \Cref{lem:kl_bound}, and $ \stil_T \le 2^j$, and the second follows by our bound on the probability of $A_j$. Letting $\delta_j = \frac{1}{2^{3j}}, \mu_j = \frac{1}{2^{4j}}$, this term can be bounded by $4\VM$. We turn now to the first term. Fix $j \in [\lceil \log T \rceil]$. By definition of the event $A_j$, and since $\stil_T \le 2^j$, plugging in our choice of $\delta_j$ we can bound, for any $M \in \cMcov^j$, 
\begin{align}
&  \Exp \bigg [ \stil_T^{\alpha} \cdot \inf_{M \in \cMaltst} \sum_{s = 1}^{\stil_T-1}  \Exp_{\Mhat \sim \xi^s}[\Exp_{p^s}[\Dklbig{\Mhat(\pi)}{M(\pi)}]] \cdot \bbI \{ \stil_T \in [2^{j-1},2^j), A_j \cap \cE_j \} \bigg ] \nonumber \\
& \le \Exp \Bigg [ \Bigg ( \stil_T^\alpha \cdot \inf_{M \in \cMaltst} \sum_{s=1}^{\stil_T-1} \Exp_{\Mhat \sim \xi^s} \brk*{\log \frac{\Prm{\Mhat}{\pi^s}(r^s,o^s)}{\Prm{M}{\pi^s}(r^s,o^s)}} + \VM\stil_T^{1/2 + \alpha} \sqrt{56 \log(2^{3j} \Ncov(\cMaltst,\rho_j,\mu_j))} \nonumber \\
& \qquad + \VM \cdot \left (4 \frac{\stil_T^{1/2 + \alpha/2}}{1-\alpha}  +  \sum_{s=1}^{\stil_T} s^{1/2 + \alpha/2} \Exp_{\Mhat \sim \xi^s}[\Exp_{\pi \sim p^s}[\D{\Mhat(\pi)}{\Mst(\pi)} ]]  \right ) \Bigg ) \cdot \bbI \{ \stil_T \in [2^{j-1},2^j), \cE_j \} \Bigg ] \nonumber \\
& \le  \Exp \Bigg [ \Bigg ( \stil_T^\alpha  \cdot \inf_{M \in \cMaltst} \sum_{s=1}^{\stil_T-1} \Exp_{\Mhat \sim \xi^s}\brk*{\log \frac{\Prm{\Mhat}{\pi^s}(r^s,o^s)}{\Prm{M}{\pi^s}(r^s,o^s)}} + \VM \stil_T^{1/2 + \alpha} \sqrt{168 \log(8 \stil_T \Ncov(\cMaltst,\rho_j,\tfrac{\stil_T^{-4}}{16}))} \nonumber \\
& \qquad + \VM \cdot \left (4 \frac{\stil_T^{1/2 + \alpha/2}}{1-\alpha}  +  \sum_{s=1}^{\stil_T} s^{1/2+\alpha/2} \Exp_{\Mhat \sim \xi^s}[ \Exp_{\pi \sim p^s}[\D{\Mhat(\pi)}{\Mst(\pi)}] ]  \right ) \Bigg ) \cdot \bbI \{ \stil_T \in [2^{j-1},2^j), \cE_j \} \Bigg ] \label{eq:inf_gain_bound1}
\end{align}
where the second inequality follows since, if $\stil_T \in [2^{j-1},2^j)$, then we can bound $2^{j} \le 2 \stil_T$. 

Since $\stil_T$ is an exploration round by definition, we know that for all $\pim[\Mhat^{\stil_T}] \in \pibm[\Mhat^{\stil_T}]$, 
there exists some $M' \in \cMalt(\pim[\Mhat^{\stil_T}])$ such that 
\begin{align*}
\sum_{s=1}^{\stil_T - 1} \Exp_{\Mhat \sim \xi^s} \brk*{\log \frac{\Prm{\Mhat}{\pi^s}(r^s,o^s)}{\Prm{M'}{\pi^s}(r^s,o^s)}} \le \log (T \log T). 
\end{align*}
By assumption we have $\pist \in \pibm[\Mhat^{\stil_T}]$, so it
follows that there exists some (random) $M' \in \cMaltst$ such that the above
inequality holds. 
Let $M'' \in \cMcov^j$ denote the model in $\cMcov^j$ such that
\begin{align*}
\left | \log \frac{\Prm{M'}{\pi}(r,o)}{\Prm{M''}{\pi}(r,o)} \right | = \left | \log \Prm{M'}{\pi}(r,o) - \log \Prm{M''}{\pi}(r,o) \right | \le \rho_j,
\end{align*}
for all $(r,o,\pi)$ for which $\sup_{\Mtil \in \cM} \Prm{\Mtil}{\pi}(r,o \mid \cE_j) > 0$, 
which is guaranteed to exist by \Cref{def:cover}. Note that by definition of $\cMcov^j$, we have $M'' \in \cMaltst$. We then have
\begin{align*}
\stil_T^{\alpha} \cdot & \inf_{M \in \cMaltst} \sum_{s=1}^{\stil_T - 1}  \Exp_{\Mhat \sim \xi^s}\brk*{\log \frac{\Prm{\Mhat}{\pi^s}(r^s,o^s)}{\Prm{M}{\pi^s}(r^s,o^s)} } \cdot \bbI \{ \stil_T \in [2^{j-1},2^j), \cE_j \} \\
& \le \stil_T^{\alpha} \sum_{s=1}^{\stil_T - 1}  \Exp_{\Mhat \sim \xi^s}\brk*{\log \frac{\Prm{\Mhat}{\pi^s}(r^s,o^s)}{\Prm{M''}{\pi^s}(r^s,o^s)} } \cdot \bbI \{ \stil_T \in [2^{j-1},2^j), \cE_j \} \\
& = \stil_T^{\alpha} \sum_{s=1}^{\stil_T - 1} \left [ \Exp_{\Mhat \sim \xi^s}\brk*{\log \frac{\Prm{\Mhat}{\pi^s}(r^s,o^s)}{\Prm{M'}{\pi^s}(r^s,o^s)}}+ \log \frac{\Prm{M'}{\pi^s}(r^s,o^s)}{\Prm{M''}{\pi^s}(r^s,o^s)} \right ] \cdot \bbI \{ \stil_T \in [2^{j-1},2^j), \cE_j \} \\
& \le \stil_T^{\alpha} \log(T\log T) + \rho_j \cdot \stil_T^{1+\alpha},
\end{align*}
where the inequality holds since on $\cE_j$, we can ensure that $\log \frac{\Prm{M'}{\pi^s}(r^s,o^s)}{\Prm{M''}{\pi^s}(r^s,o^s)} \le \rho_j$. 
Therefore, choosing $\rho_j = 2^{-3j}$, we have
\begin{align*}
\text{\eqref{eq:inf_gain_bound1}} & \le \Exp \Bigg [ \Bigg ( \stil_T^{\alpha} \log(T\log T) + \rho_j \cdot \stil_T^{1+\alpha} + \VM\stil_T^{1/2+\alpha} \sqrt{168 \stil_T \log(8 \Ncov(\cMaltst,\rho_j,\tfrac{\stil_T^{-4}}{16}))} \\
& \qquad + \VM \cdot \left (4 \frac{\stil_T^{1/2+\alpha/2}}{1-\alpha}  +  \sum_{s=1}^{\stil_T} s^{1/2+\alpha/2} \Exp_{\Mhat \sim \xi^s}[\Exp_{\pi \sim p^s}[\D{\Mhat(\pi)}{\Mst(\pi)} ]]  \right ) \Bigg ) \cdot \bbI \{ \stil_T \in [2^{j-1},2^j) \} \Bigg ] \\
& \le \Exp \Bigg [ \Bigg ( \stil_T^{\alpha} \log(T\log T) + 2^{-j} + \VM \stil_T^{\alpha + 1/2} \sqrt{168 \log(8 \stil_T \Ncov(\cMaltst,\tfrac{\stil_T^{-3}}{8},\tfrac{\stil_T^{-4}}{16}))}  \\
& \qquad + \VM \cdot \left (4 \frac{\stil_T^{1/2+\alpha/2}}{1-\alpha}  +  \sum_{s=1}^{\stil_T} s^{1/2+\alpha/2} \Exp_{\Mhat \sim \xi^s}[ \Exp_{\pi \sim p^s}[\D{\Mhat(\pi)}{\Mst(\pi)} ]]  \right ) \Bigg ) \cdot \bbI \{ \stil_T \in [2^{j-1},2^j) \} \Bigg ]  .
\end{align*}
As this holds for each $j$, and the events $\{ \stil_T \in [2^{j-1},2^j) \}$ are disjoint, we can sum over $j$ to bound
\begin{align*}
&  \sum_{j=1}^{\lceil \log T \rceil } \Exp \bigg [\stil_T^\alpha \cdot \inf_{M \in \cMaltst} \sum_{s = 1}^{\stil_T-1} \Exp_{\Mhat \sim \xi^s}[ \Exp_{p^s}[\Dklbig{\Mhat(\pi)}{M(\pi)}] ] \cdot \bbI \{ \stil_T \in [2^{j-1},2^j), A_j \cap \cE_j \} \bigg ] \\
& \qquad \le \Exp \Bigg [ \stil_T^{\alpha} \log(T\log T) + 1 + \VM \stil_T^{\alpha+1/2} \sqrt{168 \log(8 \stil_T \Ncov(\cMaltst,\tfrac{\stil_T^{-3}}{8},\tfrac{\stil_T^{-4}}{16}))}  \\
& \qquad \qquad + \VM \cdot \left (4 \frac{\stil_T^{1/2+\alpha/2}}{1-\alpha}  +  \sum_{s=1}^{\stil_T} s^{1/2+\alpha/2} \Exp_{\Mhat \sim \xi^s}[\Exp_{\pi \sim p^s}[\D{\Mhat(\pi)}{\Mst(\pi)} ]]  \right ) \Bigg ] .
\end{align*}
Finally, under \Cref{asm:covering} and \Cref{lem:cover_equiv} we can bound
\begin{align*}
\log(8 \stil_T \Ncov(\cMaltst,\tfrac{\stil_T^{-3}}{8},\tfrac{\stil_T^{-4}}{16})) & \le \log(8 \stil_T) + \log \Ncov(\cM,\tfrac{\stil_T^{-3}}{8},\tfrac{\stil_T^{-4}}{16}) \\
& \le \log(8 \stil_T) + \dcov \cdot \log \prn*{128 \Ccov \stil_T^{7}} \\
& \le 8 \dcov  \cdot \log \prn*{128 \Ccov \stil_T}. 
\end{align*}
The result follows.
\end{proof}

\subsubsection{Completing the Proof}\label{sec:upper_completing_proof}

\begin{theorem}[Full version of \Cref{thm:upper_main}]\label{thm:regret_bound_mingap2}
For any $\delta \le 1/2$, if we set $\D{\cdot}{\cdot} =
\Dkl{\cdot}{\cdot}$ and instantiate $\AlgEstD$ with
\Cref{alg:estimator}, under \Cref{ass:realizability,asm:bounded_means,ass:unique_opt,asm:smooth_kl_kl,asm:bounded_likelihood,asm:covering,asm:mingap} and
 if $\gst > 0$,
 the expected regret of
\Cref{alg:gl_alg_main_general_D} is bounded by 
\begin{align*}
\Exp\sups{\Mst}[\RegDM] & \le (1+\delta) \cst \cdot \log T + \Caec \cdot \aecflip[\veps/2](\cM) \cdot \log^{3/2}( \log T )  + \linear \left ( \Clo, \sqrt{\log T}, \log^{3/2}(\log T) \right ),
\end{align*}
for
\begin{align*}
\Clo & :=\linear \left ( \cst, \max_{M \in \cM} \gm, \VM^{13/2}, \LKL^2, \dcov, \log \Ccov, \frac{1}{\delta^2}, \frac{1}{\delmin^3}, \ncMeps[\delta/6],  \sqrt{\aec[\veps/2](\cM)} \right ) \\
\Caec & := c\cdot\frac{\VM^2 \dcov \log(\Ccov) \cdot \max_{M \in \cM} \gm \cdot (\delta^{-1} + \VM \ncMeps[\delta/6])}{\delta \delmin^3} \cdot \log(\Clo)
\end{align*}
and where $\linear (\cdot )$ denotes a function linear and
poly-logarithmic in its arguments and $c>0$ is a universal constant.
\end{theorem}
\begin{proof}[Proof of \Cref{thm:regret_bound_mingap2}]
Letting $\cTexploit$ denote the exploitation rounds, we can bound the total expected regret as
\begin{align*}
\sum_{t=1}^T \Exp[\Delst(\pi^t)] & = \Exp \left [ \sum_{t \in \cTexploit} \Delst(\pi^t) \right ] + \Exp \left [ \sum_{s = 1}^{s_T} \Delst(\pi^s) \right ] \le 2\log\log T + 3 + \Exp \left [ \sum_{s = 1}^{s_T} \Delst(\pi^s) \right ],
\end{align*}
where the inequality follows from \Cref{lem:exploit_regret}. It remains to bound the regret in the exploration rounds. By \Cref{lem:exploration_regret_docile_mingap}, we can bound this as
\begin{align*}
\Exp \bigg [ \sum_{s=1}^{s_T} \Delst(\pi^s) \bigg ] & \le \frac{8 \nmax + 2 \guncM}{\delta \guncM}  \cdot  \aecflipD[\veps/2](\cM) \cdot \Exp[\EstDbar(s_T)] +  \frac{2(1+\delta) \gst}{\delmin} \cdot   \Exp[\EstKL(s_T)]  \\
& \qquad    + (1+\delta) \cst \cdot \Exp \bigg [ \sum_{s = 1}^{s_T} \inf_{M \in \cMaltst} \Exp_{p^s}[\kl{\Mst(\pi)}{M(\pi)}]  \cdot \bbI \{ \pist \in \pibm[\Mhat^s] \}  \bigg ] .
\end{align*}
Applying \Cref{lem:info_gain_bound} with $\alpha = 0$, we can  bound
\begin{align}
\Exp \bigg [ \sum_{s = 1}^{s_T} & \inf_{M \in \cMaltst} \Exp_{p^s}[\kl{\Mst(\pi)}{M(\pi)}] \cdot \bbI \{ \pist \in \pibm[\Mhat^s] \}  \bigg ]  \le \log T  + \Exp \bigg [ \VM\sqrt{1344 \dcov s_T \cdot \log \prn*{128 \Ccov s_T}}   \nonumber \\
& \qquad + (\VM + \LKL  ) \bigg (4 \sqrt{s_T} +  \sum_{s=1}^{s_T} \sqrt{s}  \cdot \Exp_{\Mhat \sim \xi^s}[\Exp_{\pi \sim p^s}[\D{\Mhat(\pi)}{\Mst(\pi)} ]] \bigg )   \bigg ] + \log \log T  + 7\VM \nonumber \\
& \le  \log T  +  \VM\sqrt{1344 \dcov \Exp[s_T] \cdot \log \prn*{128 \Ccov \Exp[s_T]}}  \nonumber \\
& \qquad + (\VM + \LKL  ) \bigg (4\sqrt{\Exp[s_T]} +  \Exp \brk*{\sum_{s=1}^{s_T} \sqrt{s}  \cdot \Exp_{\Mhat \sim \xi^s}[\Exp_{\pi \sim p^s}[\D{\Mhat(\pi)}{\Mst(\pi)} ]] }\bigg )   + \log \log T  + 7\VM \label{eq:thm_kl_bound}
\end{align}
where the last inequality follows from \Cref{claim:concave_fun}---applied with $\alpha = \beta = 1/2$, $a = 128 \Ccov$---and the concavity of $\sqrt{x}$. Note that the condition of \Cref{claim:concave_fun} is met here since we assume $\Ccov \ge 1$. 
It follows from \Cref{lem:sT_est_exp_bound} that 
\begin{align}
\Exp \brk*{\sum_{s=1}^{s_T} \sqrt{s}  \cdot \Exp_{\Mhat \sim \xi^s}[\Exp_{\pi \sim p^s}[\D{\Mhat(\pi)}{\Mst(\pi)}] ] } & \le \Exp \brk*{(2+6\VM) \prn*{20 \dcov \cdot \log(64 \Ccov s_T) + 1} \cdot 5 \sqrt{2 s_T \log(2 s_T)}} 
\nonumber \\
& \qquad + \Exp \brk*{32(1+\VM) \log(s_T)}+ 8 \nonumber \\
& \le (2+6\VM) \prn*{20 \dcov \cdot \log(64 \Ccov \Exp[s_T]) + 1} \cdot 5 \sqrt{2 \Exp[s_T] \log(2 \Exp[s_T])} \nonumber \\
& \qquad + 32(1+\VM) \log{\Exp[s_T]} + 8 \nonumber \\
& = \bigoh \prn*{\VM \dcov \log(\Ccov \Exp[s_T]) \sqrt{\Exp[s_T] \log(\Exp[s_T])} + \VM \sqrt{\Exp[s_T]}},\label{eq:thm_kl_bound2}
\end{align}
where the last inequality follows from \Cref{claim:concave_fun}. Similarly, by \Cref{lem:est_exp_bound}, we can bound both $\Exp[\EstDbar(s_T)] $ and $\Exp[\EstKL(s_T)]$ as 
\begin{align}
\Exp[\EstDbar(s_T)] ,\Exp[\EstKL(s_T)] & \le \Exp \brk*{(2+5\VM) \prn*{20 \dcov \cdot \log(64 \Ccov s_T) + 1} \cdot \sqrt{ \log(2 s_T)}} \nonumber \\
& \qquad + \Exp \brk*{32(1+\VM) \log(s_T)} + 8 \nonumber \\
& \le  (2+5\VM) \prn*{20 \dcov \cdot \log(64 \Ccov \Exp[s_T]) + 1} \cdot \sqrt{ \log(2 \Exp[s_T])} \nonumber \\
& \qquad + 32(1+\VM) \log(\Exp[s_T]) + 8 \nonumber \\
& = \bigoh \prn*{\VM \dcov \log(\Ccov \Exp[s_T]) \sqrt{\log(\Exp[s_T])} + \VM \log(\Exp[s_T])}, \label{eq:thm_kl_bound3}
\end{align}
where the second inequality follows from \Cref{claim:concave_fun} and Jensen's inequality, which lets us pass the expectation through.
Together this gives a regret bound of
\begin{align*}
\Exp \bigg [ \sum_{s=1}^{s_T} \Delst(\pi^s) \bigg ] & \le (1+\delta) \gst \cdot \log T + (1+\delta) \gst \VM\sqrt{1344 \dcov \Exp[s_T] \cdot \log \prn*{128 \Ccov \Exp[s_T]}}  +(1+\delta) \gst( \log \log T  + 7\VM )  \\
& \qquad + (1+\delta) \gst (\VM+\LKL) \cdot \bigoh \prn*{\sqrt{\Exp[s_T]} + \VM \dcov \cdot \log( \Ccov \Exp[s_T])  \sqrt{ \Exp[s_T] \log( \Exp[s_T])} + \VM \log{\Exp[s_T]}} \\
& \qquad +\frac{8 \nmax + 2 \guncM}{\delta \guncM}  \cdot  \aecflipD[\veps/2](\cM) \cdot \bigoh \prn*{\VM \dcov \log(\Ccov \Exp[s_T]) \sqrt{\log(\Exp[s_T])} + \VM \log(\Exp[s_T])} \\
& \qquad +  \frac{2(1+\delta) \gst}{\delmin} \cdot  \bigoh \prn*{\VM \dcov \log(\Ccov \Exp[s_T]) \sqrt{\log(\Exp[s_T])} + \VM \log(\Exp[s_T])}
\end{align*}

By \Cref{lem:exploration_regret_docile_mingap} we can bound the total number of exploration rounds by
\begin{align*}
\Exp[s_T] & \le   \frac{24 \nmax^2 + 8 \nmax \guncM}{(\delta \guncM)^2} \cdot  \aecflipD[\veps/2](\cM) \cdot \Exp[\EstDbar(s_T)] +  \frac{12 \nmax}{\delta \delmin} \cdot \Exp[\EstKL(s_T)]  \\
& \qquad    + \frac{6 \nmax}{\delta} \cdot \Exp \bigg [ \sum_{s = 1}^{s_T} \inf_{M \in \cMaltst} \Exp_{p^s}[\kl{\Mst(\pi)}{M(\pi)}] \cdot \bbI \{ \pist \in \pibm[\Mhat^s] \}  \bigg ] \\
& \le \frac{24 \nmax^2 + 8 \nmax \guncM}{(\delta \guncM)^2}  \cdot  \aecflipD[\veps/2](\cM) \cdot \bigoh \prn*{\VM \dcov \log(\Ccov \Exp[s_T]) \sqrt{\log(\Exp[s_T])} + \VM \log(\Exp[s_T])} \\
& \qquad +  \frac{12 \nmax}{\delta \delmin} \cdot \bigoh \prn*{\VM \dcov \log(\Ccov \Exp[s_T]) \sqrt{\log(\Exp[s_T])} + \VM \log(\Exp[s_T])} \\
& \qquad + \frac{6 \nmax}{\delta} \cdot (\log T  +  \VM\sqrt{1344 \dcov \Exp[s_T] \cdot \log \prn*{128 \Ccov \Exp[s_T]}} + \log \log T + 7\VM) \\
& \qquad + \frac{6 \nmax}{\delta} (\VM + \LKL) \prn*{4\sqrt{\Exp[s_T]} +  \bigoh \prn*{\VM \dcov \log(\Ccov \Exp[s_T]) \sqrt{\Exp[s_T] \log(\Exp[s_T])} + \VM \sqrt{\Exp[s_T]}}}
\end{align*}
where the second inequality follows from \Cref{eq:thm_kl_bound,eq:thm_kl_bound2,eq:thm_kl_bound3}.
Using \Cref{lem:lin_log_ineq} and bounding $\guncM \le \nmax$, we can solve this for $\Exp[s_T]$ to get
\begin{align*}
\Exp[s_T] & \le \bigoht \Bigg ( \frac{\nmax}{\delta} \cdot \log T + \frac{\VM \dcov \log \Ccov \cdot \nmax^2}{(\delta \guncM)^2}  \cdot  \aecflipD[\veps/2](\cM) + \frac{\nmax \VM \dcov \log \Ccov }{\delta \delmin} \\
& \quad\qquad + \frac{\nmax^2 (\VM + \LKL)^2 ( \VM^2 \dcov^2 \log^2 \Ccov + \VM^2)}{\delta^2} \Bigg ) .
\end{align*}
Finally, by \Cref{lem:kl_bound} and \Cref{lem:gm_lb} we can lower
bound $\guncM \ge \delmin / 2 \VM$. Plugging this into the above
expression and using that
\begin{align*}
\nmax = \nmax(\cM,\delta/6) := \frac{32}{\delmin^2} \cdot \prn*{\frac{6}{\delta} + 2 \VM \ncMeps[\delta/6] } \cdot \max_{M \in \cM} \gm,
\end{align*}
gives the result, after simplifying.

\end{proof}

\subsection{Regret Bound without Uniform Regularity (\creftitle{thm:upper_main_no_mingap})}\label{sec:upper_proofs_no_mingap}

In this section, we prove \cref{thm:upper_main_no_mingap} (as well as
a more general result, \Cref{thm:upper_main_no_mingap_app}), which gives regret bounds for a variant of \mainalg,
\cref{alg:gl_alg_main_infinite_nomingap}, \mainalgb, which does not require
uniform regularity, and adapts to the gap
$\delminst := \delminm[\Mstar]$ for the underlying model $\Mstar$. \Cref{alg:gl_alg_main_infinite_nomingap} is a slightly more general version of \Cref{alg:gl_alg_main_infinite_nomingap_main}, incorporating general divergences $D$. 

Throughout this section, we let $\sst$ denote the first exploration
round $s$ such that $\Mst \in \cM^\ell$ in
\cref{alg:gl_alg_main_infinite_nomingap_main}. Note that $\sst$ is a
deterministic quantity (though, the first round $t$ in which
  $s=\sst$ is not deterministic).
  
 We remark briefly that, to bound \eqref{eq:alg_alloc_comp2} by a restriction of the \CompShort, it is critical that our estimator produced at each exploration round $s$, $\xi^s$, is only supported on $\cM^\ell$. To accomplish this, we explicitly generate a cover of $\cM^\ell$, $\cMcov^\ell$, and run an estimation procedure on this cover. As \Cref{alg:estimator}, the estimation procedure used to prove \Cref{thm:upper_main}, directly covers all of $\cM$, to prove \cref{thm:upper_main_no_mingap} we do not appeal directly to this algorithm, yet we note that the covering and estimation procedure employed by \mainalgb are essentially identical to that in \Cref{alg:estimator}, modulo the choice of which set is being covered.

\begin{algorithm}[h]
\caption{Adaptive Exploration for Allocation Estimation for classes
  without uniform regularity (\mainalgb)}
\begin{algorithmic}[1]
\State \textbf{input:} Optimality tolerance $\delta$, divergence
$D$, estimation oracle $\AlgEstD$, growth parameters $\alpha_q$, $\alpha_n$, $\alpha_\cM$.
\State $s \leftarrow 1$, $\ell \leftarrow 1$, $\veps \leftarrow \frac{\delta}{4 + 2 \delta}$.
\State $q^s \leftarrow 1- s^{-\alphaq}$, $\nsf^s \leftarrow s^{\alphan}$.
\State Compute $\xi^1 \leftarrow \AlgEstD(\{ \emptyset \})$ and $\Mhat^1 \leftarrow \Exp_{M \sim \xi^1}[M]$.
\For{$t = 1,2,3,\ldots $}
\If{$s \ge 2^\ell$} \hfill \algcommentlight{Form active set and cover}
\State $\ell \leftarrow \ell + 1$.
\State $\Delta^\ell \leftarrow \argmin_{\Delta \ge 0} \Delta \quad \text{s.t.} \quad \aecflipDM{\veps/2}{\cM}(\cM_{\Delta,\frac{1}{\Delta}}) \le s^{\alphaM}$. \label{line:cMell_aec_bound}
\State $\cM^\ell \leftarrow \cM_{\Delta^\ell, \frac{1}{\Delta^\ell}} \cap \big \{ M \in \cM \ : \ \nmepsb + \frac{1}{\delminm} + \frac{4 \gm}{\delminm} + \frac{2 \nmepsb}{\gm} + \frac{4}{\delminm \veps} \le \sqrt{\nsf^s} \big \}$. \label{line:cMell_defn}
\State $\cMcov^\ell \leftarrow$ $(\rho_\ell,\mu_\ell)$-cover of $\cM^\ell$ for $\rho_\ell \leftarrow 2^{-\ell}, \mu_\ell \leftarrow 2^{-5\ell}$, $\frakD^\ell \leftarrow \emptyset$.
\EndIf
\If{$\exists \pim[\Mhat^s] \in \pibm[\Mhat^s]$ s.t. $\forall M \in \cMalt(\pim[\Mhat^s])$, $\sum_{i=1}^{s-1} \Exp_{\Mhat \sim \xi^i} \Big [ \log \frac{\Prm{\Mhat}{\pi^i}(r^i,o^i )}{\Prm{M}{\pi^i}(r^i,o^i )}\Big ] \ge \log(t \log t)$} \arxiv{\hfill \algcommentlight{Exploit}}
\State Play $\piMhats$. \colt{\hfill \algcommentlight{Exploit}}
\Else \hfill \algcommentlight{Explore}
\State Set $p^s \leftarrow  q^s \lam^s +  (1-q^s) \omega^s$ for
\begin{align}\label{eq:alg_alloc_comp2} 
 \lam^s, \omega^s & \leftarrow \argmin_{\lam, \omega \in \simplex_\Pi} \sup_{M \in \cM^\ell \backslash \cMgl(\lam; \nsf^s)} \frac{1}{\Exp_{\Mhat \sim \xi^{s}}[\Exp_{\pi \sim \omega}[\D{\Mhat(\pi)}{M(\pi)}]]}.
\end{align}
\State Draw $\pi^s \sim p^s$, observe $r^s, o^s$, set $\frakD^\ell \leftarrow \frakD^\ell \cup \{ (\pi^s,r^s,o^s) \}$.
\State Compute estimate $\xi^{s+1} \leftarrow \AlgEstD(\frakD^\ell,\cMcov^\ell)$ and $\Mhat^{s+1} = \Exp_{M \sim \xi^{s+1}}[M]$.
\State $s \leftarrow s + 1$.
\EndIf
\EndFor
\end{algorithmic}
\label{alg:gl_alg_main_infinite_nomingap}
\end{algorithm}

\subsubsection{Bounding Regret of Explore Phase}

\begin{lemma}\label{lem:sst_bound}
We have
\begin{align*}
\sst \le \prn*{\aecflipDM{\veps/2}{\cM}(\cMLdelst)}^{\frac{1}{\alphaM}} + \prn*{ \nepsst + \frac{1}{\delminst} + \frac{4 \gst}{\delminst} + \frac{2 \nepsst}{\gst} + \frac{4}{\delminst \veps} }^{\frac{2}{\alphan}} .
\end{align*}
\end{lemma}
\begin{proof}[Proof of \Cref{lem:sst_bound}]
We will have $\Mst \in \cM^\ell$ as soon as $\Mst \in \cM'$ and $\Mst \in \cM''$ for
\begin{align*}
\cM' \leftarrow  \crl*{ M \in \cM \ : \ \nmepsb + \frac{1}{\delminm} + \frac{4 \gm}{\delminm} + \frac{2 \nmepsb}{\gm} + \frac{4}{\delminm \veps} \le \sqrt{\nsf^{2^{\ell-1}}} }, \quad \cM'' \leftarrow \cM_{\Delta^\ell,\frac{1}{\Delta^\ell} } 
\end{align*}
where
\begin{align*}
\Delta^\ell = \argmin_{\Delta \ge 0} \Delta \quad \text{s.t.} \quad \aecflipDM{\veps/2}{\cM}(\cM_{\Delta,\frac{1}{\Delta}}) \le s^{\alphaM}.
\end{align*}
Note that if this occurs for some $\ell$, then the first exploration
round $s$ such that $\Mstar\in\cM^\ell$ is $s = 2^{\ell - 1}$. From
the definition $\nsf^{2^{\ell-1}} = 2^{\alphan (\ell - 1)}$, we have $\Mst \in \cM'$ as soon as 
\begin{align*}
\sqrt{2^{\alphan (\ell - 1)}} & \ge \nepsst + \frac{1}{\delminst} + \frac{4 \gst}{\delminst} + \frac{2 \nepsst}{\gst} + \frac{4}{\delminst \veps} \\
\iff & 2^{\ell - 1} \ge \prn*{ \nepsst + \frac{1}{\delminst} + \frac{4 \gst}{\delminst} + \frac{2 \nepsst}{\gst} + \frac{4}{\delminst \veps} }^{2/\alphan}.
\end{align*}
To have $\Mst \in \cM''$, we need $\Delta^\ell \le \delminst$ and $\nepsst \le \frac{1}{\Delta^\ell}$, which will be the case once 
\begin{align*}
\aecflipDM{\veps/2}{\cM}(\cMLdelst) \le 2^{(\ell-1) \cdot \alphaM} \iff \prn*{\aecflipDM{\veps/2}{\cM}(\cMLdelst)}^{\frac{1}{\alphaM}} \le 2^{\ell - 1}.
\end{align*}
The result then follows by combining these bounds. 
\end{proof}

\begin{lemma}\label{lem:exploration_regret_infinite_nomingap}
Let $s_T$ denote the total number of exploration rounds. For $\delta
\le 1/2$, running \Cref{alg:gl_alg_main_infinite_nomingap}, we can almost surely bound
bound 
\begin{align*}
s_T & \le \frac{4}{\delta \gst} \bigg ( \sum_{s = \sst}^{s_T} 2s^{\alphaq + \alphaM} \Exp_{\Mhat \sim \xi^s}[\Exp_{p^s}[\D{\Mhat(\pi)}{\Mst(\pi)}] + \sum_{s = \sst}^{s_T} s^{\alphan} \cdot 2\cst \cdot \inf_{M \in \cMaltst} \Exp_{p^s}[\kl{\Mst(\pi)}{M(\pi)}]  \cdot \bbI \{ \pist \in \pibm[\Mhat^s] \} \\
& \qquad + \sum_{s = \sst}^{s_T} s^{3\alphan/2} \cdot 4 \gst \Exp_{\Mhat \sim \xi^s}[\Exp_{p^s}[\Dklbig{\Mst(\pi)}{\Mhat(\pi)}]] + 4 \gst (\tfrac{8}{\delta \gst} )^{\frac{1+\alphan}{\alphaq - \alphan}} \bigg ) + \sst.
\end{align*}
In addition, the regret during exploration rounds is bounded as
\begin{align*}
\Exp \bigg [ \sum_{s=\sst}^{s_T} \Delst(\pi^s) \bigg ]  & \le \Exp \bigg [  \sum_{s = \sst}^{s_T} 2s^{\alphaq + \alphaM} \Exp_{\Mhat \sim \xi^s}[\Exp_{p^s}[\D{\Mhat(\pi)}{\Mst(\pi)}] \\
& \qquad + \sum_{s = \sst}^{s_T} (1+\delta)\cst \cdot \inf_{M \in \cMaltst} \Exp_{p^s}[\kl{\Mst(\pi)}{M(\pi)}] \cdot \bbI \{ \pist \in \pibm[\Mhat^s] \} \\
& \qquad + \sum_{s = \sst}^{s_T} s^{\alphan/2} \cdot 4 \gst \Exp_{\Mhat \sim \xi^s}[\Exp_{p^s}[\Dklbig{\Mst(\pi)}{\Mhat(\pi)}]] \bigg ] + 4 \gst (\tfrac{8}{\delta \gst} )^{\frac{1}{\alphaq - \alphan}}.
\end{align*}
\end{lemma}
\begin{proof}[Proof of
  \Cref{lem:exploration_regret_infinite_nomingap}]
We first prove the bound on the regret during the exploration rounds, then use this
result to prove the bound on $s_T$.

Recall that by definition, for exploration rounds, we have $p^s \leftarrow q^s
\lam^s + (1-q^s) \omega^s$. 
We consider three cases to bound the
instantaneous expected regret, $\Delst(p^s)$,
for each $s \ge \sst$.

\paragraph{Case 1: $\Mst \in \cM^\ell \backslash \cMgl(\lam^s; \nsf^s)$}
Denote such rounds as $\cSGLest^1$. This case follows identically to Case 1 in \Cref{lem:exploration_regret_docile_mingap} with
\begin{align}
\label{eq:gamma_def_full}
\gamma^s = \frac{1+\gst \delta}{1+ q^s} \cdot \frac{1}{\Exp_{\Mhat \sim \xi^s}[\Exp_{\omega^s}[\D{\Mhat(\pi)}{\Mst(\pi)}]}.
\end{align}
We therefore omit the proof and conclude that
 \begin{align*}
\Delst(p^s) & \le -\gst \delta + \gamma^s \Exp_{\Mhat \sim \xi^s}[\Exp_{p^s}[\D{\Mhat(\pi)}{\Mst(\pi)}] \\
& \le \gamma^s \Exp_{\Mhat \sim \xi^s}[\Exp_{p^s}[\D{\Mhat(\pi)}{\Mst(\pi)}].
\end{align*}
As regret is always lower-bounded by $0$, we have $\Delst(p^s) \ge 0$, so for rounds $s \in \cSGLest^1$, we can also write
  \begin{align}\label{eq:algst_explore_case1}
\gst \delta \le \gamma^s \Exp_{\Mhat \sim \xi^s}[\Exp_{p^s}[\D{\Mhat(\pi)}{\Mst(\pi)}].
\end{align}

\paragraph{Case 2: $\Mst \in \cMgl(\lam^s; \nsf^s)$, $\pist \in \pibm[\Mhats]$}
Denote such rounds as $\cSGLest^2$, and write
\begin{align*}
\Delst(p^s) & = \left [ \Delst(p^s) - (1+\delta) \gstar \cdot \Exp_{p^s}[\kl{\Mst(\pi)}{M(\pi)}] \right ] + (1+\delta) \gstar \cdot \Exp_{p^s}[\kl{\Mst(\pi)}{M(\pi)}]
\end{align*}
for an arbitrary model $M \in \cMaltst$.
In this case, since $\Mst \in \cMgl(\lam^s; \nsf^s)$, we have that $\lam^s \in \Lambda(\Mst;\veps)$. This implies that
\begin{align*}
\Delst(\lamGL^s) \le (1 + \veps) \gstar/\betast_s \quad \text{and} \quad \inf_{M \in \cMaltst} \Exp_{\lam^s}[\kl{\Mst(\pi)}{M(\pi)}] \ge (1 - \veps) / \betast_s
\end{align*}
for some $\betast_s \le \nsf^s$. 

Since $M \in \cMaltst$, it follows that $\Exp_{\lam^s}[\kl{\Mst(\pi)}{M(\pi)}] \ge (1-\veps^s)/\betast_s$. Thus,
\begin{align*}
\Delst(p^s) - (1 + \delta) \cst \Exp_{p^s}[\kl{\Mst(\pi)}{M(\pi)}] & \le q^s \left [ \Delst(\lamGL^s) - (1+\delta) \cst \Exp_{\lamGL^s}[\kl{\Mst(\pi)}{M(\pi)}] \right ] + 1 - q^s \\
& \le q^s \left [ (1+\veps) \cst/\betast_s - (1+\delta) (1 - \veps) \cst/\betast_s \right ] + 1 - q^s \\
& = q^s  \left [ 2 \veps  - \delta (1-\veps)  \right ] \cdot \frac{\cst}{\betast_s} + 1 - q^s \\
& \overset{(a)}{=} - \frac{(1-s^{-\alphaq}) \delta}{2} \cdot \frac{\cst}{\betast_s} + s^{-\alphaq} \\
& \overset{(b)}{\le} \begin{cases} s^{-\alphaq} & s < (\frac{8 \betast_s}{\delta \cst} )^{1/\alphaq} \\
-\frac{\delta \cst}{4 \betast_s} & s \ge (\frac{8 \betast_s}{\delta \cst} )^{1/\alphaq} \end{cases} \\
& \overset{(c)}{\le} \begin{cases} s^{-\alphaq} & s < (\frac{8}{\delta \gst} )^{\frac{1}{\alphaq - \alphan}} \\
-\frac{1}{4} \delta \gst \cdot s^{-\alphan} & s \ge(\frac{8}{\delta \gst} )^{\frac{1}{\alphaq - \alphan}} \end{cases} \\
& \overset{(d)}{\le} 2 \gst \bbI \{ s < (\tfrac{8}{\delta \gst} )^{\frac{1}{\alphaq - \alphan}} \} -\frac{1}{4} \delta \gst \cdot s^{-\alphan}
\end{align*}
where $(a)$ follows from our choice of $\veps = \frac{\delta/2}{2 + \delta}$ and $q^s = 1 - s^{-\alphaq}$, $(b)$ follows from some algebra, $(c)$ uses that $\nst_s \le \nsf^s$ and $\nsf^s = s^{\alphan}$, and $(d)$ follows since we will always have $s^{-\alphaq} \le 2 \gst -\frac{1}{4} \delta \gst \cdot s^{-\alphan}$. 
Thus:
\begin{align*}
\Delst(p^s) & \le (1+\delta)\cst \Exp_{p^s}[\kl{\Mst(\pi)}{M(\pi)}] +  2 \gst \bbI \{ s < (\tfrac{8}{\delta \gst} )^{\frac{1}{\alphaq - \alphan}} \} -\frac{1}{4} \delta \gst \cdot s^{-\alphan} \\
& \le (1+\delta)\cst \Exp_{p^s}[\kl{\Mst(\pi)}{M(\pi)}] +  2 \gst \bbI \{ s < (\tfrac{8}{\delta \gst} )^{\frac{1}{\alphaq - \alphan}} \}.
\end{align*}
As this holds for all $M \in \cMaltst$, we have
\begin{align*}
\Delst(p^s) & \le (1+\delta)\cst \cdot \inf_{M \in \cMaltst} \Exp_{p^s}[\kl{\Mst(\pi)}{M(\pi)}] +  2 \gst \bbI \{ s < (\tfrac{8}{\delta \gst} )^{\frac{1}{\alphaq - \alphan}} \}.
\end{align*}
Since $\Delst(p^s) \ge 0$, this also implies that for $s \in \cSGLest^2$:
\begin{align}
\frac{1}{4} \delta \gst & \le s^{\alphan} \cdot 2\cst \cdot \inf_{M \in \cMaltst} \Exp_{p^s}[\kl{\Mst(\pi)}{M(\pi)}] + 2 \gst s^{\alphan} \bbI \{ s < (\tfrac{8}{\delta \gst} )^{\frac{1}{\alphaq - \alphan}} \} \nonumber \\
& \le s^{\alphan} \cdot 2\cst \cdot \inf_{M \in \cMaltst} \Exp_{p^s}[\kl{\Mst(\pi)}{M(\pi)}] + 2 \gst (\tfrac{8}{\delta \gst} )^{\frac{\alphan}{\alphaq - \alphan}}  \cdot \bbI \{ s < (\tfrac{8}{\delta \gst} )^{\frac{1}{\alphaq - \alphan}} \}. \label{eq:algst_explore_case2}
\end{align}

\paragraph{Case 3: $\Mst \in \cMgl(\lam^s; \nsf^s)$, $\pist \not\in \pibm[\Mhats]$}
Denote such rounds as $\cSGLest^3$, and write
\begin{align*}
\Delst(p^s) & = \left [ \Delst(p^s) -2 (1+\delta) \gst \sqrt{\nsf^s} \cdot \Exp_{\Mhat \sim \xi^s}[\Exp_{p^s}[\Dklbig{\Mst(\pi)}{\Mhat(\pi)}]] \right ] \\
& \qquad +2 (1+\delta) \gst \sqrt{\nsf^s} \cdot \Exp_{\Mhat \sim \xi^s}[\Exp_{p^s}[\Dklbig{\Mst(\pi)}{\Mhat(\pi)}]].
\end{align*}
Since $\Mst \in \cMgl(\lam^s)$, we have that $\lam^s \in \Lambda(\Mst;\veps)$. This implies that for any $M \in \cMaltst$:
\begin{align*}
\Delst(\lamGL^s) \le (1 + \veps) \gstar/\betast_s \quad \text{and} \quad \inf_{M \in \cMaltst} \Exp_{\lam^s}[\kl{\Mst(\pi)}{M(\pi)}] \ge (1 - \veps) / \betast_s
\end{align*}
for some $\betast_s \le \nsf^s$. 
Assume that we are at epoch $\ell$. By construction we have that, for $M \in \cM^\ell$, $\frac{1}{\delminm} \le \sqrt{\nsf^{2^\ell}} \iff \delminm \ge \frac{1}{\sqrt{\nsf^{2^\ell}}}$. Since $\nsf^s$ is increasing in $s$, this implies that $\delminm \ge \frac{1}{\sqrt{\nsf^s}}$. As $\xi^s$ is only supported on $\cM^\ell$, since $\pist \not\in \pibm[\Mhats]$, \Cref{lem:Mhat_alt_set_lb} implies that $\Exp_{M \sim \xi^s}[\bbI\{ M \in \cMaltst \}] \ge \frac{1}{2 \sqrt{\nsf^s}}$. Thus, we have
\begin{align*}
2 (1+\delta) \gst \sqrt{\nsf^s} \cdot \Exp_{\Mhat \sim \xi^s}[\Exp_{\lambda^s}[\Dklbig{\Mst(\pi)}{\Mhat(\pi)}]] & \ge 2 (1+\delta) \gst \sqrt{\nsf^s} \cdot \Exp_{\Mhat \sim \xi^s}[\Exp_{\lambda^s}[\Dklbig{\Mst(\pi)}{\Mhat(\pi)} \cdot \bbI \{ M \in \cMaltst \}]] \\
& \ge 2 (1+\delta) \gst \sqrt{\nsf^s} \cdot \frac{1-\veps}{2 \sqrt{\nsf^s} \betast_s} \\
& \ge \frac{ (1+\delta) (1-\veps) \gstar}{\betast_s}.
\end{align*}
This implies that
\begin{align*}
\Delst(p^s) -2 (1+\delta) \gst \sqrt{\nsf^s} \cdot \Exp_{\Mhat \sim \xi^s}[\Exp_{p^s}[\Dklbig{\Mst(\pi)}{\Mhat(\pi)}]]  & \le q^s [ (1+\veps) \gstar/\betast_s - (1+\delta)(1-\veps) \gstar/\betast_s] + 1 - q^s \\
& \le 2 \gst \bbI \{ s < (\tfrac{8}{\delta \gst} )^{\frac{1}{\alphaq - \alphan}} \} -\frac{1}{4} \delta \gst \cdot s^{-\alphan},
\end{align*}
where the final inequality follows by the same argument as in Case 2. Thus, 
\begin{align*}
\Delst(p^s) & \le  2 (1+\delta) \gst \sqrt{\nsf^s} \cdot \Exp_{\Mhat \sim \xi^s}[\Exp_{p^s}[\Dklbig{\Mst(\pi)}{\Mhat(\pi)}]] + 2 \gst \bbI \{ s < (\tfrac{8}{\delta \gst} )^{\frac{1}{\alphaq - \alphan}} \} -\frac{1}{4} \delta \gst \cdot s^{-\alphan} \\
& = s^{\alphan/2} \cdot 2(1+\delta) \gst \Exp_{\Mhat \sim \xi^s}[\Exp_{p^s}[\Dklbig{\Mst(\pi)}{\Mhat(\pi)}]] + 2 \gst \bbI \{ s < (\tfrac{8}{\delta \gst} )^{\frac{1}{\alphaq - \alphan}} \} -\frac{1}{4} \delta \gst \cdot s^{-\alphan} \\
& \le s^{\alphan/2} \cdot 2(1+\delta) \gst \Exp_{\Mhat \sim \xi^s}[\Exp_{p^s}[\Dklbig{\Mst(\pi)}{\Mhat(\pi)}]] +  2 \gst \bbI \{ s < (\tfrac{8}{\delta \gst} )^{\frac{1}{\alphaq - \alphan}} \} .
\end{align*}
Just as in Case 2, using that $\Delst(p^s) \ge 0$, this also implies that for $s \in \cSGLest^3$:
\begin{align}\label{eq:algst_explore_case3}
\frac{1}{4} \delta \gst \le s^{3\alphan/2} \cdot 2(1+\delta) \gst \Exp_{\Mhat \sim \xi^s}[\Exp_{p^s}[\Dklbig{\Mst(\pi)}{\Mhat(\pi)}]] +  2 \gst (\tfrac{8}{\delta \gst} )^{\frac{\alphan}{\alphaq - \alphan}}  \cdot \bbI \{ s < (\tfrac{8}{\delta \gst} )^{\frac{1}{\alphaq - \alphan}} \}.
\end{align}

\paragraph{Completing the Proof}
In total we have
\begin{align*}
\sum_{s=\sst}^{s_T} \Delst(p^s) & \le \sum_{s \in \cSGLest^1} \gamma^s \Exp_{\Mhat \sim \xi^s}[\Exp_{p^s}[\D{\Mhat(\pi)}{\Mst(\pi)}] + \sum_{s \in \cSGLest^2} (1+\delta)\cst \cdot \inf_{M \in \cMaltst} \Exp_{p^s}[\kl{\Mst(\pi)}{M(\pi)}]  \\
& \qquad + \sum_{s \in \cSGLest^3} s^{\alphan/2} \cdot 2(1+\delta) \gst \Exp_{\Mhat \sim \xi^s}[\Exp_{p^s}[\Dklbig{\Mst(\pi)}{\Mhat(\pi)}]] + 4 \gst (\tfrac{8 }{\delta \gst} )^{\frac{1}{\alphaq - \alphan}} .
\end{align*}
By \Cref{lem:gammas_bound_nomingap}, for $s \in \cSGLest^1$, we can bound $\gamma^s \le  (1+ \gst \delta) \cdot s^{\alphaq + \alphaM}$. This gives an upper bound on the above of
\begin{align*}
& \le \sum_{s = \sst}^{s_T} (1+\gst) s^{\alphaq + \alphaM} \Exp_{\Mhat \sim \xi^s}[\Exp_{p^s}[\D{\Mhat(\pi)}{\Mst(\pi)}] + \sum_{s = \sst}^{s_T} (1+\delta)\cst \cdot \inf_{M \in \cMaltst} \Exp_{p^s}[\kl{\Mst(\pi)}{M(\pi)}] \cdot \bbI \{ \pist \in \pibm[\Mhat^s] \} \\
& \qquad + \sum_{s = \sst}^{s_T} s^{\alphan/2} \cdot 4 \gst \Exp_{\Mhat \sim \xi^s}[\Exp_{p^s}[\Dklbig{\Mst(\pi)}{\Mhat(\pi)}]] + 4 \gst(\tfrac{8}{\delta \gst} )^{\frac{1}{\alphaq - \alphan}},
\end{align*}
which proves the regret bound.

We now bound the number of exploration rounds. Since for every $s \in \cSGLest^1$ \eqref{eq:algst_explore_case1} holds, for every $s \in \cSGLest^2$ \eqref{eq:algst_explore_case2} holds, and for every $s \in \cSGLest^3$ \eqref{eq:algst_explore_case3} holds, combining these inequalities gives 
\begin{align*}
\frac{1}{4} \delta \gst & | \cSGLest^1 \cup \cSGLest^2 \cup \cSGLest^3 | \\
& \le \sum_{s \in \cSGLest^1} \gamma^s \Exp_{\Mhat \sim \xi^s}[\Exp_{p^s}[\D{\Mhat(\pi)}{\Mst(\pi)}] + \sum_{s \in \cSGLest^2} s^{\alphan} \cdot 2\cst \cdot \inf_{M \in \cMaltst} \Exp_{p^s}[\kl{\Mst(\pi)}{M(\pi)}]   \\
& \qquad + \sum_{s \in \cSGLest^3} s^{3\alphan/2} \cdot 4 \gst \Exp_{\Mhat \sim \xi^s}[\Exp_{p^s}[\Dklbig{\Mst(\pi)}{\Mhat(\pi)}]] + 4 \gst (\tfrac{8}{\delta \gst} )^{\frac{1+\alphan}{\alphaq - \alphan}}  \\
& \le \sum_{s = \sst}^{s_T} (1+\gst)s^{\alphaq + \alphaM} \Exp_{\Mhat \sim \xi^s}[\Exp_{p^s}[\D{\Mhat(\pi)}{\Mst(\pi)}] \\
& \qquad + \sum_{s = \sst}^{s_T} s^{\alphan} \cdot 2\cst \cdot \inf_{M \in \cMaltst} \Exp_{p^s}[\kl{\Mst(\pi)}{M(\pi)}]  \cdot \bbI \{ \pist \in \pibm[\Mhat^s] \} \\
& \qquad + \sum_{s = \sst}^{s_T} s^{3\alphan/2} \cdot 4  \gst \Exp_{\Mhat \sim \xi^s}[\Exp_{p^s}[\Dklbig{\Mst(\pi)}{\Mhat(\pi)}]] + 4 \gst (\tfrac{8}{\delta \gst} )^{\frac{1+\alphan}{\alphaq - \alphan}} .
\end{align*}
Using that $| \cSGLest^1 \cup \cSGLest^2 \cup \cSGLest^3 | = s_T - \sst$ and rearranging gives the bound on $s_T$.

\end{proof}

\subsubsection{Completing the Proof}
\begin{theorem}[Full version of
  \Cref{thm:upper_main_no_mingap}]\label{thm:upper_main_no_mingap_app}
  Consider \Cref{alg:gl_alg_main_infinite_nomingap}, and suppose we
  set $\delta \le 1/2$, $\D{\cdot}{\cdot} = \Dkl{\cdot}{\cdot}$,
  $\alphaM = 1/2$, $\alphazeta = 1/8$, and $\alphaq = 1/4$, and
  instantiate $\AlgEstD$ with \Cref{alg:tempered_aggregation}. Then if \Cref{ass:realizability,asm:bounded_means,ass:unique_opt,asm:smooth_kl_kl,asm:bounded_likelihood,asm:covering} hold and 
  $\gst > 0$, the expected regret of is bounded by
\begin{align*}
\Exp\sups{\Mst}[\RegDM]  & \le (1+\delta) \gst \log T + \Caec \cdot \prn*{\aecflipDM{\veps/2}{\cM}(\cMLdelst)}^3 \cdot \log^{3/2} \log T + \Clo \cdot \log^{6/7} T
\end{align*}
for $\veps \leftarrow \frac{\delta}{4 + 2 \delta}$, 
\begin{align*}
\Caec := \bigoht \prn*{\frac{  (\VM+\LKL)\VM^3 \dcov \log(\Ccov)}{\delta \delminst}},
\end{align*}
and
\begin{align*}
\Clo := \bigoht \prn*{ \poly \prn*{\VM,\LKL,\nepsst,\dcov,\log \Ccov, \gst, \frac{1}{\delminst}, \frac{1}{\delta}, \log \log T}}.
\end{align*}
\end{theorem}
\begin{proof}[Proof of \Cref{thm:upper_main_no_mingap_app}]
The bound on the regret incurred in the exploit phase follows
identically to \Cref{thm:regret_bound_mingap2}, since the exploit test
is the same. We turn to bounding the regret in the explore phase. First, since we can incur regret of at most 1 at every round, we bound
\begin{align*}
\Exp \brk*{ \sum_{s=1}^{s_T} \delm[\Mst](p^s) } & \le \Exp \brk*{ \sum_{s=\sst}^{s_T} \delm[\Mst](p^s) } + \sst
\end{align*}

By \Cref{lem:exploration_regret_infinite_nomingap}, we can bound 
\begin{align*}
\Exp \bigg [ \sum_{s=\sst}^{s_T} \Delst(\pi^s) \bigg ]  & \le \Exp \bigg [  \sum_{s = \sst}^{s_T} (1+\gst) s^{\alphaq + \alphaM} \Exp_{\Mhat \sim \xi^s}[\Exp_{p^s}[\D{\Mhat(\pi)}{\Mst(\pi)}]  \\
& \qquad +\sum_{s = \sst}^{s_T} (1+\delta)\cst \cdot \inf_{M \in \cMaltst} \Exp_{p^s}[\kl{\Mst(\pi)}{M(\pi)}] \cdot \bbI \{ \pist \in \pibm[\Mhat^s] \} \\
& \qquad + \sum_{s = \sst}^{s_T} s^{\alphan/2} \cdot 4 \gst \Exp_{\Mhat \sim \xi^s}[\Exp_{p^s}[\Dklbig{\Mst(\pi)}{\Mhat(\pi)}]] \bigg ] + 4 \gst(\tfrac{8}{\delta \gst} )^{\frac{1}{\alphaq - \alphan}} \\
& \le \Exp \bigg [  \sum_{s = \sst}^{s_T} (1+\gst)s^{\alphaq + \alphaM} \Exp_{\Mhat \sim \xi^s}[\Exp_{p^s}[\D{\Mhat(\pi)}{\Mst(\pi)}] \\
&\qquad + \sum_{s = 1}^{s_T} (1+\delta)\cst \cdot \inf_{M \in \cMaltst} \Exp_{p^s}[\kl{\Mst(\pi)}{M(\pi)}] \cdot \bbI \{ \pist \in \pibm[\Mhat^s] \} \\
& \qquad + \sum_{s = \sst}^{s_T} s^{\alphan/2} \cdot 4 \gst \Exp_{\Mhat \sim \xi^s}[\Exp_{p^s}[\Dklbig{\Mst(\pi)}{\Mhat(\pi)}]] \bigg ] + 4 \gst (\tfrac{8}{\delta \gst} )^{\frac{1}{\alphaq - \alphan}}.
\end{align*}
Applying \Cref{lem:info_gain_bound} with $\alpha = 0$, we have
\begin{align*}
\Exp \bigg [ \sum_{s = 1}^{s_T} & \inf_{M \in \cMaltst} \Exp_{p^s}[\kl{\Mst(\pi)}{M(\pi)}] \cdot \bbI \{ \pist \in \pibm[\Mhat^s] \}  \bigg ]  \le \log T  + \Exp \bigg [ \VM s_T^{1/2} \cdot \sqrt{1344 \dcov \cdot \log \prn*{128 \Ccov s_T}}   \\
& \qquad + (\VM + \LKL  ) \bigg (4s_T^{1/2}  +  \sum_{s=1}^{s_T} s^{1/2}  \cdot \Exp_{\Mhat \sim \xi^s}[\Exp_{\pi \sim p^s}[\D{\Mhat(\pi)}{\Mst(\pi)} ]] \bigg )   \bigg ] +  \log \log T  + 7\VM .
\end{align*}
Note that for $s \ge \sst$, our estimator is applied to a cover of a set containing $\Mst$. Furthermore, note that the estimation procedure used in \Cref{alg:gl_alg_main_infinite_nomingap} is, other than the different choice of set to cover, identical to that used in \Cref{alg:estimator}. It follows that the analysis of \Cref{alg:estimator} can be applied to the estimation procedure of \Cref{alg:gl_alg_main_infinite_nomingap}, with only the mild modification accounting for the difference in the size of the cover (as we are covering $\cM^\ell$ instead of $\cM$). However, as we can upper bound the size of the cover of $\cM^\ell$ by the size of the cover of $\cM$ via \Cref{lem:cover_equiv}, this change is inconsequential.
Thus, by \Cref{lem:sT_est_exp_bound}, we can bound
\begin{align}
\Exp \brk*{\sum_{s = \sst}^{s_T} 2s^{\alphaq + \alphaM} \Exp_{\Mhat \sim \xi^s}[\Exp_{p^s}[\D{\Mhat(\pi)}{\Mst(\pi)}]} & \le \bigoh \prn*{ \Exp \brk*{\VM \dcov \log(\Ccov s_T) s_T^{\alphaq + \alphaM} \sqrt{\log(s_T)}}} \notag\\
& \le  \bigoh \prn*{\VM \dcov \log(\Ccov \Exp[s_T]) \Exp[s_T]^{\alphaq + \alphaM} \sqrt{\log(\Exp[s_T])}},\label{eq:full_est0}
\end{align}
where the second inequality uses Jensen's inequality and \Cref{claim:concave_fun} to pass the expectation through, which holds as long as $1/100 \le \alphaq + \alphaM \le 3/4$. Similarly, 
\begin{align}
& \Exp \brk*{\sum_{s = \sst}^{s_T} s^{\alphan/2} \cdot 4 \gst \Exp_{\Mhat \sim \xi^s}[\Exp_{p^s}[\Dklbig{\Mst(\pi)}{\Mhat(\pi)}]]} \le \bigoh \prn*{ \gst \VM \dcov \log(\Ccov \Exp[s_T]) \Exp[s_T]^{\alphan/2} \sqrt{\log(\Exp[s_T])}} ,\label{eq:full_est1}
\end{align}
where we have again used Jensen's inequality and \Cref{claim:concave_fun} to pass the expectation through, which holds as long $1/100 \le \alphan/2 \le 3/4$.
For $s \le \sst$, we do not have $\Mst \in \cM^\ell$, and therefore the estimation guarantees are vacuous. In this regime, using that the KL divergence is always bounded by $2\VM$ (\Cref{lem:kl_bound}), we can simply upper bound the estimation error by $2\VM$. Thus, 
\iftoggle{colt}{\begin{align*}
\Exp & \brk*{\sum_{s=1}^{s_T} s^{1/2}  \cdot \Exp_{\Mhat \sim \xi^s}[\Exp_{\pi \sim p^s}[\D{\Mhat(\pi)}{\Mst(\pi)} ]] } \\
& \le \Exp \brk*{\sum_{s=\sst}^{s_T} s^{1/2}  \cdot \Exp_{\Mhat \sim \xi^s}[\Exp_{\pi \sim p^s}[\D{\Mhat(\pi)}{\Mst(\pi)} ]] } + 2\VM (\sst)^{3/2} \\
& \le \bigoh \prn*{\VM \dcov \log(\Ccov \Exp[s_T]) \Exp[s_T]^{1/2} \sqrt{\log(\Exp[s_T])}} + 2\VM (\sst)^{3/2}
\end{align*}}{
\begin{align*}
\Exp \brk*{\sum_{s=1}^{s_T} s^{1/2}  \cdot \Exp_{\Mhat \sim \xi^s}[\Exp_{\pi \sim p^s}[\D{\Mhat(\pi)}{\Mst(\pi)} ]] } & \le \Exp \brk*{\sum_{s=\sst}^{s_T} s^{1/2}  \cdot \Exp_{\Mhat \sim \xi^s}[\Exp_{\pi \sim p^s}[\D{\Mhat(\pi)}{\Mst(\pi)} ]] } + 2\VM (\sst)^{3/2} \\
& \le \bigoh \prn*{\VM \dcov \log(\Ccov \Exp[s_T]) \Exp[s_T]^{1/2} \sqrt{\log(\Exp[s_T])}} + 2\VM (\sst)^{3/2},
\end{align*}}
which gives, applying \Cref{claim:concave_fun} again:
\begin{align}
  \Exp \bigg [ \sum_{s = 1}^{s_T} & \inf_{M \in \cMaltst}
                                    \Exp_{p^s}[\kl{\Mst(\pi)}{M(\pi)}]
                                    \cdot \bbI \{ \pist \in
                                    \pibm[\Mhat^s] \}  \bigg ] \leq
                                   \log T + \bigoh \bigg ( \VM \Exp[s_T]^{1/2}  \sqrt{\dcov \cdot \log(\Ccov \Exp[s_T])} \nonumber \\
&   + (\VM + \LKL) \prn*{ \Exp[s_T]^{1/2}  +  \VM \dcov \log(\Ccov \Exp[s_T]) \Exp[s_T]^{1/2} \sqrt{\log(\Exp[s_T])} } \bigg )   \nonumber \\
&    +                         2\VM (\sst)^{3/2} + \log \log T + 7 \VM.
  \label{eq:full_est2}
\end{align}
Combining \cref{eq:full_est0,eq:full_est1,eq:full_est2}, we have
\begin{align*}
\Exp \bigg [ \sum_{s=\sst}^{s_T} \Delst(\pi^s) \bigg ]  \le (1+\delta)\gst \log T &+ \bigoh \bigg (\VM \sqrt{\dcov \Exp[s_T] \log(\Ccov \Exp[s_T])} \\
& \quad\qquad+ (\VM+\LKL) \VM \dcov \log(\Ccov \Exp[s_T]) \Exp[s_T]^{1/2} \sqrt{\log(\Exp[s_T])}  \\
& \quad\qquad + \VM \dcov \log(\Ccov \Exp[s_T]) \sqrt{\log \Exp[s_T]}((1+\gst)\Exp[s_T]^{\alphaq+\alphaM} + \veps \gst \Exp[s_T]^{\alphan/2}) \\
& \quad\qquad+ \gst (\tfrac{1 }{\delta \gst} )^{\frac{1}{\alphaq - \alphan}} + \log \log T + \VM(\sst)^{3/2} \bigg ).
\end{align*}
To control this, it remains to bound $s_T$. By
\Cref{lem:exploration_regret_infinite_nomingap} we have the
  following almost sure bound:
\begin{align*}
s_T & \le \frac{4}{\delta \gst} \bigg ( \underbrace{ \sum_{s = \sst}^{s_T} (1+\gst)s^{\alphaq + \alphaM} \Exp_{\Mhat \sim \xi^s}[\Exp_{p^s}[\D{\Mhat(\pi)}{\Mst(\pi)}]}_{(a)} \\
& \qquad + \underbrace{\sum_{s = \sst}^{s_T} s^{\alphan} \cdot 2\cst \cdot \inf_{M \in \cMaltst} \Exp_{p^s}[\kl{\Mst(\pi)}{M(\pi)}]  \cdot \bbI \{ \pist \in \pibm[\Mhat^s] \}}_{(b)} \\
& \qquad + \underbrace{\sum_{s = 1}^{s_T} s^{3\alphan/2} \cdot 4 \gst \Exp_{\Mhat \sim \xi^s}[\Exp_{p^s}[\Dklbig{\Mst(\pi)}{\Mhat(\pi)}]]}_{(c)} + 4 \gst (\tfrac{4}{\delta \gst} )^{\frac{1 + \alphan}{\alphaq - \alphan}} \bigg ) + \sst.
\end{align*}
We bound the expectation of term $(a)$ as in \eqref{eq:full_est0}. To bound term $(b)$, we apply \Cref{lem:info_gain_bound} with $\alpha = \alphan$ to get
\begin{align*}
\Exp[(b)] & \le \Exp[ s_T^{\alphan}] \log T  + \Exp \bigg [ \VM s_T^{1/2+\alphan} \cdot \sqrt{1344 \dcov \cdot \log \prn*{128 \Ccov s_T}}   \\
& \qquad + (\VM + \LKL  ) \bigg (4 \frac{s_T^{1/2+\alphan/2}}{1-\alphan} +  \sum_{s=1}^{s_T} s^{1/2+\alphan/2}  \cdot \Exp_{\Mhat \sim \xi^s}[\Exp_{\pi \sim p^s}[\D{\Mhat(\pi)}{\Mst(\pi)} ]] \bigg )   \bigg ] + \Exp[ s_T^{\alphan}] \log \log T  + 7\VM .
\end{align*}
Again applying \Cref{lem:sT_est_exp_bound} and \Cref{claim:concave_fun}, we have
\begin{align*}
& \Exp[(c)] \le  \bigoh \prn*{\gst \VM \dcov \log(\Ccov \Exp[s_T]) \Exp[s_T]^{3\alphan/2} \sqrt{\log(\Exp[s_T])}} \end{align*}
and
\begin{align*}
\Exp \brk*{\sum_{s=1}^{s_T} s^{1/2+\alphan/2}  \cdot \Exp_{\Mhat \sim \xi^s}[\Exp_{\pi \sim p^s}[\D{\Mhat(\pi)}{\Mst(\pi)} ]]} & \le \bigoh \prn*{\VM \dcov \log(\Ccov \Exp[s_T]) \Exp[s_T]^{1/2+\alphan/2} \sqrt{\log(\Exp[s_T])}} \\
& \quad\qquad + 2\VM (\sst)^{3/2 + \alphan/2} ,
\end{align*}
as long as $1/100 \le \alphan \le 1/4$. We therefore have (using \Cref{claim:concave_fun} to pass the expectation through):
\begin{align*}
\Exp[s_T] & \le \frac{1}{\delta \gst} \cdot \bigoh \bigg ( \Exp[s_T]^{\alphan} \log T + \VM \Exp[s_T]^{1/2+\alphan} \sqrt{\dcov \log(\Ccov \Exp[s_T])} + (\VM+\LKL) \Exp[s_T]^{1/2+\alphan/2} \\
& \qquad + (\VM+\LKL)\VM \dcov \log(\Ccov \Exp[s_T]) \Exp[s_T]^{1/2+\alphan/2} \sqrt{\log(\Exp[s_T])} \\
& \qquad + \VM \dcov \log(\Ccov \Exp[s_T])(1+\gst) \Exp[s_T]^{\alphaq + \alphaM} \sqrt{\log(\Exp[s_T])} \\
& \qquad + \gst \VM \dcov \log(\Ccov \Exp[s_T]) \Exp[s_T]^{3\alphan/2} \sqrt{\log(\Exp[s_T])} + \gst (\tfrac{1 }{\delta \gst} )^{\frac{1 + \alphan}{\alphaq - \alphan}} + \VM (\sst)^{3/2 + \alphan/2}  \bigg ).
\end{align*}
We now set $\alphaM = 1/2$, $\alphan = 1/8$,
and $\alphaq = 1/4$, and note that all of the preceding parameter
restrictions are satisfied for
these choices. Furthermore, this parameter choice implies that---using
\Cref{lem:lin_log_ineq} to handle log factors---we have
\begin{align*}
\Exp[s_T] \le \bigoht \bigg ( \frac{1}{(\delta \gst)^{8/7}} \log^{8/7} T + \poly \prn*{\VM,\dcov,\log \Ccov, \LKL, \gst, \frac{1}{\delta}, \frac{1}{\gst}} + \frac{\VM (\sst)^{3/2 + \alphan/2} }{\delta \gst} \bigg ) .
\end{align*}
Plugging this into the regret bound given above, we have
\begin{align*}
\Exp \bigg [ \sum_{s=\sst}^{s_T} \Delst(\pi^s) \bigg ] & \le (1+\delta) \gst \log T + \bigoht \bigg ( \poly \prn*{\VM,\dcov,\log \Ccov, \LKL, \gst, \frac{1}{\delta}, \frac{1}{\gst}, \log \log T} \cdot \log^{6/7} T \\
& \qquad + \frac{(1 + 1/\gst) (\VM+\LKL)\VM^2 \dcov \log(\Ccov) + \VM}{\delta} \cdot (\sst)^{3/2} \cdot \log^{3/2} \log T \bigg ) .
\end{align*}
Finally, by \Cref{lem:sst_bound}, we can bound $\sst$ as
\begin{align*}
\sst \le \prn*{\aecflipDM{\veps/2}{\cM}(\cMLdelst)}^{\frac{1}{\alphaM}} + \prn*{ \nepsst + \frac{1}{\delminst} + \frac{4 \gst}{\delminst} + \frac{2 \nepsst}{\gst} + \frac{4}{\delminst \veps} }^{\frac{2}{\alphan}}  ,
\end{align*}
and, by \Cref{lem:gm_lb} and \Cref{lem:kl_bound}, we can lower bound $\gst \ge \delminst / 2 \VM$.
Plugging this in gives the final bound.
\end{proof}

\subsection{Estimation Guarantees}\label{sec:upper_est_proofs}

\begin{algorithm}[h]
\caption{Estimation with Adaptive Covering}
\begin{algorithmic}[1]
\State \textbf{input:} Class $\cM$.
\State $\ell \leftarrow 1$, $\frakD^\ell \leftarrow \emptyset$.
\State $\cMcov^\ell \leftarrow$ $(\rho_\ell,\mu_\ell)$-cover of $\cM$
for $\rho_\ell \leftarrow 2^{-\ell}, \mu_\ell \leftarrow 2^{-5\ell}$.
\State Initialize $\mathsf{TemperedAggregation}^{\ell}$ as an instance
of \cref{alg:tempered_aggregation} with $\cMcov^\ell$.
\For{$s=1,2,3,\ldots$}
	\State Receive $(\pi^s,r^s,o^s)$, $\frakD^\ell \leftarrow \frakD^\ell \cup \{ (\pi^s,r^s,o^s) \}$.
	\State $\xi^s \leftarrow \mathsf{TemperedAggregation}^\ell(\frakD^\ell)$, $\Mhat^s \leftarrow \Exp_{M \sim \xi^s}[M]$
	\If{$s \ge 2^\ell$}.
\State $\ell \leftarrow \ell + 1$, $\frakD^\ell \leftarrow \emptyset$.
\State $\cMcov^\ell \leftarrow$ $(\rho_\ell,\mu_\ell)$-cover of $\cM$
for $\rho_\ell \leftarrow 2^{-\ell}, \mu_\ell \leftarrow 2^{-5\ell}$.
\State Initialize $\mathsf{TemperedAggregation}^{\ell}$ as an instance
of \cref{alg:tempered_aggregation} with $\cMcov^\ell$.
\EndIf
\EndFor
\end{algorithmic}
\label{alg:estimator}
\end{algorithm}

In this section, we analyze \cref{alg:estimator}, which is a variant of the Tempered Aggregation
algorithm (\Cref{alg:tempered_aggregation}) designed for infinite
classes. This algorithm is used within
\Cref{alg:gl_alg_main_general_D} in order to prove \cref{thm:regret_bound_mingap2}.

\cref{alg:estimator} simply applies \cref{alg:tempered_aggregation} to
a sequence of covers for the class $\cM$ on a doubling epoch
schedule. In particular, at every epoch $\ell$, \Cref{alg:estimator}
restarts the Tempered Aggregation algorithm
(\Cref{alg:tempered_aggregation}), clearing the Tempered Aggregation
instance from the previous epoch from memory. We denote the $\ell$th instantiation of Tempered Aggregation as $\mathsf{TemperedAggregation}^\ell$.

\begin{lemma}\label{lem:sT_est_exp_bound}
Let $\tau$ denote some stopping time with respect to the filtration $(\Hsig\ind{t})_{t=1}^T$ such that $\tau \le T$ almost surely, and let $\alpha \in (0,1)$.
When running \Cref{alg:estimator} under \Cref{asm:covering}, we have
\begin{align*}
\Exp \brk*{\sum_{s=1}^{\tau} s^\alpha \cdot \Exp_{M \sim \xi^s} [ \Exp_{\pi \sim p^s} [ \Dhels{\Mst(\pi)}{M(\pi)} ] ]} \le \Exp \brk*{\prn*{20 \dcov \cdot \log(64 \Ccov \tau) + 1} \cdot \frac{2^{2\alpha}}{2^\alpha - 1} \tau^\alpha} + 4.
\end{align*}
In addition, if
\Cref{asm:bounded_likelihood} also holds, then
\iftoggle{colt}{\begin{align*}
\Exp&  \brk*{\sum_{s=1}^{\tau} s^{\alpha} \cdot \Exp_{M \sim \xi^s} [ \Exp_{\pi \sim p^s} [ \Dkl{\Mst(\pi)}{M(\pi)} ] ]} \\
& \le \Exp \brk*{(2+6\VM) \prn*{16 \dcov \cdot \log(32 \Ccov \tau) + 1} \cdot \frac{2^{2\alpha}}{2^\alpha - 1} \tau^\alpha  \sqrt{\log(2 \tau)}}  + \Exp \brk*{32(1+\VM) \log(\tau) }+ 4
\end{align*}
and
\begin{align*}
\Exp &\brk*{\sum_{s=1}^{\tau} s^{\alpha} \cdot \Exp_{M \sim \xi^s} [ \Exp_{\pi \sim p^s} [ \Dkl{M(\pi)}{\Mst(\pi)} ] ]} \\
& \le \Exp \brk*{(2+6\VM) \prn*{16 \dcov \cdot \log(32 \Ccov \tau) + 1} \cdot \frac{2^{2\alpha}}{2^\alpha - 1} \tau^\alpha  \sqrt{\log(2 \tau)}}  + \Exp \brk*{32(1+\VM) \log(\tau) }+ 4.
\end{align*}}{
\begin{align}
\Exp \brk*{\sum_{s=1}^{\tau} s^{\alpha} \cdot \Exp_{M \sim \xi^s} [ \Exp_{\pi \sim p^s} [ \Dkl{\Mst(\pi)}{M(\pi)} ] ]} & \le \Exp \brk*{(2+6\VM) \prn*{20 \dcov \cdot \log(64 \Ccov \tau) + 1} \cdot \frac{2^{2\alpha}}{2^\alpha - 1} \tau^\alpha  \sqrt{\log(2 \tau)}} \notag\\
& \qquad + \Exp \brk*{32(1+\VM) \log(\tau) }+ 8\label{eq:est_kl1}
\end{align}
and
\begin{align}
\Exp \brk*{\sum_{s=1}^{\tau} s^{\alpha} \cdot \Exp_{M \sim \xi^s} [ \Exp_{\pi \sim p^s} [ \Dkl{M(\pi)}{\Mst(\pi)} ] ]} & \le \Exp \brk*{(2+6\VM) \prn*{20 \dcov \cdot \log(64 \Ccov \tau) + 1} \cdot \frac{2^{2\alpha}}{2^\alpha - 1} \tau^\alpha  \sqrt{\log(2 \tau)}} \notag\\
& \qquad + \Exp \brk*{32(1+\VM) \log(\tau) }+ 8.\label{eq:est_kl2}
\end{align}}
\end{lemma}
\begin{proof}[Proof of \Cref{lem:sT_est_exp_bound}]
Let $\cS^k$ denote the set of $s$ values for which $\ell = k$ and note that $\cS^k = \{ 2^{k-1} + 1, 2^{k-1} + 2, \ldots, 2^k \}$. We can bound
\begin{align}
& \Exp \brk*{\sum_{s=1}^{\tau} s^\alpha \Exp_{M \sim \xi^s} [ \Exp_{\pi \sim p^s} [ \Dhels{\Mst(\pi)}{M(\pi)} ] ]}  \le  \sum_{k=1}^{\lceil \log_2 T \rceil}  \Exp \brk*{  \sum_{s \in \cS^k}s^\alpha \cdot \Exp_{M \sim \xi^s} [ \Exp_{\pi \sim p^s} [ \Dhels{\Mst(\pi)}{M(\pi)} ] ] \cdot \bbI \{ s \le \tau \} } \nonumber \\
& \qquad \qquad \qquad \qquad \qquad \le  \sum_{k=1}^{\lceil \log_2 T \rceil} 2^{\alpha k}  \cdot  \Exp \brk*{   \sum_{s \in \cS^k}  \Exp_{M \sim \xi^s} [ \Exp_{\pi \sim p^s} [ \Dhels{\Mst(\pi)}{M(\pi)} ] ] \cdot \bbI \{ s \le \tau \}} \label{eq:est_salpha_inter_bound1}
\end{align}
since we have that $\ell \le \lceil \log_2 T \rceil$ by construction, and since $s \le 2^k$ for $s \in \cS^k$. Let $A_k$ denote the event
\begin{align*}
A_k := \left \{ \forall s \in \cS^k \ : \  \sum_{i=2^{k-1}}^s \Exp_{M \sim \xi^i} [ \Exp_{\pi \sim p^i} [ \Dhels{\Mst(\pi)}{M(\pi)} ] ] \le 2 \log \frac{2^k \Ncov(\cM,\rho_k,\mu_k)}{\delta_k} + 2^k \rho_k \right \}.
\end{align*}
By \Cref{prop:temp_agg_infinite} and a union bound, $\Pr[A_k] \ge 1 - \delta_k - 2^{2k}
\mu_k$. 
Since the Hellinger distance is always bounded by 2 and $|\cS^k| \le 2^k$, we can upper bound
\begin{align*}
\text{\eqref{eq:est_salpha_inter_bound1}} & \le \sum_{k=1}^{\lceil \log_2 T \rceil} 2^{\alpha k} \prn*{\sum_{s \in \cS^k} \Exp \brk*{\Exp_{M \sim \xi^s} [ \Exp_{\pi \sim p^s} [ \Dhels{\Mst(\pi)}{M(\pi)} ] ] \cdot \bbI \{ s \le \tau, A_k \}} + 2 \cdot 2^k \Exp[\bbI \{ A_k^c \}]}.
\end{align*}
Choosing $\delta_k = 2^{-3k}$ and since $\mu_k = 2^{-5k}$, we have
\begin{align*}
\sum_{k=1}^{\lceil \log_2 T \rceil} 2^{\alpha k} \cdot 2^{k+1} \Exp[\bbI \{ A_k^c \}] & \le \sum_{k=1}^{\lceil \log_2 T \rceil} 2^{2k+1} \cdot (\delta_k + 2^{2k} \mu_k)  = \sum_{k=1}^{\lceil \log_2 T \rceil} 2^{2k+1} \cdot 2 \cdot 2^{-3k} \le 4.
\end{align*}
Note that for $s \in \cS^k$, if $s \le \tau$, this implies that $2^{k-1} \le \tau$.
On the event $A_k$, for $\rho_k = 2^{-k}$, when $\alpha > 0$ we can bound
\begin{align}\label{eq:est_bound_alphage0}
\begin{split}
& \sum_{k=1}^{\lceil \log_2 T \rceil} 2^{\alpha k}  \sum_{s \in \cS^k} \Exp \brk*{\Exp_{M \sim \xi^s} [ \Exp_{\pi \sim p^s} [ \Dhels{\Mst(\pi)}{M(\pi)} ] ] \cdot \bbI \{ s \le \tau, A_k \}} \\
& \qquad \le \Exp \brk*{\sum_{k=1}^{\lceil \log_2 T \rceil} 2^{\alpha k}  \prn*{2 \log \frac{2^k \Ncov(\cM,2^{-k},2^{-5k})}{2^{-3k}} + 1} \cdot \bbI \{ 2^{k-1} \le \tau \}} \\
& \qquad \le \Exp \brk*{\prn*{2 \log \frac{\Ncov(\cM,\tau^{-1}/2,\tau^{-5}/32)}{\tau^{-4}/16} + 1} \cdot \max_k ( \frac{2^\alpha}{2^{\alpha} - 1} 2^{\alpha k}  \cdot \bbI \{ 2^{k-1} \le \tau \}) } \\
& \qquad \le \Exp \brk*{\prn*{2 \log \frac{\Ncov(\cM,\tau^{-1}/2,\tau^{-5}/32)}{\tau^{-4}/16} + 1} \cdot \frac{2^{2\alpha}}{2^\alpha - 1} \tau^\alpha  }
\end{split}
\end{align}
where the final two inequalities follow since $2^k \le 2 \tau$. 
Under \Cref{asm:covering} we have
\begin{align*}
\log \frac{\Ncov(\cM,\tau^{-1}/2,\tau^{-5}/32)}{\tau^{-4}/16} \le 10\dcov \cdot \log (64 \Ccov \tau),
\end{align*}
which gives the first result.

\paragraph{Bound on KL Estimation Error}
By \Cref{lem:hel_to_kl}, for any $x > 0$ we have
\begin{align*}
\Dklbig{M(\pi)}{\Mtil(\pi)} \le (2 + 2\VM + x) \cdot \Dhelsbig{M(\pi)}{\Mtil(\pi)} + 32(1 + \VM^2/x + \VM^3/x^2) \cdot \exp(-x^2/8\VM^2).
\end{align*}
In particular choosing $x = \VM \sqrt{8\log s^2}$, we have
\begin{align*}
\Dklbig{M(\pi)}{\Mtil(\pi)} \le (2 + 2\VM + \VM \sqrt{8\log s}) \cdot \Dhelsbig{M(\pi)}{\Mtil(\pi)} + 32(1 + \VM ) / s.
\end{align*}
Repeating this for each step $s$, we can therefore bound
\iftoggle{colt}{\begin{align*}
\Exp & \brk*{\sum_{s=1}^{\tau} s^\alpha \Exp_{M \sim \xi^s} [ \Exp_{\pi \sim p^s} [ \Dkl{\Mst(\pi)}{M(\pi)} ] ]} \\
& \le (2+6\VM) \cdot \Exp \brk*{\sum_{s=1}^{\tau} s^\alpha \sqrt{\log s} \Exp_{M \sim \xi^s} [ \Exp_{\pi \sim p^s} [ \Dhels{\Mst(\pi)}{M(\pi)} ] ]}  + 32(1+\VM) \cdot \Exp[\log \tau].
\end{align*}}{
\begin{align*}
\Exp \brk*{\sum_{s=1}^{\tau} s^\alpha \Exp_{M \sim \xi^s} [ \Exp_{\pi \sim p^s} [ \Dkl{\Mst(\pi)}{M(\pi)} ] ]} & \le (2+6\VM) \cdot \Exp \brk*{\sum_{s=1}^{\tau} s^\alpha \sqrt{\log s} \Exp_{M \sim \xi^s} [ \Exp_{\pi \sim p^s} [ \Dhels{\Mst(\pi)}{M(\pi)} ] ]} \\
& \qquad + 32(1+\VM) \cdot \Exp[\log \tau].
\end{align*}}
The result in \cref{eq:est_kl1} then follows from a calculation nearly identical to our
above bound on Hellinger estimation error. Applying
\Cref{lem:hel_to_kl} in a similar fashion with the arguments flipped
gives \cref{eq:est_kl2}.
\end{proof}

In the following, we extend \Cref{lem:sT_est_exp_bound} to the case when $\alpha = 0$. 

\begin{lemma}\label{lem:est_exp_bound}
Let $\tau$ denote some stopping time with respect to the filtration $(\Hsig\ind{t})_{t=1}^T$ such that $\tau \le T$ almost surely.
When running \Cref{alg:estimator}, under \Cref{asm:covering}, we have
\begin{align*}
\Exp \brk*{\sum_{s=1}^{\tau}  \Exp_{M \sim \xi^s} [ \Exp_{\pi \sim p^s} [ \Dhels{\Mst(\pi)}{M(\pi)} ] ]} \le \Exp \brk*{\prn*{20 \dcov \cdot \log(64 \Ccov \tau) + 1} \cdot \prn*{2 \log \tau + 1}} + 4.
\end{align*}
In addition, if \Cref{asm:bounded_likelihood} also holds,
\begin{align*}
\Exp \brk*{\sum_{s=1}^{\tau}  \Exp_{M \sim \xi^s} [ \Exp_{\pi \sim p^s} [ \Dkl{\Mst(\pi)}{M(\pi)} ] ]} & \le \Exp \brk*{(2+5\VM) \prn*{20 \dcov \cdot \log(64 \Ccov \tau) + 1}  \prn*{2 \log \tau + 1} \cdot \sqrt{ \log(2 \tau)}} \\
& \qquad + \Exp \brk*{32(1+\VM) \log(\tau)} + 8
\end{align*}
and
\begin{align*}
\Exp \brk*{\sum_{s=1}^{\tau}  \Exp_{M \sim \xi^s} [ \Exp_{\pi \sim p^s} [ \Dkl{M(\pi)}{\Mst(\pi)} ] ]} & \le \Exp \brk*{(2+5\VM) \prn*{20 \dcov \cdot \log(64 \Ccov \tau) + 1} \prn*{2 \log \tau + 1} \cdot \sqrt{ \log(2 \tau)}} \\
& \qquad + \Exp \brk*{32(1+\VM) \log(\tau)} + 8.
\end{align*}
\end{lemma}
\begin{proof}[Proof of \Cref{lem:est_exp_bound}]
This follows identically to \Cref{lem:sT_est_exp_bound} but replacing \eqref{eq:est_bound_alphage0} with
\begin{align*}
& \sum_{k=1}^{\lceil \log_2 T \rceil} 2^{\alpha k}  \sum_{s \in \cS^k} \Exp \brk*{\Exp_{M \sim \xi^s} [ \Exp_{\pi \sim p^s} [ \Dhels{\Mst(\pi)}{M(\pi)} ] ] \cdot \bbI \{ s \le \tau, A_k \}} \\
& \qquad \le \Exp \brk*{\sum_{k=1}^{\lceil \log_2 T \rceil}   \prn*{2 \log \frac{2^k \Ncov(\cM,2^{-k},2^{-5k})}{2^{-3k}} + 1} \cdot \bbI \{ 2^{k-1} \le \tau \}} \\
& \qquad \le \Exp \brk*{\prn*{2 \log \frac{\Ncov(\cM,\tau^{-1}/2,\tau^{-5}/32)}{\tau^{-4}/16} + 1} \cdot (2 \log \tau + 1) } .
\end{align*}

The bound on the KL estimation error also follows from the same reasoning as in \Cref{lem:sT_est_exp_bound}.

\end{proof}

\subsection{Supporting Lemmas}\label{sec:upper_misc_results}

\begin{lemma}\label{lem:Mhat_alt_set_lb}
Consider running either \Cref{alg:gl_alg_main_general_D} or \Cref{alg:gl_alg_main_infinite_nomingap}. Assume that $\pist \not\in \pibm[\Mhat^s]$ and that $\min_{M \in \cM \
  : \ \xi^s(M) > 0} \delminm \ge \Delta$. Then $\Exp_{M \sim
  \xi^s}[\bbI \{ M \in \cMaltst \} ] \ge \frac{1}{2}
\Delta$. 
\end{lemma}
\begin{proof}[Proof of \Cref{lem:Mhat_alt_set_lb}]
Recall that $\Mhat^s = \Exp_{M \sim \xi^s}[M]$, so $\pi \in \pibm[\Mhat^s]$ implies that $\pi \in \argmax_{\pi' \in \Pi} \Exp_{M \sim \xi^s}[\fm(\pi')]$. If $\pist \not\in \pibm[\Mhat^s]$, then there exists some $\pitil$ such that $\Exp_{M \sim \xi^s}[\fm(\pitil)] > \Exp_{M \sim \xi^s}[\fm(\pist)]$. Since $\fm(\pi) \in [0,1]$, this implies that
\begin{align*}
0 & < \Exp_{M \sim \xi^s}[\fm(\pitil) - \fm(\pist)]  \le \Exp_{M \sim \xi^s}[\bbI \{ M \in \cMaltst \}] -  \Exp_{M \sim \xi^s}[(\fm(\pist) - \fm(\pitil) ) \cdot \bbI \{ M \not\in \cMaltst \}].
\end{align*}
For $M \not\in \cMaltst$, we have $\fm(\pist) - \fm(\pitil) \ge \delminm \ge \Delta$. Thus, the above implies
\begin{align*}
& 0 < \Exp_{M \sim \xi^s}[\bbI \{ M \in \cMaltst \}] - \Delta \cdot  \Exp_{M \sim \xi^s}[ \bbI \{ M \not\in \cMaltst \}] \\
\iff & \Delta \cdot ( 1 - \Exp_{M \sim \xi^s}[\bbI \{ M \in \cMaltst \}]) < \Exp_{M \sim \xi^s}[\bbI \{ M \in \cMaltst \}].
\end{align*}
Rearranging gives $\Exp_{M \sim \xi^s}[\bbI \{ M \in \cMaltst \} ]  \ge \frac{\Delta}{1+\Delta} \ge \frac{1}{2} \Delta$. 
\end{proof}

\begin{lemma}\label{lem:gammas_bound}
When running \Cref{alg:gl_alg_main_general_D}, on rounds $s$ for which $\Mst \in \cM \backslash
\cMgl(\lambda^s; \nmax)$, we have
\begin{align*}
\gamma^s \le \frac{(1+\delta)(4 \nmax + 2 \delta \guncM)}{\delta \guncM}  \cdot  \aecflipD[\veps/2](\cM),
\end{align*}
for $\gamma^s$ as defined in \eqref{eq:gamma_def}. 
\end{lemma}
\begin{proof}[Proof of \Cref{lem:gammas_bound}]
Recall that $\omega^s$ and $\lambda^s$ are chosen to minimize 
\begin{align*}
\sup_{M \in \cM \backslash \cMgl(\lam^s; \nmax)} \frac{1}{\Exp_{\Mhat \sim \xi^s}[\Exp_\omega[\D{\Mhat(\pi)}{M(\pi)}]}.
\end{align*}
Since $\Mst \in \cM \backslash \cMgl(\lambda^s; \nmax)$, we can therefore bound
\begin{align*}
\frac{1}{\Exp_{\Mhat \sim \xi^s}[\Exp_{\omega^s}[\D{\Mhat(\pi)}{\Mst(\pi)}]]} & \le \inf_{\omega \in \simplex_\Pi} \sup_{M \in \cM \backslash \cMgl(\lam^s; \nmax)} \frac{1}{\Exp_{\Mhat \sim \xi^s}[\Exp_\omega[\D{\Mhat(\pi)}{M(\pi)}]]} \\
& = \inf_{\lam, \omega \in \simplex_\Pi} \sup_{M \in \cM \backslash \cMgl(\lam; \nmax)} \frac{1}{\Exp_{\Mhat \sim \xi^s}[\Exp_\omega[\D{\Mhat(\pi)}{M(\pi)}]}.
\end{align*}
Recall that we set
\begin{align*}
\nmax =  \prn*{\frac{1}{\delmin \veps} + \frac{2 \VM \ncMeps}{\delmin} } \cdot \max_{M \in \cM} \frac{32 \gm}{\delmin} .
\end{align*}
By \Cref{lem:kl_bound}, under \Cref{asm:bounded_likelihood}, we can bound $\Dkl{M(\pi)}{M'(\pi)} \le 2 \VM$ for all $M,M' \in \cM$ and $\pi \in \Pi$. It follows from \Cref{lem:gm_lb} that 
\begin{align*}
\frac{2\VM}{\delmin} \ge \frac{1}{\min_{M \in \cM : \gm > 0} \gm},
\end{align*}
so
\begin{align*}
\nmax \ge \prn*{\frac{1}{\delmin \veps} + \frac{ \ncMeps}{\min_{M \in \cM : \gm > 0} \gm} } \cdot \max_{M \in \cM} \frac{32 \gm}{\delmin}.
\end{align*}
Given this, straightforward calculation shows that $\nmax$ meets the condition required by \Cref{lem:nM_aec_bound}, so \Cref{lem:nM_aec_bound} implies
\begin{align*}
 \inf_{\lam, \omega \in \simplex_\Pi} \sup_{M \in \cM \backslash \cMgl(\lam; \nmax)} \frac{1}{\Exp_{\Mhat \sim \xi^s}[\Exp_\omega[\D{\Mhat(\pi)}{M(\pi)}]} & \le  \inf_{\lam, \omega \in \simplex_\Pi} \sup_{M \in \cM \backslash \cMgl[\veps/2](\lam)} \frac{1}{\Exp_{\Mhat \sim \xi^s}[\Exp_\omega[\D{\Mhat(\pi)}{M(\pi)}]} \\
 & = \aecflipD[\veps/2](\cM;\xi^s) \\
 & \le \aecflipD[\veps/2](\cM).
 \end{align*}
By our choice of $q$ we have $\frac{1}{1-q} = \frac{4 \nmax + 2 \delta \guncM}{\delta \guncM}$. We can then bound
\begin{align*}
\gamma^s & = \frac{1+\delta}{1-q} \cdot  \frac{1}{\Exp_{\Mhat \sim \xi^s}[\Exp_{\omega^s}[\D{\Mhat(\pi)}{\Mst(\pi)}]}  \le \frac{(1+\delta)(4 \nmax + 2 \delta \guncM)}{\delta \guncM} \cdot \aecflipD[\veps/2](\cM).
\end{align*}
\end{proof}

\begin{lemma}\label{lem:gammas_bound_nomingap}
Consider running \Cref{alg:gl_alg_main_infinite_nomingap}. Then on rounds $s$ for which $\Mst \in \cM^\ell \backslash \cMgl(\lambda^s;
\nsf^s)$, we have
\begin{align*}
\gamma^s \le  (1+ \gst \delta) \cdot s^{\alphaq + \alphaM}
\end{align*}
for $\gamma^s$ as defined in \eqref{eq:gamma_def_full}, and $\alphaq,\alphaM$ parameters of \Cref{alg:gl_alg_main_infinite_nomingap}.
\end{lemma}
\begin{proof}[Proof of \Cref{lem:gammas_bound_nomingap}]
Assume we are at epoch $\ell$.
Recall that $\omega^s$ and $\lambda^s$ minimize
\begin{align*}
\sup_{M \in \cM^\ell \backslash \cMgl(\lam^s; \nsf^s)} \frac{1}{\Exp_{\Mhat \sim \xi^s}[\Exp_\omega[\D{\Mhat(\pi)}{M(\pi)}]}.
\end{align*}
Since $\Mst \in \cM^\ell \backslash \cMgl(\lambda^s;\nsf^s)$, we have can therefore bound
\begin{align*}
\frac{1}{\Exp_{\Mhat \sim \xi^s}[\Exp_{\omega^s}[\D{\Mhat(\pi)}{\Mst(\pi)}]]} & \le \inf_{\omega \in \simplex_\Pi} \sup_{M \in \cM^\ell \backslash \cMgl(\lam^s; \nsf^s)} \frac{1}{\Exp_{\Mhat \sim \xi^s}[\Exp_\omega[\D{\Mhat(\pi)}{M(\pi)}]]} \\
& = \inf_{\lam, \omega \in \simplex_\Pi} \sup_{M \in \cM^\ell \backslash \cMgl(\lam; \nsf^s)} \frac{1}{\Exp_{\Mhat \sim \xi^s}[\Exp_\omega[\D{\Mhat(\pi)}{M(\pi)}]}.
\end{align*}
By construction, for every $M \in \cM^\ell$, we have
\begin{align*}
 \nmepsb + \frac{1}{\delminm} + \frac{4 \gm}{\delminm} + \frac{2 \nmepsb}{\gm} + \frac{4}{\delminm \veps} \le \sqrt{\nsf^s} .
 \end{align*}
This implies that
 \begin{align*}
\zeta:=  \min_{M \in \cM^\ell} \min \crl*{ \frac{\gm}{\gm + 2 \nmepsb}, \frac{\delminm \veps}{4}} \ge \frac{1}{\sqrt{\nsf^s}}
 \end{align*}
 and
 \begin{align*}
 \max_{M \in \cM^\ell} \max \crl*{ \nmepsb, \frac{4 \gm}{\delminm}} \le \sqrt{\nsf^s},
 \end{align*}
 which together imply that
  \begin{align*}
 \max_{M \in \cM^\ell} \max \crl*{ \nmepsb, \frac{4 \gm}{\delminm}, \frac{2\gm}{\zeta \delminm}} \le \nsf^s
 \end{align*}
By \Cref{lem:nM_aec_bound} and since $\cM^\ell$ is constructed such that $\inf_{M \in \cM^\ell} \delminm > 0$, we can therefore bound 
\begin{align*}
\inf_{\lam, \omega \in \simplex_\Pi} \sup_{M \in \cM^\ell \backslash \cMgl(\lam; \nsf^s)} \frac{1}{\Exp_{\Mhat \sim \xi^s}[\Exp_\omega[\D{\Mhat(\pi)}{M(\pi)}]} & \le \inf_{\lam, \omega \in \simplex_\Pi} \sup_{M \in \cM^\ell \backslash \cMgl[\veps/2](\lam)} \frac{1}{\Exp_{\Mhat \sim \xi^s}[\Exp_\omega[\D{\Mhat(\pi)}{M(\pi)}]} \\
& \le \aecflipDM{\veps/2}{\cM}(\cM_{\Delta^\ell,\frac{1}{\Delta^\ell}};\xi^s)
\end{align*}
where the last inequality holds by the definition of $\cM^\ell$ and $\Delta^\ell$.
Note that by construction, $\xi^s$ is only supported on $\cM^\ell$, so we can bound $\aecflipDM{\veps/2}{\cM}(\cM_{\Delta^\ell,\frac{1}{\Delta^\ell}};\xi^s) \le \aecflipDM{\veps/2}{\cM}(\cM_{\Delta^\ell,\frac{1}{\Delta^\ell}})$.
By construction, we also have
$\aecflipDM{\veps/2}{\cM}(\cM_{\Delta^\ell,\frac{1}{\Delta^\ell}}) \le
2^{\ell \alphaM} \le s^{\alphaM}$. Lastly, by our choice for $q^s$ we have $\frac{1}{1-q^s} = s^{\alphaq}$. We can then bound
\begin{align*}
\gamma^s & = \frac{1+\gst \delta}{1-q^s} \cdot  \frac{1}{\Exp_{\Mhat \sim \xi^s}[\Exp_{\omega^s}[\D{\Mhat(\pi)}{\Mst(\pi)}]}  \le (1+\delta) \cdot s^{\alphaq + \alphaM}.
\end{align*}
\end{proof}

\subsubsection{Likelihood Ratios}

\begin{lemma}\label{lem:information_gain_lrt_bound}
Consider running either \Cref{alg:gl_alg_main_general_D} or \Cref{alg:gl_alg_main_infinite_nomingap}.
Under \Cref{asm:bounded_likelihood}, with probability at least
$1-\delta$, we can bound, for any $M \in \cM$, $s$, and $\beta_i > 0$,
\begin{align*}
\sum_{i=1}^s \Exp_{\Mhat \sim \xi^i}[\Exp_{\pi \sim p^i}[\Dklbig{\Mhat(\pi)}{M(\pi)}]] & \le \sum_{i=1}^{s} \Exp_{\Mhat \sim \xi^i} \brk*{\log \frac{\Prm{\Mhat}{\pi^i}(r^i,o^i)}{\Prm{M}{\pi^i}(r^i,o^i)}} + \VM \sqrt{56 s \log 1/\delta}  \\
& \qquad + \VM \cdot  \left (\sum_{i=1}^s 1/\beta_i +  \sum_{i=1}^s\beta_i  \Exp_{\Mhat \sim \xi^i}[\Exp_{\pi \sim p^i}[\D{\Mhat(\pi)}{\Mst(\pi)} ]] \right ).
\end{align*}
\end{lemma}
\begin{proof}[Proof of \Cref{lem:information_gain_lrt_bound}]
By Theorem 2 of \cite{shamir2011variant}, under
\Cref{asm:bounded_likelihood} we have that with probability at least
$1-\delta$,
\begin{align}\label{eq:lrt_conc}
\sum_{i=1}^s \Exp_{(r,o) \sim \Mst} \left [  \Exp_{\Mhat \sim \xi^i}\brk*{\log \frac{\Prm{\Mhat}{\pi}(r,o)}{\Prm{M}{\pi}(r,o)}} \mid \cH^{i-1} \right ] \le \sum_{i=1}^{s} \Exp_{\Mhat \sim \xi^i} \brk*{\log \frac{\Prm{\Mhat}{\pi^i}(r^i,o^i)}{\Prm{M}{\pi^i}(r^i,o^i)}}  + \VM \sqrt{56 s \log 1/\delta}.
\end{align}
Note that
\begin{align*}
\Exp \left [  \Exp_{\Mhat \sim \xi^i} \brk*{\log \frac{\Prm{\Mhat}{\pi^i}(r^i,o^i)}{\Prm{M}{\pi^i}(r^i,o^i)}} \mid \cH^{i-1} \right ] & = \Exp_{\pi \sim p^i} \left [ \Exp_{(r,o) \sim \Mst(\pi)} \left [ \Exp_{\Mhat \sim \xi^i} \brk*{\log \frac{\Prm{\Mhat}{\pi}(r,o)}{\Prm{M}{\pi}(r,o)}} \right ] \right ] \\
& = \Exp_{\pi \sim p^i} \left [ \Exp_{\Mhat \sim \xi^i} \brk*{ \Exp_{(r,o) \sim \Mst(\pi)} \left [ \log \frac{\Prm{\Mhat}{\pi}(r,o)}{\Prm{M}{\pi}(r,o)} \right ]} \right ] 
\end{align*}
By Lemma B.4 of \cite{foster2022complexity}, we can bound, for any $\Mhat \in \cM$,
\begin{align*}
& \left | \Exp_{(r,o) \sim \Mst(\pi)} \left [  \log \frac{\Prm{\Mhat}{\pi}(r,o)}{\Prm{M}{\pi}(r,o)} \right ] - \Exp_{(r,o) \sim \Mhat(\pi)} \left [ \log \frac{\Prm{\Mhat}{\pi}(r,o)}{\Prm{M}{\pi}(r,o)} \right ] \right |  \\
& \qquad \le  \sqrt{\frac{1}{2} \left (\Exp_{(r,o) \sim \Mst(\pi)} \left [ \log^2 \frac{\Prm{\Mhat}{\pi}(r,o)}{\Prm{M}{\pi}(r,o)} \right ] + \Exp_{(r,o) \sim \Mhat(\pi)} \left [ \log^2 \frac{\Prm{\Mhat}{\pi}(r,o)}{\Prm{M}{\pi}(r,o)} \right ] \right ) \cdot \Dhelsbig{\Mst(\pi)}{\Mhat(\pi)}} \\
& \qquad \le \sqrt{2\VM^2 \cdot \Dhelsbig{\Mst(\pi)}{\Mhat(\pi)}}
\end{align*}
where the second inequality follows under the subgaussian assumption,
\Cref{asm:bounded_likelihood}. It follows that for any $\Mhat\in\cM$,
\iftoggle{colt}{ \begin{align*}
 \Exp_{\pi \sim p^i} & \left [ \Exp_{(r,o) \sim \Mst(\pi)} \left [ \log \frac{\Prm{\Mhat}{\pi}(r,o)}{\Prm{M}{\pi}(r,o)} \right ] \right ] \\
 & \ge \Exp_{\pi \sim p^i} \left [ \Exp_{(r,o) \sim \Mhat(\pi)} \left [ \log \frac{\Prm{\Mhat}{\pi}(r,o)}{\Prm{M}{\pi}(r,o)} \right ] \right ] - \sqrt{2\VM^2 \cdot \Exp_{\pi\sim p^i}[\Dhelsbig{\Mst(\pi)}{\Mhat(\pi)}]} \\
 & \overset{(a)}{=} \Exp_{\pi \sim p^i} [\Dklbig{\Mhat(\pi)}{M(\pi)}] - \sqrt{2\VM^2 \cdot \Exp_{\pi\sim p^i}[\Dhelsbig{\Mst(\pi)}{\Mhat(\pi)}]} \\
 & \overset{(b)}{\ge} \Exp_{\pi \sim p^i} [\Dklbig{\Mhat(\pi)}{M(\pi)}] - \VM^2 \cdot (\beta \Exp_{\pi\sim p^i}[\Dhelsbig{\Mst(\pi)}{\Mhat(\pi)}] + 1/\beta)
 \end{align*}}{
  \begin{align*}
 \Exp_{\pi \sim p^i} \left [ \Exp_{(r,o) \sim \Mst(\pi)} \left [ \log \frac{\Prm{\Mhat}{\pi}(r,o)}{\Prm{M}{\pi}(r,o)} \right ] \right ] & \ge \Exp_{\pi \sim p^i} \left [ \Exp_{(r,o) \sim \Mhat(\pi)} \left [ \log \frac{\Prm{\Mhat}{\pi}(r,o)}{\Prm{M}{\pi}(r,o)} \right ] \right ] - \sqrt{2\VM^2 \cdot \Exp_{\pi\sim p^i}[\Dhelsbig{\Mst(\pi)}{\Mhat(\pi)}]} \\
 & \overset{(a)}{=} \Exp_{\pi \sim p^i} [\Dklbig{\Mhat(\pi)}{M(\pi)}] - \sqrt{2\VM^2 \cdot \Exp_{\pi\sim p^i}[\Dhelsbig{\Mst(\pi)}{\Mhat(\pi)}]} \\
 & \overset{(b)}{\ge} \Exp_{\pi \sim p^i} [\Dklbig{\Mhat(\pi)}{M(\pi)}] - \VM^2 \cdot (\beta \Exp_{\pi\sim p^i}[\Dhelsbig{\Mst(\pi)}{\Mhat(\pi)}] + 1/\beta)
 \end{align*}}
 where $(a)$ follows by the definition of KL divergence, and $(b)$ from AM-GM for any $\beta > 0$. 
Since this bound holds uniformly for all $\Mhat\in\cM$, this implies that
\begin{align*}
\Exp \left [  \Exp_{\Mhat \sim \xi^i} \brk*{\log \frac{\Prm{\Mhat}{\pi^i}(r^i,o^i)}{\Prm{M}{\pi^i}(r^i,o^i)}} \mid \cH^{i-1} \right ]  & \ge  \Exp_{\Mhat \sim \xi^i}[\Exp_{\pi \sim p^i} [\Dklbig{\Mhat(\pi)}{M(\pi)}]] \\
& \qquad - \VM^2 \cdot (\beta \Exp_{\Mhat \sim \xi^i}[\Exp_{\pi\sim p^i}[\Dhelsbig{\Mst(\pi)}{\Mhat(\pi)}]] + 1/\beta).
\end{align*}
 Combining this with \eqref{eq:lrt_conc}, and using \Cref{asm:D_to_hel} proves the result.
\end{proof}

\begin{lemma}\label{lem:kl_bound}
Under \Cref{asm:bounded_likelihood}, we have, for any $M,M' \in \cM$,
\begin{align*}
\Dkl{M'(\pi)}{M(\pi)} \le 2\VM, \quad \forall \pi \in \Pi .
\end{align*}
\end{lemma}
\begin{proof}[Proof of \Cref{lem:kl_bound}]
We have
\begin{align*}
\Dkl{M'(\pi)}{M(\pi)}  = \Exp_{(r,o) \sim M'(\pi)} \left [ \log \frac{\Prm{M'}{\pi}(r,o)}{\Prm{M}{\pi}(r,o)} \right ] \le \Exp_{(r,o) \sim M'(\pi)} \left [ \left | \log \frac{\Prm{M'}{\pi}(r,o)}{\Prm{M}{\pi}(r,o)} \right | \right ] \le 2\VM
\end{align*}
where the last inequality holds under \Cref{asm:bounded_likelihood}.
\end{proof}

\begin{lemma}\label{lem:hel_to_kl}
Under \Cref{asm:bounded_likelihood}, for any $x > 0$ and $M, \Mtil \in \cM$, we have
\begin{align*}
\Dklbig{M(\pi)}{\Mtil(\pi)} \le (2 + 2\VM + x) \cdot \Dhels{M(\pi)}{\Mtil(\pi)} + 32(1+\VM^2/x + \VM^3/x^2) \cdot \exp(-x^2/8\VM^2).
\end{align*}
\end{lemma}
\begin{proof}[Proof of \Cref{lem:hel_to_kl}]
Fix $\pi$. Define
\begin{align*}
\cE & := \crl*{\left |  \log
      \frac{\Prm{M}{\pi}(r,o)}{\Prm{\Mtil}{\pi}(r,o)} -
      \Dklbig{M(\pi)}{\Mtil(\pi)} \right | \le x }
      \intertext{and for $j\in\bbN$,}
\cE_j & := \crl*{ e^{j-1} \cdot x  < \left | \log \frac{\Prm{M}{\pi}(r,o)}{\Prm{\Mtil}{\pi}(r,o)} - \Dklbig{M(\pi)}{\Mtil(\pi)}  \right |  \le  e^j \cdot x }.
\end{align*}
Note that $\cE,(\cE_j)_{j=1}^\infty$ form a partition of the probability space. Furthermore, since $\Dklbig{M(\pi)}{\Mtil(\pi)} = \Exp_{o \sim M(\pi)} [\log \frac{\Prm{M}{\pi}(r,o)}{\Prm{\Mtil}{\pi}(r,o)}]$, under \Cref{asm:bounded_likelihood} we have that $\Prm{M}{\pi}(\cE_j) \le 2 \exp(-x^2 e^{2(j-1)}/\VM^2)$ and $\Prm{M}{\pi}(\cE^c) \le 2\exp(-x^2/\VM^2)$. Now, 
\begin{align*}
\Dklbig{M(\pi)}{\Mtil(\pi)} & = \int  \log \frac{\Prm{M}{\pi}(r,o)}{\Prm{\Mtil}{\pi}(r,o)} \rmd \Prm{M}{\pi}(r,o) \\
& = \int_{\cE}  \log \frac{\Prm{M}{\pi}(r,o)}{\Prm{\Mtil}{\pi}(r,o)} \rmd \Prm{M}{\pi}(r,o) + \sum_{j=1}^{\infty} \int_{\cE_j} \log \frac{\Prm{M}{\pi}(r,o)}{\Prm{\Mtil}{\pi}(r,o)} \rmd \Prm{M}{\pi}(r,o).
\end{align*}
Using that $\Dklbig{M(\pi)}{\Mtil(\pi)} \le 2\VM$ by \Cref{lem:kl_bound}, we can bound
\begin{align*}
\sum_{j=1}^{\infty} \int_{\cE_j} \log \frac{\Prm{M}{\pi}(r,o)}{\Prm{\Mtil}{\pi}(r,o)} \rmd \Prm{M}{\pi}(r,o) & \le \sum_{j=1}^{\infty} (e^j x + 2\VM) \cdot \int_{\cE_j}  \rmd \Prm{M}{\pi}(r,o) \\
& = \sum_{j=1}^{\infty} (e^j x + 2\VM) \cdot \Prm{M}{\pi}(\cE_j) \\
& \le \sum_{j=1}^{\infty} (e^j x + 2\VM) \cdot 2 \exp(-x^2 e^{2(j-1)}/\VM^2) \\
& \le \int_{0}^{\infty} (e^j x + 2\VM) \cdot 2\exp(-x^2 e^{2(j-1)}/\VM^2) \rmd j \\
& \le 4(x+\VM) \int_{0}^{\infty} e^j \exp(-e^{j} \cdot x^2  e^{-2}/\VM^2) \rmd j \\
& = 4(x+\VM) \VM^2 e^2 \exp(-x^2  e^{-2}/\VM^2)/x^2  \\
& \le 32(\VM^2/x + \VM^3/x^2) \cdot \exp(-x^2/8\VM^2).
\end{align*}
We turn now to the first term. Note that we can write
\begin{align*}
\int_{\cE}  \log \frac{\Prm{M}{\pi}(r,o)}{\Prm{\Mtil}{\pi}(r,o)} & \rmd \Prm{M}{\pi}(r,o)  = \int_{\cE, \Prm{\Mtil}{\pi}(r,o) > \Prm{M}{\pi}(r,o)}  \prn*{ \log \frac{\Prm{M}{\pi}(r,o)}{\Prm{\Mtil}{\pi}(r,o)} + \frac{\Prm{\Mtil}{\pi}(r,o)}{\Prm{M}{\pi}(r,o)} - 1} \rmd \Prm{M}{\pi}(r,o) - \Prm{\Mtil}{\pi}(\cE^c ) \\
& + \int_{\cE, \Prm{\Mtil}{\pi}(r,o) \le \Prm{M}{\pi}(r,o)}  \prn*{\frac{\Prmb[M](o \mid \pi)}{\Prm{\Mtil}{\pi}(r,o)} \log \frac{\Prm{M}{\pi}(r,o)}{\Prm{\Mtil}{\pi}(r,o)} + 1 - \frac{\Prm{M}{\pi}(r,o)}{\Prm{\Mtil}{\pi}(r,o)} } \rmd \Prm{\Mtil}{\pi}(r,o) + \Prm{M}{\pi}(\cE^c)
\end{align*}
Following the proof of Lemma 4 of \cite{yang1998asymptotic} and using that $\log \frac{\Prm{M}{\pi}(r,o)}{\Prm{\Mtil}{\pi}(r,o)} \le \Dklbig{M(\pi)}{\Mtil(\pi)} + x \le 2\VM +x$ on $\cE$, we can bound this as
\begin{align*}
 & \le (2 + 2\VM + x) \int_{\cE} (\sqrt{\rmd \Prm{M}{\pi}(r,o)} - \sqrt{\rmd \Prm{\Mtil}{\pi}(r,o)})^2  + \Prm{M}{\pi}(\cE^c)   \\
& \le (2 + 2\VM + x) \int (\sqrt{\rmd \Prm{M}{\pi}(r,o)} - \sqrt{\rmd \Prm{\Mtil}{\pi}(r,o)})^2  + \Prm{M}{\pi}(\cE^c) \\
& \le  (2 + 2\VM + x) \cdot \Dhelsbig{M(\pi)}{\Mtil(\pi)} + 2\exp(-x^2/\VM^2).
\end{align*}

\end{proof}

\subsubsection{Covering Numbers}

\begin{lemma}\label{lem:cover_equiv}
For any subset $\cM' \subseteq \cM$, there exists some $(\rho,\mu)$-cover $\cMcov' \subseteq \cM'$ for $\cM'$ such that $|\cMcov'| \le \Ncov(\cM,\rho/2,\mu)$.
\end{lemma}
\begin{proof}[Proof of \Cref{lem:cover_equiv}]
Let $\cMcov$ denote a $(\rho/2,\mu)$-cover of $\cM$ with event $\cE$. Throughout the proof we use the shorthand $(r,o,\pi) \in \cE$ to denote that there exists $M \in \cM$ such that $\Prm{M}{\pi}(r,o \mid  \cE) > 0$.
By definition, it follows that for any $M \in \cM$, there exists $M' \in \cMcov$ such that
\begin{align}\label{eq:cover_defn_subset_cov}
\sup_{r,o,\pi \ : \ (r,o,\pi) \in \cE } \left | \log \frac{\Prm{M}{\pi}(r,o)}{\Prm{M'}{\pi}(r,o)} \right | \le \rho/2.
\end{align}
Let $\cMcov' = \emptyset$ and consider running the following procedure for every $M' \in \cMcov$:
\begin{enumerate}
\item Choose a single $M \in \cM'$ such that $\sup_{r,o,\pi  :
    (r,o,\pi) \in \cE } \left | \log
    \frac{\Prm{M}{\pi}(r,o)}{\Prm{M'}{\pi}(r,o)} \right | \le \rho/2$
  (if such an $M$ exists). 
\item If there exists an $M \in \cM'$ in step 1, set $\cMcov' \leftarrow \cMcov' \cup \{ M \}$. Otherwise $\cMcov'$ remains unchanged.
\end{enumerate}
By construction $\cMcov' \subseteq \cM'$, and $|\cMcov'| \le |\cMcov|$. We claim that $\cMcov'$ is a $(\rho,\mu)$-cover of $\cM'$. To see why, take some $M \in \cM'$. Let $M' \in \cMcov$ denote the point realizing \eqref{eq:cover_defn_subset_cov} for $M$. Let $M''$ denote the point chosen in the above procedure for $M'$. Note that there must exist some $M''$ chosen for this $M'$ since \eqref{eq:cover_defn_subset_cov} holds for $M$, so in particular $M \in \cM'$ satisfies the condition of step 1 in the above procedure. 
Then,
\begin{align*}
\sup_{r,o,\pi \ : \ (r,o,\pi) \in \cE } \left | \log \frac{\Prm{M}{\pi}(r,o)}{\Prmb[M''](o|\pi)} \right | & = \sup_{r,o,\pi \ : \ (r,o,\pi) \in \cE } \left | \log \frac{\Prm{M}{\pi}(r,o)}{\Prm{M'}{\pi}(r,o)} +  \log \frac{\Prm{M'}{\pi}(r,o)}{\Prmb[M''](o|\pi)} \right | \\
& \le \sup_{r,o,\pi \ : \ (r,o,\pi) \in \cE } \left | \log \frac{\Prm{M}{\pi}(r,o)}{\Prm{M'}{\pi}(r,o)} \right | + \sup_{r,o,\pi \ : \ (r,o,\pi) \in \cE }\left |  \log \frac{\Prmb[M''](o|\pi)}{\Prm{M'}{\pi}(r,o)} \right | \\
& \le \rho/2 + \rho/2 = \rho
\end{align*}
where the last inequality follows by our choice of $M'$ and the definition of $M''$. Thus, it follows that $\cMcov'$ is a $(\rho,\mu)$-cover of $\cM'$.
\end{proof}

\subsubsection{Further Lemmas}

\begin{lemma}\label{claim:concave_fun}
For $a > 0$, $\alpha \le [0,3/4]$, and $\beta > 0$, the function $x^{\alpha} \log^\beta (ax)$ is concave in $x$ for $x \ge \frac{1}{a} \exp \prn*{\frac{4 \beta}{\alpha}}$ when $\beta \le 1$, and for $x \ge  \max \crl*{ \frac{1}{a} \exp \prn*{\sqrt{\frac{8(\beta - 1)\beta}{\alpha}}}, \frac{1}{a} \exp \prn*{\frac{4 \beta}{\alpha}} }$ when $\beta > 1$. 
\end{lemma}
\begin{proof}[Proof of \Cref{claim:concave_fun}]
By some calculation, we have
\begin{align*}
\frac{\rmd^2}{\rmd^2 x} \prn*{x^{\alpha} \log^\beta (ax)} & = (-1+\beta) \beta x^{-2+\alpha} \log^{\beta - 2}(ax) + (-\beta + 2 \alpha \beta)x^{-2+\alpha} \log^{\beta - 1}(ax) \\
& \qquad + (-1+\alpha)\alpha x^{-2+\alpha} \log^{\beta}(ax).
\end{align*}
If we restrict to $x\geq{}1/a$, then to show that the function is concave it then suffices to show that
\begin{align*}
(-1+\beta) \beta \log^{ - 2}(ax) + (-\beta + 2 \alpha \beta) \log^{- 1}(ax) + (-1+\alpha)\alpha  \le 0
\end{align*}
which, since $\alpha \le 3/4$, is implied by
\begin{align*}
(-1+\beta) \beta \log^{ - 2}(ax) + \frac{1}{2} \beta \log^{- 1}(ax) \le \frac{1}{4} \alpha 
\end{align*}
which is further implied by
\begin{align*}
(-1+\beta) \beta \log^{ - 2}(ax)  \le \frac{1}{8} \alpha \quad \text{and} \quad  \frac{1}{2} \beta \log^{- 1}(ax) \le \frac{1}{8} \alpha .
\end{align*}
The former condition is met for $x > 1/a$ for all $\beta \in (0,1]$. For $\beta > 1$, it is met as long as
\begin{align*}
\frac{8(\beta - 1)\beta}{\alpha} \le \log^2(ax) \iff x \ge  \frac{1}{a} \exp \prn*{\sqrt{\frac{8(\beta - 1)\beta}{\alpha}}}.
\end{align*}
The latter condition is met for
\begin{align*}
x \ge \frac{1}{a} \exp \prn*{\frac{4 \beta}{\alpha}}.
\end{align*}
\end{proof}

\begin{lemma}\label{lem:lin_log_ineq}
For all $B, C, n \ge 1$, if $x \le C \log^n(Bx)$, then $x \le C (2n)^n  \log^n(2nBC)$.
\end{lemma}
\begin{proof}[Proof of \Cref{lem:lin_log_ineq}]
This is a direct consequence of Lemma A.4 of \cite{wagenmaker2022reward}.
\end{proof}




\section{Proofs for Examples}\label{sec:ex_proofs}

In this section, we provide proofs for the examples given in \Cref{sec:upper}. We begin in \Cref{sec:ex_regular_model} by introducing a condition which implies $\nsf\sups{\Mstar}_\veps$ is bounded, and is easy to verify for many classes of interest. Next, in \Cref{sec:gauss_bandits_proof}, we consider a variety of structured bandit settings, and in \Cref{sec:contextual_proofs} extend this to contextual bandits with finitely many arms. In \Cref{sec:inf_arm_proofs}, we provide proofs for the informative arm example of \Cref{sec:motivating_example}. Finally, in \Cref{sec:tabular_proofs}, we consider tabular MDPs.

\subsection{Preliminaries: Regular Models}\label{sec:ex_regular_model}
To bound the quantity $\nepsst=\nsf\sups{\Mstar}_\veps$ for the examples we consider, it will be helpful to introduce the following notion of a \emph{regular model}. 

\begin{definition}[Regular Model]\label{def:regular_class2}
We say instance $M \in \cM$ is a \emph{regular model} if there exists some constant $\LMc > 0$ such that, for any $M' \in \cMalt(M)$ with $\Dkl{M(\pim)}{M'(\pim)} > 0$, there exists $M'' \in \cMalt(M)$ such that $\Dkl{M(\pim)}{M''(\pim)} = 0$ and, for all $\pi \in \Pi$,
\iftoggle{colt}{\begin{align}\label{eq:reg_class_cond2}
\begin{split}
|\Dkl{M(\pi)}{M'(\pi)} - \Dkl{M(\pi)}{M''(\pi)}| & \le \sqrt{\LMc \Dkl{M(\pim)}{M'(\pim)}}  \\
& \qquad + \LMc \Dkl{M(\pim)}{M'(\pim)}. 
\end{split}
\end{align}}
{\begin{align}\label{eq:reg_class_cond2}
\begin{split}
|\Dkl{M(\pi)}{M'(\pi)} - \Dkl{M(\pi)}{M''(\pi)}| & \le \sqrt{\LMc \Dkl{M(\pim)}{M'(\pim)}} \\
& \qquad + \LMc \Dkl{M(\pim)}{M'(\pim)}.
\end{split}
\end{align}}
\end{definition}

Our definition of a regular model is a direct generalization of
existing notions of class regularity found in the literature
\citep{degenne2020structure}. As we will see, for a variety of
standard bandit classes (including multi-armed bandits, linear
bandits, and Lipschitz bandits), as well as tabular MDPs, one can show
that$ \Mst$ is a regular model with $\LMst=\LM\sups{\Mstar}$ bounded by
a polynomial function of problem parameters. Intuitively, $\Mst$ will
be a regular model if, for any instance $M' \in \cMalt(\Mst)$ for
which it is sufficient to pull $\pist$ in order to distinguish $\Mst$
and $M'$ (thereby ruling out $M'$ while incurring no regret), then
there exists some other instance $M'' \in \cMalt(\Mst)$ which is
``close'' to $M'$ in a certain sense, and which cannot be
distinguished from $\Mst$ by simply pulling $\pist$. As the following
result shows, the quantity $\nepsst$ can be bounded whenever $\Mst$ is a regular model.

\begin{proposition}\label{prop:regular_class_to_nM}
If $M$ is a regular model with $\delminm > 0$, we can bound
\begin{align*}
\nmeps \le \frac{2 \gm}{\delminm} \cdot \prn*{1 + \LMc + \frac{2 \gm}{\veps \delminm} \cdot \LMc} .
\end{align*}
\end{proposition}

Given \Cref{prop:regular_class_to_nM}, for many of the examples in
this section, rather than bounding $\nepsst$ directly, we first show that $\Mst$ is a regular model with $\LMst$ well-bounded, and then use \Cref{prop:regular_class_to_nM} to obtain a bound on $\nepsst$.

\begin{proof}[Proof of \Cref{prop:regular_class_to_nM}]
To prove this result, it suffices to show that, for every normalized
allocation $\lam \in \Lambda(M,\veps)$ with normalization factor
$\nsf$, there exists some allocation $\eta \in \R_+^\Pi$ such that 1)
$\eta(\pi) = \nsf \lam(\pi)$ for $\pi \neq \pim$, and $\eta(\pim) \le
\nbar$ for some well-bounded $\nbar$, and 2) $\delm(\eta) \le (1+2\veps) \gm$ and $\Im(\eta) \ge 1 - 2\veps$.

Fix some $\lambda \in \Lambda(M,\veps)$ with normalization factor $\nsf > 0$. Note that if $M$ is a regular model, then $\Im[M](\bbI_{\pim}) = 0$. Since $\lambda \in \Lambda(M,\veps)$, this implies that $\lam(\pim) < 1$.
Let $\lam'$ denote the allocation $\lam'(\pim) = 1 - \zeta$ and $\lam'(\pi) = \frac{\zeta}{1-\lam(\pim)} \lam(\pi)$ for $\pi \neq \pim$, for some $\zeta$ to be chosen. 

We have 
\begin{align}
\delm(\lamGL') = \frac{\zeta}{1-\lamGL(\pim)} \delm(\lamGL)  \le \frac{\zeta}{1- \lamGL(\pim)} (1 + \veps) \gm/\nsf .\label{eq:lambdaprime_gap}
\end{align}
Take $M' \in \cMalt(M)$ such that $\Dkl{M(\pim)}{M'(\pim)} > 0$ and let $M'' \in \cMalt(M)$ denote the instance guaranteed to exist under \Cref{def:regular_class2}. We then have that, for any $\pi$,
\begin{align*}
\Dkl{M(\pi)}{M'(\pi)} & \ge \Dkl{M(\pi)}{M''(\pi)} -  \sqrt{ \LMc \Dkl{M(\pim)}{M'(\pim)}} - \LMc \Dkl{M(\pim)}{M'(\pim)} \\
& \ge \Dkl{M(\pi)}{M''(\pi)} - (1 + \alpha) \LMc  \Dkl{M(\pim)}{M'(\pim)} - \frac{1}{\alpha}
\end{align*}
where the last inequality follows for any $\alpha > 0$ by AM-GM. 
Then 
\begin{align*}
& \sum_{\pi} \lamGL'(\pi) \Dkl{M(\pi)}{M'(\pi)} \\
& = \sum_{\pi \neq \pim} \frac{\zeta}{1-\lamGL(\pim)} \lamGL(\pi) \Dkl{M(\pi)}{M'(\pi)} + (1-\zeta) \Dkl{M(\pim)}{M'(\pim)} \\
& \ge \sum_{\pi \neq \pim} \frac{\zeta \lamGL(\pi)}{1-\lamGL(\pim)}  \Big ( \Dkl{M(\pi)}{M''(\pi)} - (1+\alpha) \LMc \Dkl{M(\pim)}{M'(\pim)} - \frac{1}{\alpha} \Big ) \\
& \qquad + (1-\zeta) \Dkl{M(\pim)}{M'(\pim)} \\
& = \sum_{\pi \neq \pim} \frac{\zeta}{1-\lamGL(\pim)} \lamGL(\pi) \Dkl{M(\pi)}{M''(\pi)} + \Big ( (1-\zeta) - (1+\alpha) \LMc \zeta  \Big ) \Dkl{M(\pim)}{M'(\pim)} - \frac{\zeta}{\alpha} .
\end{align*}
We have $M'' \in \cMalt(M)$, so by definition
\begin{align*}
\sum_\pi \lamGL(\pi) \Dkl{M(\pi)}{M''(\pi)} \ge (1-\veps)/\nsf.
\end{align*}
However, $\Dkl{M(\pim)}{M''(\pim)} = 0$ by assumption, so it follows that
\begin{align*}
\sum_{\pi \neq \pim} \frac{\zeta}{1-\lamGL(\pim)} \lamGL(\pi) \Dkl{M(\pi)}{M''(\pi)} \ge \frac{\zeta}{1-\lamGL(\pim)} \cdot \frac{1-\veps}{\nsf}.
\end{align*}
By assumption, $\delm(\lamGL) \le (1+\veps) \gm/\nsf$, and we can also lower bound $\delm(\lamGL) \ge (1-\lamGL(\pim)) \delminm$. Rearranging these implies that
\begin{align*}
(1-\lambda(\pim)) \nsf \le (1+\veps) \gm/ \delminm.
\end{align*} 
Set $\alpha = \frac{(1+\veps) \gm}{ \veps \delminm}$, then we have
\begin{align*}
\frac{\zeta}{1-\lamGL(\pim)} \cdot \frac{1-\veps}{\nsf} - \frac{\zeta}{\alpha} & = \frac{\zeta}{1-\lamGL(\pim)} \cdot \frac{1-\veps}{\nsf} - \frac{\zeta \veps}{(1+\veps) \gm /\delminm} \\
& \ge \frac{\zeta}{1-\lamGL(\pim)} \cdot \frac{1-\veps}{\nsf} - \frac{\zeta \veps}{(1-\lambda(\pim)) \nsf} \\
& = \frac{\zeta}{1-\lamGL(\pim)} \cdot \frac{1-2\veps}{\nsf} .
\end{align*}
Furthermore, with this choice of $\alpha$, if 
\begin{align*}
\zeta \le \frac{1}{1 + (1+\alpha) \LMc } = \frac{1}{1 + (1 + (1+\veps) \gm / \veps \delminm ) \LMc},
\end{align*}
we have 
\begin{align*}
\Big ( (1-\zeta) - (1+\alpha) \LMc \zeta  \Big ) \Dkl{M(\pim)}{M'(\pim)} \ge 0.
\end{align*}
We therefore have that, for any $M' \in \cMalt(M)$ with $\Dkl{M(\pim)}{M'(\pim)} > 0$,
that
\begin{align*}
\sum_{\pi} \lam'(\pi) \Dkl{M(\pi)}{M'(\pi)} \ge  \frac{\zeta}{1-\lamGL(\pim)} \cdot \frac{1-2\veps}{\nsf}.
\end{align*}
Now consider $M' \in \cMalt(M)$ with $\Dkl{M(\pim)}{M'(\pim)} = 0$. In this case we have
\begin{align*}
\sum_{\pi} \lam'(\pi) \Dkl{M(\pi)}{M'(\pi)} = \frac{\zeta}{1 - \lam(\pim)} \cdot \sum_{\pi} \lam(\pi) \Dkl{M(\pi)}{M'(\pi)} \ge  \frac{\zeta}{1 - \lam(\pim)} \cdot \frac{1 - \veps}{\nsf}
\end{align*}
where the inequality follows since $\lam \in \Lambda(M,\veps)$ with normalization factor $\nsf$ by assumption. 
Together these bounds then imply that:
\begin{align*}
\Im(\lamGL') \ge \frac{\zeta}{1-\lamGL(\pim)} \cdot \frac{1-2\veps}{\nsf}.
\end{align*}
Combining this with our bound on $\delm(\lamGL')$ in
\cref{eq:lambdaprime_gap} implies that $\lamGL' \in \Lambda(M;2\veps)$ with parameter $\nsf' = \frac{1-\lamGL(\pim)}{\zeta} \cdot \nsf$.

To conclude, define the allocation $\eta := \nsf' \lam'$. Then for $\pi \neq \pim$:
\begin{align*}
\eta(\pi) = \frac{1-\lam(\pim)}{\zeta} \cdot \nsf \lam'(\pi) = \nsf \lam(\pi) 
\end{align*}
and
\begin{align*}
\eta(\pim) = \frac{1-\lam(\pim)}{\zeta} \cdot \nsf \lam'(\pim) \le  \frac{1-\lam(\pim)}{\zeta} \cdot \nsf  
\end{align*}
It follows then that $\delm(\eta) = \nsf \delm(\lam) \le (1+\veps)
\gm$, and $\Im(\eta) \ge \nsf' \Im(\lam') \ge 1 - 2\veps$. Therefore, $\eta$
satisfies the desired condition. Since $\lambda \in \Lambda(M,\veps)$, we have 
\begin{align*}
\delm(\lam) \le (1+\veps) \gm/\nsf \implies (1-\lam(\pim)) \nsf \le (1+\veps)\gm/\delminm,
\end{align*}
and thus
\begin{align*}
\eta(\pim) \le \frac{(1+\veps) \gm}{\delminm \cdot \zeta} \le \frac{2\gm (1 + (1 + 2 \gm / \veps \delminm ) \LMst)}{\delminm}.
\end{align*}
which proves the result. 
\end{proof}

\subsection{Structured Bandits with Gaussian Noise}\label{sec:gauss_bandits_proof}
In this section, we consider the problem of structured bandits with
Gaussian noise, in which $\cO=\crl{\emptyset}$, and the mean reward
functions belong to a given function class
$\cF$. Concretely, we consider the model class
\begin{align*}
\cM = \crl*{M(\pi) = \cN(f(\pi),\sigma^2) \ : \ f \in \cF}.
\end{align*}
We set 
\begin{align}\label{eq:bandits_Dkl}
\D{M(\pi)}{\Mbar(\pi)} \leftarrow \Dkl{M(\pi)}{\Mbar(\pi)} =
\frac{1}{2\sigma^2} (\fm(\pi) - \fm[\Mbar](\pi))^2
\end{align}
for $D$ the divergence used by \mainalgb.
  In general in the following examples we take $\sigma = 1$ for simplicity.

We begin by verifying that the basic regularity conditions required by
our results are satisfied for generic classes $\cF$, then provide
bounds on the \CompShort for specific classes of interest.
\begin{lemma}\label{lem:gauss_bandits_satisfy_asm}
  For bandits with Gaussian noise:
  \begin{enumerate}
    \item For $D \leftarrow D_{\mathsf{KL}}$, \Cref{asm:smooth_kl,asm:D_to_hel,asm:bounded_likelihood} hold with
    parameters
    \begin{align*}
      \LKL = \VM = \frac{2 \sqrt{2}}{\sigma}.
    \end{align*}
  \item We can bound $\Ncov(\cM,\rho,\mu)$ by the covering number of $\cM$ in the distance $d(M,M') \ldef \sup_{\pi \in \Pi} | \fm(\pi) -
    \fm[M'](\pi)|$ at tolerance $\frac{\sigma^2 \cdot \rho}{2 + \sqrt{2\sigma^2
        \log(2/\mu)}}$. Furthermore, it suffices to take $\cE := \{ |r| \le 1 + \sqrt{2\sigma^2 \log(2/\mu)} \}$.
  \end{enumerate}

\end{lemma}
\begin{proof}[Proof of \Cref{lem:gauss_bandits_satisfy_asm}]
In this setting, we have that for any $M,M' \in \cM$ and any $\pi \in \Pi$,
\begin{align*}
\Dkl{M(\pi)}{M'(\pi)} = \frac{1}{2\sigma^2} ( \fm[M](\pi) - \fm[M'](\pi))^2.
\end{align*}
For $\Mbar \in \cM$, we therefore have
\begin{align*}
\left |\kl{M(\pi)}{M'(\pi)} - \kl{\Mbar(\pi)}{M'(\pi)} \right | & = \frac{1}{2 \sigma^2} \left | (\fm(\pi) - \fm[M'](\pi))^2 -(\fm[\Mbar](\pi) - \fm[M'](\pi))^2 \right | \\
& \le \frac{2}{\sigma^2} | \fm(\pi) - \fm[\Mbar](\pi) | \\
& = \frac{2\sqrt{2}}{\sigma} \sqrt{\D{M(\pi)}{\Mbar(\pi)}},
\end{align*}
where the inequality follows from the Mean Value Theorem and the
assumption that $\fm[\Mbar](\pi) \in [0,1]$ for all $\pi \in \Pi$.
This verifies that \Cref{asm:smooth_kl} holds with $\LKL = \frac{2\sqrt{2}}{\sigma}$. 

To show that \Cref{asm:bounded_likelihood} is met, we note that
for all $M,\Mbar,\Mbar'\in\cM$,
\begin{align*}
\log \frac{\Prm{\Mbar}{\pi}(r,o)}{\Prm{M}{\pi}(r,o)} & = \frac{1}{2 \sigma^2} (r - \fm(\pi))^2 -\frac{1}{2\sigma^2} (r - \fm[\Mbar](\pi))^2  \\
& = \frac{1}{2 \sigma^2} \left [ \fm(\pi)^2 + \fm[\Mbar](\pi)^2 - 2 r (\fm(\pi) - \fm[\Mbar](\pi)) \right ] \\
& = \frac{1}{2 \sigma^2} \left [  \fm(\pi)^2 + \fm[\Mbar](\pi)^2 - 2 \fm[\Mbar'](\pi) (\fm(\pi) - \fm[\Mbar](\pi)) + 2 (\fm[\Mbar'](\pi) - r) (\fm(\pi) - \fm[\Mbar](\pi)) \right ] \\
& =  \Exp_{(r,o) \sim \Mbar'(\pi)}\left [ \log \frac{\Prm{\Mbar}{\pi}(r,o)}{\Prm{M}{\pi}(r,o)} \right ] + \frac{1}{\sigma^2} (\fm[\Mbar'](\pi) - r) (\fm(\pi) - \fm[\Mbar](\pi)).
\end{align*}
It follows that
\begin{align*}
\log \frac{\Prm{\Mbar}{\pi}(r,o)}{\Prm{M}{\pi}(r,o)} - \Exp_{(r,o) \sim \Mbar'(\pi)}\left [ \log \frac{\Prm{\Mbar}{\pi}(r,o)}{\Prm{M}{\pi}(r,o)} \right ] 
\end{align*}
is sub-gaussian with parameter $\Exp_{r \sim \Mbar'(\pi)}[(\frac{1}{\sigma^2} (\fm[\Mbar'](\pi) - r) (\fm(\pi) - \fm[\Mbar](\pi)))^2] \le \frac{4}{\sigma^2}$, which verifies \Cref{asm:bounded_likelihood} with $\VM^2 = \frac{8}{\sigma^2}$.

Finally, we bound the covering number. Let $\cE := \{ |r| \le 1 + \sqrt{2\sigma^2 \log(2/\mu)} \}$. Elementary manipulations show that $\Prm{\Mbar}{\pi}(\cE^c) \le \mu$ for any $\Mbar \in \cM$ and $\pi$. Using the same calculation as above, we have
\begin{align*}
\log \frac{\Prm{M}{\pi}(r,o)}{\Prm{M'}{\pi}(r,o)} & = \frac{1}{2 \sigma^2} (r - \fm(\pi))^2 -\frac{1}{2\sigma^2} (r - \fm[M'](\pi))^2 \le \frac{1 + |r|}{\sigma^2} \cdot |\fm(\pi) - \fm[M'](\pi)|
\end{align*}
where the inequality follows from the Mean Value Theorem. We therefore
have that for any $M,M'\in\cM$,
\begin{align*}
\sup_{r,o,\pi \ : \ |r| \le 1 + \sqrt{2\sigma^2 \log(2/\mu)}} \left | \log \frac{\Prm{M}{\pi}(r,o)}{\Prm{M'}{\pi}(r,o)} \right | \le \frac{2 + \sqrt{2\sigma^2 \log(2/\mu)}}{\sigma^2} \cdot \sup_\pi |\fm(\pi) - \fm[M'](\pi)|.
\end{align*}
It follows that if we can form a $\frac{\sigma^2 \cdot \rho}{2 +
  \sqrt{2\sigma^2 \log(2/\mu)}}$-cover of $\cM$ in the distance
$d(M,M') = \sup_{\pi \in \Pi} | \fm(\pi) - \fm[M'](\pi)|$, this will
serve as an $(\rho,\mu)$ cover of $\cM$. 
\end{proof}

\subsubsection{Discrete Structured Bandits}\label{sec:discrete_structured_bandit}

As a first example of bandits with Gaussian noise, we present an additional class that satisfies the uniformly regular assumption. 

\begin{example}[Discrete Structured Bandits]\label{ex:discrete_structured_bandit}
Fix $\delmin > 0$, and consider a discrete reward space $\cR \subseteq [0,1]$ satisfying $\min_{r,r' \in \cR} | r - r' | \ge \delmin$. Consider any function class $\cF \subseteq (\Pi \rightarrow \cR)$ defined such that each $f \in \cF$ has a unique optimal decision. Let our model class be defined as
\begin{align*}
\cM = \{ M(\pi) = \cN(f(\pi),1) \mid f \in \cF \}.
\end{align*} I
t is straightforward to show that \Cref{asm:smooth_kl_kl} and \Cref{asm:bounded_likelihood} are met with $\LKL, \VM \le 4$, and that \Cref{asm:covering} is satisfied with $\dcov$ scaling with the log-covering number of $\cF$ in the distance $d(f,f') = \sup_{\pi \in \Pi} | f(\pi) - f'(\pi) |$, and $\Ccov = \bigoh(1)$. Furthermore, \Cref{asm:mingap} is satisfied by construction of $\cR$ and $\cF$, and we can bound $\ncMeps \le \frac{2}{\delmin^2}$.\footnote{It is not difficult to see that, given the construction of $\cR$, once the optimal arm has been played $\frac{2}{\delmin^2}$ times, no additional information can be extracted from playing it.} We thus have the following corollary to \Cref{thm:upper_main}.

\begin{corollary}
In the discrete structured bandits setting considered above, if $\gst > 0$, the regret of \mainalg is bounded as
\begin{align*}
 \Exp\sups{\Mst}[\RegDM] & \le (1+\veps) \cst \cdot \log (T) + \aecflip[\veps/12](\cM)\cdot \poly \prn*{ \max_{M \in \cM} \gm, \dcov, \tfrac{1}{\veps}, \tfrac{1}{\delmin}, \log \log T} \cdot \log^{1/2}(T)
\end{align*}
\end{corollary}
As we show in \Cref{ex:eluder_bandits}, the \CompShort in this setting can be bounded in terms of the eluder dimension of $\cF$. 
\end{example}

\begin{proof}[Proof for \cref{ex:discrete_structured_bandit}]
  We provide calculations for the discrete structured bandits setting
  of
  \Cref{ex:discrete_structured_bandit}. First, note that \Cref{asm:smooth_kl_kl,asm:bounded_likelihood,asm:covering}
  all hold by \Cref{lem:gauss_bandits_satisfy_asm}. To bound $\ncM$, consider $M \in \cM$ and $M' \in \cMalt(M)$. Note
  that either $\Dkl{M(\pim)}{M'(\pim)} = 0$, in which case there is no
  advantage to playing $\pim$, or, due to the discretization of the
  means, $\Dkl{M(\pim)}{M'(\pim)} \ge \frac{1}{2} \delmin^2$. Thus for
  any allocation $\eta\in\bbR^{\Pi}_{+}$, as long as $\eta({\pim}) \ge \frac{2}{\delmin^2}$, we have
  $\eta({\pim}) \Dkl{M(\pim)}{M'(\pim)} \ge 1$. It follows that there
  is no advantage to choosing $\eta({\pim})$ larger than
  $\frac{2}{\delmin^2}$, so we can bound
  $\ncM \le \frac{2}{\delmin^2}$.
\end{proof}

\subsubsection{Multi-Armed Bandits (\creftitle{ex:mab_upper})}
In this section we prove the result in \Cref{ex:mab_upper}. First,
note that for any $M \in \cM$, we have $\gm \le
A/\delminm$. \Cref{asm:smooth_kl,asm:D_to_hel,asm:bounded_likelihood}
are met due to \Cref{lem:gauss_bandits_satisfy_asm}, and with
constants $\LKL,\VM \le 4$, $\dcov = \bigoh(A)$, and $\Ccov =
\bigoh(1)$. By \Cref{lem:mab_regular}---stated and proven
below---$\Mst$ is a regular model with $\LMst = \sqrt{2}$ as long as
$\fm[\Mst](\pist) < 1$.  It remains to bound the \CompShort.
\begin{proof}[\pfref{prop:aec_bound_mab}]
It is immediate to see that $\Cexp(\cM,\veps) \le \bigoh(\frac{A}{\veps})$ by choosing the exploration distribution to be uniform over $A$. By \Cref{lem:aec_Cexp_bound}, we then 
\begin{align*}
\aecflipM{\veps}{\cM}(\cMst) \le c_1 \cdot \frac{A^3}{\veps^2 \delLst^6}
\end{align*}
for a universal constant $c_1$.
By \Cref{prop:regular_class_to_nM}, we have
\begin{align*}
\nst_{\veps/36} \le c_2 \cdot \frac{\gst}{\delminst} \cdot \prn*{1 + \frac{\gst}{\veps \delminst}} \le c_2 \cdot \frac{A^2}{\veps (\delminst)^4},
\end{align*}
for a universal constant $c_2$, so we can lower bound $\delLst \ge c_3 \veps (\delminst)^4/A^2$, giving
\begin{align*}
\aecflipM{\veps}{\cM} (\cMst) \le c_4 \cdot \frac{A^{15}}{\veps^8 (\delminst)^{24}}.
\end{align*}
\end{proof}

\begin{lemma}\label{lem:mab_regular}
In the multi-armed bandit setting of \Cref{ex:mab_upper}, any model $\Mst\in\cM$ is a regular model with $\LMst = \sqrt{3}$ as long as $\fm[\Mst](\pist) < 1$. 
\end{lemma}
\begin{proof}[Proof of \Cref{lem:mab_regular}]
Let $M' \in \cMalt(\Mst)$ and assume that
$\Dkl{\Mst(\pist)}{M'(\pist)} > 0$. 

\paragraph{Case 1: $\fm[M'](\pim[M']) + \fm[\Mst](\pist) - \fm[M'](\pist) \le 1$}
Let
$M''$ denote the Gaussian bandit instance given by $\fm[M''](\pi) = \fm[M'](\pi) +
\fm[\Mst](\pist) - \fm[M'](\pist)$; by our assumption, $M'' \in \cM$, and $M'' \in \cMalt(\Mst)$. Furthermore, $\fm[M''](\pist) = \fm[\Mst](\pist)$ which implies $\Dkl{\Mst(\pist)}{M''(\pist)} = 0$ as desired. Finally, for any $\pi$, we have
\begin{align*}
\fm[M'](\pi) - \fm[M''](\pi)  =  \fm[\Mst](\pist) - \fm[M'](\pist).
\end{align*}
This implies that for all $\pi$, since all models in $\cM$ have unit
Gaussian rewards, using the expression for the KL divergence given in \eqref{eq:bandits_Dkl}:
\begin{align*}
| \Dkl{\Mstar(\pi)}{M'(\pi)} - \Dkl{\Mstar(\pi)}{M''(\pi)} | \le |\fm[\Mst](\pist) - \fm[M'](\pist)| = \sqrt{2 \Dkl{\Mst(\pist)}{M'(\pist)}}
\end{align*}
which implies the condition of \Cref{def:regular_class2} is met with $\LMst = \sqrt{2}$. 

\paragraph{Case 2: $\fm[M'](\pim[M']) + \fm[\Mst](\pist) - \fm[M'](\pist) > 1$}
For this case
the model $M''$ constructed in Case 1 will not be in $\cM$. Assume first that $\fm[\Mst](\pist) \ge \fm[M'](\pim[M'])$ and in this case define $M''$ to be the instance
\begin{align*}
\fm[M''](\pi) = \begin{cases} \min \{ \fm[\Mst](\pist), \fm[M'](\pi) + \fm[\Mst](\pist) - \fm[M'](\pist) \} & \pi \neq \pim[M'] \\
\fm[\Mst](\pist) + \delta  & \pi = \pim[M'] \end{cases}
\end{align*}
or some $\delta > 0$ such that $\fm[\Mst](\pist) + \delta < 1$ (note that such a $\delta$ exists since we have assumed $\fm[\Mst](\pist) < 1$). Note that we now have $M'' \in \cM$, and $\fm[M''](\pi) < \fm[M''](\pim[M'])$ for all $\pi \neq \pim[M']$, so $M'' \in \cMalt(\Mst)$. Furthermore, we have $\fm[M''](\pist) = \fm[\Mst](\pist)$, so $\Dkl{\Mst(\pist)}{M''(\pist)} = 0$. For $\pi \neq \pim[M']$, if $\min \{ \fm[\Mst](\pist), \fm[M'](\pi) + \fm[\Mst](\pist) - \fm[M'](\pist) \} = \fm[\Mst](\pist)$, this implies that
\begin{align*}
\fm[\Mst](\pist) \le \fm[M'](\pi) + \fm[\Mst](\pist) - \fm[M'](\pist) \implies \fm[\Mst](\pist) - \fm[M'](\pi) \le \fm[\Mst](\pist) - \fm[M'](\pist).
\end{align*}
Since we have assumed $\fm[\Mst](\pist) \ge \fm[M'](\pim[M'])$, this implies that
\begin{align*}
|\fm[M''](\pi) - \fm[M'](\pi)| = | \fm[\Mst](\pist) - \fm[M'](\pi)| \le |\fm[\Mst](\pist) - \fm[M'](\pist)|
\end{align*}
So by the expression given for the KL divergence in \eqref{eq:bandits_Dkl}, we have
\begin{align}\label{eq:mab_regular_kl_cond}
| \Dkl{\Mstar(\pi)}{M'(\pi)} - \Dkl{\Mstar(\pi)}{M''(\pi)} | \le \sqrt{3 \Dkl{\Mst(\pist)}{M'(\pist)}}.
\end{align}
For $\pi \neq \pim[M']$ with $\min \{ \fm[\Mst](\pist), \fm[M'](\pi) + \fm[\Mst](\pist) - \fm[M'](\pist) \} =  \fm[M'](\pi) + \fm[\Mst](\pist) - \fm[M'](\pist)$, the bound on $| \Dkl{\Mstar(\pi)}{M'(\pi)} - \Dkl{\Mstar(\pi)}{M''(\pi)} |$ follows identically to Case 1. For $\pi = \pim[M']$, since we have assumed that $ \fm[\Mst](\pist) \ge \fm[M'](\pim[M']) $ we have
\begin{align*}
|\fm[M''](\pim[M']) - \fm[M'](\pim[M'])| & = \fm[\Mst](\pist) - \fm[M'](\pim[M']) + \delta \le \fm[\Mst](\pist) - \fm[M'](\pist) + \delta.
\end{align*}
For small enough $\delta$, this implies that \eqref{eq:mab_regular_kl_cond} is satisfied for $\pi = \pim[M']$ as well.

Consider now the case where $\fm[\Mst](\pist) < \fm[M'](\pim[M'])$. In this case define $M''$ by
\begin{align*}
\fm[M''](\pi) = \begin{cases} \min \{ \fm[M'](\pim[M']) - \delta, \fm[M'](\pi) + \fm[\Mst](\pist) - \fm[M'](\pist) \} & \pi \neq \pim[M'] \\
\fm[M'](\pim[M'])  & \pi = \pim[M'] \end{cases}
\end{align*}
for $\delta > 0$ small enough that $\fm[M'](\pim[M']) - \delta > \fm[\Mst](\pist)$. Note that $M''$ and $M'' \in \cMalt(\Mst)$ by construction, and that $\fm[M''](\pist) = \fm[\Mst](\pist)$ by our choice of $\delta$, so $\Dkl{\Mst(\pist)}{M''(\pist)} = 0$. For $\pi \neq \pim[M']$, if $ \min \{ \fm[M'](\pim[M']) - \delta, \fm[M'](\pi) + \fm[\Mst](\pist) - \fm[M'](\pist) \} = \fm[M'](\pim[M']) - \delta$, then we have
\begin{align*}
| \fm[M''](\pi) - \fm[M'](\pi)| & = | \fm[M'](\pim[M']) - \delta - \fm[M'](\pi) | \\
& \le \fm[M'](\pim[M']) - \fm[M'](\pi) + \delta \\
& \le |\fm[\Mst](\pist) - \fm[M'](\pist)| + \delta
\end{align*}
where for the final inequality we have used that $ \min \{ \fm[M'](\pim[M']) - \delta, \fm[M'](\pi) + \fm[\Mst](\pist) - \fm[M'](\pist) = \fm[M'](\pim[M']) - \delta$. It follows that \eqref{eq:mab_regular_kl_cond} is satisfied for this $\pi$ for sufficiently small $\delta$. If we instead have $ \min \{ \fm[M'](\pim[M']) - \delta, \fm[M'](\pi) + \fm[\Mst](\pist) - \fm[M'](\pist) \}= \fm[M'](\pi) + \fm[\Mst](\pist) - \fm[M'](\pist)$, then the bound on $| \Dkl{\Mstar(\pi)}{M'(\pi)} - \Dkl{\Mstar(\pi)}{M''(\pi)} |$ follows identically to Case 1.

For $\pi = \pim[M']$, we have $|\fm[M''](\pim[M']) - \fm[M'](\pim[M'])| = 0$.
This proves the result.

\end{proof}

\subsubsection{Structured Bandits with Bounded Eluder Dimension (\creftitle{ex:eluder_bandits})}\label{sec:eluder_ex_proofs}

In this section, we give generic bounds on the uniform exploration
coefficient and \CompText for structured bandit classes with bounded eluder
dimension (cf. \cref{def:eluder}). These result are used by subsequent
examples, including linear bandits.

\begin{lemma}\label{lem:Cexp_eluder_bound}
  Let $\cM\crl*{M(\pi)=\cN(f(\pi),1)\mid{}f\in\cF}$. Then for all $\veps>0$, we have
\begin{align*}
\Cexp(\cM, \veps) \le \frac{16 \dE(\cF,\sqrt{\veps}/2)}{\veps}.
\end{align*}
\end{lemma}
\begin{proof}[Proof of \Cref{lem:Cexp_eluder_bound}]
  Let $\xi\in\simplex(\cM)$. Recall the expression for KL divergence between Gaussians of unit
variance:
\begin{align*}
\Exp_{\Mbar \sim \xi}[\Exp_{p}[\Dkl{\Mbar(\pi)}{M(\pi)}]] =
  \frac{1}{2} \Exp_{\Mbar \sim \xi}[\Exp_{p}[(\fm(\pi) -
  \fm[\Mbar](\pi))^2]].
\end{align*}
Abbreviate $\dE := \dE(\cF,\sqrt{\veps}/2)$ and let $\{ \pi_1,
\ldots, \pi_{\dE} \}$ denote a maximal sequence of $\veps$-independent points. By the definition of the eluder dimension, for any $\pi\in\Pi$ and any $M,\Mbar \in \cM$, we have:
\begin{align*}
\sqrt{\sum_{i=1}^{\dE} (\fm(\pi_i) - \fm[\Mbar](\pi_i))^2} \le \sqrt{\veps/2} \implies |\fm(\pi) - \fm[\Mbar](\pi)| \le \sqrt{\veps/2}.
\end{align*}
Now, set $p$ to be the uniform distribution over $\{ \pi_1, \ldots,
\pi_{\dE} \}$. Assume that $M,M'\in\cM$ are such that 
\[
\max_{M''\in\crl{M,M'}}\En_{\Mbar\sim\xi}\Exp_{p}[\Dkl{\Mbar(\pi)}{M''(\pi)}]
=\frac{1}{2}\max_{M''\in\crl{M,M'}}\En_{\Mbar\sim\xi}\Exp_{p}[(\fm[M''](\pi) - \fm[\Mbar](\pi))^2] \le
\veps/(16\dE),
\]
Markov's
inequality implies that for each $M''\in\crl{M,M'}$, with probability
at least $3/4$ over the draw of $\Mbar\sim\xi$,
\[
\Exp_{p}[(\fm[M''](\pi) - \fm[\Mbar](\pi))^2] \leq{} \veps/(2\dE).
\]
Taking a union bound, we conclude that with probability at least $1/2$
over the draw of $\Mbar\sim\xi$,
\begin{equation}
  \label{eq:good_mbar}
\max_{M''\in\crl{M,M'}}\Exp_{p}[(\fm[M''](\pi) - \fm[\Mbar](\pi))^2] \leq{} \veps/(2\dE).
\end{equation}
Going forward, let $\Mbar\in\cM$ be any model such that
\eqref{eq:good_mbar} holds; we have just proven that such a model
exists. It follows from the maximality of $\pi_1,\ldots,\pi_{\dE}$ and the
definition of $p$ that for all $\pitil\in\Pi$, and $M''\in\cM$,
\begin{align*}
\Exp_{p}[(\fm[M''](\pi) - \fm[\Mbar](\pi))^2] \le \veps/(2 \dE) \implies ( \fm[M''](\pitil) - \fm[\Mbar](\pitil))^2 \le \veps/2.
\end{align*}
In particular, since this holds for both $M \in \{ M',M'' \}$, and
since \eqref{eq:good_mbar} holds, we have that for all $\pi$,
\begin{align*}
\Dkl{M'(\pi)}{M''(\pi)} & = \frac{1}{2} (\fm[M'](\pi) - \fm[M''](\pi))^2 \\
& \le (\fm[M'](\pi) - \fm[\Mbar](\pi))^2 + (\fm[\Mbar](\pi) - \fm[M''](\pi))^2 \\
& \le \veps.
\end{align*}
As this is the condition required by \Cref{def:uniform_exp}, it
suffices to take $\Cexpxi(\veps) = 16 \dE(\cF,\sqrt{\veps}/2) /
\veps$. Since this bound holds uniformly for all choices of $\xi$, the
result follows.
\end{proof}

\begin{proof}[Proof of \Cref{prop:aec_bound_eluder}]
  The bound \[
    \Cexp(\cMst,\delta)
    \le \frac{16\dE(\cF,\sqrt{\delta}/2)}{\delta}.
  \]
    follows from \Cref{lem:Cexp_eluder_bound}, since $\dE(\cF',\delta)
    \le \dE(\cF,\delta)$ for all $\cF'\subseteq\cF$.

By \Cref{lem:aec_Cexp_bound} we can bound $\aecflipM{\veps}{\cM}(\cMst)$:
\begin{align*}
\aecflipM{\veps}{\cM}(\cMst) \le \Cexp(\cMst,\delta) \quad \text{for} \quad \delta & = \min_{M \in \cMst} \min \crl*{ \min \left \{ \frac{1}{81 \LKL}, \frac{\delminm}{34 \VM} \right \} \cdot \frac{\veps}{2 \gm/\delminm + \nmepsc{\veps/36}{M}}, \frac{\delminm}{3}}^2.
\end{align*}
    
Lemma \Cref{lem:gl_const_Cexp_bound} implies that for all $M \in \cMst$, 
\begin{align*}
\gm \le \Cexp(\cMst,\tfrac{1}{4} (\delminm)^2) \le \frac{64 \dE(\cF,\frac{1}{2} \delminm)}{(\delminm)^2} \le \frac{64 \dE(\cF,\frac{1}{2} \DelLst)}{\DelLst^2}
\end{align*}
where we have used that the eluder dimension increases as its scale
$\veps$ decreases; by \Cref{lem:gauss_bandits_satisfy_asm}, it
suffices to take $\LKL = \VM = 2$.  A sufficient value for $\delta$ is therefore
\begin{align*}
\delta = c \cdot \frac{\veps^2 \delLst^8}{\dE(\cF,\frac{1}{2} \delLst)^2}
\end{align*}
The result follows.

\end{proof}

\subsubsection{Linear Bandits (\creftitle{ex:linear_bandits})}\label{sec:linear_bandit_proofs}
\begin{proof}[Proof of \Cref{prop:linear_bandit}]
The result follows directly from \Cref{prop:aec_bound_eluder}, since it is known that linear bandits have eluder dimension which scales as $\dE(\cF,\veps) = \bigoh(d \cdot \log 1/\veps)$ \citep{russo2013eluder}.
\end{proof}

In what follows, we prove \Cref{prop:lin_bandit_nm_bound}, providing sufficient conditions under which it is
possible to bound the regularity constant $\LMst$ (and hence
$\nst_\veps$) for linear bandits.

We begin with an geometric assumption on $\Theta$, $\cX$, and
$\thetast$, which we will show ensures that $\Mst(\pi)=\cN(\tri*{x_\pi,\thetast},1)$ is a regular
model.  To state our condition, we denote, for any vectors $x$ and
$y$, we define $x_y$ and $x_{\bar{y}}$ to be unique vectors satisfying $x = x_y + x_{\bar{y}}$ for $x_y \parallel y$ and $x_{\bar{y}} \perp y$. 

\begin{assumption}[Regular Linear Bandits]\label{asm:regular_lin_band}
The sets $\Theta,\cX$ and model parameter $\thetast$ satisfy:
\begin{enumerate}
\item $\Theta$ is a convex polytope.
\item For all $\theta \in \Theta$, we have that there exists some $\delta_\theta > 0$ such that $\{ \theta' \in \R^d \ : \ \| \theta' - \theta \|_2 \le \delta_{\theta} \} \subseteq \Theta$. 
\item Letting $\xst \in \cX$ denote the optimal action for $\thetast$, we have
\begin{align*}
\crl*{ \theta \in \R^d \ : \ \| \theta - \thetast \|_2 \le \max_{x \in \cX, x \neq \xst} \delst(x) / \| x_{\xstbar} \|_2 } \subseteq \Theta.
\end{align*}
\end{enumerate}
\end{assumption}
The first two points above are quite mild. The primary restriction of
\Cref{asm:regular_lin_band} is Point 3, which requires that $\thetast$
is located sufficiently far within the interior of $\Theta$. Using \Cref{asm:regular_lin_band} we can state the full version of \Cref{prop:lin_bandit_nm_bound}.

\begin{proposition}[Full Version of \Cref{prop:lin_bandit_nm_bound}]\label{prop:lin_bandit_nm_bound_full}
If $\Theta,\cX$, and $\thetast$ satisfy \Cref{asm:regular_lin_band}, then $\nepsst$ is bounded by a polynomial function of $d, 1/\delminst, 1/\veps, \gst$, and a geometry-dependent term scaling with the structure of $\cX$ and $\Theta$. 
\end{proposition}

\begin{remark}[Comparison to Existing Work]
We remark that \Cref{asm:regular_lin_band} is similar to the conditions required by existing works which achieve instance-optimality in linear bandits with polynomial lower-order terms \citep{tirinzoni2020asymptotically,kirschner2021asymptotically}. Though neither of these works explicitly states such a condition, closer inspection of their analysis reveals it is indeed required. In particular, the proof of Lemma 1 of \cite{tirinzoni2020asymptotically} relies on a result from \cite{degenne2020structure} which shows that a condition analogous to \Cref{def:regular_class2} is met for linear bandits. However, the proof given in \cite{degenne2020structure} appears to only hold when $\Theta$ is unbounded, or a condition such as \Cref{asm:regular_lin_band} holds. As \cite{tirinzoni2020asymptotically} assumes that $\Theta$ is bounded, their results therefore only appear to hold if a condition similar to \Cref{asm:regular_lin_band} also holds. Similarly, in the proof of Lemma 10 of \cite{kirschner2021asymptotically}, it is assumed that for every arm $x \neq x_{\pist}$, there exists some instance in the alternate set with optimal arm $x$. To satisfy this condition, it appears that an assumption similar to \Cref{asm:regular_lin_band} is required. 

Thus, while not stated explicitly in the existing literature, it therefore seems that all existing results which obtain reasonable lower-order terms require an assumption similar to \Cref{asm:regular_lin_band}. Removing this assumption (or showing it is necessary) is an interesting direction for future work. 
\end{remark}

\begin{proof}[Proof of \Cref{prop:lin_bandit_nm_bound_full}]
Under \Cref{asm:regular_lin_band}, this follows directly from \Cref{lem:regular_lin_band} and \Cref{prop:regular_class_to_nM}.
\end{proof}

\begin{lemma}\label{lem:regular_lin_band}
Under \Cref{asm:regular_lin_band}, the linear bandit model $\Mst$ is
regular for some $\LMst<\infty$ whose value depends on the geometry
of $\Theta$ and $\cX$.
\end{lemma}
\begin{proof}[Proof of \Cref{lem:regular_lin_band}]
Fix some $\thetast \in \Theta$ and let $\xst$ denote its optimal arm. 
Let $\Thetaalt(\thetast) \subseteq \Theta$ denote parameters with optimal arm $x \neq \xst$. 
Assume there exists some $\theta \in \Thetaalt(\thetast)$ such that $\inner{\theta - \thetast}{\xst} \neq 0$ (if this is not the case, $\Mst$ immediately satisfies \Cref{def:regular_class2} and we are done).
Let $\Theta_x = \{ \theta \in \Theta \ : \ \inner{x}{\theta} \ge \inner{\xst}{\theta} \}$, $\Thetast = \{ \theta \in \R^d \ : \ \inner{\theta - \thetast}{\xst} = 0 \}$, and $\Thetast_x = \Theta_x \cap \Thetast$. We first show that, under \Cref{asm:regular_lin_band}, $\Thetast_x \neq \emptyset$ for all $x \in \cX$, $x \neq \xst$. We then use this to show that $\Mst(\pi)=\cN(\tri*{x_\pi,\thetast},1)$ is a regular model.

\paragraph{Part 1: $\Thetast_x \neq \emptyset$}
Fix some $x \in \cX$ with $x \neq \xst$. Consider $\theta = \thetast + a x_{\xstbar}$ for some $a\in\bbR$ to be chosen. By construction we have $\inner{\theta}{\xst} = \inner{\thetast}{\xst}$, which implies that $\theta \in \Thetast$ for all $a\in\bbR$. We wish to choose $a$ large enough that $\inner{\theta}{x} \ge \inner{\theta}{\xst}$. Note that
\begin{align*}
\inner{\theta}{x} = \inner{\thetast}{x} + a \inner{x_{\xstbar}}{x_{\xstbar} + x_{\xst}} = \inner{\thetast}{x} + a \| x_{\xstbar} \|_2^2
\end{align*}
and $\inner{\theta}{\xst} = \inner{\thetast}{\xst}$. Thus, to satisfy $\inner{\theta}{x} \ge \inner{\theta}{\xst}$, we need
\begin{align*}
a \| x_{\xstbar} \|_2^2 \ge \inner{\thetast}{\xst - x} \iff a \ge \Delst(x) /  \| x_{\xstbar} \|_2^2.
\end{align*}
Let $a = \Delst(x) /  \| x_{\xstbar} \|_2^2$, then it follows that $\inner{\theta}{x} \ge \inner{\theta}{\xst}$. Furthermore, we can bound
\begin{align*}
\| \theta - \thetast \|_2 \le a \| x_{\xstbar} \|_2 = \Delst(x) / \| x_{\xstbar} \|_2 .
\end{align*}
Under \Cref{asm:regular_lin_band}, it follows that $\theta \in \Theta$.

\paragraph{Part 2: $\Mst$ is a Regular Model}
Let $\Thetabar_x = \{ \theta - \thetast \ : \ \theta \in \Theta_x
\}$. Note that, since $\Theta$ is a convex polytope, and $\Theta_x$
simply adds a linear inequality constraint, $\Theta_x$ is also
convex. Let $\Thetastbar_x = \{ \phi \in \Thetabar_x \ : \
\inner{\phi}{\xst} = 0 \}$. From Part 1, we have $\Thetastbar_x \neq \emptyset$. 
Lemma 23 of \cite{kirschner2021asymptotically} then gives that there
exists a geometry-dependent constant $C(\Theta,\cX)$ such that, for all $\phi \in \Thetabar_x$:
\begin{align*}
\min_{\phi' \in \Thetastbar_x} \| \phi - \phi' \|_2 \le C(\Theta,\cX) \cdot |\inner{\phi}{\xst}|.
\end{align*}
This implies that for all $\theta \in \Theta_x$, we have:
\begin{align*}
\min_{\theta' \in \Thetast_x} \| \theta - \theta' \|_2 \le C(\Theta,\cX) \cdot |\inner{\theta - \thetast}{\xst}|.
\end{align*}
Now consider some $\theta \in \Thetaalt(\thetast)$, and assume that
$\inner{\theta - \thetast}{\xst} \neq 0$ (by assumption such a $\theta$ exists). Assume that $\theta$ has optimal arm $x$, which implies that $\theta \in \Theta_x$. By what we have just shown, we know that there exists some $\theta' \in \Theta$ with $\inner{\theta'}{x} > \inner{\theta'}{\xst}$ so that $\theta' \in \Thetaalt(\thetast)$, $\inner{\theta' - \thetast}{\xst} = 0$, and
\begin{align*}
\| \theta - \theta' \|_2 \le  C(\Theta,\cX) \cdot |\inner{\theta - \thetast}{\xst}|.
\end{align*}
Note that, for any $x' \in \cX$, we have
\begin{align*}
| \Dkl{\thetast(x')}{\theta(x')} - \Dkl{\thetast(x')}{\theta'(x')} | & = \frac{1}{2} | \inner{\thetast - \theta}{x'}^2 - \inner{\thetast - \theta'}{x'}^2| \\
& \le 2 \max_{\theta'' \in \Theta} | \inner{\theta''}{x'} | \cdot |\inner{\theta - \theta'}{x'} | \\
& \le 2 \max_{\theta'' \in \Theta} \| \theta'' \|_2 \| x' \|_2^2  \cdot  \| \theta - \theta' \|_2.
\end{align*}
Furthermore, note that
\begin{align*}
\sqrt{\Dkl{\thetast(\xst)}{\theta(\xst)}} = \frac{1}{\sqrt{2}} | \inner{\thetast - \theta}{\xst} |,
\end{align*}
so
\begin{align*}
| \Dkl{\thetast(x')}{\theta(x')} - \Dkl{\thetast(x')}{\theta'(x')} | \le \prn*{2\sqrt{2} C(\Theta,\cX) \max_{\theta'' \in \Theta, x'' \in \cX} \| \theta'' \|_2 \| x'' \|_2^2} \cdot \sqrt{\Dkl{\thetast(\xst)}{\theta(\xst)}}.
\end{align*}

As $\theta \in \Thetast(\thetast)$ was arbitrary, we have therefore shown that $\Mst$ is a regular model with
\begin{align*}
\LMst = \prn*{2\sqrt{2} \cdot C(\Theta,\cX) \cdot  \max_{\theta'' \in \Theta, x'' \in \cX} \| \theta'' \|_2 \| x'' \|_2^2}^2.
\end{align*}

\end{proof}

\subsubsection{Generalized Linear Models (\creftitle{ex:glm})}

\begin{proof}[Proof Sketch for \cref{ex:glm}]
The bound on the \CompShort follows as in \Cref{prop:linear_bandit},
  using that the eluder dimension for generalized linear models is
  bounded as
  $\bigoh(d \cdot (\frac{\sigmamax}{\sigmamin})^2 \cdot \log
  \frac{\sigmamax}{\veps})$ \citep{russo2013eluder}. For the other regularity assumptions, note that by the Mean Value Theorem, we have
  \begin{align*}
    |\Dkl{M(\pi)}{M'(\pi)} - \Dkl{M(\pi)}{M''(\pi)}| & = \frac{1}{2} | (\link(\inner{\theta}{x}) - \link(\inner{\theta'}{x}))^2 - (\link(\inner{\theta}{x}) - \link(\inner{\theta''}{x}))^2| \\
                                                     & \le 2 \sigmamax | \inner{\theta' - \theta''}{x}|
  \end{align*}
  and
  \begin{align*}
    \sqrt{\Dkl{M(\pi)}{M'(\pi)}} = \frac{1}{\sqrt{2}} | \link(\inner{\theta}{x}) - \link(\inner{\theta'}{x})| \ge \frac{\sigmamin}{\sqrt{2}} | \inner{\theta - \theta'}{x} | .
  \end{align*}
  In light of these inequalities, bounds on all relevant regularity
  parameters for generalized linear bandits follow from similar
  reasoning to the proofs for linear bandits. In particular,
  the conclusion of \Cref{lem:regular_lin_band} holds for generalized linear bandits
  under \Cref{asm:regular_lin_band}, with $\LMst$ as in
  \Cref{lem:regular_lin_band}, but scaled by
  $(\frac{\sigmamax}{\sigmamin})^2$).
  
\end{proof}

\subsection{Contextual Bandits with Finitely Many Actions
  (\creftitle{ex:cb})}\label{sec:contextual_proofs}

In this setting we take $\D{M(\pi)}{\Mbar(\pi)} \leftarrow
\Dkl{M(\pi)}{\Mbar(\pi)}$ for $D$ the divergence employed by \mainalgb. Note that we have
\begin{align*}
\Dkl{M(\pi)}{\Mbar(\pi)} = \frac{1}{2} \Exp_{x \sim \pX}[ \Exp_{a \sim \pi(x)} [ (\fm(x,a) - \fm[\Mbar](x,a))^2]].
\end{align*}

\begin{lemma}\label{lem:con_bandits_satisfy_asm}
  In the contextual bandits setting of \cref{ex:cb}:
  \begin{enumerate}
    \item \Cref{asm:smooth_kl,asm:D_to_hel,asm:bounded_likelihood} hold with
    parameters
    \begin{align*}
      \LKL = \VM = 2 \sqrt{2}
    \end{align*}
    and $\D{\cdot}{\cdot}=\Dkl{\cdot}{\cdot}$.
  \item We can bound $\Ncov(\cM,\rho,\mu)$ by the covering number of $\cM$ in the distance $d(M,M') = \sup_{x \in \cX, a \in \cA} | \fm(x,a) -
    \fm[M'](x,a)|$ at tolerance $\frac{\sigma^2 \cdot \rho}{2 + \sqrt{2 \log(2/\mu)}}$. Furthermore, it suffices to take
  $\cE := \{ |r| \le 1 + \sqrt{2 \log(2/\mu)} \}$.
  \end{enumerate}

\end{lemma}
\begin{proof}[Proof of \Cref{lem:con_bandits_satisfy_asm}]
Using the expression for the KL divergence given above, for any $M,M',\Mbar \in \cM$ and $\pi \in \Pi$, we have
\begin{align*}
& \left |\kl{M(\pi)}{M'(\pi)} - \kl{\Mbar(\pi)}{M'(\pi)} \right | \\
& \qquad = \frac{1}{2} \left | \Exp_{x \sim \pX}[ \Exp_{a \sim \pi(x)} [ (\fm(x,a) - \fm[M'](x,a))^2]] - \Exp_{x \sim \pX}[ \Exp_{a \sim \pi(x)} [ (\fm[\Mbar](x,a) - \fm[M'](x,a))^2]] \right | \\
& \qquad \le \frac{1}{2} \Exp_{x \sim \pX}\brk*{\Exp_{a \sim \pi(x)} \brk*{ \left | (\fm(x,a) - \fm[M'](x,a))^2 - (\fm[\Mbar](x,a) - \fm[M'](x,a))^2 \right | }} \\
& \qquad \le 2 \Exp_{x \sim \pX}\brk*{\Exp_{a \sim \pi(x)} \brk*{ \left | \fm(x,a)  - \fm[\Mbar](x,a)  \right | }} \\
& \qquad \le 2 \sqrt{\Exp_{x \sim \pX}\brk*{\Exp_{a \sim \pi(x)} \brk*{ \left ( \fm(x,a)  - \fm[\Mbar](x,a)  \right )^2 }}} \\
& \qquad = 2\sqrt{2} \sqrt{\Dkl{M(\pi)}{\Mbar(\pi)}}.
\end{align*}
This verifies \Cref{asm:smooth_kl} holds with $\LKL =
2\sqrt{2}$. \cref{asm:D_to_hel} is immediate.

To show that \Cref{asm:bounded_likelihood} is met, we note that
\begin{align*}
\log \frac{\Prm{\Mbar}{\pi}(r,o)}{\Prm{M}{\pi}(r,o)} & = \log \frac{\Prm{\Mbar}{\pi}(r \mid o) \Prm{\Mbar}{\pi}(o)}{\Prm{M}{\pi}(r \mid o) \Prm{M}{\pi}(o)}  =  \log \frac{\Prm{\Mbar}{\pi}(r \mid o)}{\Prm{M}{\pi}(r \mid o)},
\end{align*}
where the second equality holds because the context distribution is identical for all models. 
As the reward likelihoods conditioned on the context are Gaussian, a
calculation similar to \Cref{lem:gauss_bandits_satisfy_asm} shows that
\Cref{asm:bounded_likelihood} with $\VM = 2\sqrt{2}$. The covering
number bound also follows from the same reasoning as \Cref{lem:gauss_bandits_satisfy_asm}.
\end{proof}

\begin{lemma}\label{lem:con_bandits_Cexp_bound}
For the contextual bandit setting described above, we can bound
\begin{align*}
\Cexp(\cM,\veps) \le \frac{4A}{\veps}.
\end{align*}
\end{lemma}
\begin{proof}[Proof of \Cref{lem:con_bandits_Cexp_bound}]
Fix some $\xi \in \simplex_\cM$ and let $\piexp$ be uniform over $\cA$ for each context. Then, for any $p' \in \simplex_\Pi$, we have
\begin{align*}
\Exp_{\pi \sim p'}[\Dkl{M(\pi)}{M'(\pi)}] & =  \frac{1}{2} \Exp_{\pi \sim p'}[ \Exp_{x \sim \pX}[ \Exp_{a \sim \pi(x)} [ (\fm(x,a) - \fm[M'](x,a))^2]]] \\
& \le \Exp_{\Mbar \sim \xi}[\Exp_{\pi \sim p'}[ \Exp_{x \sim \pX}[ \Exp_{a \sim \pi(x)} [ (\fm(x,a) - \fm[\Mbar](x,a))^2 + (\fm[\Mbar](x,a) - \fm[M'](x,a))^2]]]] \\
& \le \sum_{a \in \cA} \Exp_{\Mbar \sim \xi}[ \Exp_{x \sim \pX} [ (\fm(x,a) - \fm[\Mbar](x,a))^2 + (\fm[\Mbar](x,a) - \fm[M'](x,a))^2]] \\
& = A \Exp_{\Mbar \sim \xi}[\Exp_{x \sim \pX} [ \Exp_{a \sim \pexp(x)}[ (\fm(x,a) - \fm[\Mbar](x,a))^2 + (\fm[\Mbar](x,a) - \fm[M'](x,a))^2]]] \\
& \le 2A \Exp_{\Mbar \sim \xi}[\Dkl{\Mbar(\piexp)}{M(\piexp)} + \Dkl{\Mbar(\piexp)}{M'(\piexp)}].
\end{align*}
It follows that if
\begin{align*}
\Exp_{\Mbar \sim \xi}[\Dkl{\Mbar(\piexp)}{M(\piexp)}] \le \frac{\veps}{4A} \quad \text{and} \quad \Exp_{\Mbar \sim \xi}[\Dkl{\Mbar(\piexp)}{M'(\piexp)}] \le \frac{\veps}{4A},
\end{align*}
then we can bound $\Exp_{\pi \sim p'}[\Dkl{M(\pi)}{M'(\pi)}] \le
\veps$. Thus, choosing $\pexp\in\simplex(\Pi)$ to place probability
mass 1 on $\piexp$, a sufficient bound on $\Cexp(\cM,\veps)$ is $4A/\veps$.
\end{proof}

\begin{proof}[Proof of \Cref{prop:con_band_aec_bound}]
The bound on $\Cexp(\cMst,\veps)$ follows from
\Cref{lem:con_bandits_Cexp_bound}. Hence, by \Cref{lem:aec_Cexp_bound} we can bound $\aecflipM{\veps}{\cM}(\cMst)$ as:
\begin{align*}
\aecflipM{\veps}{\cM}(\cMst) \le \frac{4A}{\delta} \quad \text{for} \quad \delta & =  \min_{M \in \cMst} \min \crl*{ \min \left \{ \frac{1}{81 \LKL}, \frac{\delminm}{34 \VM} \right \} \cdot \frac{\veps}{2 \gm/\delminm + \nmepsc{\veps/36}{M}}, \frac{\delminm}{3}}^2.
\end{align*}
 By \Cref{lem:gl_const_Cexp_bound},
we have that for all $M \in \cMst$, 
\begin{align*}
\gm \le \Cexp(\cMst,\frac{1}{4} (\delminm)^2) \le \frac{16 A}{\delLst^2}.
\end{align*}
By \Cref{lem:con_bandits_satisfy_asm}, we can take $\LKL =  \VM = 2
\sqrt{2}$. A sufficient choice for $\delta$ is therefore
\begin{align*}
\delta = c \cdot \frac{\veps^2 \delLst^8}{A^2}.
\end{align*}
The result follows.
\end{proof}

\subsection{Informative Arms (\creftitle{ex:informative_arm_upper})}\label{sec:inf_arm_proofs}
In this section, we provide calculations for the bandits with
informative arms setting in \Cref{ex:informative_arm_upper}.  We first
show that \cref{asm:smooth_kl_kl,asm:bounded_likelihood,asm:covering}
are satisfied.
\begin{itemize}
\item   \Cref{lem:gauss_bandits_satisfy_asm} If $\pi \in [A]$, then the
  response is simply Gaussian, so by
  \Cref{lem:gauss_bandits_satisfy_asm}, the condition of
  \Cref{asm:smooth_kl_kl} is met with $\LKL = 2$. If
  $\pi = \picirc_i$, then by the Mean Value Theorem we have
  \begin{align*}
    & \left | \Dkl{M(\pi)}{M'(\pi)} - \Dkl{\Mbar(\pi)}{M'(\pi)} \right |  \\
    & \qquad = \left |  \sum_{a \in [A]} \Prm{M}{\pi}(a) \log \frac{\Prm{M}{\pi}(a)}{\Prm{M'}{\pi}(a)} -  \sum_{a \in [A]} \Prm{\Mbar}{\pi}(a) \log \frac{\Prm{\Mbar}{\pi}(a)}{\Prm{M'}{\pi}(a)}  \right | \\
    & \qquad \le \prn*{1 + \max_{a \in [A]} \max \crl*{ \left | \log \frac{\Prm{M}{\pi}(a)}{\Prm{M'}{\pi}(a)} \right |, \left | \log \frac{\Prm{\Mbar}{\pi}(a)}{\Prm{M'}{\pi}(a)} \right |}} \cdot \sum_{a \in [A]} |\Prm{M}{\pi}(a) - \Prm{\Mbar}{\pi}(a) | \\
    & \qquad = \prn*{1 + \max_{a \in [A]} \max \crl*{ \left | \log \frac{\Prm{M}{\pi}(a)}{\Prm{M'}{\pi}(a)} \right |, \left | \log \frac{\Prm{\Mbar}{\pi}(a)}{\Prm{M'}{\pi}(a)} \right |}} \cdot \Dtv{\Prm{M}{\pi}}{\Prm{\Mbar}{\pi}} \\
    & \qquad \le \prn*{1 + \max_{a \in [A]} \max \crl*{ \left | \log \frac{\Prm{M}{\pi}(a)}{\Prm{M'}{\pi}(a)} \right |, \left | \log \frac{\Prm{\Mbar}{\pi}(a)}{\Prm{M'}{\pi}(a)} \right |}} \cdot \sqrt{\frac{1}{2} \Dkl{\Prm{M}{\pi}}{\Prm{\Mbar}{\pi}}}.
  \end{align*}
  Using the bound on the log-likelihood ratio given above, this
  verifies that \Cref{asm:smooth_kl_kl} holds with
  $\LKL = \max \{ 2, 1+\log \frac{A}{1-\beta} \}$.
\item 
If $\pi \in [A]$,
  then the since the response is Gaussian, by
  \Cref{lem:gauss_bandits_satisfy_asm}, the condition of
  \Cref{asm:bounded_likelihood} is met with $\VM = 2$. If
  $\pi = \picirc_i$, then for $M \in \cM$, either the observation is
  distributed as $1/A$, so $\Prm{M}{\pi}(r,o) = 1/A$ for all
  $o \in [A]$, or $i$ is the informative arm for instance $M$, in
  which case $\Prm{M}{\pi}(r,o) = (1-\beta)/A$ for $o \neq \pim$, and
  $\Prm{M}{\pi}(r,o) = \beta + (1-\beta)/A$ for $o = \pim$ (note that
  we can disregard $o = \perp$ since it occurs with probability 0 if
  an informative arm is pulled). The log-likehood ratio is then at
  most
  \begin{align*}
    \log \frac{\beta + (1-\beta)/A}{(1-\beta)/A} \le \log \frac{A}{1-\beta}.
  \end{align*}
  Thus, \Cref{asm:bounded_likelihood} is satisfied with
  $\VM = \max \{ 2, \log \frac{A}{1-\beta} \}$.
\item Using \Cref{lem:gauss_bandits_satisfy_asm}, it is easy to see that
  \Cref{asm:covering} is met with $\dcov = \bigoh(A)$ and
  $\Ccov = \bigoh(1)$.
\end{itemize}
To bound the parameter $\ncM$, consider $M \in \cM$ and $M' \in
\cMalt(M)$. Note that either $\Dkl{M(\pim)}{M'(\pim)} = 0$, in which
case there is no advantage to playing $\pim$, or, due to the
discretization of the means, $\Dkl{M(\pim)}{M'(\pim)} \ge \frac{1}{2}
\delmin^2$. Thus, for any allocation $\eta\in\bbR^{\Pi}_+$, as long as $\eta({\pim}) \ge \frac{2}{\delmin^2}$, we have $\eta({\pim}) \Dkl{M(\pim)}{M'(\pim)} \ge 1$. It follows that there is no advantage to choosing $\eta({\pim})$ larger than $\frac{2}{\delmin^2}$, so we can bound $\ncM \le \frac{2}{\delmin^2}$.

\subsubsection{Bounding the \CompText}
We begin with some basic observations. First, 
since we restrict $\cM$ to only contain instances with a single optimal decision,
if $\fm(\pim) = \lfloor \frac{1}{\delmin} \rfloor \delmin$ for some $M \in \cM$, this implies that $\fm(\pi) < \lfloor \frac{1}{\delmin} \rfloor \delmin$ for all $\pi \neq \pim$. Fix some $M \in \cM$ satisfying $\fm(\pim) = \lfloor \frac{1}{\delmin} \rfloor \delmin$. It follows that, for every $M' \in \cMalt(M)$, it must be the case that $\fm[M'](\pim) < \lfloor \frac{1}{\delmin} \rfloor \delmin$.
Therefore, since $M$ and $M'$ have different reward means at $\pim$, and since this holds for all $M' \in \cMalt(M)$, $M$ can be distinguished from every $M' \in
\cMalt(M)$ by playing $\pim$.
In this case, then, $\gm = 0$, so any $\veps$-optimal Graves-Lai allocation must put all its mass on $\pim$, implying $\Lambda(M,\veps) =
\{ \bbI_{\pim} \}$. Denote such instances $M$ with $\fm(\pim) = \lfloor \frac{1}{\delmin} \rfloor \delmin$ as $\cMbar$.
Note that for $M$ with $\fm(\pim) < \lfloor \frac{1}{\delmin} \rfloor \delmin$, we have $\bbI_{\picirc\subs{M}} \in \Lambda(M,\veps)$.

We proceed to bound the value of the \ Fix $\Mbar \in \conv(\cM)$. For
a first case, assume that that $\fm[\Mbar](\pim[\Mbar]) \le  \lfloor \frac{1}{\delmin} \rfloor \delmin - \frac{1}{2} \delmin$ and let $k = \argmax_{i \in [N]} \Prm{\Mbar}{\picirc_i}(o = \pim[\Mbar])$ denote the index of the most informative arm for $\Mbar$. Let $\lambda = \bbI_{\picirc_k}$, and note that $\cMgl(\lambda)$ contains only instances in $\cM$ that have informative arm $k$. Let $\cM' = \{ M \in \cM \ : \ \picirc\subs{M} \neq \picirc\subs{k} \} \cup \cMbar$. Then $\cM \backslash \cMgl(\lambda) \subseteq \cM'$. Let $\omega = \frac{1}{2} \unif(\{ \picirc_i \}_{i \in [N]}) + \frac{1}{2} \unif([A])$. Then,
\begin{align*}
\aec(\cM,\Mbar) & \le \sup_{M \in \cM'}  \frac{1}{\Exp_{\pi \sim \omega}[\Dkl{\Mbar(\pi)}{M(\pi)}]} \\
& \le \sup_{M \in \cM, \picirc\subs{M} \neq \picirc\subs{k}}  \frac{2N}{\sum_{i=1}^N \Dkl{\Mbar(\picirc_i)}{M(\picirc_i)}} + \sup_{M \in \cMbar} \frac{2A}{\sum_{\pi \in [A]} \Dkl{\Mbar(\pi)}{M(\pi)}}
\end{align*}
If $\picirc\subs{M} \neq \picirc\subs{k}$, this implies that $o \sim M(\picirc_i)$ is uniform on $[A]$ for $\picirc_i \neq \picirc\subs{M}$, and $o \sim M(\picirc_i)$ is distributed as $\beta \bbI_{\pim[M]} + (1-\beta)\unif([A])$ for $\picirc_i = \picirc\subs{M}$. 
Note that since $k = \argmax_{i \in [N]} \Prm{\Mbar}{\picirc_i}(o = \pim[\Mbar])$ and $\Mbar \in \conv(\cM)$, we can have at most 
$\Prm{\Mbar}{\picirc_i}(o) \le 1/A + \beta/2 $ for all $o$ if $i \neq k$, since if this were not the case, then $i$ must be $k$. It follows from Pinsker's inequality that for $M$ with $\picirc\subs{M} \neq \picirc\subs{k}$:
\begin{align*}
\Dkl{\Mbar(\picirc\subs{M})}{M(\picirc\subs{M})} & \ge 2 \Dtv{M(\picirc\subs{M})}{\Mbar(\picirc\subs{M})}^2 \\
& \ge 2 | \Prm{M}{\picirc\subs{M}}(o = \pim) - \Prm{\Mbar}{\picirc\subs{M}}(o = \pim)|^2 \\
& = 2 | \beta - 1/A - \beta/2 |^2 \\
& \ge 2 | \beta/4 |^2 
\end{align*}
where the last inequality uses our assumption that $\beta \ge 4/A$. For $M \in \cMbar$, since $\fm[\Mbar](\pim[\Mbar]) \le  \lfloor \frac{1}{\delmin} \rfloor \delmin - \frac{1}{2} \delmin$, we have that 
\begin{align*}
\Dkl{\Mbar(\pi)}{M(\pi)} \ge \frac{1}{8} \delmin^2.
\end{align*}
Thus, we can bound 
\begin{align*}
\aec(\cM,\Mbar) \le \frac{64N}{\beta^2} + \frac{16A}{\delmin^2}.
\end{align*}

Now, consider the second case where $\Mbar$ has $\fm[\Mbar](\pim[\Mbar]) >  \lfloor \frac{1}{\delmin} \rfloor \delmin - \frac{1}{2} \delmin$. Note that in this case we must have $|\pibm[\Mbar]| = 1$. Set $\lambda = \bbI_{\pim[\Mbar]}$. Then we have that $\cMgl(\lambda)$ contains every instance except the single instance with $\fm(\pim[\Mbar]) = \lfloor \frac{1}{\delmin} \rfloor \delmin$. Let $\omega = \bbI_{\pim[\Mbar]}$. Note that for any instance with $\fm(\pim[\Mbar]) < \lfloor \frac{1}{\delmin} \rfloor \delmin$, i.e. every instance in $\cM \backslash \cMgl(\lambda)$, we have
\begin{align*}
\Dkl{\Mbar(\pim[\Mbar])}{M(\pim[\Mbar])} \ge \frac{1}{8} \delmin^2.
\end{align*}
It follows that with such an $\Mbar$, we can bound
\begin{align*}
\aec(\cM,\Mbar) \le  \frac{8}{\delmin^2}.
\end{align*}

\subsection{Tabular Reinforcement Learning
  (\creftitle{sec:tabular_results})}\label{sec:tabular_proofs}
In this section, we prove all of the claims in
\cref{sec:tabular_results} concerning tabular
reinforcement learning.

Throughout this section we let $M_{sh}(a)$ denote the joint
distribution of the next state and reward if we play action $a$ in state $s$ at step $h$ on MDP $M\in\cM$. We also define
\begin{align*}
\wmb{M}{\pi}_h(s,a) = \Prm{M}{\pi}(s_h = s, a_h = a)
\end{align*}
as the state-action visitation probabilities on MDP $M$ under policy
$\pi$ (and define $\wmb{M}{\pi}_h(s)$ analogously). We let
$r\sups{M}_h(s,a) = \Exp_{r \sim \Rm_h(s,a)}[r]$ denote the mean
reward on MDP $M$ at $(s,a,h)$, and let $r_h(s_h,a_h)$ the denote the
realized (random reward) at step $h$. We let $r := (r_1(s_1,a_1),\ldots, r_H(s_H,a_H))$ denote the vector of all random rewards in a given episode. $\tau = (s_1,\ldots, s_H)$ denotes a trajectory of states, and $\tau_h = s_h$ the $h$th state in the trajectory. 
We denote the $Q$-value function on $M$ for policy $\pi$ by
\begin{align*}
\Qm{M}{\pi}_h(s,a) = \Emb{M}{\pi} \brk*{ \sum_{h' = h}^H r_{h'}(s_{h'},a_{h'}) \mid s_h = s, a_h = a }
\end{align*} 
and the value function by $\Vm{M}{\pi}_h(s) = \Exp_{a \sim \pi_h(s)}[\Qm{M}{\pi}_h(s,a)]$. We denote the value of a policy by $\Vm{M}{\pi}_1 := \Vm{M}{\pi}_1(s_1)$. For any function $V : \cS \rightarrow \bbR$ we denote
\begin{align*}
\bbP\sups{M}_h[V](s,a) = \Exp_{s' \sim \Pm_h(\cdot \mid s,a)}[V(s')].
\end{align*}
For all results in this section concerning general divergences, we take $\D{\cdot}{\cdot} \leftarrow \Dkl{\cdot}{\cdot}$. 

\begin{proof}[Proof of \Cref{prop:tabular_aec_bound}]
To bound the \CompShort, we first move from KL divergence to Hellinger
distance. Since we always have $\Dkl{\Mbar(\pi)}{M(\pi)} \ge \Dhels{\Mbar(\pi)}{M(\pi)}$, we upper bound
\begin{align*}
\aecflipM{\veps}{\cM}(\cMst) \le \sup_{\xi \in \simplex_\cM} \inf_{\lambda,\omega \in \simplex_\Pi} \sup_{M \in \cMst \backslash \cMgl(\lambda)} \frac{1}{\Exp_{\Mbar \sim \xi}[\Exp_{\omega}[\Dhels{\Mbar(\pi)}{M(\pi)}]]}.
\end{align*}
We then apply \Cref{lem:aec_Cexp_bound} to bound this by
$\CexpD(\cMst,\delta)$, with $\D{\cdot}{\cdot} \leftarrow
\Dhels{\cdot}{\cdot}$. The bound on $\CexpD(\cMst,\veps)$ then follows
directly from \Cref{lem:tabular_Cexp_bound}, and gives
\begin{align*}
\aecflipM{\veps}{\cM}(\cMst) & \le \frac{SAH^2 \cdot \log^2 H}{\delta^2} \quad \text{for} \quad \delta = \min_{M \in \cMst} \min \crl*{ \min \left \{ \frac{1}{81 \LKL}, \frac{\delminm}{34 \VM} \right \} \cdot \frac{\veps}{2 \gm/\delminm + \nmepsc{\veps/36}{M}}, \frac{\delminm}{3}}^2.
\end{align*}

By \Cref{lem:tabular_satisfies_asm}, we have that \Cref{asm:bounded_likelihood,asm:smooth_kl} hold with
\begin{align*}
\LKL = \VM = 4 H + \max_{\Mbar,\Mbar' \in \cM} \max_{\pi \in \Pi} \max_{\tau \in \cT}  \left | \log \frac{\Prm{\Mbar'}{\pi}(\tau) }{\Prm{\Mbar}{\pi}(\tau)} \right |.
\end{align*}
As we assume every transition has probability at least $\pmin$, we can lower bound $\Prm{\Mbar}{\pi}(\tau) \ge \pmin^H$, so it suffices to take
\begin{align*}
\LKL = \VM = H(4 + \log 1/\pmin).
\end{align*}

By \Cref{lem:gl_const_Cexp_bound}, for $M \in \cMst$ we can bound
\begin{align*}
\gm \le \Cexp(\cMst, \tfrac{1}{4} (\delminm)^2) \le c \cdot \frac{SAH^2 \cdot \log^2 H}{\delLst^4}.
\end{align*}

A sufficient choice of $\delta$ is therefore
\begin{align*}
\delta = c \cdot \frac{\veps^2 \delLst^{12}}{S^2 A^2 H^6 \cdot \log^4 H \cdot \log^2 1/\pmin}.
\end{align*}

\end{proof}

\subsubsection{Tabular MDPs are Regular Classes}

\begin{lemma}\label{lem:tabular_regular_class}
If $\Mst$ is such that $\rem[\Mst]_h(s,a) \in [0,1/H^2)$ for all $(s,a,h)$, then in the setting of $\cM \leftarrow \cMtab(\pmin)$, $\Mst$ is a regular model with constant
\begin{align*}
\LMst =  \frac{96}{\delminst} 
\end{align*}
\end{lemma}
\begin{proof}[Proof of \Cref{lem:tabular_regular_class}]
Take some MDP $M' \in \cMalt(\Mst)$ such that $\Dkl{\Mst(\pist)}{M'(\pist)} > 0$. Let $M''$ be such that
\begin{align*}
\Dkl{\Mst_{sh}(\pist(s,h))}{M''_{sh}(\pist(s,h))} = 0, \quad \forall s,h
\end{align*}
and
\begin{align*}
\Dkl{M_{sh}'(a)}{M''_{sh}(a)} = 0, \quad \forall s,h, a \neq \pist(s,h),
\end{align*}
so that $M''$ is the MDP which is identical to $M^{\star}$ on optimal
actions, and identical to $M'$ on suboptimal actions (recall that
optimal actions for $\Mstar$ are unique). By construction, we have that
$\Dkl{\Mst(\pist)}{M''(\pist)} = 0$. Furthermore, it is not difficult to see that
$M'' \in \cM$. In particular, to verify that $\Pm[M'']_h(s' \mid s,a) \ge \Pmin$ for each $(s,a,h,s')$, we note that since $\Mst,M' \in \cM$, for every $(s,a,h,s')$, we have $\Pm[\Mst]_h(s' \mid s,a) \ge \Pmin$ and $\Pm[M']_h(s' \mid s,a) \ge \Pmin$. By construction, we have that $\Pm[M'']_h( \cdot \mid s,a)$ is identical to either $\Pm[\Mst]_h( \cdot \mid s,a)$ or $\Pm[M']_h( \cdot \mid s,a)$ for each $(s,a,h)$, so it follows that $\Pm[M'']_h(s' \mid s,a) \ge \Pmin$. The remaining conditions for inclusion in $\cM$ are immediate.
We consider two cases.

\paragraph{Case 1: $M'' \in \cMalt(\Mst)$}
For $\pi \in \Pi$, by \Cref{lem:mdp_kl} we have
\begin{align*}
\Dkl{\Mst(\pi)}{M'(\pi)} & = \sum_{s,a,h} \wmb{\Mst}{\pi}_h(s,a) \Dkl{\Mst_{sh}(a)}{M'_{sh}(a)}, 
\end{align*}
and
\begin{align*}
\Dkl{\Mst(\pi)}{M''(\pi)} & = \sum_{s,a,h} \wmb{\Mst}{\pi}_h(s,a) \Dkl{\Mst_{sh}(a)}{M''_{sh}(a)} \\
& = \sum_{s,a,h} \wmb{\Mst}{\pi}_h(s,a) \Dkl{\Mst_{sh}(a)}{M'_{sh}(a)} \cdot \bbI \{ a \neq \pist(s,h) \}
\end{align*}
so that
\begin{align*}
| \Dkl{\Mst(\pi)}{M'(\pi)}  - \Dkl{\Mst(\pi)}{M''(\pi)} | & =  \sum_{s,h} \wmb{\Mst}{\pi}_h(s,\pist(s,h)) \Dkl{\Mst_{sh}(\pist(s,h))}{M'_{sh}(\pist(s,h))}  \\
& \le \sup_{s,h} \frac{\wmb{\Mst}{\pi}_h(s,\pist(s,h))}{\wmb{\Mst}{\pist}_h(s,\pist(s,h))} \cdot \Dkl{\Mst(\pist)}{M'(\pist)} \\
& \le \sup_{s,h} \frac{1}{\wmb{\Mst}{\pist}_h(s)} \cdot \Dkl{\Mst(\pist)}{M'(\pist)} \\
& \le \frac{1}{\delminst} \cdot \Dkl{\Mst(\pist)}{M'(\pist)}
\end{align*}
where the last inequality follows from \Cref{lem:min_visit_to_mingap}.
Thus, in this case, $\Mst$ is a regular model with
\begin{align*}
\LMst = \frac{1}{\delminst}.
\end{align*}

\paragraph{Case 2: $M'' \not\in \cMalt(\Mst)$}

Let $(\stil,\atil,\htil)$ be such that
$\Qm{M'}{\pist}_{\htil}(\stil,\atil) >
\Qm{M'}{\pist}_{\htil}(\stil,\pist(\stil,\htil))$, and note that such
a tuple is guaranteed to exist by \Cref{lem:mdp_alternate_local} since $M' \in \cMalt(\Mst)$. 
Let $\Mtil''$ denote an MDP that is identical to $M''$ everywhere except for at $(\stil,\htil,\atil)$, where we set $\rem[\Mtil'']_{\htil}(\stil,\atil)$ so that  
\begin{align}\label{eq:tab_regular_alt_pert}
\Qm{\Mtil''}{\pist}_{\htil}(\stil,\atil) = \Qm{\Mtil''}{\pist}_{\htil}(\stil,\pist(\stil,\htil)) + \delta 
\end{align}
for some $\delta > 0$ to be chosen. This will ensure $\atil$ is the optimal action in $(\stil,\htil)$, so $\pim[\Mtil''] \neq \pist$. By construction we have that $\Mst$ and $\Mtil''$ behave identically on $\pist$, which implies that $\Qm{\Mtil''}{\pist}_{\htil}(\stil,\pist(\stil,\htil)) = \Qm{\Mst}{\pist}_{\htil}(\stil,\pist(\stil,\htil))$. Furthermore, by assumption we have $\rem[\Mst]_{h}(s,a) < 1/H^2$ for all $(s,a,h)$, which implies $\Qm{\Mst}{\pist}_{\htil}(\stil,\pist(\stil,\htil)) < 1/H$. As $\Qm{\Mtil''}{\pist}_{\htil}(\stil,\atil) = \rem[\Mtil'']_{\htil}(\stil,\atil) + \Prmb[\Mtil'']_{\htil}[\Vm{\Mtil''}{\pist}_{\htil+1}](\stil,\atil) \ge \rem[\Mtil'']_{\htil}(\stil,\atil)$, it follows that for small enough $\delta$, we can ensure \eqref{eq:tab_regular_alt_pert} is met with $\rem[\Mtil'']_{\htil}(\stil,\atil) < 1/H$, so that $\Mtil'' \in \cM$.

If we can show that, for all $\pi$, $| \Dkl{\Mst(\pi)}{M'(\pi)} - \Dklbig{\Mst(\pi)}{\Mtil''(\pi)}|$
is bounded by some function of $\Dkl{\Mst(\pist)}{M'(\pist)}$, we are then done. We proceed to show this. First, note that, similar to Case 1:
\begin{align}\label{eq:tab_regular_alt_kl}
\begin{split}
| \Dkl{\Mst(\pi)}{M'(\pi)} &  - \Dklbig{\Mst(\pi)}{\Mtil''(\pi)} | \\
& = \Big | \sum_{s,h} \wmb{\Mst}{\pi}_h(s,\pist(s,h)) \Dkl{\Mst_{sh}(\pist(s,h))}{M'_{sh}(\pist(s,h))} \Big | \\
& \qquad + \wmb{\Mst}{\pi}_{\htil}(\stil,\atil) \Big | \Dklbig{\Mst_{\stil \htil}(\atil)}{M'_{\stil \htil}(\atil)} - \Dklbig{\Mst_{\stil \htil}(\atil)}{\Mtil''_{\stil \htil}(\atil)} \Big | \\
& \le \sup_{s,h} \frac{1}{\wmb{\Mst}{\pist}_h(s)} \cdot \Dkl{\Mst(\pist)}{M'(\pist)} \\
& \qquad + \frac{1}{2} \wmb{\Mst}{\pi}_{\htil}(\stil,\atil) \Big |  (\rem[\Mst]_{\htil}(\stil,\atil) - \rem[M']_{\htil}(\stil,\atil))^2 - (\rem[\Mst]_{\htil}(\stil,\atil) - \rem[\Mtil'']_{\htil}(\stil,\atil))^2 \Big |
\end{split}
\end{align}
where the inequality follows by what we showed in Case 1, and since $M'$ and $\Mtil''$ have identical transitions at $(\stil,\atil,\htil)$, so the contribution to the KL divergence from the transitions cancels, leaving only the KL divergence between unit-variance Gaussians. 
By the Mean Value Theorem and since rewards are in $[0,1/H]$, we have
\begin{align*}
\frac{1}{2} \Big |  (\rem[\Mst]_{\htil}(\stil,\atil) - \rem[M']_{\htil}(\stil,\atil))^2 - (\rem[\Mst]_{\htil}(\stil,\atil) - \rem[\Mtil'']_{\htil}(\stil,\atil))^2 \Big | & \le \frac{1}{H} |  \rem[M']_{\htil}(\stil,\atil) - \rem[\Mtil'']_{\htil}(\stil,\atil) | .
\end{align*}
Thus, it suffices to bound $|  \rem[M']_{\htil}(\stil,\atil) - \rem[\Mtil'']_{\htil}(\stil,\atil) |$. 

By assumption $\Qm{M'}{\pist}_{\htil}(\stil,\atil) > \Qm{M'}{\pist}_{\htil}(\stil,\pist(\stil,\htil))$. We can then ensure
\begin{align*}
\Qm{\Mtil''}{\pist}_{\htil}(\stil,\atil) - \Qm{\Mtil''}{\pist}_{\htil}(\stil,\pist(\stil,\htil)) \le \Qm{M'}{\pist}_{\htil}(\stil,\atil) - \Qm{M'}{\pist}_{\htil}(\stil,\pist(\stil,\htil)) 
\end{align*}
for $\delta$ sufficiently small. 
This is equivalent to, abbreviating $\am := \pist(\stil,\htil)$:
\begin{align*}
& \rem[\Mtil'']_{\htil}(\stil,\atil) + \Prmb[\Mtil'']_{\htil}[\Vm{\Mtil''}{\pist}_{\htil+1}](\stil,\atil) - \rem[\Mtil'']_{\htil}(\stil,\am) - \Prmb[\Mtil'']_{\htil}[\Vm{\Mtil''}{\pist}_{\htil+1}](\stil,\am) \\
&  \qquad \le \rem[M']_{\htil}(\stil,\atil) + \Prmb[M']_{\htil}[\Vm{M'}{\pist}_{\htil+1}](\stil,\atil) - \rem[M']_{\htil}(\stil,\am) - \Prmb[M']_{\htil}[\Vm{M'}{\pist}_{\htil+1}](\stil,\am).
\end{align*}
By construction we have that $\Vm{\Mtil''}{\pist}_h(s) = \Vm{\Mst}{\pist}_h(s)$ for all $(s,h)$, $\rem[\Mtil'']_{\htil}(\stil,\am) = \rem[\Mst]_{\htil}(\stil,\am)$, and $ \Prmb[\Mtil'']_{\htil}[\Vm{\Mtil''}{\pist}_{\htil+1}](\stil,\am) =  \Prmb[\Mst]_{\htil}[\Vm{\Mst}{\pist}_{\htil+1}](\stil,\am)$, since $\Mtil''$ behaves identically to $M$ on actions taken by $\pist$. Furthermore, we have $\Prmb[\Mtil'']_{\htil}[\Vm{\Mtil''}{\pist}_{\htil+1}](\stil,\atil) = \Prmb[M']_{\htil}[\Vm{\Mst}{\pist}_{\htil+1}](\stil,\atil)$ since $\Mtil''$ behaves identically to $M'$ on actions not taken by $\pist$, other than the reward at $(\stil,\atil,\htil)$. Using these simplifications and rearranging, we get
\iftoggle{colt}{\begin{align*}
| \rem[\Mtil'']_{\htil}(\stil,\atil) - \rem[M']_{\htil}(\stil,\atil)| & \le | \rem[M']_{\htil}(\stil,\am) - \rem[\Mst]_{\htil}(\stil,\am)| + | \Prmb[\Mst]_{\htil}[\Vm{\Mst}{\pist}_{\htil+1}](\stil,\am) - \Prmb[M']_{\htil}[\Vm{M'}{\pist}_{\htil+1}](\stil,\am) | \\
& \qquad + | \Prmb[M']_{\htil}[\Vm{\Mst}{\pist}_{\htil+1}](\stil,\atil) - \Prmb[M']_{\htil}[\Vm{M'}{\pist}_{\htil+1}](\stil,\atil)| \\
& \le | \rem[M']_{\htil}(\stil,\am) - \rem[\Mst]_{\htil}(\stil,\am)| + | \Prmb[\Mst]_{\htil}[\Vm{\Mst}{\pist}_{\htil+1}](\stil,\am) - \Prmb[M']_{\htil}[\Vm{\Mst}{\pist}_{\htil+1}](\stil,\am) | \\
& \qquad + | \Prmb[M']_{\htil}[\Vm{\Mst}{\pist}_{\htil+1}](\stil,\am) - \Prmb[M']_{\htil}[\Vm{M'}{\pist}_{\htil+1}](\stil,\am) | \\
& \qquad + | \Prmb[M']_{\htil}[\Vm{\Mst}{\pist}_{\htil+1}](\stil,\atil) - \Prmb[M']_{\htil}[\Vm{M'}{\pist}_{\htil+1}](\stil,\atil)|.
\end{align*}}{
\begin{align*}
| \rem[\Mtil'']_{\htil}(\stil,\atil) - \rem[M']_{\htil}(\stil,\atil)| & \le | \rem[M']_{\htil}(\stil,\am) - \rem[\Mst]_{\htil}(\stil,\am)| + | \Prmb[\Mst]_{\htil}[\Vm{\Mst}{\pist}_{\htil+1}](\stil,\am) - \Prmb[M']_{\htil}[\Vm{M'}{\pist}_{\htil+1}](\stil,\am) | \\
& \qquad + | \Prmb[M']_{\htil}[\Vm{\Mst}{\pist}_{\htil+1}](\stil,\atil) - \Prmb[M']_{\htil}[\Vm{M'}{\pist}_{\htil+1}](\stil,\atil)| \\
& \le | \rem[M']_{\htil}(\stil,\am) - \rem[\Mst]_{\htil}(\stil,\am)| + | \Prmb[\Mst]_{\htil}[\Vm{\Mst}{\pist}_{\htil+1}](\stil,\am) - \Prmb[M']_{\htil}[\Vm{\Mst}{\pist}_{\htil+1}](\stil,\am) | \\
& \qquad + | \Prmb[M']_{\htil}[\Vm{\Mst}{\pist}_{\htil+1}](\stil,\am) - \Prmb[M']_{\htil}[\Vm{M'}{\pist}_{\htil+1}](\stil,\am) | + | \Prmb[M']_{\htil}[\Vm{\Mst}{\pist}_{\htil+1}](\stil,\atil) - \Prmb[M']_{\htil}[\Vm{M'}{\pist}_{\htil+1}](\stil,\atil)|.
\end{align*}}
Since rewards are unit-variance Gaussian, we have
\begin{align*}
| \rem[M']_{\htil}(\stil,\am) - \rem[\Mst]_{\htil}(\stil,\am)| & \le \sqrt{2 \Dklbig{\Mst_{\htil,\stil}(\am)}{M'_{\htil,\stil}(\am)}}  \le \sqrt{\frac{2}{\wmb{\Mst}{\pist}_{\htil}(\stil)} \Dkl{\Mst(\pist)}{M'(\pist)}}.
\end{align*}

Since $\Vm{\Mst}{\pist}_{\htil+1} \in [0,1]$, we have
\begin{align*}
 | \Prmb[\Mst]_{\htil}[\Vm{\Mst}{\pist}_{\htil+1}](\stil,\am) - \Prmb[M']_{\htil}[\Vm{\Mst}{\pist}_{\htil+1}](\stil,\am) | & \le \sum_{s'} | \Probm_{\htil}(s' \mid \stil,\am) - \Probm[M']_{\htil}(s' \mid \stil,\am) | \\
&  \le 2\Dtv{\Probm_{\htil}(\cdot \mid \stil,\am)}{\Probm[M']_{\htil}(\cdot \mid \stil,\am)} \\
& \le \sqrt{2\Dklbig{\Probm_{\htil}(\cdot \mid \stil,\am)}{\Probm[M']_{\htil}(\cdot \mid \stil,\am)}} \\
& \le \sqrt{\frac{2}{\wmb{\Mst}{\pist}_{\htil}(\stil)}\Dklbig{\Probm_{\htil}(\cdot \mid \stil,\am)}{\Probm[M']_{\htil}(\cdot \mid \stil,\am)}} \\
& \le \sqrt{\frac{2}{\wmb{\Mst}{\pist}_{\htil}(\stil)} \Dkl{\Mst(\pist)}{M'(\pist)}}.
\end{align*}

By \Cref{lem:mdp_value_kl_bound} we have
\begin{align*}
| \Prmb[M']_{\htil}[\Vm{\Mst}{\pist}_{\htil+1}](\stil,\am) - \Prmb[M']_{\htil}[\Vm{M'}{\pist}_{\htil+1}](\stil,\am) | & \le \sum_{s'} \Probm[M']_{\htil}(s' \mid \stil,\am) | \Vm{\Mst}{\pist}_{\htil+1}(s') - \Vm{M'}{\pist}_{\htil+1}(s')| \\
& \le \sum_{s'} \Probm[M']_{\htil}(s' \mid \stil,\am) \cdot \sqrt{\frac{8H}{\wmb{\Mst}{\pist}_{\htil+1}(s')} \cdot \Dkl{\Mst(\pist)}{M'(\pist)}} \\
& \le \sup_{s} \sqrt{\frac{8H}{\wmb{\Mst}{\pist}_{\htil+1}(s)} \cdot \Dkl{\Mst(\pist)}{M'(\pist)}}
\end{align*}
and similarly
\begin{align*}
| \Prmb[M']_{\htil}[\Vm{\Mst}{\pist}_{\htil+1}](\stil,\atil) - \Prmb[M']_{\htil}[\Vm{M'}{\pist}_{\htil+1}](\stil,\atil)| & \le \sup_{s} \sqrt{\frac{8H}{\wmb{\Mst}{\pist}_{\htil+1}(s)} \cdot \Dkl{\Mst(\pist)}{M'(\pist)}}.
\end{align*}

Altogether then:
\begin{align*}
| \rem[\Mtil'']_{\htil}(\stil,\atil) - \rem[M']_{\htil}(\stil,\atil)| & \le \prn*{\sqrt{\frac{8}{\wmb{\Mst}{\pist}_{\htil}(\stil)}}  + \sup_{s} \sqrt{\frac{32H}{\wmb{\Mst}{\pist}_{\htil+1}(s)}}} \cdot \sqrt{\Dkl{\Mst(\pist)}{M'(\pist)}} \nonumber \\
& \le \sup_{s,h} \sqrt{\frac{96H}{\wmb{\Mst}{\pist}_h(s)}} \cdot \sqrt{\Dkl{\Mst(\pist)}{M'(\pist)}}. \label{eq:mdp_regular_reward_diff_bound}
\end{align*}
Combining this with \eqref{eq:tab_regular_alt_kl}, we have
\begin{align*}
| \Dkl{\Mst(\pi)}{M'(\pi)} &  - \Dklbig{\Mst(\pi)}{\Mtil''(\pi)} | \\
& \le  \sup_{s,h} \frac{1}{\wmb{\Mst}{\pist}_h(s)} \cdot \Dkl{\Mst(\pist)}{M'(\pist)} + \sup_{s,h} \sqrt{\frac{96}{ H \wmb{\Mst}{\pist}_h(s)}} \cdot \sqrt{\Dkl{\Mst(\pist)}{M'(\pist)}} \\
& \le \frac{1}{\delminst} \cdot \Dkl{\Mst(\pist)}{M'(\pist)} + \sqrt{\frac{96}{H \delminst}} \cdot \sqrt{\Dkl{\Mst(\pist)}{M'(\pist)}}
\end{align*}
where the second inequality uses \Cref{lem:min_visit_to_mingap}. Thus, in this case $\Mst$ is a regular model with
\begin{align*}
\LMst = \frac{96}{\delminst}. 
\end{align*}

\end{proof}

\subsubsection{Tabular MDPs Satisfy Basic Assumptions}

\begin{lemma}\label{lem:tabular_satisfies_asm}
Tabular MDPs with unit-variance Gaussian rewards satisfy \Cref{asm:bounded_likelihood,asm:smooth_kl,asm:D_to_hel} with
\begin{align*}
\LKL  = \VM = 8 H + \max_{\Mbar,\Mbar' \in \cM} \max_{\pi \in \Pi} \max_{\tau \in \cT}  \left | \log \frac{\Prm{\Mbar'}{\pi}(\tau) }{\Prm{\Mbar}{\pi}(\tau)} \right |,
\end{align*}
and $\D{\cdot}{\cdot}=\Dkl{\cdot}{\cdot}$, where $\cT \ldef \cS^H$ and $\Prm{\Mbar}{\pi}(\tau)$ denotes the probability of observing state sequence $\tau \in \cT$ on $\Mbar$ when playing policy $\pi$. 
\end{lemma}
\begin{proof}[Proof of \Cref{lem:tabular_satisfies_asm}]
We verify each assumption separately. 

\paragraph{Verifying \Cref{asm:smooth_kl}}
Fix some $M,M',\Mbar \in \cM$. Our goal is to bound
\begin{align*}
\left |\kl{M(\pi)}{M'(\pi)} - \kl{\Mbar(\pi)}{M'(\pi)} \right | .
\end{align*}
Let $\Mtil$ denote the MDP with transitions identical to $M$ but rewards identical to $\Mbar$. Then
\begin{align*}
\left |\kl{M(\pi)}{M'(\pi)} - \kl{\Mbar(\pi)}{M'(\pi)} \right | & \le \left |\kl{M(\pi)}{M'(\pi)} - \Dklbig{\Mtil(\pi)}{M'(\pi)} \right | \\
& \qquad + \left |\Dklbig{\Mtil(\pi)}{M'(\pi)} - \kl{\Mbar(\pi)}{M'(\pi)} \right |
\end{align*}
We bound these terms separately. First, by \Cref{lem:mdp_kl} we have
\iftoggle{colt}{\begin{align*}
& \Dkl{M(\pi)}{M'(\pi)}  = \sum_{s,a,h} \wmb{M}{\pi}_h(s,a) \Dkl{M_{sh}(\pi(s,h))}{M'_{sh}(\pi(s,h))} \\
& \qquad = \sum_{s,a,h} \wmb{M}{\pi}_h(s,a) \brk*{\frac{1}{2} (\rem_h(s,a) - \rem[M']_h(s,a))^2 + \Dkl{\Probm_{h}(\cdot \mid s,\pi(s,h))}{\Probm[M']_{h}(\cdot \mid s, \pi(s,h))}}
\end{align*}}{
\begin{align*}
\Dkl{M(\pi)}{M'(\pi)} & = \sum_{s,a,h} \wmb{M}{\pi}_h(s,a) \Dkl{M_{sh}(\pi(s,h))}{M'_{sh}(\pi(s,h))} \\
& = \sum_{s,a,h} \wmb{M}{\pi}_h(s,a) \brk*{\frac{1}{2} (\rem_h(s,a) - \rem[M']_h(s,a))^2 + \Dkl{\Probm_{h}(\cdot \mid s,\pi(s,h))}{\Probm[M']_{h}(\cdot \mid s, \pi(s,h))}}
\end{align*}}
and, given our definition of $\Mtil$, 
\iftoggle{colt}{
\begin{align*}
\Dklbig{\Mtil(\pi)}{M'(\pi)} &  = \sum_{s,a,h} \wmb{M}{\pi}_h(s,a) \bigg [ \frac{1}{2} (\rem[\Mbar]_h(s,a) - \rem[M']_h(s,a))^2 \\
& \qquad + \Dkl{\Probm_{h}(\cdot \mid s,\pi(s,h))}{\Probm[M']_{h}(\cdot \mid s, \pi(s,h))} \bigg ].
\end{align*}
}{
\begin{align*}
\Dklbig{\Mtil(\pi)}{M'(\pi)} &  = \sum_{s,a,h} \wmb{M}{\pi}_h(s,a) \brk*{\frac{1}{2} (\rem[\Mbar]_h(s,a) - \rem[M']_h(s,a))^2 + \Dkl{\Probm_{h}(\cdot \mid s,\pi(s,h))}{\Probm[M']_{h}(\cdot \mid s, \pi(s,h))}}.
\end{align*}
}
Thus,
\begin{align*}
& \left |\kl{M(\pi)}{M'(\pi)} - \Dklbig{\Mtil(\pi)}{M'(\pi)} \right | \\
& \qquad = \left | \frac{1}{2} \sum_{s,a,h} \wmb{M}{\pi}_h(s,a) \brk*{(\rem_h(s,a) - \rem[M']_h(s,a))^2 - (\rem[\Mbar]_h(s,a) - \rem[M']_h(s,a))^2} \right | \\
& \qquad \overset{(a)}{\le} \sum_{s,a,h} \wmb{M}{\pi}_h(s,a) |\rem_h(s,a) - \rem[\Mbar]_h(s,a)|  \\
& \qquad \le  \sum_{s,a,h} \wmb{M}{\pi}_h(s,a) \sqrt{2\Dkl{M_{sh}(a)}{\Mbar_{sh}(a)}} \\
& \qquad \le   \sqrt{2H \cdot \sum_{s,a,h} \wmb{M}{\pi}_h(s,a) \Dkl{M_{sh}(a)}{\Mbar_{sh}(a)}} \\
& \qquad = \sqrt{2H \cdot \Dkl{M(\pi)}{\Mbar(\pi)}},
\end{align*}
where $(a)$ holds by the Mean Value Theorem and the assumption that reward means are in $[0,1]$, and the final equality holds by \Cref{lem:mdp_kl}.

We turn now to bounding the second term. Let $\cT = \cS^H$ denote the space of all possible state trajectories. Let $\Prm{M}{\pi}(\tau = \cdot)$ denote the probability of observing $\tau \in \cT$ when playing policy $\pi$ on $M$. We then have
\begin{align*}
\Dklbig{\Mtil(\pi)}{M'(\pi)} & = \int \log \frac{\Prm{\Mtil}{\pi}(r,\tau)}{\Prm{M'}{\pi}(r,\tau)} \rmd  \Prm{\Mtil}{\pi}(r,\tau) \\
& = \int \int   \log \frac{\Prm{\Mbar}{\pi}(r \mid \tau) \Prm{M}{\pi}(\tau)}{\Prm{M'}{\pi}(r \mid \tau) \Prm{M'}{\pi}(\tau)} \rmd \Prm{\Mbar}{\pi}(r \mid \tau) \rmd \Prm{M}{\pi}(\tau) \\
& = \int  \log \frac{\Prm{M}{\pi}(\tau)}{ \Prm{M'}{\pi}(\tau)}  \rmd \Prm{M}{\pi}(\tau) + \int   \prn*{ \int  \log \frac{\Prm{\Mbar}{\pi}(r \mid \tau) }{\Prm{M'}{\pi}(r \mid \tau) } \rmd \Prm{\Mbar}{\pi}(r \mid \tau)} \rmd \Prm{M}{\pi}(\tau) \\
& = \sum_{\tau \in \cT} \Prm{M}{\pi}(\tau) \log \frac{\Prm{M}{\pi}(\tau)}{ \Prm{M'}{\pi}(\tau)} + \sum_{\tau \in \cT} \Prm{M}{\pi}(\tau) \Dkl{\Prm{\Mbar}{\pi}(r \mid \tau)}{\Prm{M'}{\pi}(r \mid \tau)}.
\end{align*}
It follows that
\begin{align*}
\left | \Dklbig{\Mtil(\pi)}{M'(\pi)} - \Dkl{\Mbar(\pi)}{M'(\pi)} \right | & \le \left | \sum_{\tau \in \cT} \Prm{M}{\pi}(\tau) \log \frac{\Prm{M}{\pi}(\tau)}{ \Prm{M'}{\pi}(\tau)} - \sum_{\tau \in \cT} \Prm{\Mbar}{\pi}(\tau) \log \frac{\Prm{\Mbar}{\pi}(\tau)}{ \Prm{M'}{\pi}(\tau)} \right | \\
& \qquad + \sum_{\tau \in \cT} |\Prm{M}{\pi}(\tau) - \Prm{\Mbar}{\pi}(\tau)| \Dkl{\Prm{M}{\pi}(r \mid \tau)}{\Prm{M'}{\pi}(r \mid \tau)} .
\end{align*}
Note that, since rewards at each state are independent,
\begin{align*}
\Dkl{\Prm{M}{\pi}(r \mid \tau)}{\Prm{M'}{\pi}(r \mid \tau)} = \sum_{h=1}^H \Dkl{\Prm{M}{\pi}(r_h \mid \tau)}{\Prm{M'}{\pi}(r_h \mid \tau)} & \le H,
\end{align*}
since rewards have means are in $[0,1]$ and are unit Gaussian.
This implies 
\begin{align*}
\sum_{\tau \in \cT} |\Prm{M}{\pi}(\tau) - \Prm{\Mbar}{\pi}(\tau)| \Dkl{\Prm{M}{\pi}(r \mid \tau)}{\Prm{M'}{\pi}(r \mid \tau)} & \le H \sum_{\tau \in \cT} |\Prm{M}{\pi}(\tau) - \Prm{\Mbar}{\pi}(\tau)| \\
& = H \Dtv{M(\pi)}{\Mbar(\pi)} \\
& \le H \sqrt{\frac{1}{2} \Dkl{M(\pi)}{\Mbar(\pi)}}.
\end{align*}

Note that $\frac{\rmd}{\rmd x} x \log \frac{x}{y} = 1 + \log \frac{x}{y}$, so by the Mean Value Theorem we have
\begin{align*}
& \left | \Prm{M}{\pi}(\tau) \log \frac{\Prm{M}{\pi}(\tau)}{\Prm{M'}{\pi}(\tau)} - \Prm{\Mbar}{\pi}(\tau) \log \frac{\Prm{\Mbar}{\pi}(\tau)}{\Prm{M'}{\pi}(\tau)} \right | \\
& \qquad \le \prn*{ 1 +  \max \left \{ \left | \log \frac{ \Prm{M}{\pi}(\tau) }{\Prm{M'}{\pi}(\tau)} \right | , \left | \log \frac{ \Prm{\Mbar}{\pi}(\tau)  }{\Prm{M'}{\pi}(\tau)} \right | \right \}} \cdot | \Prm{M}{\pi}(\tau) - \Prm{\Mbar}{\pi}(\tau) | \\
& \qquad \le \prn*{ 1 + \max_{\Mbar' \in \cM} \max_{\tau' \in \cT}  \left | \log \frac{\Prm{\Mbar'}{\pi}(\tau') }{\Prm{M'}{\pi}(\tau')} \right | } \cdot | \Prm{M}{\pi}(\tau) - \Prm{\Mbar}{\pi}(\tau) |.
\end{align*}
It follows that
\begin{align*}
& \left |  \sum_{\tau \in \cT} \Prm{M}{\pi}(\tau) \log \frac{\Prm{M}{\pi}(\tau)}{\Prm{M'}{\pi}(\tau)} -  \sum_{\tau \in \cT} \Prm{\Mbar}{\pi}(\tau) \log \frac{\Prm{\Mbar}{\pi}(\tau)}{\Prm{M'}{\pi}(\tau)} \right | \\
& \qquad \le \prn*{ 1 + \max_{\Mbar' \in \cM} \max_{\tau' \in \cT}  \left | \log \frac{\Prm{\Mbar'}{\pi}(\tau') }{\Prm{M'}{\pi}(\tau')} \right | }  \cdot \sum_{\tau \in \cT} | \Prm{M}{\pi}(\tau) - \Prm{\Mbar}{\pi}(\tau) | \\
& \qquad = \prn*{ 1 + \max_{\Mbar' \in \cM} \max_{\tau' \in \cT}  \left | \log \frac{\Prm{\Mbar'}{\pi}(\tau') }{\Prm{M'}{\pi}(\tau')} \right | }  \cdot \Dtv{M(\pi)}{\Mbar(\pi)} \\
& \qquad \le \prn*{ 1 + \max_{\Mbar' \in \cM} \max_{\tau' \in \cT}  \left | \log \frac{\Prm{\Mbar'}{\pi}(\tau') }{\Prm{M'}{\pi}(\tau')} \right | }  \cdot \sqrt{\frac{1}{2} \Dkl{M(\pi)}{\Mbar(\pi)}}.
\end{align*}
This verifies \Cref{asm:smooth_kl} with
\begin{align*}
\LKL = 1 + \sqrt{2H} + H + \max_{M' \in \cM, \Mbar' \in \cM} \max_{\pi \in \Pi} \max_{\tau \in \cT}  \left | \log \frac{\Prm{\Mbar'}{\pi}(\tau) }{\Prm{M'}{\pi}(\tau)} \right | .
\end{align*}

\paragraph{Verifying \Cref{asm:D_to_hel}}
That $\Dhels{\Mbar(\pi)}{\Mbar'(\pi)} \le \D{\Mbar(\pi)}{\Mbar'(\pi)}$ is immediate, since KL always upper bounds Hellinger squared.

\paragraph{Verifying \Cref{asm:bounded_likelihood}}
We have
\begin{align*}
\log \frac{\Prm{\Mbar}{\pi}(r,o)}{\Prm{M}{\pi}(r,o)} = \log \frac{\Prm{\Mbar}{\pi}(\tau)}{\Prm{M}{\pi}(\tau)} +\log \frac{\Prm{\Mbar}{\pi}(r \mid \tau)}{\Prm{M}{\pi}(r \mid \tau)} = \log \frac{\Prm{\Mbar}{\pi}(\tau)}{\Prm{M}{\pi}(\tau)} + \sum_{h=1}^H \log \frac{\Prm{\Mbar}{\pi}(r_h \mid \tau)}{\Prm{M}{\pi}(r_h \mid \tau)}.
\end{align*}
Using the same calculation as in \Cref{lem:gauss_bandits_satisfy_asm}, we have that $\log \frac{\Prm{\Mbar}{\pi}(r_h \mid \tau)}{\Prm{M}{\pi}(r_h \mid \tau)}$ is 8-subgaussian, since rewards are unit-variance Gaussian. As $\log \frac{\Prm{\Mbar}{\pi}(r_h \mid \tau)}{\Prm{M}{\pi}(r_h \mid \tau)}$ and $\log \frac{\Prm{\Mbar}{\pi}(r_{h'} \mid \tau)}{\Prm{M}{\pi}(r_{h'} \mid \tau)}$ are independent for $h \neq h'$, it follows that $\log \frac{\Prm{\Mbar}{\pi}(r \mid \tau)}{\Prm{M}{\pi}(r \mid \tau)} $ is $8H$-subgaussian. 

Furthermore, bounding
\begin{align*}
\log \frac{\Prm{\Mbar}{\pi}(\tau)}{\Prm{M}{\pi}(\tau)} \le \sup_{\Mbar,M \in \cM} \sup_{\pi \in \Pi}  \sup_{\tau \in \cT} \left | \log \frac{\Prm{\Mbar}{\pi}(\tau) }{\Prm{M}{\pi}(\tau)} \right | =: V_\cT,
\end{align*}
we have that
$ \log \frac{\Prm{\Mbar}{\pi}(\tau)}{\Prm{M}{\pi}(\tau)}$ is $V_{\cT}^2$-subgaussian. 
Since the sum of subgaussian random variables is subgaussian, it follows that $\log \frac{\Prm{\Mbar}{\pi}(r,o)}{\Prm{M}{\pi}(r,o)} = \log \frac{\Prm{\Mbar}{\pi}(\tau)}{\Prm{M}{\pi}(\tau)} +\log \frac{\Prm{\Mbar}{\pi}(r \mid \tau)}{\Prm{M}{\pi}(r \mid \tau)} $ is $(V_{\cT}^2 + 8H)$-subgaussian, which verifies \Cref{asm:bounded_likelihood}. 

\end{proof}

\begin{lemma}\label{lem:tabular_covering}
Let $\pmin := \inf_{M \in \cM} \inf_{h,s',s,a}\Probm[M]_h(s' \mid
s,a)$ and assume $\cM$ is such that $\pmin > 0$. We can construct a
$(\rho,\mu)$-cover of $\cM$ with respect to $\cE := \{ |r_h| \le 1 +
\sqrt{2\log(2H/\mu)}, \forall h \in [H] \}$, with
\begin{align*}
\Ncov(\cM,\rho,\mu) \le \frac{1}{\min \{ \frac{\rho \pmin}{4H}, 2 \pmin \}^{S^2 AH}} \cdot \frac{(2H(2 + \sqrt{\log(H/\mu)}))^{SAH}}{\rho^{SAH}}.
\end{align*}
\end{lemma}
\begin{proof}[Proof of \Cref{lem:tabular_covering}]
Throughout this proof, we use $M = \{
(\Probm_h)_{h=1}^H,(\rem_h)_{h=1}^H \}\in\cM$ to denote the MDP in
$\cM$ with $\Rm_h(s,a)=\cN(\rem_h(s,a),1)$; for brevity, we $\cS$,
$\cA$, and $s_1$ to be fixed and the dependence on them.
  
Observe that for any models $M,M'\in\cM$, we have
\begin{align*}
\left | \log \frac{\Prm{M}{\pi}(r,o)}{\Prm{M'}{\pi}(r,o)} \right | & \le \left | \log \frac{\Prm{M}{\pi}(\tau)}{\Prm{M'}{\pi}(\tau)} \right | + \left | \log \frac{\Prm{M}{\pi}(r \mid \tau)}{\Prm{M'}{\pi}(r \mid \tau)} \right | \\
& =  \sum_{h=1}^H \left | \log \frac{\Probm_h(\tau_{h+1} \mid \tau_h, \pi(\tau_h,h)) }{\Probm[M']_h(\tau_{h+1} \mid \tau_h, \pi(\tau_h,h))} \right | + \sum_{h=1}^H \left | \log \frac{\Prm{M}{\pi}(r_h \mid \tau)}{\Prm{M'}{\pi}(r_h \mid \tau)} \right |.
\end{align*}
Let $\cI_\veps = \{ \veps, 2 \veps, \ldots, \lfloor 1/\veps \rfloor \veps \}$, so that $|\cI_\veps| \le  1/\veps $. Let $\cP_\veps$ denote an $\veps$ cover of $\simplex_\cS$ in the $\ell_\infty$-norm, so that for any $P \in \simplex_\cS$, there exists some $P' \in \cP_\veps$ such that $\sup_{s \in \cS} |P_s - P'_s| \le \veps$. It suffices to choose $\cP_\veps = \cI_{\veps}^S \cap \simplex_{\cS}$, so we can bound $|\cP_\veps| \le 1/\veps^S$.
Let 
\begin{align*}
\cMcov = \left \{ M = \{ (\Probm_h)_{h=1}^H,(\rem_h)_{h=1}^H \} \ : \ \Probm_h(\cdot \mid s,a) \in \cP_{\veps_1}, \rem_h(s,a) \in \cI_{\veps_2}, \quad \forall s,a,h \right \}
\end{align*}
for parameters $\veps_1,\veps_2>0$ to be chosen. Then 
\begin{align*}
\cMcov = (|\cP_{\veps_1}| |\cI_{\veps_2}|)^{SAH} \le \frac{1}{\veps_1^{S^2 AH}} \cdot \frac{1}{\veps_2^{SAH}}.
\end{align*}
We will show that $\cMcov$ is a $(\rho,\mu)$-cover of $\cM$ for
appropriately chosen $\cE$ and $\veps_1,\veps_2>0$. 

Let $\cE := \{ |r_h| \le 1 + \sqrt{2\log(2H/\mu)}, \forall h \in [H] \}$. As we assume rewards are unit-variance Gaussian and have means in $[0,1]$, it is straightforward to see that $\Pr[\cE^c] \le \mu$. Fix $M$ and let $M' \in \cMcov$ denote the instance such that 
\begin{align*}
| \rem_h(s,a) - \rem[M']_h(s,a) | \le \veps_2 \quad \text{and} \quad \sup_{s' \in \cS} | \Probm_h(s' \mid s,a) - \Probm[M']_h(s' \mid s,a) | \le \veps_1, \quad \forall s,a,h.
\end{align*}
Note that such an instance is guaranteed to exist by definition of $\cMcov$.

By a similar argument as in \Cref{lem:gauss_bandits_satisfy_asm}, we can bound, on $\cE$,
\begin{align*}
\sum_{h=1}^H \left | \log \frac{\Prm{M}{\pi}(r_h \mid \tau)}{\Prm{M'}{\pi}(r_h \mid \tau)}\right | & \le \sum_{h=1}^H (1+ |r_h|) \cdot \sup_{s,a} |\rem_h(s,a) - \rem[M']_h(s,a)| \\
& \le H(2 + \sqrt{2\log(2H/\mu)}) \cdot \sup_{s,a,h} |\rem_h(s,a) - \rem[M']_h(s,a)| \\
& \le H(2 + \sqrt{2\log(2H/\mu)}) \cdot \veps_2.
\end{align*}
We also have
\begin{align*}
\sum_{h=1}^H \left | \log \frac{\Probm_h(\tau_{h+1} \mid \tau_h, \pi(\tau_h,h)) }{\Probm[M']_h(\tau_{h+1} \mid \tau_h, \pi(\tau_h,h))} \right | & \le H \cdot \sup_{h,s',s,a} \left | \log  \frac{\Probm_h(s' \mid s,a) }{\Probm[M']_h(s' \mid s,a)} \right | \\
& \le H \cdot \sup_{|x| \le \veps_1} \sup_{h,s',s,a} \left | \log  \frac{\Probm_h(s' \mid s,a) }{\Probm[M]_h(s' \mid s,a) - x} \right | \\
& \le H \cdot \sup_{h,s',s,a} \frac{\veps_1}{\Probm[M]_h(s' \mid s,a) - \veps_1}
\end{align*}
where the last inequality holds as long as $\inf_{h,s',s,a}\Probm[M]_h(s' \mid s,a) - \veps_1 > 0$. Denoting $\pmin := \inf_{M \in \cM} \inf_{h,s',s,a}\Probm[M]_h(s' \mid s,a)$, for $\cMcov$ to be a $(\rho,\mu)$-cover, it therefore suffices that
\begin{align*}
H(2 + \sqrt{2\log(2H/\mu)}) \cdot \veps_2 \le \rho/2, \quad  \frac{2H \veps_1}{\pmin} \le \rho/2, \quad \text{and} \quad \pmin \ge \veps_1/2
\end{align*}
so it suffices to take
\begin{align*}
\veps_1 = \min \{ \frac{\rho \pmin}{4H}, 2 \pmin \} \quad \text{and} \quad  \veps_2 = \frac{\rho}{2H(2 + \sqrt{\log(H/\mu)}) } .
\end{align*}
The result now follows from our bound on $|\cMcov|$.
\end{proof}

\subsubsection{Tabular MDPs have Bounded Uniform Exploration Coefficient}
\begin{lemma}\label{lem:tabular_Cexp_bound}
  For the tabular MDP class $\cM$ in \cref{eq:tabular_class}, we can bound, for all $\veps > 0$,
\begin{align*}
\CexpD(\cM,\veps) \le \frac{320000SAH^2 \cdot \log^2 H}{\veps^2}.
\end{align*}
for $\D{\cdot}{\cdot} \leftarrow \Dhels{\cdot}{\cdot}$.
\end{lemma}
\begin{proof}[Proof of \Cref{lem:tabular_Cexp_bound}]
  Let $\xi\in\Delta(\cM)$ be given. Define
\begin{align*}
  \pexp = \argmin_{p\in\simplex_\Pi}\max_{q \in \simplex_\Pi}\sum_{s,a,h}\Exp_{\pi \sim q} \brk*{\En_{\Mbar\sim\xi} \brk*{\frac{(\wmb{\Mbar}{\pi}_h(s,a))^2}{\En_{\pi'\sim{}p}[\wmb{\Mbar}{\pi'}_h(s,a)]}}}.
\end{align*}
We first show that, for any $M \in \cM$ and any $\pi$,
\begin{align*}
\Exp_{\Mbar \sim \xi}[\Dhels{\Mbar(\pi)}{M(\pi)}]  \le \sqrt{SAH^2 \cdot \Exp_{\Mbar \sim \xi}[\Exp_{\pi \sim \pexp}[\Dhels{\Mbar(\pi)}{M(\pi)}]]}.
\end{align*}
Consider any policy $\pi$. We can bound
\begin{align*}
& \Exp_{\Mbar \sim \xi}[ \Dhels{\Mbar(\pi)}{M(\pi)}] \\
& \overset{(a)}{\le} 100\log(H) \cdot \sum_{s,a,h} \Exp_{\Mbar \sim \xi} \brk*{\wmb{\Mbar}{\pi}_h(s,a) \Dhels{\Mbar_{sh}(a)}{M_{sh}(a)}} \\
&  \overset{(b)}{\le} 100\log(H) \cdot \sqrt{\sum_{s,a,h} \Exp_{\Mbar \sim \xi} \brk*{ \frac{\wmb{\Mbar}{\pi}_h(s,a)^2}{\Exp_{\pi' \sim \pexp}[\wmb{\Mbar}{\pi'}_h(s,a)]}}} \cdot \sqrt{\sum_{s,a,h} \Exp_{\pi' \sim \pexp}\brk*{\Exp_{\Mbar \sim \xi} \brk*{\wmb{\Mbar}{\pi'}_h(s,a) D_{\mathsf{H}}^4\big ( \Mbar_{sh}(a),M_{sh}(a) \big ) }}} \\
&  \overset{(c)}{\le} 200\log(H) \cdot \sqrt{\sum_{s,a,h} \Exp_{\Mbar \sim \xi} \brk*{ \frac{\wmb{\Mbar}{\pi}_h(s,a)^2}{\Exp_{\pi' \sim \pexp}[\wmb{\Mbar}{\pi'}_h(s,a)]}}} \cdot \sqrt{\sum_{s,a,h} \Exp_{\pi' \sim \pexp}\brk*{\Exp_{\Mbar \sim \xi} \brk*{\wmb{\Mbar}{\pi'}_h(s,a) \Dhels{\Mbar_{sh}(a)}{M_{sh}(a)}}}}
\end{align*}
where $(a)$ follows from Lemma A.13 of \cite{foster2021statistical},
$(b)$ follows from Cauchy-Schwarz and Jensen's inequality, and $(c)$
follows because the Hellinger distance is always bounded by 2. Now note that, by definition of $\pexp$, we have
\begin{align*}
\sum_{s,a,h} \Exp_{\Mbar \sim \xi} \brk*{ \frac{\wmb{\Mbar}{\pi}_h(s,a)^2}{\Exp_{\pi' \sim \pexp}[\wmb{\Mbar}{\pi'}_h(s,a)]}} \le \min_{p \in \simplex_\Pi} \max_{q \in \simplex_\Pi}\sum_{s,a,h}\Exp_{\pi \sim q} \brk*{\En_{\Mbar\sim\xi} \brk*{\frac{(\wmb{\Mbar}{\pi}_h(s,a))^2}{\En_{\pi'\sim{}p}[\wmb{\Mbar}{\pi'}_h(s,a)]}}}
\end{align*}
and by the minimax theorem, we can bound
\begin{align*}
\min_{p \in \simplex_\Pi}  \max_{q \in \simplex_\Pi}\sum_{s,a,h}\Exp_{\pi \sim q} \brk*{\En_{\Mbar\sim\xi} \brk*{\frac{(\wmb{\Mbar}{\pi}_h(s,a))^2}{\En_{\pi'\sim{}p}[\wmb{\Mbar}{\pi'}_h(s,a)]}}} & = \max_{q \in \simplex_\Pi} \min_{p \in \simplex_\Pi}  \sum_{s,a,h}\Exp_{\pi \sim q} \brk*{\En_{\Mbar\sim\xi} \brk*{\frac{(\wmb{\Mbar}{\pi}_h(s,a))^2}{\En_{\pi'\sim{}p}[\wmb{\Mbar}{\pi'}_h(s,a)]}}} \\
& \le \max_{q \in \simplex_\Pi}   \sum_{s,a,h}\Exp_{\pi \sim q} \brk*{\En_{\Mbar\sim\xi} \brk*{\frac{(\wmb{\Mbar}{\pi}_h(s,a))^2}{\En_{\pi'\sim{}q}[\wmb{\Mbar}{\pi'}_h(s,a)]}}} \\
& \le \max_{q \in \simplex_\Pi}   \sum_{s,a,h}\Exp_{\pi \sim q} \brk*{\En_{\Mbar\sim\xi} \brk*{\frac{\wmb{\Mbar}{\pi}_h(s,a)}{\En_{\pi'\sim{}q}[\wmb{\Mbar}{\pi'}_h(s,a)]}}} \\
& = SAH.
\end{align*}
By Lemma A.9 of \cite{foster2021statistical}, since
$\mu\sups{\Mbar}(s,a,h)\ldef{}\frac{1}{H} \wmb{\Mbar}{\pi'}_h(s,a)$ forms a valid
distribution on $\cS\times\cA\times\brk{H}$, we can upper bound 
\begin{align*}
\sum_{s,a,h} \wmb{\Mbar}{\pi'}_h(s,a) \Dhels{\Mbar_{sh}(a)}{M_{sh}(a)} \le H \Dhels{\Mbar(\pi')}{M(\pi')}.
\end{align*}
Altogether then, we have shown that for all $\pi\in\Pi$,
\begin{align*}
\Exp_{\Mbar \sim \xi}[\Dhels{\Mbar(\pi)}{M(\pi)}]  \le 200\log(H) \cdot \sqrt{SAH^2 \cdot \Exp_{\Mbar \sim \xi}[\Exp_{\pi \sim \pexp}[\Dhels{\Mbar(\pi)}{M(\pi)}]]}
\end{align*}
as desired.
Since the Hellinger distance is a metric and satisfies the triangle inequality, this in particular implies that, for any $M,M'$,
\begin{align*}
\Dhels{M(\pi)}{M'(\pi)} & \le 2\Exp_{\Mbar \sim \xi}[\Dhels{\Mbar(\pi)}{M(\pi)}] + 2\Exp_{\Mbar \sim \xi}[\Dhels{\Mbar(\pi)}{M'(\pi)}] \\
& \le 400 \log H \cdot \sqrt{SAH^2 \cdot \Exp_{\Mbar \sim \xi}[\Exp_{\pi \sim \pexp}[\Dhels{\Mbar(\pi)}{M(\pi)}]]} \\
& \qquad + 400 \log H \sqrt{SAH^2 \cdot \Exp_{\Mbar \sim \xi}[\Exp_{\pi \sim \pexp}[\Dhels{\Mbar(\pi)}{M'(\pi)}]]}.
\end{align*}
Thus, if
\begin{align*}
\Exp_{\Mbar \sim \xi}[\Exp_{\pi \sim \pexp}[\Dhels{\Mbar(\pi)}{M''(\pi)}]] \le \frac{\veps^2}{320000 SAH^2 \cdot \log^2 H}
\end{align*}
for both $M'' \in \{M,M'\}$, then $\Dhels{M(\pi)}{M'(\pi)} \le \veps$. It follows that a sufficient choice for $\CexpD$ is $320000SAH^2 \log^2 H/\veps^2$. 
\end{proof}

\subsubsection{Supporting Lemmas}
\begin{lemma}\label{lem:mdp_alternate_local}
If $M$ has a unique optimal policy $\pim$ and $M' \in \cMalt(M)$, then there exists some $(\stil,\atil,\htil)$ such that
\begin{align*}
\Qm{M'}{\pim}_{\htil}(\stil,\atil) > \Vm{M'}{\pim}_{\htil}(\stil).
\end{align*}
\end{lemma}
\begin{proof}[Proof of \Cref{lem:mdp_alternate_local}]
Assume that this is not the case, i.e. that for all $(s,a,h)$,
\begin{align*}
\Qm{M'}{\pim}_{h}(s,a) \le \Vm{M'}{\pim}_{h}(s) = \Qm{M'}{\pim}_h(s,\pim(s,h)).
\end{align*}
Our goal is to show that in this case $\pim[M'] = \pim$, which contradicts the fact that $M' \in \cMalt(M)$. We proceed by induction.

\paragraph{Base Case}
Let $h = H$ and assume that for all $(s,a)$,
\begin{align*}
\Qm{M'}{\pim}_{H}(s,a) \le  \Qm{M'}{\pim}_H(s,\pim(s,h)).
\end{align*}
This contradicts the assumption that $\pim$ is unique.

\paragraph{Inductive Case}
Assume that $\pim[M'](s,h') = \pim(s,h')$ for all $s$ and $h' > h$. This then implies that $\Vm{M'}{\pim}_{h+1}(s) = \Vm{M'}{\pim[M']}_{h+1}(s)$ for all $s$. It follows that for all $a$
\begin{align*}
 \Qm{M'}{\pim}_{h}(s,a) & = \rem[M']_h(s,a) + \Prmb[M']_h[\Vm{M'}{\pim}_{h+1}](s,a)  = \rem[M']_h(s,a) + \Prmb[M']_h[\Vm{M'}{\pim[M']}_{h+1}](s,a)  =  \Qm{M'}{\pim[M']}_{h}(s,a)
\end{align*}
so in particular $\Qm{M'}{\pim}_h(s,\pim(s,h)) =
\Qm{M'}{\pim[M']}_h(s,\pim(s,h))$. Since we have assumed that for all $(s,a)$
\begin{align*}
\Qm{M'}{\pim}_{h}(s,a) \le \Qm{M'}{\pim}_h(s,\pim(s,h))
\end{align*}
we have
\begin{align*}
\Qm{M'}{\pim[M']}_{h}(s,\pim[M'](s,h)) \le \Qm{M'}{\pim[M']}_h(s,\pim(s,h)).
\end{align*}
However, since each $M \in \cM$ has a unique optimal action at each state, this is a contradiction unless $\pim[M'](s,h) = \pim(s,h)$, which proves the inductive hypothesis. The result follows.
\end{proof}

\begin{lemma}\label{lem:mdp_value_kl_bound}
For MDPs $M,M'$ with unit variance Gaussian rewards, we have
\begin{align*}
\Vm{M'}{\pi}_h(s) - \Vm{M}{\pi}_h(s) \le \sqrt{\frac{8H}{\wmb{M}{\pi}_{h}(s)}  \cdot \Dkl{M(\pi)}{M'(\pi)}}.
\end{align*}
\end{lemma}
\begin{proof}[Proof of \Cref{lem:mdp_value_kl_bound}]
In the Gaussian reward setting, we have
\begin{align*}
\rem[M']_{h'}(s',a') - \rem_{h'}(s',a') & \le \sqrt{(\rem[M']_{h'}(s',a') - \rem_{h'}(s',a'))^2} \le \sqrt{2\Dklbig{M_{h',s'}(a')}{M'_{h',s'}(a')}}.
\end{align*}
Furthermore, since $\Vm{M'}{\pi}_{h'+1}(s') \in [0,1]$, we have
\begin{align*}
\Prmb[M']_{h'}[\Vm{M'}{\pi}_{h'+1}](s',a') - \Prmb[M]_{h'}[\Vm{M'}{\pi}_{h'+1}](s',a') & \le \sum_{s''} |\Probm[M']_{h'}(s''\mid s',a') - \Probm[M]_{h'}(s''\mid s',a')| \\
& = 2 \Dtv{\Probm[M']_{h'}( \cdot \mid s',a')}{\Probm[M]_{h'}( \cdot \mid s',a')} \\
& \le \sqrt{2 \Dkl{\Probm[M']_{h'}( \cdot \mid s',a')}{\Probm[M]_{h'}( \cdot \mid s',a')}} \\
& \le \sqrt{2 \Dklbig{M_{h',s'}(a')}{M'_{h',s'}(a')}}.
\end{align*}
By \Cref{lem:mdp_sim_lemma}, Jensen's inequality, and \Cref{lem:mdp_kl}, it follows that
\begin{align*}
 \Vm{M'}{\pi}_h(s) - \Vm{M}{\pi}_h(s) & \le  \sum_{h'=h}^H \sum_{s',a'} \wmb{M}{\pi}_{h'}(s',a' \mid s_h = s) \cdot 2\sqrt{2 \Dklbig{M_{h',s'}(a')}{M'_{h',s'}(a')}} \\
 & \le    2\sqrt{2H \sum_{h'=h}^H \sum_{s',a'} \wmb{M}{\pi}_{h'}(s',a' \mid s_h = s) \Dklbig{M_{h',s'}(a')}{M'_{h',s'}(a')}} \\
 & \le   2\sqrt{\frac{2H}{\wmb{M}{\pi}_{h}(s)} \cdot \sum_{h'=h}^H \sum_{s',a'} \wmb{M}{\pi}_{h'}(s',a') \Dklbig{M_{h',s'}(a')}{M'_{h',s'}(a')}} \\
 & \leq   2\sqrt{\frac{2H}{\wmb{M}{\pi}_{h}(s)}  \Dkl{M(\pi)}{M'(\pi)}}
\end{align*}
where we have used that, for $h < h'$,
\begin{align*}
\wmb{M}{\pi}_{h'}(s',a') = \sum_{s''} \wmb{M}{\pi}_{h'}(s',a' \mid s_h = s'') \wmpi(s'',h) \ge \wmb{M}{\pi}_{h'}(s',a' \mid s_h = s) \wmb{M}{\pi}_{h}(s).
\end{align*}
\end{proof}

\begin{lemma}\label{lem:min_visit_to_mingap}
For any $M \in \cM$ for which $\pim$ is unique, we have
\begin{align*}
\delminm \le \min_{s,h} \wmb{M}{\pim}_h(s).
\end{align*}
\end{lemma}
\begin{proof}[Proof of \Cref{lem:min_visit_to_mingap}]
Let $\pitil$ denote the policy that differs from policy $\pim$ only at
the state $\stil$ and layer $\htil$ given by $(\stil,\htil) = \argmin_{s,h} \wmb{M}{\pim}_h(s)$. Note that this implies, in particular, that $\wmb{M}{\pim}_{\htil}(\stil) = \wmb{M}{\pitil}_{\htil}(\stil)$ since $\pitil$ and $\pim$ take identical actions up to step $\htil$. 
By the Performance-Difference Lemma \citep{kakade2003sample}, we have
\begin{align*}
\Vm{M}{\pim}_1 - \Vm{M}{\pitil}_1 & = \sum_{h=1}^H \sum_{s,a} \wmb{M}{\pitil}_h(s,a) (\Vm{M}{\pim}_h(s) - \Qm{M}{\pim}_h(s,a)) \\
 & \overset{(a)}{=} \wmb{M}{\pitil}_{\htil}(\stil,\pitil(\stil,\htil)) (\Vm{M}{\pim}_{\htil}(\stil) - \Qm{M}{\pim}_{\htil}(\stil,\pitil(\stil,\htil))) \\
 & = \wmb{M}{\pim}_{\htil}(\stil)(\Vm{M}{\pim}_{\htil}(\stil) - \Qm{M}{\pim}_{\htil}(\stil,\pitil(\stil,\htil))) \\
 & \le \wmb{M}{\pim}_{\htil}(\stil)
\end{align*}
where $(a)$ holds since $\Vm{M}{\pim}_h(s) = \Qm{M}{\pim}_h(s,a)$ for
all $(s,a,h)$ with $\wmb{M}{\pitil}(s,a,h) > 0$ other than at
$(\stil,\htil)$. By assumption, the optimal policy is unique, so
$\Vm{M}{\pim}_1 - \Vm{M}{\pitil}_1 > 0$, and thus
\begin{align*}
\delminm = \min_{\pi \in \Pi \ : \ \Vm{M}{\pim}_1 - \Vm{M}{\pi}_1 > 0} \Vm{M}{\pim}_1 - \Vm{M}{\pi}_1 \le \Vm{M}{\pim}_1 - \Vm{M}{\pitil}_1 \le  \wmb{M}{\pim}_{\htil}(\stil) = \min_{s,h} \wmb{M}{\pim}_h(s).
\end{align*}
\end{proof}

\begin{lemma}[Lemma E.15 of \cite{dann2017unifying}]\label{lem:mdp_sim_lemma}
For MDPs $M,M'$ and policy $\pi$, we have
\begin{align*}
 \Vm{M'}{\pi}_h(s) - \Vm{M}{\pi}_h(s) = \sum_{h'=h}^H \sum_{s',a'} \wmb{M}{\pi}_{h'}(s',a' \mid s_h = s) \cdot & \Big [ \rem[M']_{h'}(s',a') - \rem_{h'}(s',a') + \Prmb[M']_{h'}[\Vm{M'}{\pi}_{h'+1}](s',a') \\
 & \qquad - \Prmb[M]_{h'}[\Vm{M'}{\pi}_{h'+1}](s',a') \Big ].
\end{align*}
\end{lemma}

\begin{lemma}\label{lem:mdp_kl}
For MDPs $M,M'$ and policy $\pi$, we have
\begin{align*}
\Dkl{M(\pi)}{M'(\pi)} = \sum_{h=1}^H \sum_{s,a} \wmb{M}{\pi}_h(s,a) \Dkl{M_{hs}(a)}{M'_{hs}(a)}.
\end{align*}
\end{lemma}
\begin{proof}[Proof of \Cref{lem:mdp_kl}]
This is a standard calculation; see e.g. \citep{simchowitz2019non,tirinzoni2021fully}.
\end{proof}


\section{Proofs and Additional Results from \creftitle{sec:lower}}

\newcommand{\mpq}{\sss{M},(p,q)}
\newcommand{\mbarpq}{\sss{\Mbar},(p,q)}
\newcommand{\mpqp}{\sss{M},(p',q')}
\newcommand{\mbarpqp}{\sss{\Mbar},(p',q')}

\subsection{Technical Lemmas}

Throughout this section, when the class $\cM$ is clear from context,
we define
\begin{align}
\Lambdammax := \{ \lambda \in \simplex_\Pi \ : \ \exists \nsf
  \leq\nbar \text{ s.t. } \delm(\lambda) \le (1+\veps) \gm / \nsf,
  \Imfull{\lambda}
  \ge (1-\veps)/\nsf \}.
\end{align}

\begin{lemma}[Derandomization]
  \label{lem:derandomize}
  Let $\nbar>0$ be given. For any $p\in\simplex(\simplex(\Pi))$,
  defining $\lambdabar_p=\En_{\lam\sim{}p}\brk{\lam}\in\Delta(\Pi)$,
  we have that for all $M\in\cM$,
  \begin{align*}
    \indic\crl*{\lambdabar_p\notin\Lambdammax}
    \leq \delta^{-1}\cdot\bbP_{\lambda\sim{}p}\brk*{\lambda\notin\Lambda(M; \veps/2,\nbar)},
  \end{align*}
  where $\delta\ldef{}\frac{\veps}{2}\cdot{}\min\crl*{1,\frac{\gm}{\nbar}}
  $.
\end{lemma}
\begin{proof}[\pfref{lem:derandomize}]%
  \newcommand{\Lambdatmp}{\Lambda(M; \veps/2, \nbar)}%
  \newcommand{\nlam}{\nsf_\lam}%
  Let $M\in\cM$ be fixed and abbreviate $\Im(\lambda) = \Imfull{\lambda}$.
  Fix $p\in\simplex(\simplex(\Pi))$. For any $\lambda\in\Lambdatmp$,
  let $\nsf_\lam>0$ denote the least $\nsf>0$ such that
  \[
\delm(\lambda) \le (1+\veps/2) \gm / \nsf,\mathand
  \Imfull{\lambda}
  \ge (1-\veps/2)/\nsf.
\]
Define
\[
  \nsf = \prn*{\En_{\lam\sim{}p}\brk*{\frac{1}{\nsf_\lam}\mid{}\lam\in\Lambdatmp}}^{-1},
  \]
  and note that by Jensen's inequality,
  \[
    \nsf\leq{} \En_{\lam\sim{}p}\brk*{\nsf_\lam\mid{}\lam\in\Lambdatmp}\leq\nbar.
  \]
We first observe that since $\delm\in\brk{0,1}$,
  \begin{align*}
    \delm(\lambdabar_p)
    &\leq
      \En_{\lam\sim{}p}\brk*{\delm(\lam)\mid\lam\in\Lambdatmp} +
      \bbP_{\lam\sim{}p}\brk*{\lam\notin\Lambdatmp}\\
    &\leq
      (1+\veps/2)\gm\cdot{}\En_{\lam\sim{}p}\brk*{\frac{1}{\nsf_\lam}\mid\lam\in\Lambdatmp}
      + \bbP_{\lam\sim{}p}\brk*{\lam\notin\Lambdatmp}\\
        &=
      (1+\veps/2)\frac{\gm}{\nsf} + \bbP_{\lam\sim{}p}\brk*{\lam\notin\Lambdatmp}.
  \end{align*}
Next, note that the map $\lambda\mapsto{}\Im(\lambda)$ is concave and
non-negative (it is an
infimum over non-negative linear functions), so we have
\begin{align*}
  \Im(\lambdabar_p)
  &\geq{}\En_{\lam\sim{}p}\brk*{\Im(\lam)}
  \\
&\geq\En_{\lambda\sim{}p}\brk*{\Im\prn*{\lam}\mid\lam\in\Lambdatmp}\cdot{}\bbP_{\lam\sim{}p}\brk*{\lam\in\Lambdatmp}\\
&\geq(1-\veps/2)\En_{\lambda\sim{}p}\brk*{\frac{1}{\nlam}\mid\lam\in\Lambdatmp}\cdot{}\bbP_{\lam\sim{}p}\brk*{\lam\in\Lambdatmp}
  \\
  &=(1-\veps/2)\frac{1}{\nsf}\cdot{}\bbP_{\lam\sim{}p}\brk*{\lam\in\Lambdatmp}\\
  &=(1-\veps/2)\frac{1}{\nsf}\cdot{}\prn*{1-\bbP_{\lam\sim{}p}\brk*{\lam\notin\Lambdatmp}}.
\end{align*}
It follows that as long as
\begin{align*}
  \bbP_{\lam\sim{}p}\brk*{\lam\notin\Lambdatmp}
  \leq{} \delta\ldef{}\frac{\veps}{2}\cdot\min\crl*{1, \frac{\gm}{\nbar}}
  \leq{}  \frac{\veps}{2}\cdot\min\crl*{1, \frac{\gm}{\nsf}},
\end{align*}
we have
\[
\delm(\lambdabar_p) \le (1+\veps) \gm / \nsf,\mathand
  \Imfull{\lambdabar_p}
  \ge (1-\veps)/\nsf,
\]
so that $\lambdabar_p\in\Lambda(M; \veps,
\nsf)\subseteq\Lambda(M; \veps,\nbar)$.
      We conclude that
      \begin{align*}
        \indic\crl*{\lambdabar_p\notin\Lambda(M; \veps,\nbar)}
\leq{} \indic\crl*{
        \bbP_{\lambda\sim{}p}\brk*{\lambda\notin\Lambdatmp}>\delta}
        \leq{} \delta^{-1}\cdot{}\bbP_{\lambda\sim{}p}\brk*{\lambda\notin\Lambdatmp}.
      \end{align*}
    \end{proof}

    \begin{restatable}{lemma}{lowregretfeasible}
  \label{lem:low_regret_feasible}
  Let $M\in\cM$ be given, and suppose \pref{asm:mingap} holds. Fix
  $T\in\bbN$ and consider an 
  algorithm $\mathbb{A}$ such that for all $M'\in\cMalt(M)\cup\crl{M}$,
  \[
    \Empa\brk*{\RegDM}
    \leq{} R\sups{M'}\cdot{}\log(T).
  \]
  for some bound $\Rm\geq{}2$. Define $\etam\in\Aspace$ via
  \[
    \etam(\pi) = \Ema\brk*{\frac{T(\pi)}{\log(T)}}.
  \]
  Then if
  \[
\log(T) \geq{} \frac{6}{\veps}\log\prn*{\sup_{M'\in\cMalt(M)\cup\crl{M}}\frac{R\sups{M'}}{\delmin\sups{M'}}\cdot\log(T)},
\]
we must have
\[
\Imfull{\etam} \geq{} (1-\veps).
\]
\end{restatable}

\begin{proof}[\pfref{lem:low_regret_feasible}]
  Throughout this proof we will use that $\pim$ is uniquely defined
  for all $M\in\cM$ by \cref{asm:mingap}. Note that
\begin{align*}
\Ema\brk{\RegDM} \le \Rm\log T \implies \sum_{\pi \neq \pim} \Ema\brk{T(\pi)} \le \frac{\Rm \log T}{\delminm}.
\end{align*}
Fix some $M' \in \cMalt(M)$. Then $\pim \neq \pimp$ (recall that under \pref{asm:mingap}, each $M \in \cM$ has a unique optimal), so
\begin{align*}
\Ema\brk{T(\pimp)} \le \frac{\Rm \log T}{\delminm}, \quad \Empa\brk{T(\pimp)} \ge T - \frac{\Rm[M'] \log T}{\delminm[M']}.
\end{align*}
By Lemma H.1 of \cite{simchowitz2019non}, we have that for any $\hist\ind{T}$-measurable variable $Z \in [0,1]$, that 
\begin{align*}
\sum_{\pi} \Ema\brk{T(\pi)} \KL(M(\pi),M'(\pi)) \ge d(\Ema\brk{Z},\Empa\brk{Z})
\end{align*}
for $d(x,y) = x \log \frac{x}{y} + (1-x) \log \frac{1-x}{1-y}$. Choosing $Z = T(\pimp)/T$, and using that
\begin{align*}
d(x,y) \ge (1-x) \log \frac{1}{1-y} - \log 2,
\end{align*}
we have
\begin{align*}
\sum_{\pi} \Ema\brk{T(\pi)} \KL(M(\pi),M'(\pi)) & \ge \left ( 1 - \frac{\Rm \log T}{\delminm T} \right ) \log \frac{T}{T - ( T - \frac{\Rm[M'] \log T}{\delminm[M']})} - \log 2 \\
&  = \left ( 1 - \frac{\Rm \log T}{\delminm T} \right ) \left ( \log T - \log \frac{\Rm[M'] \log T}{\delminm[M']} \right ) - \log 2.
\end{align*}
Now, if
\[
\log(T) \geq{} \frac{2}{\veps}\log\prn*{ \frac{\sup_{M'\in\cMalt(M)\cup\crl{M}}\Rm[M'] \log
  T}{\delminm[M']}\vee{}2}, 
\]
we have $\log(2)\leq{}\veps\log(T)$, 
\[
  \log T - \log\prn*{ \frac{\Rm[M'] \log T}{\delminm[M']}}
  \geq{} (1-\veps)\log(T),
\]
and $1 - \frac{\Rm \log T}{\delminm T} \ge (1-\veps)$, so we can
lower bound
\begin{align*}
  \sum_{\pi} \Ema\brk{T(\pi)} \KL(M(\pi),M'(\pi))
  \geq{} \prn*{(1-\veps)^2-\veps}\log(T) \geq{} (1-3\veps)\log(T).
\end{align*}
As this is true for every $M' \in \cMalt(M)$, the result follows.
\end{proof}

\optregretfeasible*

\begin{proof}[\pfref{lem:opt_regret_feasible}]
  Immediate consequence of \pref{lem:low_regret_feasible}.
\end{proof}

\subsection{Proof of \creftitle{thm:lower_t}}
\newcommand{\decgl}[1][\alpha]{\mathsf{dec}_{#1}^{\mathrm{GL}}}

\begin{theorem}[Full Statement of \creftitle{thm:lower_t}]
  \label{thm:lower_t_general}
  Let $\veps>0$ and $\cMsub \subseteq \cM$ be given. Let $\crl{\nbarm}_{M\in\cMsub}$ be a
collection of non-negative scalars indexed by $\cMsub$, and set
$\delta\ldef{}\frac{\veps}{2}\cdot{}\min\crl*{1,
  \inf_{M\in\cMsub}\frac{\gm}{\nbarm}}$.
  Unless
\[
  T >
  \frac{\delta}{8}\cdot\sup_{\Mbar\in\cMall} \aecM{2\veps}{\cM}(\cMsub,\Mbar),
\]
any algorithm must have, for some $M \in \cMsub$:
\[
  \Pma\brk*{\lambdahat\notin\Lambdammaxm}
\geq\frac{\delta}{6}.
\]
\end{theorem}

\begin{proof}[\pfref{thm:lower_t_general}]%
  \newcommand{\opt}{\mathsf{opt}}%
    \newcommand{\Lambdatmp}{\Lambda(M; 2\veps, \nbar)}%
Fix $\veps>0$, and let an algorithm $\Alg$ be given. For any $\Mbar\in\cMall$, define $\qmbar =
\Pmbara(\lambdahat=\cdot)\in\simplex(\simplex(\Pi))$, and let
$\ommbar\ldef\Embara\brk*{\frac{1}{T}\sum_{t=1}^{T}p\ind{t}}\in\simplex(\Pi)$.

Fix $\alpha>0$ and $\Mbar\in\cMall$ be fixed. Define
\begin{align*}
  M = \argmax_{M\in\cMsub}\crl*{
  \bbP_{\lambda\sim{}\qmbar}\brk*{\lambda\notin\Lambdammaxm}
  \mid{} \En_{\pi\sim{}\ommbar}\brk*{\Dkl{\Mbar(\pi)}{M(\pi)}}\leq\alpha^2
  };
\end{align*}
we assume that such an $M\in\cMsub$ does exist, as otherwise the claim we will prove is trivial.
It is immediate from this definition that we have
\begin{align*}
  \Pmbara\brk*{\lambdahat\notin\Lambdammaxm}
  &=  \bbP_{\lambda\sim\qmbar}\brk*{\lambda\notin\Lambdammaxm} \\
  &= \sup_{M\in\cMsub}\crl*{
  \bbP_{\lambda\sim{}\qmbar}\brk*{\lambda\notin\Lambdammaxm}
  \mid{} \En_{\pi\sim{}\ommbar}\brk*{\Dkl{\Mbar(\pi)}{M(\pi)}}\leq\alpha^2
  }\\
  &\geq{}
    \inf_{q\in\simplex(\simplex(\Pi)),\om\in\simplex(\Pi)}\sup_{M\in\cMsub}\crl*{
    \bbP_{\lambda\sim{}q}\brk*{\lambda\notin\Lambdammaxm}
  \mid{} \En_{\pi\sim{}\om}\brk*{\Dkl{\Mbar(\pi)}{M(\pi)}}\leq\alpha^2
    } \\
    & \rdef \opt,
\end{align*}
with the convention that this value is zero if the set
$\crl*{M\in\cMsub \mid{}\En_{\pi\sim{}\om}\brk*{\Dkl{\Mbar(\pi)}{M(\pi)}}\leq\alpha^2}$
is empty. In addition, we have
\begin{align}
  \En_{\pi\sim\ommbar}\brk*{\Dkl{\Mbar(\pi)}{M(\pi)}}
  \leq\alpha^2.
  \label{eq:decgl_kl}
\end{align}

Now, define
$\delta\ldef{}\frac{\veps}{2}\cdot{}\min\crl*{1,\inf_{M\in\cMsub}\frac{\gm}{\nbarm}}$,
  and let $\lambdabar_q\ldef\En_{\lambda\sim{}q}\brk*{\lambda}$. By
  \pref{lem:derandomize}, we have
  \begin{align}
    \opt &= \inf_{q\in\simplex(\simplex(\Pi)),\om\in\simplex(\Pi)}\sup_{M\in\cMsub}\crl*{
    \bbP_{\lambda\sim{}q}\brk*{\lambda\notin\Lambdammaxm}
  \mid{} \En_{\pi\sim{}\om}\brk*{\Dkl{\Mbar(\pi)}{M(\pi)}}\leq\alpha^2 }\notag
            \\
         &\geq{} \delta\cdot{}\inf_{q\in\simplex(\simplex(\Pi)),\om\in\simplex(\Pi)}\sup_{M\in\cMsub}\crl*{
           \indic\crl*{\lambdabar_q\notin\Lambda(M; 2\veps,\nbarm)}
  \mid{} \En_{\pi\sim{}\om}\brk*{\Dkl{\Mbar(\pi)}{M(\pi)}}\leq\alpha^2\notag
           } \\
         &\geq{} \delta\cdot{}\inf_{\lam\in\simplex(\Pi),\om\in\simplex(\Pi)}\sup_{M\in\cMsub}\crl*{
           \indic\crl*{\lam\notin\Lambda(M; 2\veps,\nbarm)}
  \mid{} \En_{\pi\sim{}\om}\brk*{\Dkl{\Mbar(\pi)}{M(\pi)}}\leq\alpha^2\notag
           } \\
         &\geq{} \delta\cdot{}\inf_{\lam\in\simplex(\Pi),\om\in\simplex(\Pi)}\sup_{M\in\cMsub}\crl*{
           \indic\crl*{\lam\notin\Lambda(M; 2\veps)}
  \mid{} \En_{\pi\sim{}\om}\brk*{\Dkl{\Mbar(\pi)}{M(\pi)}}\leq\alpha^2\notag
           } \\
    &= \delta\cdot\indic\crl*{\alpha^2\geq{} \prn*{\aecM{2\veps}{\cM}(\cMsub,\Mbar)}^{-1}}.\label{eq:decgl_lb}
  \end{align}
  We conclude that
  \begin{align}
    \Pmbara\brk*{\lambdahat\notin\Lambdammaxm}
    \geq{} \delta\cdot\indic\crl*{\alpha^2\geq{} \prn*{\aecM{2\veps}{\cM}(\cMsub,\Mbar)}^{-1}}.\label{eq:decgl_intermediate}
  \end{align}

To proceed, using Lemma A.11 of \citet{foster2021statistical}, we
have
\begin{align*}
  \Pma\brk*{\lambdahat\notin\Lambdammaxm}
  &\geq{} \frac{1}{3}     \Pmbara\brk*{\lambdahat\notin\Lambdammaxm}
    - \frac{4}{3}\Dkl{\Pmbara}{\Pma}\\
  &\geq{} \frac{\delta}{3}\indic\crl*{\alpha^2\geq{} \prn*{\aecM{2\veps}{\cM}(\cMsub,\Mbar)}^{-1}}
    - \frac{4}{3}\Dkl{\Pmbara}{\Pma}.
\end{align*}
Using \pref{eq:decgl_kl} gives
\[
\Dkl{\Pmbara}{\Pma}=\Ema\brk*{\sum_{t=1}^{T}\En_{\pi\sim{}p\ind{t}}\Dkl{\Mbar(\pi)}{M(\pi)}}
  =T\cdot\En_{\pi\sim\ommbar}\brk*{\Dkl{\Mbar(\pi)}{M(\pi)}} \leq \alpha^2T,
\]
so we have
\begin{align*}
  \Pma\brk*{\lambdahat\notin\Lambdammaxm}
  &\geq{} \frac{\delta}{3}\indic\crl*{\alpha^2\geq{} \prn*{\aecM{2\veps}{\cM}(\cMsub,\Mbar)}^{-1}}
    - \frac{4}{3}\alpha^2T.
\end{align*}
We set $\alpha^2=\frac{\delta}{8T}$, so that
\begin{align*}
  \Pma\brk*{\lambdahat\notin\Lambdammaxm}
  &\geq{} \frac{\delta}{6}\cdot{}\indic\crl*{\alpha^2\geq{}
  \prn*{\aecM{2\veps}{\cM}(\cMsub,\Mbar)}^{-1}} \\
  & =  \frac{\delta}{6}\cdot{}\indic\crl*{T \leq{} \frac{\delta}{8}\cdot\aecM{2\veps}{\cM}(\cMsub,\Mbar)}.
\end{align*}
By taking the supremum over all possible choices for $\Mbar\in\cMall$, we conclude that
unless
\[
T > \frac{\delta}{8}\cdot\sup_{\Mbar\in\cMall}\aecM{2\veps}{\cM}(\cMsub,\Mbar),
\]
the algorithm must have $\Pma\brk*{\lambdahat\notin\Lambdammaxm}
\geq\frac{\delta}{6}$.
\end{proof}

\subsection{Proof of \creftitle{thm:lower_logt}}

 Recall that for $M\in\cMall$ we define
\begin{align*}
  \pibm = \argmax_{\pi\in\Pi}\fm(\pi)
\end{align*}
as the set $\pibm\subseteq\Pi$ of all optimal decisions $\pi$ with
$\fm(\pi)=\max_{\pi'\in\Pi}\fm(\pi')$. Unless otherwise stated, the
results in this subsection do not make use of
\pref{asm:mingap}. For $\Mbar\in\cMall$ and $\cMsub \subseteq \cM$, we define
\begin{align*}
  \cMoptsub(\Mbar) = \crl*{M\in\cMsub \mid{} \pibm\subseteq\pibmbar,\;\;
  \Dkl{\Mbar(\pi)}{M(\pi)}=0\ \forall{}\pi\in\pibmbar}.
\end{align*}
For a subset $\Pi'\subseteq\Pi$, let
\[
N_{\neg\Pi'}=\abs*{\crl*{t\in\brk{T}\mid{}\pi\ind{t}\notin\Pi'}}.
\]
Note that for all $M\in\cMoptsub(\Mbar)$, since $\pibm\subseteq\pibmbar$, we have
\[
\Nmbar\leq\Nm.
\]

\lowerlogt*

\begin{proof}[\pfref{thm:lower_logt}]
For each $M\in\cM$, if $\Ema\brk*{\RegDM}\leq{}2\gm\log(T)$, then
$\Ema\brk*{\Nm}\leq{} 2\frac{\gm}{\delminm}\log(T)$. The result now
follows by appealing to \pref{thm:lower_logt_general} with
$\nbarm=\nmax$ and
$R=2\sup_{M\in\cMsub}\frac{\gm}{\delminm}\log(T)$.
\end{proof}

\begin{theorem}
  \label{thm:lower_logt_general}
  Let $T\in\bbN$, $\veps>0$, $R\geq{}1$, and $\cMsub \subseteq \cM$ be given. Let $\crl{\nbarm}_{M\in\cMsub}$
        be a collection of non-negative scalars
        indexed by $\cMsub$. Define
  $\delta=\frac{\veps}{2}\cdot\min\crl*{1,\inf_{M\in\cMsub}\frac{\gm}{\nbarm}}$. Unless
  \begin{align*}
    R \geq
\frac{\delta^2}{192}\cdot\sup_{\Mbar\in\cMall}\aecM{2\veps}{\cM}(\cMoptsub(\Mbar),\Mbar),
  \end{align*}
  there is no algorithm that simultaneously ensures that
  \begin{enumerate}
  \item $\Ema\brk*{\Nm}\leq{}R,\;\;\forall{}M\in\cMsub$.
  \item $\Pma\brk*{\lambdahat\notin\Lambdammaxm}\leq{} \frac{\delta}{12},\;\;\forall{}M\in\cMsub$.
  \end{enumerate}

\end{theorem}

\begin{proof}[\pfref{thm:lower_logt_general}]
  Let $\veps>0$ be fixed. To prove the result, it suffices to lower bound the constrained minimax value
  \begin{align}
    \label{eq:constrained_minimax}
    \fM\ldef\sup_{\Mbar\in\cMall}\inf_{\Alg}\crl*{
    \sup_{M\in\cMoptsub(\Mbar)}
\Pma\brk*{\lambdahat\notin\Lambdammaxm}
    \mid{} \Empa\brk{N_{\neg\mb{\pi}\subs{M'}}} \leq{}
    R\ \ \forall{}M'\in\cMoptsub(\Mbar)}.
  \end{align}
We begin by appealing to the following technical lemma.
\begin{lemma}
  \label{lem:reduction}
  Let $\Mbar\in\cMall$ and $T\in\bbN$ be given. Consider any algorithm $\Alg$ with
  the property that for all $M\in\cMoptsub(\Mbar)$,
  \[
    \Ema\brk{\Nm} \leq{} R
  \]
   for some $R\geq{}1$.
   For any $\beta\in(0,1)$, there exists a modified algorithm $\Alg'$
  with the following properties:
  \begin{itemize}
  \item $\Pmap\brk*{\Nmbar > \ceil{\frac{R}{\beta}}}=0$ for all
    models $M\in\cMall$.
  \item For all $M\in\cMoptsub(\Mbar)$,
    \[
\Pma\brk*{\lambdahat\notin\Lambdammaxm}
      \geq \Pmap\brk*{\lambdahat\notin\Lambdammaxm}- \beta.
    \]
  \end{itemize}
  
\end{lemma}
By \pref{lem:reduction}, for any $\beta\in(0,1)$, we have
    \begin{align*}
      \fM \geq \sup_{\Mbar\in\cMall}\inf_{\Alg}\crl*{
      \sup_{M\in\cMoptsub(\Mbar)}
\Pma\brk*{\lambdahat\notin\Lambdammaxm}
      \mid{} \Pmpa\brk*{\Nmbar >
      \ceil*{\frac{R}{\beta}}}=0\ \ \forall{}M'\in\cMall} - \beta.
    \end{align*}
    Now, consider an arbitrary choice for $\Mbar$ above. We lower
    bound the minimax value using another technical lemma.
\begin{lemma}
  \label{lem:logt_intermediate}
  Let $T\in\bbN$ and $\veps>0$ be given. Let $\crl{\nbarm}_{M\in\cM}$ be a
collection of non-negative scalars indexed by $\cM$. 
Consider any algorithm $\Alg$ with the property that
\[
  \Pmbara\brk*{\Nmbar>R} = 0
\]
for some $R\geq{}1$. For any $\Mbar\in\cMall$, if we set 
$\delta\ldef{}\frac{\veps}{2}\cdot{}\min\crl*{1,
  \inf_{M\in\cMoptsub(\Mbar)}\frac{\gm}{\nbarm}}$, then unless
\[
  R >
  \frac{\delta}{8}\cdot\aecM{2\veps}{\cM}(\cMoptsub(\Mbar),\Mbar),
\]
the algorithm must have
\[
  \sup_{M\in\cMoptsub(\Mbar)}\Pma\brk*{\lambdahat\notin\Lambdammaxm}
\geq\frac{\delta}{6}.
\]
\end{lemma}    
    Bounding $\ceil*{\frac{R}{\beta}}\leq{}
    \frac{2R}{\beta}$, it follows from \pref{lem:logt_intermediate}
    that unless
    \begin{align*}
      \frac{2R}{\beta}> \frac{\delta}{8}\cdot\aecM{2\veps}{\cM}(\cMoptsub(\Mbar),\Mbar),
    \end{align*}
    where $\delta\ldef{}\frac{\veps}{2}\cdot{}\min\crl*{1,
  \inf_{M\in\cMoptsub(\Mbar)}\frac{\gm}{\nbarm}}$,
    we have
    \begin{align*}
\fM \geq      \inf_{\Alg}\crl*{
      \sup_{M\in\cMoptsub(\Mbar)}
\Pma\brk*{\lambdahat\notin\Lambdammaxm}
      \mid{} \Pmpa\brk*{\Nmbar >
      \ceil*{\frac{R}{\beta}}}=0\ \ \forall{}M'\in\cMall} -
        \beta
      \geq{} \frac{\delta}{6} - \beta.
    \end{align*}
    To conclude, we set $\beta=\frac{\delta}{12}$ and maximize over $\Mbar\in\cMall$.
\end{proof}

\begin{proof}[\pfref{lem:reduction}]
  Fix $\beta\in(0,1)$ and let $C\ldef{}\ceil{\frac{R}{\beta}}$.
  Fix $\Alg=(p,q)$ and consider the algorithm $\Alg' =
  (p',q')$ defined implicitly as follows. For $t=1,\ldots,T$:
  \begin{itemize}
  \item Sample $\pi\ind{t}\sim{}p\ind{t}(\cdot\mid{}\hist\ind{t-1})$.
    \item If
      $\abs{\crl*{i\leq{}t\mid{}\pi\ind{i}\notin\pibmbar}}=C$,
      break and play an arbitrary decision $\pi\in\pibmbar$ until round
      $T$.
    \end{itemize}
    Return $\lambdahat\sim{}q(\cdot\mid\hist\ind{T})$.

    It is immediate from this construction that $\Alg' = (p',q')$ has
    $\Nmbar\leq{}C$ almost surely under all
    possible models $M\in\cMall$. We now focus on bounding the performance. Let $T_0$ be the greatest
value of $t$ for which $\abs{\crl*{i\leq{}t\mid{}\pi\ind{i}\notin\pibmbar}}\leq{}C$. First, observe that for all
$M\in\cMoptsub(\Mbar)$, since the algorithms behave identically in law
whenever $T_0=T$,  
\begin{align*}
  \Pmap\brk*{\lambdahat\in\Lambdammaxm}
  &\geq{}   \Pmap\brk*{\lambdahat\in\Lambdammaxm \wedge{}T_0=T} \\
  &=   \Pma\brk*{\lambdahat\in\Lambdammaxm \wedge T_0=T} \\
  &= \Pma\brk*{\lambdahat\in\Lambdammaxm \wedge \Nmbar\leq{}C}.
\end{align*}
By the union bound, we have
\begin{align*}
  \Pma\brk*{\lambdahat\in\Lambdammaxm \wedge \Nmbar\leq{}C}
  \geq{} \Pma\brk*{\lambdahat\in\Lambdammaxm} - \Pma\brk*{\Nmbar>C}.
\end{align*}
Finally, we observe that by Markov's inequality we have
\begin{align*}
  \Pma\brk*{\Nmbar>C}
  \leq{} \frac{\Ema\brk*{\Nmbar}}{C}
  \leq{}  \frac{\Ema\brk*{\Nm}}{C} \leq{} \beta,
\end{align*}
where we have used that $\Nmbar\leq\Nm$, since
$\pibm\subseteq\pibmbar$. Rearranging, we obtain
\[
  \Pma\brk*{\lambdahat\notin\Lambdammaxm}
      \geq \Pmap\brk*{\lambdahat\notin\Lambdammaxm}- \beta.
    \]
  \end{proof}

  \begin{proof}[\pfref{lem:logt_intermediate}]%
      \newcommand{\opt}{\mathsf{opt}}%
    \newcommand{\Lambdatmp}{\Lambda(M; 2\veps, \nbar)}%
Fix $\veps>0$, and let an algorithm $\Alg$ be given. For any $\Mbar\in\cMall$, define $\qmbar =
\Pmbara(\lambdahat=\cdot)\in\simplex(\simplex(\Pi))$, and let
$\ommbar\ldef\Embara\brk*{\frac{1}{\Nmbar}\sum_{t:\pi\ind{t}\notin\pibmbar}p\ind{t}}\in\simplex(\Pi)$,
with the convention that the value inside the expectation is zero
whenever $\Nmbar=0$.\footnote{If $\Nmbar=0$ almost surely under $\Mbar$, we can take
  $R=0$, in which case the statement of the lemma is vacuous.}

Fix $\alpha>0$ and $\Mbar\in\cMall$ be fixed. Define
\begin{align*}
  M = \argmax_{M\in\cMoptsub(\Mbar)}\crl*{
  \bbP_{\lambda\sim{}\qmbar}\brk*{\lambda\notin\Lambdammaxm}
  \mid{} \En_{\pi\sim{}\ommbar}\brk*{\Dkl{\Mbar(\pi)}{M(\pi)}}\leq\alpha^2
  };
\end{align*}
we assume that such an $M\in\cMoptsub(\Mbar)$ does exist, as otherwise the claim we will prove is trivial.
It is immediate from this definition that we have
\begin{align*}
  & \Pmbara\brk*{\lambdahat\notin\Lambdammaxm} \\
  &=  \bbP_{\lambda\sim\qmbar}\brk*{\lambda\notin\Lambdammaxm} \\
  &= \sup_{M\in\cMoptsub(\Mbar)}\crl*{
  \bbP_{\lambda\sim{}\qmbar}\brk*{\lambda\notin\Lambdammaxm}
  \mid{} \En_{\pi\sim{}\ommbar}\brk*{\Dkl{\Mbar(\pi)}{M(\pi)}}\leq\alpha^2
  }\\
  &\geq{}
    \inf_{q\in\simplex(\simplex(\Pi)),\om\in\simplex(\Pi)}\sup_{M\in\cMoptsub(\Mbar)}\crl*{
    \bbP_{\lambda\sim{}q}\brk*{\lambda\notin\Lambdammaxm}
  \mid{} \En_{\pi\sim{}\om}\brk*{\Dkl{\Mbar(\pi)}{M(\pi)}}\leq\alpha^2
    }\rdef \opt,
\end{align*}
with the convention that this value is zero if the set
$\crl*{M\in\cMoptsub(\Mbar)\mid{}\En_{\pi\sim{}\om}\brk*{\Dkl{\Mbar(\pi)}{M(\pi)}}\leq\alpha^2}$
is empty. In addition, we have
\begin{align}
  \En_{\pi\sim\ommbar}\brk*{\Dkl{\Mbar(\pi)}{M(\pi)}}
  \leq\alpha^2.
  \label{eq:kl_logt_intermediate}
\end{align}

Now, define
$\delta\ldef{}\frac{\veps}{2}\cdot{}\min\crl*{1,\inf_{M\in\cMoptsub(\Mbar)}\frac{\gm}{\nbarm}}$,
  and let $\lambdabar_q\ldef\En_{\lambda\sim{}q}\brk*{\lambda}$. By
  \pref{lem:derandomize}, we have
  \begin{align}
    \opt &= \inf_{q\in\simplex(\simplex(\Pi)),\om\in\simplex(\Pi)}\sup_{M\in\cMoptsub(\Mbar)}\crl*{
    \bbP_{\lambda\sim{}q}\brk*{\lambda\notin\Lambdammaxm}
  \mid{} \En_{\pi\sim{}\om}\brk*{\Dkl{\Mbar(\pi)}{M(\pi)}}\leq\alpha^2 }\notag
            \\
         &\geq{} \delta\cdot{}\inf_{q\in\simplex(\simplex(\Pi)),\om\in\simplex(\Pi)}\sup_{M\in\cMoptsub(\Mbar)}\crl*{
           \indic\crl*{\lambdabar_q\notin\Lambda(M; 2\veps,\nbarm)}
  \mid{} \En_{\pi\sim{}\om}\brk*{\Dkl{\Mbar(\pi)}{M(\pi)}}\leq\alpha^2\notag
           } \\
         &\geq{} \delta\cdot{}\inf_{\lam\in\simplex(\Pi),\om\in\simplex(\Pi)}\sup_{M\in\cMoptsub(\Mbar)}\crl*{
           \indic\crl*{\lam\notin\Lambda(M; 2\veps,\nbarm)}
  \mid{} \En_{\pi\sim{}\om}\brk*{\Dkl{\Mbar(\pi)}{M(\pi)}}\leq\alpha^2\notag
           } \\
         &\geq{} \delta\cdot{}\inf_{\lam\in\simplex(\Pi),\om\in\simplex(\Pi)}\sup_{M\in\cMoptsub(\Mbar)}\crl*{
           \indic\crl*{\lam\notin\Lambda(M; 2\veps)}
  \mid{} \En_{\pi\sim{}\om}\brk*{\Dkl{\Mbar(\pi)}{M(\pi)}}\leq\alpha^2\notag
           } \\
    &= \delta\cdot\indic\crl*{\alpha^2\geq{} \prn*{\aecM{2\veps}{\cM}(\cMoptsub(\Mbar),\Mbar)}^{-1}}.\label{eq:decgl_lb}
  \end{align}
Hence, we have
  \begin{align}
    \Pmbara\brk*{\lambdahat\notin\Lambdammaxm}
    \geq{} \delta\cdot\indic\crl*{\alpha^2\geq{} \prn*{\aecM{2\veps}{\cM}(\cMoptsub(\Mbar),\Mbar)}^{-1}}.\label{eq:decgl_intermediate}
  \end{align}

To proceed, using Lemma A.11 of \citet{foster2021statistical}, we
have
\begin{align*}
  \Pma\brk*{\lambdahat\notin\Lambdammaxm}
  &\geq{} \frac{1}{3}     \Pmbara\brk*{\lambdahat\notin\Lambdammaxm}
    - \frac{4}{3}\Dkl{\Pmbara}{\Pma}\\
  &\geq{} \frac{\delta}{3}\indic\crl*{\alpha^2\geq{} \prn*{\aecM{2\veps}{\cM}(\cMoptsub(\Mbar),\Mbar)}^{-1}}
    - \frac{4}{3}\Dkl{\Pmbara}{\Pma}.
\end{align*}
Now, recall that from the definition, we have that for all $M\in\cMoptsub(\Mbar)$,
\begin{align*}
  \Dkl{\Pmbara}{\Pma}&=\Ema\brk*{\sum_{t:\pi\ind{t}\notin\pibmbar}\En_{\pi\sim{}p\ind{t}}\Dkl{\Mbar(\pi)}{M(\pi)}}\\
  &=\Ema\brk*{\frac{\Nmbar}{\Nmbar}\sum_{t:\pi\ind{t}\notin\pibmbar}\En_{\pi\sim{}p\ind{t}}\Dkl{\Mbar(\pi)}{M(\pi)}}\\
  &\leq{}   R\cdot{}\Ema\brk*{\frac{1}{\Nmbar}\sum_{t:\pi\ind{t}\notin\pibmbar}\En_{\pi\sim{}p\ind{t}}\Dkl{\Mbar(\pi)}{M(\pi)}}\\
    &=R\cdot\En_{\pi\sim\ommbar}\brk*{\Dkl{\Mbar(\pi)}{M(\pi)}} \leq \alpha^2R,
\end{align*}
where the first inequality uses that $\Pmbara\brk*{\Nmbar>R}=0$, and
the second inequality uses \pref{eq:kl_logt_intermediate}.

With this, we have
\begin{align*}
  \Pma\brk*{\lambdahat\notin\Lambdammaxm}
  &\geq{} \frac{\delta}{3}\indic\crl*{\alpha^2\geq{} \prn*{\aecM{2\veps}{\cM}(\cMoptsub(\Mbar),\Mbar)}^{-1}}
    - \frac{4}{3}\alpha^2R.
\end{align*}
We set $\alpha^2=\frac{\delta}{8R}$, so that
\begin{align*}
  \Pma\brk*{\lambdahat\notin\Lambdammaxm}
  &\geq{} \frac{\delta}{6}\cdot{}\indic\crl*{\alpha^2\geq{}
  \prn*{\aecM{2\veps}{\cM}(\cMoptsub(\Mbar),\Mbar)}^{-1}} \\
  & =  \frac{\delta}{6}\cdot{}\indic\crl*{R \leq{} \frac{\delta}{8}\cdot\aecM{2\veps}{\cM}(\cMoptsub(\Mbar),\Mbar)}.
\end{align*}
We conclude that
unless
\[
R > \frac{\delta}{8}\cdot\sup_{\Mbar\in\cMall}\aecM{2\veps}{\cM}(\cMoptsub(\Mbar),\Mbar),
\]
the algorithm must have $\Pma\brk*{\lambdahat\notin\Lambdammaxm}
\geq\frac{\delta}{6}$.

\end{proof}

\subsection{Proofs for Lower Bound Examples}\label{sec:lb_ex_proofs}

\begin{proof}[\pfref{ex:mab_lower}]%
  \newcommand{\Lambdamdel}[1][M]{\Lambda(#1; \delta)}%
  \newcommand{\nsfi}{\nsf_i}%
  \newcommand{\nul}{\underline{\nsf}}%
  Let $\Delta\in(0,1)$ and $A\geq{}2$ be given and set
  \[
    \cM = \crl*{M(\pi)=\cN(\fm(\pi),1/2)\mid{}\fm\in\brk*{0,1}^{A}}
  \]
  and
  \begin{align*}
  \cMsub = \crl*{ M \in \cM \ : \ \delminm\geq{}\Delta/2}
  \end{align*}
  Define $\Mbar\in\cM$ via $\fmbar(\pi) =
  \Delta\indic\crl{\pi=A}$. Fix $\veps\in(0,1/2)$ and define a subclass
  \begin{align*}
\cM' = \crl{\Mbar}\cup\crl*{M_i}_{i\in\brk{A-1}}
  \end{align*}
  via
  \[
    \fmi(\pi) = \Delta\indic\crl{\pi=A} + \veps\cdot\Delta\indic\crl{\pi=i}.
  \]
  Since $\veps\leq{}1/2$, we have $\cM'\subseteq\cMopt(\Mbar)$ and $\cM' \subseteq \cMsub$. In addition,
  we have
  \[
    \Dkl{\Mbar(\pi)}{M_i(\pi)} = (\fmbar(\pi)-\fmi(\pi))^2.
  \]
  Let $\cM'' \subseteq \cM$ denote the set of instances such that, for $M' \in \cM''$, $\Dkl{M_i(A)}{M'(A)} = 0$, and $M' \in \cMalt(M_i)$, for all $i \in [A-1]$. Then, 
  \begin{align}\label{eq:mab_aec_lb_Im}
  \begin{split}
  \Imfull[M_i]{\lambda} & = \inf_{M' \in \cMalt(M_i)} \Exp_{\pi \sim \lambda}[\Dkl{M_i(\pi)}{M'(\pi)}] \\
  & \le \inf_{M' \in \cM''} \Exp_{\pi \sim \lambda}[\Dkl{M_i(\pi)}{M'(\pi)}] \\
  & = \min_{j\in\brk{A-1}}\crl*{\lambda_i\cdot(1-\veps)^2\Delta^2\indic\crl{j=i}
      + \lambda_j\cdot\Delta^2\indic\crl{j\neq{}i}}.
     \end{split}
  \end{align}
We also have that
  \[
    \gmi = \gsf\ldef{} \frac{(A-2)}{\Delta} + \frac{1}{(1-\veps)\Delta},
  \]
  where we have used again that $\veps\leq{}1/2$.

  Fix any pair $\lam,\om\in\simplex_\Pi$ and consider the value
  \begin{align*}
    \sup_{M\in\cM'\setminus\cMgl(\lambda)}\crl*{
    \frac{1}{
    \En_{\pi\sim{}\omega}\brk*{\Dkl{\Mbar(\pi)}{M(\pi)}}
    }}.
  \end{align*}
Pick
  $\delta\leq{}\veps/32$, and let $\cJ\subseteq\brk{A-1}$ be the set of
  models $i$ for which $\lambda\in\Lambdamdel[M_i]$. For each such
  model, by definition, there exists $\nsf_i>0$ such that
  \[
    \delm[M_i](\lambda)\leq{}(1+\delta)\frac{\gsf}{\nsf_i},\mathand\Imfull[M_i]{\lambda}
  \geq{} (1-\delta)\frac{1}{\nsf_i}.
  \]
  Define $\nbar=\max_{i\in\cJ}\nsf_i$ and
  $\nul=\min_{i\in\cJ}\nsf_i$. Using \eqref{eq:mab_aec_lb_Im}, it is immediate that for all
  $j\in\brk{A-1}$,
  $\lambda_j\geq{}(1-\delta)\frac{\Delta^{-2}}{\nul}$, and that for
  all $i\in\cJ$,
  \[
\lambda_i \geq{} (1-\delta)\frac{(1-\veps)^{-2}\Delta^{-2}}{\nsf_i}.
\]
In particular, this implies that for all $i\in\cJ$,
\begin{align*}
  (1-\delta)\frac{(A-2)\Delta^{-1}}{\nul}
  + (1-\delta)\frac{(1-\veps)^{-1}\Delta^{-1}}{\nsf_i} 
&   \leq{} \delm[M_i](\lambda) \\
&  \leq{} (1+\delta)\frac{\gsf}{\nsf_i} \\
&  =   (1+\delta)\frac{(A-2)\Delta^{-1}}{\nsf_i}
  + (1+\delta)\frac{(1-\veps)^{-1}\Delta^{-1}}{\nsf_i},
\end{align*}
or by rearranging,
\begin{align*}
  (1-\delta)\frac{(A-2)\Delta^{-1}}{\nul}
& \leq{}   (1+\delta)\frac{(A-2)\Delta^{-1}}{\nsf_i}
                                            + 2\delta\frac{(1-\veps)^{-1}\Delta^{-1}}{\nsf_i},\\
  & \leq{}   (1+\delta)\frac{(A-2)\Delta^{-1}}{\nsf_i}
    + 4\delta\frac{\Delta^{-1}}{\nsf_i},\\
  & \leq{}   (1+2\delta)\frac{(A-2)\Delta^{-1}}{\nsf_i}
\end{align*}
as long as $A\geq{}6$. Since this holds uniformly for all $i\in\cJ$,
rearranging once more gives
\[
  \nbar \leq \frac{(1+2\delta)}{1-\delta}\nul
  \leq{} (1+2\delta)^2\nul\leq{}(1+8\delta)\nul,
\]
where we have used that $\delta\leq{}1/2$.

Now, observe that for all $i\in\cJ$, we have
\begin{align*}
  \delm[M_i](\lambda)
  &\geq{} (1-\delta)(A -\abs{\cJ}-1)\frac{1}{\Delta\nul}
  + (1-\delta) \abs{\cJ}\frac{1}{(1-\veps)^2\Delta\nbar} + (1-\delta)\frac{1}{(1-\veps)\Delta\nsf_i},\\
    &\geq{} (1-\delta)(A -\abs{\cJ}-1)\frac{1}{\Delta\nul}
  + \frac{(1-\delta)}{1+8\delta}
      \abs{\cJ}\frac{1}{(1-\veps)^2\Delta\nul} + (1-\delta)\frac{1}{(1-\veps)\Delta\nsf_i},\\
  &\geq{} \frac{(1-\delta)}{\Delta\nul}
\prn*{ (A -\abs{\cJ}-1)
    + \abs{\cJ}\frac{1-8\delta}{(1-\veps)^2}}+ (1-\delta)\frac{1}{(1-\veps)\Delta\nsf_i}, \\
    &\geq{} \frac{(1-\delta)}{\Delta\nul}
\prn*{ (A -\abs{\cJ}-1)
    + \abs{\cJ}\frac{1}{(1-\veps)}}+(1-\delta)\frac{1}{(1-\veps)\Delta\nsf_i},
\end{align*}
where we have used that $\frac{1}{1+x}\geq{}1-x$, and that
$\delta\leq{}\veps/8$. Suppose that $\abs{\cJ}\geq{}\frac{A}{2}$. Then
we have
\begin{align*}
  (A -\abs{\cJ}-1)
  + \abs{\cJ}\frac{1}{(1-\veps)}
  \geq
  \frac{A}{2}\prn*{1 + \frac{1}{1-\veps}} - 1
  \geq{} (1+\veps/2)A - 1\\
    \geq{} (1+\veps/2)(A - 1),
\end{align*}
so that
\begin{align*}
  \delm[M_i](\lambda)
  \geq{} \frac{(1-\delta)(1+\veps/2)}{\Delta\nul}(A-1) + (1-\delta)\frac{1}{(1-\veps)\Delta\nsf_i}.
\end{align*}
Noting that $\nul\leq\nsf_i$ and $\delta\leq{}\veps/8$,  we further
have
\begin{align*}
  \delm[M_i](\lambda)
  \geq{} (1+\veps/4)\frac{A-1}{\Delta}\frac{1}{\nsf_i} + (1-\delta)\frac{1}{(1-\veps)\Delta\nsf_i}.
\end{align*}
Observe that the right-hand side above is greater than
$(1+\delta)\frac{\gsf}{\nsf_i}$ if and only if
\begin{align*}
  (1+\veps/4)(A-1) > (1+\delta)(A-1) + \frac{2\delta}{1-\veps},
\end{align*}
which is satisfied if $\delta\leq{}\veps/32$. In this case, we have
\[
  \delm[M_i](\lambda) > (1+\delta)\frac{\gsf}{\nsf_i},
\]
which contradicts the assumption that $i\in\cJ$. It follows that we
must have $\abs{\cJ}<A/2$.

Now, to conclude, select $i=\argmin_{i\in\brk{A-1}\setminus\cJ}\om_i$,
and consider the value
  \begin{align*}
    \sup_{M\in\cM'\setminus\cMgl(\lambda)}\crl*{
    \frac{1}{
    \En_{\pi\sim{}\omega}\brk*{\Dkl{\Mbar(\pi)}{M(\pi)}}
    }}
\geq{}    
    \frac{1}{
    \En_{\pi\sim{}\omega}\brk*{\Dkl{\Mbar(\pi)}{M_i(\pi)}}
         }
    =
    \frac{1}{
      \om_i\cdot{}\veps^2\Delta^2
      },
  \end{align*}
  where the first inequality follows because $\lambda\notin\Lambda(M_i;
  \delta)$ by definition, and the equality follows from the
  construction of $\Mbar$ and $M_i$. Since
  $\sum_{i\in\brk{A-1}\setminus\cJ}\om_i\leq{}1$ and
  $\abs*{\brk{A-1}\setminus\cJ}\geq{}\frac{A}{2}$, we must have
  $\om_i\leq\frac{2}{A}$, so that
  \begin{align*}
    \frac{1}{
      \om_i\cdot{}\veps^2\Delta^2
      } \geq{} \frac{A}{2\veps^2\Delta^2}
  \end{align*}
  as desired. 
 To complete the proof, note that this holds uniformly for all choices for $\lambda$
  and $\omega$, and that
  \begin{align*}
  \aecM{\veps}{\cM}(\cMsub,\Mbar) & =  \inf_{\lam, \omega \in \simplex_\Pi} \sup_{M\in\cMsub \setminus\cMgl(\lambda)}\crl*{
    \frac{1}{
    \En_{\pi\sim{}\omega}\brk*{\Dkl{\Mbar(\pi)}{M(\pi)}}
    }} \\
    &\ge \inf_{\lam, \omega \in \simplex_\Pi} \sup_{M\in\cM'\setminus\cMgl(\lambda)}\crl*{
    \frac{1}{
    \En_{\pi\sim{}\omega}\brk*{\Dkl{\Mbar(\pi)}{M(\pi)}}
    }}.
  \end{align*}
  To obtain the parameter setting in
  the theorem statement, we rescale $\Delta\gets{}2\Delta$ and $\veps\gets32\veps$.
\end{proof}

\begin{proof}[\pfref{ex:tabular_lower}]
  We reduce the lower bound to that of multi-armed bandits via a
  standard tree construction \citep{osband2016lower,domingues2021episodic}; as the argument is standard, we only
  sketch the approach. Assume without loss of generality that $H$ is a
  multiple of $2$. Set $H=\log_2(S/2)$.  Consider a sub-class $\cM'\subseteq\cM$
  defined as follows. All models $M\in\cM'$ have identical,
  deterministic dynamics given by a binary tree. Each layer $h$ has
  $2^{h-1}$ states, so that layer $H$ has $S/4$ states, and the total
  number of states is $S-1$.
  The agent begins from
  a root state $s_1$ deterministically. For $h\leq{}H-1$, there are two available
  actions, $\mathsf{left}$ and $\mathsf{right}$. Choosing
  $\mathsf{left}$ leads to the left successor for the current layer,
  and $\mathsf{right}$ leads to the right successor for the current
  layer. There are no rewards for layer $h\leq{}H-1$. For layer $H$,
  there are $A$ available actions, and rewards are arbitrary, subject
  to the constraint that the mean lies in $\brk{0,1}$ and the noise
  follows $\cN(0,1/2)$.

  It is clear that the class $\cM'$ is equivalent to
  the class of multi-armed bandit instances with $SA/4$ actions. As a
  consequence, the lower bound follows from \cref{ex:mab_lower}.

\end{proof}

\begin{proof}[\pfref{ex:revealing_revisited}]
  \newcommand{\Lambdamdel}[1][M]{\Lambda(#1; 1/2)}%
  \newcommand{\nsfi}{\nsf_i}%
  \newcommand{\nul}{\underline{\nsf}}%
  \newcommand{\picircm}[1][M]{\picirc\subs{#1}}
  Let $\Delta\in(0,1/6)$, $\beta\in(0,1)$, and $A,N\geq{}2$ be given.
  Consider the
  reference model $\Mbar\in\cMall$ defined as follows:
        \begin{itemize}
    \item For each bandit arm $k\in\brk{A}$, we have
      $\fmbar(k) = \frac{1}{2}+\Delta\indic\crl{k=A}$ and
      $r\sim{}\cN(\fmbar(k),1)$. There are no observations,
      i.e. $o=\perp$ almost surely.
    \item For each revealing arm $\picirc_k$, we receive zero reward
      almost surely (so $\fmbar(\picirc_k)=0$) and
      $o\sim\unif(\brk{A})$.
    \end{itemize}
    We define a subclass 
    \begin{align*}
      \cM'=\crl{M_j}_{j\in\brk{N}}\subset\cMoptsub(\Mbar)
    \end{align*}
    as follows
    \begin{itemize}
    \item For each bandit arm $k\in\brk{A}$, we have
      $f\sups{M_j}(k) = \frac{1}{2}+\Delta\indic\crl{k=A}$ and
      $r\sim{}\cN(f\sups{M_j}(k),1)$. There are no observations,
      i.e. $o=\perp$ almost surely.
    \item For each revealing arm $\picirc_k$, we receive zero reward
      almost surely (so $f\sups{M_j}(\picirc_k)=0$). We have
      \begin{align*}
        o \sim\left\{
        \begin{array}{ll}
         \unif(\brk{A}),&\quad{}k\neq{}j,\\
          \beta\indic_{i}+(1-\beta)\unif(\brk{A}),&\quad{}k=j.
        \end{array}
        \right.
      \end{align*}   
    \end{itemize}
    Note that $\cM' \subseteq \cMsub$. 
    For all $j\in\brk{N}$, a direct calculation gives
    \begin{align*}
      \Dkl{\Mbar(\picirc_j)}{M_{j}(\picirc_j)}
      =
      \frac{A-1}{A}\log\prn*{\frac{1}{1-\beta}}
      + \frac{1}{A}\log\prn*{\frac{1}{1+\beta(A-1)}}
      \rdef \alpha
    \end{align*}
    and
    \begin{align}
      \label{eq:rev0}
      \Dkl{\Mbar(\pi)}{M_{j}(\pi)}
      = \alpha\cdot\indic\crl{\pi=\picirc_j}.
    \end{align}
    In addition, it is straightforward to see that
    $\alpha\leq2\beta$ whenever $\beta\leq{}1/2$. Next we calculate
    that for any $j\in\brk{N}$ and $M\in\cM$ with
    $\picircm=\picirc_j$, and $\pim\neq{}A$,
    \begin{align*}
      \Dkl{M_j(\picirc_j)}{M(\picirc_j)}
      = \beta\log\prn*{1 + \frac{\beta{}A}{1-\beta}} \rdef \gamma,
    \end{align*}
which has $\gamma\leq\bigoh(\beta)$ whenever $\beta\leq{}1/A$ and $\gamma\geq{}\beta\log(1+\beta{}A)$. Lastly,
we have that for all $i\in\brk{A}$,
all $M,M'\in\cM$ have
\begin{align*}
  \Dkl{M(i)}{M'(i)} =\frac{1}{2}(\fm(i)-\fmp(i))^2.
\end{align*}
Let $\cM''$ denote the set of instances such that for $M' \in \cM''$, $\pim[M'] \neq A$, and $\fm[M'](A) = \frac{1}{2} + \Delta$, so that $\Dkl{M_j(A)}{M'(A)} = 0$ and $\cM'' \subseteq \cMalt(M_j)$ for all $j \in \brk{N}$. Using the above calculations and the definition of $\Im[M_j](\lam;\cM)$, we can then compute, for all $j \in \brk{N}$,
\begin{align}\label{eq:rev_inf}
\Im[M_j](\lam;\cM) & \le \inf_{M' \in \cM''} \Exp_{\lam}[\Dkl{M_j(\pi)}{M'(\pi)}]  = \frac{\Delta^2}{2}\cdot{}\min_{k\in\brk{A-1}}\lambda_k + \gamma\cdot\lambda_{\picirc_j}
\end{align}
and
  \[
    \gmj = \gsf\ldef{} \min\crl*{2\frac{A-1}{\Delta},
      \prn*{\frac{1}{2} + \Delta} \frac{1}{\gamma} } = \prn*{\frac{1}{2} + \Delta} \frac{1}{\gamma}
  \]
  whenever $\gamma\geq{}\Delta/2(A-1)$.

  Fix any pair $\lam,\om\in\simplex_\Pi$ and consider the value
  \begin{align*}
    \sup_{M\in\cM'\setminus\cMgl[1/2](\lambda)}\crl*{
    \frac{1}{
    \En_{\pi\sim{}\omega}\brk*{\Dkl{\Mbar(\pi)}{M(\pi)}}
    }}.
  \end{align*}
  Let $\cJ\subseteq\brk{N}$ be the set of
  models $j$ for which $\lambda\in\Lambdamdel[M_j]$. For each such
  model, by definition, there exists $\nsf_j>0$ such that
  \[
    \delm[M_j](\lambda)\leq{}(1+1/2)\frac{\gsf}{\nsf_j}\leq{}\frac{1}{\gamma\nsf_j},\mathand\Imfull[M_j]{\lambda}
  \geq{} \frac{1}{2\nsf_j}.
  \]
  Define $\nbar=\max_{j\in\cJ}\nsf_j$ and
  $\nul=\min_{j\in\cJ}\nsf_j$. Let us begin with some basic observations.
  \begin{itemize}
  \item Since $\delm[M_j]=\delm[\Mbar]$ for all $j$, we have
    $\delm[M_j](\lambda)\leq{}\frac{1}{\gamma\nsf_{j'}}$ for all
    $j,j'\in\cJ$, and hence
    \begin{equation}
      \label{eq:rev1}
      \delm[M_j](\lambda)\leq{}\frac{1}{\gamma\nbar}.
    \end{equation}
  \item Any $j\in\cJ$ must have
    \begin{equation}
      \label{eq:rev2}
      \lambda_{\picirc_j}\gamma\geq{}\frac{1}{4\nsf_j}\geq{}\frac{1}{4\nbar}.
    \end{equation}
    To see this, observe that if it were not the case, we would need
    $\min_{i\in\brk{A-1}}\lambda_i\frac{\Delta^2}{2}\geq{}\frac{1}{4\nsf_j}$
    to satisfy the constraint that $\Imfull[M_j]{\lambda}
  \geq{} \frac{1}{2\nsf_j}$ (by \pref{eq:rev_inf}). But if this were
  to occur, we would have
  \begin{align*}
    \frac{A-1}{2\Delta\nsf_j} \leq \delm[M_j](\lambda)
    \leq{} \frac{1}{\gamma\nsf_j},
  \end{align*}
  which would contradict the assumption that $\gamma\geq{}2\Delta/(A-1)$.
  \end{itemize}
Combing the inequalities \pref{eq:rev1} and \pref{eq:rev2}, it follows that any $j\in\cJ$ must have
\begin{align*}
  \frac{\abs{\cJ}}{8\gamma\nsf_j}
  \leq{} \frac{1}{2}\sum_{k\in\cJ}\lambda_{\picirc_k}
  \leq{}\delm[M_j](\lambda) \leq{} \frac{1}{\gamma\nsf_j}, 
\end{align*}
which implies that $\abs{\cJ}\leq{}8$. Hence, as long as $N\geq{}16$,
we have $\abs{\brk{N}\setminus\cJ}\geq{}N/2$.

To conclude, select $k=\argmin_{j\in\brk{N}\setminus\cJ}\om_{\picirc_j}$,
and consider the value
  \begin{align*}
    \sup_{M\in\cM'\setminus\cMgl[1/2](\lambda)}\crl*{
    \frac{1}{
    \En_{\pi\sim{}\omega}\brk*{\Dkl{\Mbar(\pi)}{M(\pi)}}
    }}
\geq{}    
    \frac{1}{
    \En_{\pi\sim{}\omega}\brk*{\Dkl{\Mbar(\pi)}{M_k(\pi)}}
         }
    =
    \frac{1}{
      \om_{\picirc_k}\cdot{}\alpha
      },
  \end{align*}
  where the first inequality follows because $\lambda\notin\Lambda(M_k;
  1/2)$ by definition, and the equality follows from \pref{eq:rev0}. Since
  $\sum_{j\in\brk{N}\setminus\cJ}\om_{\picirc_j}\leq{}1$ and
  $\abs*{\brk{N}\setminus\cJ}\geq{}\frac{N}{2}$, we must have
  $\om_{\picirc_k}\leq\frac{2}{N}$, so that
  \begin{align*}
        \frac{1}{
      \om_{\picirc_k}\cdot{}\alpha
    }
    \geq{} \frac{N}{2\alpha}.
  \end{align*}
  as desired. Since this holds uniformly for all choices for $\lambda$
  and $\omega$, the proof is completed. 
\end{proof}

\subsection{Lower Bound on Regret for Algorithms
  with Well-Behaved Tails}
\label{app:lower_regret}

In this section, we present an additional result,
\pref{prop:regret_lower}, which shows that for algorithms for which
the tail behavior is ``well-behaved'' in a certain sense, the
\CompText directly leads to lower bounds on the least possible value
of $T$ for which any algorithm can achieve (approximate) instance-optimality.

\begin{restatable}{theorem}{regretlower}
  \label{prop:regret_lower}
      Let the time horizon $T\in\bbN$, $\veps\in(0,1/2)$, and $\cMsub \subseteq \cM$ be given. Suppose that there exists an algorithm
      $\Alg$ with the property that for all $M\in\cMsub$, 
\begin{enumerate}
\item $\Ema\brk*{\RegDM}\leq{}(1+\veps)\gm[M]\log(T)$.
\item For all $\pi\in\Pi$, if $\Ema\brk*{T(\pi)}\neq{}0$, then $\Ema\brk*{T(\pi)}\geq{}1$.
\item 
  $\sqrt{\Ema\brk*{(\RegDM)^2}}\leq{}2\gm\log(T)$.

\end{enumerate}
In addition, assume that 1) $\gm\geq{}1$ for all $M\in\cMsub$, 2)
\pref{asm:mingap} holds,
3) \pref{asm:bounded_likelihood} holds with parameter
$\VM\geq{}1$, 
and 4) that
\[
    \log(T) \geq{}
  \frac{12}{\veps}\log\prn*{\sup_{M\in\cM}\frac{2\gm}{\delminm}\cdot\log(T)}.
\]
Then if we define $\delta=\veps \cdot \min \{ 1, \inf_{M\in\cMsub} \frac{\gm}{3 \gm/\delminm + \ncMeps} \}$, it
must be the case that
\begin{align*}
\log^3(T)
  \geq{} \frac{\delta^2}{C}\cdot\sup_{\Mbar\in\cMall}\aecM{4\veps}{\cM}(\cMoptsub(\Mbar),\Mbar),
\end{align*}
for $C\leq{} \bigoh\prn*{(\sup_{M\in\cM}\frac{\gm}{\delminm})^4\cdot\frac{\VM^2\log(\delta^{-1})}{\veps^2}}$.
\end{restatable}

We prove \pref{prop:regret_lower} by combining \pref{thm:lower_logt}
with an another technical result,
\pref{prop:regret_allocation_general} (stated and proven in the sequel), which
shows that for algorithms that satisfy the assumptions of
\pref{prop:regret_lower}, it is possible to use the empirical decision
frequencies to compute an allocation that is approximately optimal
with high probability.

\begin{proof}[\pfref{prop:regret_lower}]
  Define $\nbarm=3\frac{\gm}{\delminm} + \ncMeps$, and set
  \[
    \delta\ldef{}\veps\cdot\min\crl*{1,\inf_{M\in\cMsub}\frac{\gm}{\nbarm}}.
  \]
  Assume that
  \[
    \log(T) \geq{}
  \frac{12}{\veps}\log\prn*{\sup_{M\in\cM}\frac{2\gm}{\delminm}\cdot\log(T)}.
\]
  Let $\Alg$ be the algorithm in the statement of the proposition, and
  let $\Alg'$ be the modified algorithm created through
  \pref{prop:regret_allocation_general} with parameter
  $\delta$. By assumption, we have that
  $\sqrt{\Ema\brk*{\Nm}}\leq{}R\ldef{}2\sup_{M\in\cM}\frac{\gm}{\delminm}\log(T)$. We define $n =
c\cdot\frac{\log(24\delta^{-1})}{\veps^2}\cdot\frac{R^3\VM^2}{\log(T)}$
for a sufficiently large numerical constant $c$ and
$T'=T\cdot{}n$. \pref{prop:regret_allocation_general} implies that for
time $T'$, the algorithm $\Alg'$ satisfies
  \[
    \sqrt{\Emap\brk*{\Nm}}\leq{}R'\ldef{}R\cdot{}n
  \]
  and
  \[
\Pmap\brk*{\lambdahat\in\Lambda(M;2\veps,\nbarm)}
  \geq{} 1-\frac{\delta}{24}.
\]
On the other hand, since the precondition of
\pref{thm:lower_logt_general} is now satisfied with parameter $R'$, we
have that unless
\begin{align}
  \label{eq:intermediate_condition}
R' \geq
\frac{\delta^2}{192}\cdot\sup_{\Mbar\in\cMall}\aecM{4\veps}{\cM}(\cMoptsub(\Mbar),\Mbar),
\end{align}
the algorithm must have
\[
\Pmap\brk*{\lambdahat\in\Lambda(M;2\veps,\nbarm)}
\leq{} 1-\frac{\delta}{12}.
\]
As $\frac{\delta}{12}>\frac{\delta}{24}$, this is a contradiction
unless \pref{eq:intermediate_condition} holds.

\end{proof}

\begin{proposition}
  \label{prop:regret_allocation_general}
  Let the time horizon $T\in\bbN$ and $\cMsub \subseteq \cM$ be given. Let $\Alg$ be an algorithm
  with the property that for all $M\in\cMsub$, 
\begin{enumerate}
\item $\Ema\brk*{\RegDM}\leq{}(1+\veps)\gm[M]\log(T)$ for some
  $\veps\in(0,1)$.
\item For all $\pi\in\Pi$, if $\Ema\brk*{T(\pi)}\neq{}0$, then $\Ema\brk*{T(\pi)}\geq{}1$.
\item 
  $\sqrt{\Ema\brk*{\Nm^2}}\leq{}R$ for some
  $R\geq{}2$.
\end{enumerate}
In addition, assume that 1) $\gm\geq{}1$ for all $M\in\cMsub$, 2)
\pref{asm:mingap} holds,
3) \pref{asm:bounded_likelihood} holds with parameter
$\VM\geq{}1$, 
and 4) that
\[
    \log(T) \geq{}
  \frac{12}{\veps}\log\prn*{\sup_{M\in\cM}\frac{2\gm}{\delminm}\cdot\log(T)}.
\]
Then for any $\delta\in(0,e^{-1})$, if we define $n =
c\cdot\frac{\log(\delta^{-1})}{\veps^2}\cdot\frac{R^3\VM^2}{\log(T)}$ for a sufficiently
large numerical constant $c$, 
there exists an algorithm $\bbA'$ that, using
$T'\ldef{} T\cdot{}n$ rounds, returns a normalized
allocation $\lambdahat\in\simplex_\Pi$ such that 
\[
\Pmap\brk*{\lambdahat\in\Lambda(M;2\veps,\nbarm)}
  \geq{} 1-\delta,
\]
for $\nbarm\leq{}3\frac{\gm}{\delminm} + \ncMeps$ and that $\sqrt{\Emap\brk*{\Nm^2}}\leq{}R\cdot{}n$ and
\begin{align*}
  \Emap\brk*{\RegDM[T']}
  \leq{} (1+\veps)\gm\log(T)\cdot{}n
\end{align*}
for all $M\in\cMsub$.
\end{proposition}

\begin{proof}[\pfref{prop:regret_allocation_general}]
We first state a technical lemma regarding robust mean estimation.
  \begin{lemma}
  \label{lem:mean_estimation}
  Let $X\in\bbR^{d}$ be a random variable with
  $\mu\ldef{}\En\brk{X}$. Assume that $\nrm{\mu}_0\leq{}s$, where $s$
  is a known upper bound. For any $\delta\leq{}e^{-1}$, there exists
    an estimator $\muhat_n$ that, given $n$ independent samples from
    $X$, ensures that with probability at least $1-\delta$,
            \begin{align*}
    \nrm*{\muhat_n-\mu}_{1}
    \leq{} 24\sqrt{\frac{2s\cdot\En\nrm{X-\mu}_2^2\cdot\log(\delta^{-1})}{n}}.
            \end{align*}
            In addition, $\nrm{\muhat_n}_0\leq{}s$ with probability $1$.
          \end{lemma}
          Throughout, we will use that since \pref{asm:mingap} holds,
          $\pim$ is unique for all $M\in\cM$. 
         Fix $M \in \cMsub$. Let $\etahat\in\Aspace$ denote the vector of empirical decision
          frequencies when $\Alg$ is run with horizon $T$,
          i.e. $\etahat(\pi) = T(\pi)$. Let
          \[
            \etam = \Ema\brk*{\etahat}.
          \]
          For parameters $n\in\bbN$ and $\delta\leq{}e^{-1}$ we define $\bbA'$ as follows:
  \begin{itemize}
  \item Run $\bbA$ a total of $n$ times independently (so that $T'=T\cdot{}n$), and let $\etahat\ind{1},\ldots,\etahat\ind{n}$
    be the empirical decision frequencies.
  \item Apply the algorithm from \pref{lem:mean_estimation} to
    $\etahat\ind{1},\ldots,\etahat\ind{n}$ with parameters $\delta$
    and $s=2R$, and let
    $\etacheck\in\Rplus^{\Pi}$ be the resulting vector (note that we can
    take $\etacheck$ to have non-negative entries without loss of
    generality).
  \item Set $\pihat=\argmax_{\pi\in\Pi}\etacheck$, and set
    $\wt{\eta}(\pi) = \etacheck(\pi)\indic\crl{\pi\neq\pihat}$ and $\etatil(\pihat) = \ncMeps[\veps] \cdot \log T$ (note that $\ncMeps[\veps]$ is a class-dependent quantity, and so is known to the learner).
  \item Set $\lambdahat=\etatil/\nrm*{\etatil}_1$.
  \end{itemize}
  It is immediate that this construction
  satisfies $\sqrt{\Emap\brk*{\Nmu^2}}\leq{}R\cdot{}n$ and
  \begin{align*}
  \Emap\brk*{\RegDM[T']}
  \leq{} (1+\veps)\gm\log(T)\cdot{}n,
  \end{align*}
  so it remains to show that $\lambdahat$ is a near-optimal allocation
  with high probability when $n$ is chosen appropriately.
  
We start by applying \pref{lem:mean_estimation}. To do so, we carry
out some prerequisite calculations. First, by assumption, we have $\etam(\pi)\geq{}1$ if
$\etam\neq{}0$. Using this, along with \pref{asm:mingap}, we have
\begin{align*}
  \nrm*{\etam}_0
  \leq{}1+ \sum_{\pi\neq\pim}\etam(\pi)
  \leq 1 + \Ema\brk*{\Nmu}
  \leq{} 1+R\leq{}2R.
\end{align*}
Second,
\begin{align*}
  \Ema\nrm*{\etahat-\etam}_2^2
  &\leq{}   \Ema\brk*{\prn*{\etahat(\pim)-\etam(\pim)}^2}
    + \sum_{\pi\neq\pim}\Ema\brk*{\etahat(\pi)^2}\\
    &\leq{}   \Ema\brk*{\prn*{\etahat(\pim)-\etam(\pim)}^2}
  + \Ema\brk*{\Nmu^2}.
\end{align*}
Furthermore,
\begin{align*}
  \Ema\brk*{\prn*{\etahat(\pim)-\etam(\pim)}^2}
&  =
  \Ema\brk*{\prn*{\sum_{\pi\neq\pim}\etahat(\pi)-\sum_{\pi\neq\pim}\etam(\pi)}^2} \\
&  \leq{} \Ema\brk*{\prn*{\sum_{\pi\neq\pim}\etahat(\pi)}^2} \\
&  = \Ema\brk*{\Nmu^2}.
\end{align*}
so that
\begin{align*}
  \Ema\nrm*{\etahat-\etam}_2^2
  \leq{} 2 \Ema\brk*{\Nmu^2}
  \leq{} 2R^2.
\end{align*}

As a result, \pref{lem:mean_estimation} implies that with probability
$1-\delta$,
\begin{align}
  \label{eq:est_error_eta}
  \nrm*{\etacheck-\etam}_1
  \leq{} \sqrt{C_1\frac{\log(\delta^{-1})}{n}} \rdef \vepsstat
\end{align}
where $C_1=\bigoh(R^{3})$.

Next, we appeal to \pref{lem:low_regret_feasible}, which implies that
as long as
\begin{equation}
  \label{eq:robust_logt_min}
  \log(T) \geq{}
  \frac{12}{\veps}\log\prn*{\sup_{M\in\cM}\frac{2\gm}{\delminm}\cdot\log(T)},
\end{equation}
we have
\begin{align*}
  \delm(\etam)\leq{}(1+\veps/2)\gm\log(T),\mathand\Imfull{\etam} \geq{} (1-\veps/2)\log(T).
\end{align*}
Applying \pref{eq:est_error_eta} and using that $\delm\in\brk{0,1}$ and $\Dkl{M(\pi)}{M'(\pi)} \le 2\VM$ by \Cref{lem:kl_bound},
we have
\begin{align*}
  \delm(\etacheck)\leq{}(1+\veps/2)\gm\log(T)+\vepsstat,\mathand\Imfull{\etacheck} \geq{} (1-\veps/2)\log(T)-\vepsstat\cdot{}2\VM.
\end{align*}
Thus, as long as
\[
\vepsstat \leq \frac{1}{2}\veps\cdot{}\min\crl*{\gm\log(T),(2\VM)^{-1}\log(T)},
\]
we have
\begin{align}\label{eq:robust_etachk_gl}
  \delm(\etacheck)\leq{}(1+\veps)\gm\log(T),\mathand\Imfull{\etacheck} \geq{} (1-\veps)\log(T).
\end{align}
Next, we claim that $\pihat=\pim$. To see this, note that since
$\Ema\brk*{\RegDM}\leq{}2\gm\log(T)$, we have
\begin{align*}
  \etam(\pim)
  \geq{} T - 2\frac{\gm}{\delminm}\frac{\log(T)}{T} \geq{}\frac{3}{4}T
\end{align*}
as long as $T>8\frac{\gm}{\delminm}\log(T)$, which is implied by
the condition \pref{eq:robust_logt_min}. Hence, as long as
$\vepsstat\leq{}\frac{T}{4}$, we have
\[
\etacheck(\pim)>\frac{T}{2},
\]
which implies that $\pihat=\pim$. 
By definition of $\ncMeps$, and since $\etacheck$ satisfies \eqref{eq:robust_etachk_gl} above, we then have that setting $\etacheck(\pihat) = \ncMeps \log (T)$ does not affect the regret, and only decreases the information gain by a factor of $\veps \log(T)$. It follows that
\begin{align*}
  \delm(\etatil)\leq{}(1+\veps)\gm\log(T),\mathand\Imfull{\etatil} \geq{} (1-2\veps)\log(T).
\end{align*}
We conclude that $\lambdahat\in\Lambda(M; 2\veps,\nbar)$ for
\[
  \nbar=\nrm*{\etatil}_1/\log(T).
\]
To wrap up, we compute that
\begin{align*}
  \nrm*{\etatil}_1
  \leq{} \sum_{\pi\neq\pim}\etam(\pi)
  + \vepsstat + \ncMeps \log(T)
  \leq{} 2\frac{\gm}{\delminm}\log(T) + \vepsstat + \ncMeps \log(T)
  \leq{} 3\frac{\gm}{\delminm}\log(T) + \ncMeps \log(T)
\end{align*}
whenever $\vepsstat\leq{}\gm\log(T)$.
  
\end{proof}

\begin{proof}[\pfref{lem:mean_estimation}]
  From Proposition 1 of \citet{lugosi2019mean}, we have that for any
  $\delta\leq{}e^{-1}$, there exists an
  estimator $\mutil_n$ that, given $n$ independent samples of $X$,
  ensures that with probability at least $1-\delta$,
  \begin{align*}
    \nrm*{\mutil_n-\mu}_{2}
    \leq{} 12\sqrt{\frac{\En\nrm{X-\mu}_2^2\cdot\log(\delta^{-1})}{n}}.
  \end{align*}
  Given $\mutil_n$, we define
  $\muhat_n=\argmin_{u\in\bbR^{d}:\nrm*{u}_0\leq{}s}\nrm*{u-\mutil_n}_2$. Since
  $\nrm*{\mu}_0\leq{}s$, we have
  $\nrm*{\mutil_n-\muhat_n}_{2}=\min_{u:\nrm{u}_0\leq{}s}\nrm{\mutil_n-u}_2\leq{}\nrm{\mutil_n-\mu}_2$. It follows that
  \begin{align*}
    \nrm*{\muhat_n-\mu}_{2}
    \leq{}   \nrm*{\mutil_n-\muhat_n}_{2}
    + \nrm*{\mutil_n-\mu}\leq{}2\nrm*{\mutil_n-\mu}_2.
  \end{align*}
  Finally, we note that since $\muhat_n$ and $\mu$ are both
  $s$-sparse, $\nrm*{\muhat_n-\mu}_1\leq{}\sqrt{2s}\nrm{\muhat_n-\mu}_2$.
\end{proof}


\end{document}
